\documentclass[11pt]{article}
\usepackage{mathrsfs,amsmath,amsfonts,amssymb,bm,bbm,dsfont,pifont,amscd,stmaryrd,euscript,amsthm,color,epsfig,xr,yfonts,tikz,verbatim}

\usepackage{tikz-cd} 
\usepackage{authblk}
\usepackage[utf8]{inputenc} 
\usepackage[T1]{fontenc}    
\usepackage{url}            
\usepackage{booktabs}       
\usepackage{amsfonts}  
\usepackage{amsmath} 
\usepackage{nicefrac}       
\usepackage{microtype}      
\usepackage{xcolor}
\usepackage{hyperref}  
\usepackage{graphicx}
\usepackage{array}
\usepackage{subcaption,mathrsfs}
\usepackage{natbib}
\usepackage{wrapfig}
\usepackage{tikz}
\usepackage{include/titletoc}
\usepackage[makeroom]{cancel}
\usepackage{algorithm}
\usepackage{algpseudocode}

\newcommand{\bgamma}{\boldsymbol{\gamma}}
\newcommand{\stageonesamples}{\tilde{n}}
\newcommand{\meffective}{m^{\dag}}
\newcommand{\deffective}{d^{\dag}}
\newcommand{\bt}{\boldsymbol{t}}

\newenvironment{talign*}
 {\csname align*\endcsname}
 {\endalign}
\newenvironment{talign}
 {\csname align\endcsname}
 {\endalign}

\allowdisplaybreaks
\usetikzlibrary{decorations.markings}

\definecolor{mydarkblue}{rgb}{0,0.08,0.45}

\hypersetup{ %
    pdftitle={},
    pdfauthor={},
    pdfsubject={},
    pdfkeywords={},
    pdfborder=0 0 0,
    pdfpagemode=UseNone,
    colorlinks=true,
    linkcolor=mydarkblue,
    citecolor=mydarkblue,
    filecolor=mydarkblue,
    urlcolor=mydarkblue,
    pdfview=FitH
}

\usepackage[utf8]{inputenc}
\usepackage{booktabs} 
\usepackage{hyperref}

\usepackage{bbm} 
\usepackage{bm}
\usepackage{verbatim}
\usepackage{float}
\usepackage{color,soul}
\usepackage{enumitem}
\usepackage{mathtools}
\usepackage{hhline}
\usepackage[title]{appendix}
\usepackage[nameinlink]{cleveref} 
\usepackage[font=small,labelfont=bf,
   justification=justified,
   format=plain]{caption}
\usepackage{tikz}
\usepackage{wrapfig,booktabs}
\usepackage{xspace}
\usepackage{cancel}
\usepackage{sidecap}

\graphicspath{ {../figures/} }

\DeclarePairedDelimiterX{\infdivx}[2]{(}{)}{%
  #1\;\delimsize\|\;#2%
}

\newcommand{\one}{\mathbbm{1}}

\usetikzlibrary{shapes.geometric, arrows, bayesnet, calc, positioning}

\newcommand{\ba}{\mathbf{a}}

\newcommand{\bg}{\mathbf{g}}

\newcommand{\bY}{\mathbf Y}

\newcommand{\bx}{\mathbf x}

\newcommand{\bz}{\mathbf z}
\newcommand{\bv}{\mathbf v}

\newcommand{\bo}{\mathbf o}
\newcommand{\bomega}{\boldsymbol{\omega}}

\newcommand{\calZ}{\mathcal{Z}}
\newcommand{\calR}{\mathcal{R}}
\newcommand{\calF}{\mathcal{F}}
\newcommand{\calE}{\mathcal{E}}
\newcommand{\calH}{\mathcal{H}}

\newcommand{\calA}{\mathcal{A}}
\newcommand{\calC}{\mathcal{C}}

\newcommand{\calG}{\mathcal{G}}

\newcommand{\calB}{\mathcal{B}}
\newcommand{\calN}{\mathcal{N}}
\newcommand{\calO}{\mathcal{O}}

\newcommand{\calL}{\mathcal{L}}

\newcommand{\calX}{\mathcal{X}}

\newcommand{\calD}{\mathcal{D}}

\newcommand{\x}{\mathbf{x}}

\newcommand{\bl}{{\boldsymbol\ell}}

\newcommand{\R}{\mathbb{R}}
\newcommand{\N}{\mathbb{N}}

\newcommand{\bh}{\boldsymbol{h}}
\newcommand{\bs}{\boldsymbol{s}}

\newcommand{\balpha}{\boldsymbol\alpha}
\newcommand{\bbeta}{\boldsymbol\beta}

\newcommand{\closer}[3]{{\kern-#1ex{#2}\kern-#3ex}}

\DeclareMathOperator*{\argmin}{arg\,min}

\mathchardef\mhyphen="2D

\DeclareMathOperator{\E}{\mathbb{E}}

\newcommand\reallywidehat[1]{\arraycolsep=0pt\relax%
\begin{array}{c}
\stretchto{
  \scaleto{
    \scalerel*[\widthof{\ensuremath{#1}}]{\kern-.5pt\bigwedge\kern-.5pt}
    {\rule[-\textheight/2]{1ex}{\textheight}} 
  }{\textheight} %
}{0.5ex}\\           
#1\\                 
\rule{-1ex}{0ex}
\end{array}
}

\def\Id{\mathrm{Id}}

\definecolor{azure}{rgb}{0.0, 0.5, 1.0}
\definecolor{airforceblue}{rgb}{0.36, 0.54, 0.66}
\definecolor{darkgreen}{rgb}{0.0, 0.2, 0.13}

\usepackage{standalone}
\usepackage{tikz}
\usepackage{pgfplots}
\usepackage{pgf}
\usetikzlibrary{calc}
\usetikzlibrary{positioning}
\usetikzlibrary{angles,quotes}
\usetikzlibrary{backgrounds}
\usetikzlibrary{fit}
\usetikzlibrary{arrows}
\usetikzlibrary{arrows.meta}
\usetikzlibrary{shapes.symbols}
\usetikzlibrary{shadings}
\usetikzlibrary{shapes}
\usetikzlibrary{fadings}
\usetikzlibrary{bayesnet}
\usetikzlibrary{matrix}
\usetikzlibrary{plotmarks}
\usetikzlibrary{intersections}
\usetikzlibrary{pgfplots.fillbetween}
\pgfplotsset{compat=1.14}
\usepackage{siunitx}

\usepackage{colortbl}
\definecolor{mediumgray}{gray}{0.7}
\definecolor{lightgray}{gray}{0.85}
\definecolor{lightlightgray}{gray}{0.9}
\definecolor{C1}{HTML}{1F77B4}
\definecolor{C2}{HTML}{FF7F0E}
\definecolor{C3}{HTML}{2CA02C}
\definecolor{C4}{HTML}{D62728}
\definecolor{C5}{HTML}{9467BD}
\colorlet{C1light}{C1!70!white}
\colorlet{C2light}{C2!70!white}
\colorlet{C3light}{C3!70!white}
\colorlet{C4light}{C4!70!white}
\colorlet{C5light}{C5!70!white}
\colorlet{C1lighter}{C1!50!white}
\colorlet{C2lighter}{C2!50!white}
\colorlet{C3lighter}{C3!50!white}
\colorlet{C4lighter}{C4!50!white}
\colorlet{C5lighter}{C5!50!white}
\colorlet{C1vlight}{C1!20!white}
\colorlet{C2vlight}{C2!20!white}
\colorlet{C3vlight}{C3!20!white}
\colorlet{C4vlight}{C4!20!white}
\colorlet{C5vlight}{C5!20!white}
\colorlet{linkcolor}{violet}

\crefname{enumi}{}{}
\crefname{enumii}{}{}

\usepackage{xr}
\makeatletter

\newcommand*{\addFileDependency}[1]{
\typeout{(#1)}
\@addtofilelist{#1}
\IfFileExists{#1}{}{\typeout{No file #1.}}
}\makeatother

\newcommand{\steve}[1]{#1}
\usepackage{tikz}
   \usetikzlibrary{calc}

\newtheorem{thm}{Theorem}[section]
\newtheorem{cor}{Corollary}[section]
\newtheorem{lem}{Lemma}[section]
\newtheorem{defi}{Definition}
\newtheorem{prop}{Proposition}[section]
\newtheorem{rem}{Remark}[section]

\newtheorem{ass}{Assumption}[section]
\newtheorem*{ass*}{Assumption}

\makeatletter
\newcommand\subsubsubsection{\@startsection{subsubsubsection}{4}{\z@}%
  {-3.25ex\@plus -1ex \@minus -.2ex}%
  {1.5ex \@plus .2ex}%
  {\normalfont\normalsize\bfseries}}
\newcommand\subsubsubsectionmark[1]{}
\newcommand{\bPhi}{\boldsymbol{\Phi}}
\makeatother

\title{Nonparametric Instrumental Variable Regression with Observed Covariates}

\author[1]{Zikai Shen$^{*}$}
\author[2]{Zonghao Chen$^{*}$}
\author[3]{Dimitri Meunier}
\author[4]{Ingo Steinwart}
\author[3]{Arthur Gretton$^{\dagger}$}
\author[3]{Zhu Li$^{\dagger}$}

\affil[1]{Department of Statistical Science,
University College London}
\affil[2]{Department of Computer Science,
University College London}
\affil[3]{Gatsby Computational Neuroscience Unit,
University College London}
\affil[4]{Department of Mathematics, University of
Stuttgart}

\date{}
\begin{document}

\maketitle
\setcounter{footnote}{0}
\renewcommand{\thefootnote}{\fnsymbol{footnote}} 
\footnotetext[1]{$,^\dagger$ Equal contribution in random order.}
\begin{abstract}
We study the problem of nonparametric instrumental variable regression with observed covariates, which we refer to as NPIV-O. Compared with standard nonparametric instrumental variable regression (NPIV), the additional observed covariates facilitate causal identification and enables  heterogeneous causal effect estimation. However, the presence of observed covariates introduces two challenges for its theoretical analysis. First, it induces a partial identity structure, which renders previous NPIV analyses—based on measures of ill-posedness, stability conditions, or link conditions—inapplicable. Second, it imposes anisotropic smoothness on the structural function. To address the first challenge, we introduce a novel \emph{Fourier measure of partial smoothing}; for the second challenge, we extend the existing kernel 2SLS instrumental variable algorithm with observed covariates, termed KIV-O, to incorporate Gaussian kernel lengthscales adaptive to the anisotropic smoothness. We prove upper $L^2$-learning rates for KIV-O and the first $L^2$-minimax lower learning rates for NPIV-O. Both rates interpolate between known optimal rates of NPIV and nonparametric regression (NPR). Interestingly, we identify a gap between our upper and lower bounds, which arises from the choice of kernel lengthscales tuned to minimize a projected risk. Our theoretical analysis also applies to proximal causal inference, an emerging framework for causal effect estimation that shares the same conditional moment restriction as NPIV-O.
\end{abstract}

\section{Introduction}
\label{sec:intro}

We consider the problem of identifying and estimating the \emph{causal effect} of a treatment variable $X$ on an outcome variable $Y$, where their relationship is confounded by an unobserved confounder $\epsilon$. Despite the existence of unobserved confounding, it is nonetheless possible to identify the causal effect  by leveraging an \emph{instrumental variable} $Z$. For instance, if one aims to identify the causal effect of smoking ($X$) on the risk of lung disease ($Y$), which may be potentially confounded by an individual's occupation and early childhood environment ($\epsilon$), the cigarette cost ($Z$) would be a valid instrument as it only affects the risk of lung disease via smoking~\citep{leigh2004instrumental}. For cases where $X$ and $Y$ are continuous (and possibly multivariate) random variables, \emph{Nonparametric Instrumental Variables Regression} (NPIV) has received significant attention. NPIV is particularly valuable as it avoids imposing potentially misspecified parametric or semiparametric assumptions when such structure is not warranted~\citep{newey_2003, horowitz2011applied}. Prior literature has explored various algorithms for NPIV. These include methods based on: kernel density estimation \citep{HallHorowitz2005, darolles2011nonparametric}, sieve minimum distance estimators \citep{chen2012estimation, newey_2003, Chen_2018} and more recently Reproducing Kernel Hilbert Spaces (RKHSs) with the Kernel Instrumental Variables (KIV) algorithm \citep{singh2019kernel, meunier2024nonparametricinstrumentalregressionkernel}, which is a nonparametric generalization of the two-stage least squares (2SLS) algorithm. Another family of nonparametric algorithms is based on min-max optimization \citep{bennett2019deep, dikkala2020minimax, bennett2023minimax}. We defer a full discussion of these approaches to \Cref{sec:related_work}.

Practitioners often have access to \emph{observed covariates} $O$. These observed covariates encode individual-level characteristics, which allow for the estimation of \emph{heterogenous causal effects}. Returning to the smoking example, confounders such as an individual's occupation fall into this subset as it is readily observable. Such extra information enables the estimation of heterogeneous treatment effects---for instance, the causal effect of smoking on lung disease specifically for manual workers. In this work, we refer to the NPIV framework that incorporates observed covariates as NPIV-O. Formally, we introduce the following NPIV-O model,
\begin{align}
\label{eq:npiv_o_intro}
    Y = f_{\ast}(X, O) + \epsilon, \quad \mathbb{E}[\epsilon \mid Z, O] = 0  .
\end{align}
We refer to $f_{\ast}$ as the \emph{heterogeneous dose response curve}, and it is our target of interest. Here, $\epsilon$ is an unobserved confounder that affects $Y$ additively.  By Eq. \eqref{eq:npiv_o_intro}, we implicitly assume that $Z$ can only possibly affect $Y$ through $X$, a condition known as \emph{exclusion restriction}. The mean independence $\mathbb{E}[\epsilon\mid Z,O] = 0$ is a relaxation of the stronger \emph{unconfoundedness assumption}, that requires $(Z,O)$ to be independent of $\epsilon$. We also require that $Z$ and $X$ are not independent, a condition known as \emph{instrumental relevance}. A random variable $Z$ that satisfies these requirements is referred to as a valid instrumental variable. For a detailed discussion of our assumptions, we refer to \Cref{sec:theory}. We also note that the observed covariates $O$ aid identification by capturing non-linear confounding effects, thereby relaxing the strict additivity assumption required for \emph{all} confounders in classical NPIV. 

Incorporating observed covariates in NPIV estimation poses both an algorithmic and a theoretical challenge. To understand this, define the conditional expectation operator
\begin{align*}
    T : L^2(P_{XO}) \to L^2(P_{ZO}), \quad f \mapsto \E[f(X, O)\mid Z, O].
\end{align*}
Eq.~\eqref{eq:npiv_o_intro} is equivalent to the following \emph{conditional moment restriction} for $f_{\ast}$:
\begin{align}\label{eq:inverse_problem}
    \E[Y\mid Z, O]=(Tf_\ast)(Z, O).
\end{align}
We refer the reader to \Cref{ass:T_injective} for the technical assumption ensuring unique identification of $f_{\ast}$ from this equation.  
With the presence of $O$, $T$ acts as an identity operator on the infinite dimensional function space
\begin{align}
\label{eq:calF_1_espace}
    \calF_1 = \{f : f(X, O) = f_1(O), \text{ } f_1 \in L^2(P_{O})\} \subset L^2(P_{XO}),
\end{align}
rendering the operator non-compact. 
A naive application of kernel instrumental variables without observed covariates \citep{singh2019kernel} augments both $X$ and $Z$ to $(X,O)$ and $(Z,O)$, without exploiting the fact that $T$ is partially an identity operator. As a result, this algorithm is not consistent, a point we elaborate on in \Cref{sec:algorithm}. From a theoretical standpoint, $T$ being partially an identity operator fundamentally alters the statistical properties of NPIV-O estimation compared with classical NPIV, posing significant challenges as detailed below.

The first challenge is to characterize the degree of ill-posedness of the inverse problem, Eq. \eqref{eq:inverse_problem}, while accounting for the fact that $T$ is partially an identity operator. We consider another function space $\calF_2 = \{f : f(X, O) = f_2(X), \;\; f_2 \in L^2(P_{X})\} \subset L^2(P_{XO})$. In this case, under mild conditions on the conditional distributions~\citep[Assumption A.1]{darolles2011nonparametric}, $T$ when treated as a mapping from $\calF_2$ to $L^2(P_{ZO})$ is a compact operator whose smoothing effect can be quantified through the rate of decay of its singular values. \emph{This mixed behaviour of $T$ reveals the nature of NPIV-O as a hybrid between NPIV and Nonparametric Regression (NPR).} Existing theoretical analyses of NPIV estimator mostly preclude the case where $T$ acts as a partial identity operator \citep{blundell2007semi, chen2011rate, chen2012estimation, Chen_2018, meunier2024nonparametricinstrumentalregressionkernel, kim2025optimality,chen2024adaptiveestimationuniformconfidence}, making them unsuitable for the NPIV-O setting in Eq. \eqref{eq:npiv_o_intro}. We refer the reader to \Cref{sec:related_work} and \Cref{rem:partial_smoothing} for a more in-depth discussion.

Another challenge is the \emph{anisotropic smoothness} of the heterogenous dose response curve $f_\ast$ across the treatment $X$ and observed covariates $O$.
In real-world applications, the treatment $X$ (e.g. smoking) is often one dimensional, while the observed covariates $O$ (e.g. occupation, age, gene) are of higher dimensionality. This is because practitioners tend to adjust for as many observed covariates as possible, in an effort to overcome unobserved confounding. 
Therefore, a desirable algorithm would adapt to the \emph{intrinsic smoothness} of $f_\ast$, thereby \steve{adapts appropriately to the high intrinsic smoothness when} the directional smoothness is highly anisotropic~\citep{hoffman2002random}.  
To achieve this desirable property, we modify the KIV-O algorithm to select kernel lengthscales adaptively to the varying directional smoothness of $f_\ast$. In contrast, existing NPIV algorithms are typically analyzed under an isotropic smoothness assumption, meaning their established convergence rates are dictated by the worst smoothness across all dimensions~\citep{singh2019kernel,singh2024kernel}. This limitation of these theoretical guarantees becomes increasingly severe in high dimensions.

In this paper, we tackle the above two challenges and make the following contributions. 

\begin{enumerate}[itemsep=5.0pt,topsep=0pt,leftmargin=*]
    \item \emph{Fourier measure of partial smoothing effect of $T$}:
    In \Cref{sec:ill_posedness}, we introduce a new framework based on Fourier spectra which quantifies the \emph{partial} smoothing effect of $T$ for functions $f:\calX\times\calO\to\R$ with only high frequencies on $X$; while conversely quantifies the \emph{partial} anti-smoothing effect of $T$ for functions $f:\calX\times\calO\to\R$ with only low frequencies on $X$. 
    Our Fourier measure of partial smoothing effect resembles existing ones based on sieves~\citep{blundell2007semi,chen2024adaptiveestimationuniformconfidence,kim2025optimality}, but features two key distinctions:
    1) Our framework takes into account the partial identity structure of $T$ caused by the existence of observed covariates $O$.
    2) Our framework aligns, through Bochner's theorem, with the RKHS of a continuous translational invariant kernel. This alignment will be useful for our next contribution.
    \item \emph{Upper and minimax lower $L^2(P_{XO})$-learning rate}: 
    Under the above framework that quantifies the \emph{partial} smoothing effect of $T$, we prove an upper learning rate (\Cref{thm:upper_rate}) for a kernel based algorithm for instrumental variable
    regression with observed covariates proposed in \citet{singh2024kernel}, termed KIV-O. Furthermore, we prove the first minimax lower learning rate (\Cref{thm:lower_rate}) for NPIV-O defined in Eq.~\eqref{eq:npiv_o_intro}. All our bounds hold in the strong $L^2(P_{XO})$ norm rather than the pseudo-metric $\|T(\cdot)\|_{L^2(P_{ZO})}$ considered in \citet{singh2019kernel}. 
    For the following two edge cases: 1)
    No observed covariates: NPIV-O reduces to classical NPIV; 2) No hidden confounding: instrument variables $Z$ are unnecessary and NPIV-O reduces to NPR, our upper bound of KIV-O matches the minimax lower bound, recovering earlier results on minimax optimality of classical kernel instrumental variable regression~\citep{meunier2024nonparametricinstrumentalregressionkernel} and kernel ridge regression~\citep{hang2021optimal,fischer2020sobolev}, respectively established under analogous assumptions to our work. In the general intermediate case \steve{where $T$ exhibits a partial identity structure}, both our upper and lower learning rates interpolate accordingly, however, there exists a gap between the upper and minimax lower bound. \steve{We posit that this gap is fundamental and we provide insights on why this gap emerges in \Cref{sec:challenge_optimal}. 
    \item \emph{Adaptivity to model intrinsic smoothness:} We modify the existing KIV-O algorithm to select kernel lengthscales separately for each dimension, adaptive to the varying directional smoothness of $f_\ast$. 
    We prove that its learning rate takes into account the anisotropic smoothness of $f_\ast$
    across the treatment $X$ and observed covariates $O$. 
    Compared with existing KIV algorithms and their associated analyses that assume isotropic smoothness~\citep{fischer2020sobolev,meunier2024nonparametricinstrumentalregressionkernel,singh2023kernel,singh2024kernel}, our learning rate is adaptive to the target function's intrinsic smoothness, and alleviates the slow rate caused by the need to account for the worst-case smoothness, when the anisotropic smoothness is highly imbalanced. 
    \item \emph{Interpretable anisotropic smoothness assumption}: 
    Another key feature of our upper and lower bounds is that they highlight the separate contribution of the partial smoothing effect of $T$ and the anisotropic smoothness of $f_{\ast}$, characterized by an anisotropic Besov space. 
    In contrast, much work in the NPIV literature employ a generalized source condition with respect to the unknown conditional expectation operator $T$ \citep{engl1996regularization, singh2019kernel, mastouri2021proximal, singh2023kernel, bozkurt2025densityratiobasedproxycausal,HallHorowitz2005}, which is less interpretable because $T$ is unknown a priori and cannot reveal the separate contribution of the intrinsic (anisotropic) smoothness of $f_{\ast}$ and the smoothness of $T$. 
    }

    \end{enumerate}

\subsection{Organization of the paper}
An outline of the paper is as follows. 
In \Cref{sec:setup}, we introduce the RKHS-based 2SLS algorithm for instrumental variable
regression with observed covariates, referred to as KIV-O.
In \Cref{sec:related_work}, we discuss related work on NPIV in the literature.
In \Cref{sec:theory}, we present the main assumptions and theoretical results, and discuss the interpretation of our findings.
In \Cref{sec:challenge_optimal}, we highlight the fundamental challenges towards obtaining minimax optimal rates of KIV-O.

\section{Setup}\label{sec:setup}

Consider $P$ the joint data-generating probability measure over $(Z,O,X,Y)$, where $Z\in \calZ := [0,1]^{d_z}$ denotes the instrument, $O \in \calO:=[0,1]^{d_o}$ denotes the observed covariates, $Y\in \mathbb{R}$ denotes the outcome variable, and $X \in \calX := [0,1]^{d_x}$ denotes the treatment variable. We use $p$ to denote the probability density functions; for example, $p(\bx\mid \bz, \bo)$ denotes the density of the conditional distribution $P_{X\mid Z=\bz, O=\bo}$. 
As stated in Section~\ref{sec:intro}, we define the conditional expectation operator $T$:
\begin{talign*}
    T : L^2(P_{XO}) \to L^2(P_{ZO}), \quad 
    f \mapsto \left( (\bz,\bo) \mapsto \int_{\calX} f(\bx,\bo)p(\bx|\bz,\bo)\;\mathrm{d}\bx \right) .
\end{talign*}

\textit{Notations:} 
Let $\mathbb{N}_{+}$ denote the set of positive integers and $\mathbb{N}=\mathbb{N}_{+} \cup\{0\}$ denote the set of non-negative integers. We use boldfaced letters, such as $\bx$, to denote a vector in $\mathbb{R}^{d}$ for $d\geq 1$. Specifically, $\bx = [x_1, \ldots, x_d]^\top \in \calX \subset \R^d$. 
For a distribution $P$ defined on a measurable space $(\mathcal{X}, \calB(\calX))$ and $0<p < \infty$, $L^p(P)$ is the space of functions $h: \mathcal{X} \rightarrow \mathbb{R}$ such that $\|h\|_{L^p(P)}:=\mathbb{E}_{X \sim P}\left[|h(X)|^p\right]^{\frac{1}{p}}<\infty$ and $L^{\infty}(P)$ is the space of functions that are bounded $P$-almost everywhere. When $P$ is the Lebesgue measure $\mathcal{L}_{\mathcal{X}}$ over $\mathcal{X}$, we write $L^p(\mathcal{X}) := L^p\left(\mathcal{L}_{\mathcal{X}}\right)$. For $H$ a separable Hilbert space, we let $L^p(\mathcal{X};H)$ denote the space of Bochner 2-integrable functions from $\mathcal{X}$ to $H$ with norm $\|F\|^2_{L^2(\mathcal{X}; H)} = \int_{\mathcal{X}}\|F(\bx)\|^2_{H}\;\mathrm{d}\bx$. Two Banach spaces $E_1,E_2$ are said to be isometrically isomorphic, denoted $E_1\cong E_2$, if there exists an isometric isomorphism $S$, such that $\| S h\|_{E_2} = \|h \|_{E_1}$ for all $h \in E_1$.
Two Banach spaces $E_1,E_2$ are said to be norm equivalent, denoted $E_1\simeq E_2$, if $E_1,E_2$ coincide as sets and there are constants $c_1, c_2>0$ such that $c_1\|h\|_{E_1} \leq\|h\|_{E_2} \leq c_2\|h\|_{E_1}$ holds for all $h \in E_1$.
For an operator $T: E_1 \rightarrow E_2$, $\|T\|$ denotes its operator norm and $T^\ast$ denotes its adjoint. For two Hilbert spaces $H_1, H_2$, $S_2(H_1, H_2)$ is the Hilbert space of Hilbert-Schmidt operators from $H_1$ to $H_2$. For two numbers $\alpha$ and $\beta$, we let $\alpha \wedge \beta=\min (\alpha, \beta)$ and $\alpha \vee \beta=\max (\alpha, \beta)$. $\lesssim$ (resp. $\gtrsim$) means $\leq$ (resp. $\geq$) up to positive multiplicative constants. 

\subsection{Algorithm}
\label{sec:algorithm}
In this section, we introduce a kernel two-stage least-squares approach for instrumental variable regression with observed covariates, which we term the KIV-O algorithm. KIV-O algorithm adopts a sample splitting strategy. In Stage I, we learn the conditional expectation operator $T$ with dataset $\calD_1 := \{(\tilde{\bz}_{i}, \tilde{\bo}_i, \tilde{\bx}_i)\}_{i=1}^{\stageonesamples}$ (see Eq.~\eqref{eq:hat_F_xi}); in Stage II, \steve{we perform regression of the outcome $Y$ on the features learned in Stage I with dataset $\mathcal{D}_2 := \{(\bz_i, \bo_i, y_i)\}_{i=1}^{n}$ (see Eq.~\eqref{eq:hat_f_lambda}).}  The KIV-O algorithm is a generalization of the KIV algorithm proposed in \citet{singh2019kernel}. We note that many existing NPIV learning methods employ a two-stage estimation procedure, see for instance \citet{hartford2017deep, singh2019kernel, xu2020learning, li2024regularized, petrulionyte2024functional, khoury2025learning}. 

We now briefly review the relevant reproducing kernel Hilbert space (RKHS) theory, following \citet{berlinet2004reproducing}. 
For a domain $\calX\subseteq \mathbb{R}^{d}$, a Hilbert space $\calH$ of functions $f:\calX \to \mathbb{R}$ is called a \emph{Reproducing Kernel Hilbert Space} (RKHS) if the evaluation functional $\delta_{\bx}:\calH \to \mathbb{R}$ defined by $f \mapsto f(\bx)$ is continuous for every $\bx\in \calX$. Every RKHS $\calH$ has a unique symmetric, positive definite \emph{reproducing kernel} $k : \calX\times \calX\to \mathbb{R}$, which satisfies $k(\bx,\cdot)\in \calH$ for all $\bx\in \calX$ and $\langle f, k(\cdot, \bx)\rangle_{\calH} = f(\bx)$ for all $f\in \calH$ and $\bx\in \calX$ (the reproducing property). To describe the KIV-O algorithm, we introduce RKHSs $\calH_X$ on $\calX$, $\calH_{O,1}$ and $\calH_{O,2}$ on $\calO$ and $\calH_{Z}$ on $\calZ$. The reasoning for defining two distinct RKHSs on $\calO$ will be clear later in the algorithm. We denote the associated unique reproducing kernels via $k_{X}: \calX\times \calX\to \mathbb{R}, k_{O,1} : \calO\times \calO\to \mathbb{R}, k_{O,2}: \calO\times \calO\to \mathbb{R}, k_{Z} : \calZ\times \calZ\to \mathbb{R}$. We denote the canonical feature map of $\calH_{X}$ as $\phi_{X}(\bx) := k_{X}(\bx,\cdot)$, and similarly for feature maps $\phi_{O,1}, \phi_{O,2}, \phi_{Z}$.

\begin{ass}\label{ass: kernel_technical}
All kernels ($k_X, k_{O,1}, k_{O,2}$ and $k_{Z}$) are measurable and bounded.
\end{ass}
An immediate consequence of \Cref{ass: kernel_technical} is that the embedding $I_{P_{X}} : \calH_{X} \to L^2(P_{X})$, which maps a function $f\in \calH_{X}$ to its $P_{X}$-equivalence class $[f]_{P_{X}}$ is well-defined, compact and Hilbert-Schmidt~\citep[Lemma 2.3]{steinwart2012mercer}. We define $[\calH_X]_{P_X} \subseteq L^2(P_X)$ as the image of $I_{P_{X}}$. For $\beta>0$, we denote by $[\calH_X]_{P_X}^\beta$ the $\beta$-th \emph{power space}, as introduced in \citet[Theorem 4.6]{steinwart2012mercer}. For $0\leq \beta\leq 1$, this space is shown to be isomorphic to the $\beta$-interpolation space $[L^2(P_X),[\calH_X]_{P_X}]_{\beta, 2}$ \citep[Theorem 4.6]{steinwart2012mercer}. 
It is known that the $1$-interpolation space $[\calH_X]_{P_X}^{1}$ is isometrically isomorphic to the closed subspace $(\operatorname{ker} I_{P_{X}})^{\perp}$ of $\calH_X$ via $I_{P_{X}}$~\citep[Lemma 2.12]{steinwart2012mercer}. 
For $\beta \geq 1$, the space contains functions that are smoother than those in $\calH_X$. The same definitions and properties hold for $\calH_{Z}$, $\calH_{O,1}$ and $\calH_{O,2}$ as well.

For two Hilbert spaces $H,H'$, we let $H\otimes H'$ denote their tensor product Hilbert space, defined as 
$H\otimes H' := \overline{\mathrm{span}\{u\otimes u':u\in H,u'\in H'\}}$, 
where $u\otimes u'$ is the linear rank-one operator $H'\rightarrow H$ defined by $(u\otimes u') v' = \langle u', v' \rangle_{H'} u$ \citep[Section 12]{aubin2011applied}. 
In the case of RKHSs, the tensor product $\calH_{ZO,1} := \calH_Z \otimes \calH_{O,1}$ and $\calH_{XO,2} := \mathcal{H}_{X} \otimes \calH_{O,2}$ are the unique RKHSs associated with the product kernels $k_{ZO,1}((\bz,\bo), (\bz',\bo')) = k_{Z}(\bz,\bz') \cdot k_{O,1}(\bo,\bo')$ and $k_{XO,2}((\bx,\bo), (\bx',\bo')) = k_{X}(\bx,\bx') \cdot k_{O,2}(\bo,\bo')$, respectively~\citep{berlinet2004reproducing}. 
We define the embedding $I_{P_{XO,2}} : \calH_{XO,2} \to L^2(P_{XO})$ which maps a function $f\in \calH_{XO,2}$ to its $P_{XO}$-equivalence class $[f]_{P_{XO}}$, and define the $\beta$-th power spaces as $[\calH_{XO,2}]_{P_{XO}}^\beta$.
An analogous construction applies to $\mathcal{H}_{Z O, 1}$, yielding the spaces $[\calH_{Z O, 1}]_{P_{Z O}}^\beta$. 
In the rest of the paper, we omit the subscript and use the notation $[\cdot]$ to denote equivalence classes in $L^2$. 

\vspace{1mm}
We are now ready to present the KIV-O algorithm. \\  
\indent \textit{Stage I.} The action of the operator $T$ on the RKHS $\calH_{XO,2}$ can be represented with the aid of the \emph{conditional mean embedding} (CME)~\citep{song2009hilbert, park2021measuretheoreticapproachkernelconditional, klebanov2020rigorous, lietal2022optimal}. We define the CME $F_{\ast}$ as the mapping from $\calZ\times \calO$ to $\calH_{X}$, given by $(\bz,\bo)\mapsto \mathbb{E}[\phi_{X}(X)\mid Z = \bz,O = \bo]$. Equipped with the CME, we note that the image of $T$ acting on a function $f\in \calH_{XO,2}$ admits the following representation: for any $(\bz,\bo)\in \calZ\times \calO$,
\begin{align*}
    &(Tf)(\bz,\bo) = \mathbb{E}[f(X,O)\mid Z=\bz, O = \bo] = \mathbb{E}[\langle f, \phi_{X}(X) \otimes \phi_{O,2}(O)\rangle_{\calH_{XO,2}} \mid Z = \bz, O = \bo]\\
    &= \langle f, \mathbb{E}[\phi_{X}(X)\mid Z = \bz, O = \bo]\otimes \phi_{O,2}(\bo)\rangle_{\calH_{XO,2}} = \langle f, F_{\ast}(\bz,\bo)\otimes \phi_{O,2}(\bo)\rangle_{\calH_{XO,2}},
\end{align*}
where the second equality follows from the \emph{reproducing property} and the third equality requires a Bochner integrable feature map $\phi_X$ (true for bounded kernels) from \Cref{ass: kernel_technical}~\citep[Definition A.5.20]{steinwart2008support}. 
Note that the feature map $\phi_{X}$ is projected by the conditional expectation of the conditional distribution $P_{X\mid Z,O}$, while the feature map $\phi_{O,2}$ remains unprojected. This is the key distinction from classical KIV. 
In Stage I, our goal is to estimate the CME, $F_{\ast}$, by performing a regularized least squares regression in a vector-valued RKHS $\calG$ induced by the operator-valued kernel ~\citep{grunewalder2012conditional, lietal2022optimal} 
\begin{align}
\label{eq:op_val_kernel}
    K := k_{ZO,1}\Id_{\calH_X} : (\calZ\times \calO)\times (\calZ\times \calO)\to \calL(\calH_{X}),
\end{align}
where $\calL(\calH_{X})$ denotes the space of bounded linear operators $\calH_{X}\to \calH_{X}$, and $\mathrm{Id}_{\calH_{X}}\in \calL(\calH_{X})$ denotes the identity operator on $\calH_{X}$. 
An important property of $\calG$ is that it is isometrically isomorphic to the space $S_2(\mathcal{H}_{Z O, 1}, \mathcal{H}_X)$ of Hilbert-Schmidt operators from $\mathcal{H}_{Z O, 1}$ to $\mathcal{H}_X$. On the other hand, by \citet[Theorem 12.6.1]{aubin2011applied}, $S_2(L^2(P_{ZO}), \calH_X)$ is isometrically isomorphic to the Bochner space $L^2(P_{ZO}, \calH_X)$, and we denote this isomorphism as $\Psi$. 
We can define vector-valued $\beta$-th power spaces~\citep[Definition 4]{lietal2022optimal}:
\begin{talign}\label{vv_interpolation_space}
    [\mathcal{G}]^\beta := \Psi\left(S_2\left([\mathcal{H}_{ZO,1}]^\beta, \mathcal{H}_{X} \right)\right)=\left\{F \mid F=\Psi(C), C \in S_2\left([\mathcal{H}_{ZO,1}]^\beta, \mathcal{H}_{X} \right)\right\} .
\end{talign}
The space $[\mathcal{G}]^\beta$ generalizes the definition of scalar-valued \emph{power space} to vector-valued RKHSs, quantifying the smoothness of $F_\ast$ relative to the RKHS $\calG$ (see Eq.~\eqref{eq:F_ast_power_space}). We refer the reader to \citet{carmeli2006vector, carmeli2010vector} for definitions and properties of more general vector-valued RKHSs. 

Given $\calD_1 = \{(\tilde{\bz}_i, \tilde{\bo}_i, \tilde{\bx}_i)\}_{i=1}^{\stageonesamples}$ sampled i.i.d from
the joint distribution $P_{XZO}$, a regularized estimator of $F_\ast$ is obtained as the solution to the following optimization problem:
\begin{talign}
\label{eq:hat_F_xi}
    \hat{F}_{\xi} := \underset{F \in \mathcal{G}}{\argmin } \frac{1}{\stageonesamples} \sum_{i=1}^{\stageonesamples} \left\|\phi_{X}(\tilde{\bx}_i) - F(\tilde{\bz}_i, \tilde{\bo}_i)\right\|_{\mathcal{H}_{X}}^2 + \xi\|F\|_{\mathcal{G}}^2,
\end{talign}
where $\xi >0 $ denotes the Stage I regularization parameter. 

\textit{Stage II.} In Stage II, we perform regularized least squares regression in the RKHS $\calH_{XO,2}$, using features derived from the estimated conditional mean embedding $\hat{F}_{\xi}$. Specifically, the features are $\hat{F}_{\xi}(Z, O) \otimes \phi_{O,2}(O)$.
Given $\calD_2 = \{(\bz_i,\bo_i, y_i)\}_{i=1}^{n}$ i.i.d sampled from the joint distribution $P_{ZOY}$ and independent of $\calD_1$, the regularized estimator $\hat{f}_{\lambda}$ is defined as: 
\begin{talign}\label{eq:hat_f_lambda}
    \hat{f}_{\lambda} &:= \underset{f\in \calH_{XO,2}}{\argmin}  \frac{1}{n}\sum\limits_{i=1}^n \left(y_i - \left\langle f, \hat{F}_{\xi}(\bz_i,\bo_i) \otimes \phi_{O,2}(\bo_i) \right\rangle_{\calH_{XO,2}} \right)^2 + \lambda \|f\|^2_{\calH_{X}\otimes \calH_{O,2}},
\end{talign}
where $\lambda >0 $ denotes the Stage II regularization parameter.\footnote{The naive extension to observed covariates in KIV \citet{singh2019kernel} considers augmenting $X,Z$ to $(X,O), (Z,O)$. This approach is not consistent because Stage I would then require estimating the conditional mean embedding $(\bz,\bo)\mapsto \E[\phi_X(X)\otimes\phi_O(O)\mid Z=\bz,O=\bo]$, which is \emph{not} Hilbert-Schmidt and for which vector-valued kernel ridge regression is not consistent, see also \citep[Appendix B.9]{mastouri2021proximal} for an illustration.}
Owing to the favourable properties of kernel ridge regression, $\hat{f}_\lambda$ admits a closed-form expression, given in \Cref{sec:closed_form_kivo} in the Supplement.
Upon learning $\hat{f}_\lambda$, the quantity $\hat{f}_\lambda(\bx^\ast, \bo^\ast)$ represents the estimated heterogenous dose response of a new treatment $\bx^\ast$ on a new individual with observed covariates $\bo^\ast$. The estimated dose response curve evaluated at $\bx^\ast$ can then be obtained as the expectation of $\hat{f}_{\lambda}(\bx^{\ast},O)$ with respect to the marginal distribution of the observed covariates. 

\vspace{1mm}
Our primary goal is to study the $L^2(P_{XO})$-risk:
\begin{align}\label{eq:L2_risk}
    \|\hat{f}_\lambda - f_{\ast}\|_{L^2(P_{XO})}.
\end{align}
To this end, we need to impose regularity conditions on the regression targets in both stages. 
Specifically, we characterize the regularity of the conditional mean embedding $F_\ast: \calZ\times\calO \to \calH_X$ through a \emph{dominating mixed-smoothness Sobolev space}, as discussed in \Cref{sec:mixed_smooth_ss}; and we characterize the regularity of the function $f_\ast: \calX \times \calO \to \R$ through an \emph{anisotropic Besov space}, as introduced in \Cref{sec:bg_anb}.
It is thus natural to use two different kernels $k_{O, 1}$ and $k_{O,2}$, because the regularity of $F_\ast$ and $f_\ast$ with respect to $\calO$ might not be the same.
Since the choice of kernel in both stages is dependent on the regularity of their respective regression targets $F_\ast$ and $f_\ast$, we provide a more in-depth description and justification of the kernels we use in stages I and II in Remarks~\ref{rem:kernel_stage_1_choice} and \ref{rem:kernel_stage_2_choice} respectively.  

\subsection{Mixed-smoothness Sobolev spaces}
\label{sec:mixed_smooth_ss}
In this section, we introduce vector-valued mixed-smoothness Sobolev spaces to characterize the smoothness of the conditional mean embedding (CME) $F_{\ast}: (\bz, \bo) \mapsto \E[\phi_{X}(X) \mid Z=\bz, O=\bo]$ in Stage I. 
In fact, the smoothness of $F_{\ast}$ can be identified via the differentiability of the conditional density.

Let $(\mathbb{N}^{+})^{d}$ be the set of all multi-indices $\balpha = (\alpha_1,\dots,\alpha_d)$ with $\alpha_i\in \mathbb{N}$ and $|\balpha| = \sum_{i=1}^{d}\alpha_i$.  
For $\balpha\in \mathbb{N}^{d}$ and $f: \R^d \to \R$,  $\partial^{\balpha}$ denotes the classical (pointwise) partial derivative, and $D^{\balpha} f$ denotes the corresponding weak (distributional) partial derivative. 

\begin{ass}\label{assn: technical}
Let $m_o, m_z \in \mathbb{N}^+$. For any $\bx\in\calX$, the map $(\bz,\bo)\mapsto p(\bx\mid\bz,\bo)$ has bounded, continuous derivatives of order $m_o$ with respect to $\bo$ and order $m_z$ with respect to $\bz$ on the interior of $\calZ\times \calO$. 
\begin{align*}
    \rho := \max_{|\balpha|\leq m_z}\max_{|\bbeta|\leq m_o}\sup_{\bx \in \calX ,\bz\in\calZ,\bo\in\calO} \left|\partial^{\balpha}_{\bz}\partial^{\bbeta}_{\bo}p(\bx\mid \bz,\bo)\right|< \infty. 
\end{align*}
\end{ass}
The differentiability conditions on the conditional density imposed in \Cref{assn: technical} imply that $F_\ast$ belong to a certain vector-valued dominating mixed-smoothness Sobolev space, as defined below.
\begin{defi}[Vector-valued dominating mixed-smoothness Sobolev space]
Let $H$ be a Hilbert space. Let $m_z,m_o\in \mathbb{N}^+$. We define
\begin{align*}
    MW^{m_z,m_o}_2(\calZ\times \calO; H):=\left\{ F \mid F \in L^2(\calZ\times \calO; H), \|F\|_{MW^{m_z,m_o}_2(\calZ\times \calO; H)} <\infty \right\} .
\end{align*}
where $\|F\|_{MW^{m_z,m_o}_2(\calZ\times \calO; H)} := \sum_{|\balpha|\leq m_z}\sum_{|\bbeta|\leq m_o} \| D^{\balpha}_{\bz} D^{\bbeta}_{\bo} F \|_{L^2(\calZ\times \calO; H)}$.
\end{defi}
The real-valued dominating mixed-smoothness Sobolev space~\citep{schmeisser1987unconditional, schmeisser2007recent, SICKEL2009748} $MW^{m_z,m_o}_2(\calZ\times\calO; \mathbb{R})$ is a special case of $MW^{m_z,m_o}_{2}(\calZ\times\calO;H)$ when $H = \mathbb{R}$. 
When $d_z =0$ (or $d_o = 0$), we recover the vector-valued Sobolev spaces $W^{m_z}_2(\calZ; H)$ (or $W^{m_o}_2(\calO; H)$) as defined in \citet[Section 12.7]{aubin2011applied}.
\Cref{assn: technical} implies that $\|D_{\bz}^{\balpha}D^{\bbeta}_{\bo}F_{\ast}\|_{L^2(\calZ\times\calO;\calH_{X})}$ is bounded for any multi-indices  $|\balpha|\leq m_z$ and $|\bbeta|\leq m_o$. Hence, $F_{\ast}\in MW^{m_z, m_o}_2(\calZ\times \calO;\calH_{X})$. 

\vspace{1mm}
Now we are ready to state our choice of kernels $k_Z, k_{O,1}$ in Stage I.
\begin{rem}[Choice of Stage I kernels $k_Z, k_{O,1}$]
\label{rem:kernel_stage_1_choice}
\steve{We let $k_{Z}$ and $k_{O, 1}$ be any positive definite kernels} such that their associated RKHSs \steve{$\calH_{Z}, \calH_{O,1}$} are \steve{respectively} norm equivalent to real-valued Sobolev spaces $W_2^{t_z}(\calZ)$ and $W_2^{t_o}(\calO)$, where $t_z > \frac{d_z}{2},t_o > \frac{d_o}{2}$. Following \cite{chen2025nested}, we say that $k_{Z}, k_{O, 1}$ are Sobolev reproducing kernels of smoothness $t_z, t_o$. An important example of Sobolev reproducing kernel is the Matérn-$\nu$ kernel whose RKHS is norm equivalent to a Sobolev space $W^{t}_2$ of smoothness $t = \nu+d/2$~\citep[Corollary 10.48]{wendland2004scattered}. Since all Sobolev reproducing kernels are bounded and measurable, $k_{Z}, k_{O,1}$ satisfy \Cref{ass: kernel_technical}. 

With the above choice of $k_Z$ and $k_{O,1}$, it follows from \Cref{lem:equivalence_G} that the vector-valued RKHS $\calG$ associated with the operator-valued kernel in Eq.~\eqref{eq:op_val_kernel} is norm equivalent to the mixed-smoothness Sobolev space $MW^{t_z,t_o}(\calZ\times\calO;\calH_X)$.
Since $F_\ast \in MW^{m_z,m_o}_2(\calZ\times\calO;\calH_X)$ as established above, $F_\ast$ lies in the appropriate power space of $\calG$ for suitably chosen $(t_z,t_o)$.
Consequently, \citet{lietal2022optimal,meunier2024optimalratesvectorvaluedspectral} show that estimating the CME via Eq.~\eqref{eq:hat_F_xi} achieves the minimax-optimal rate in both the $L^2(\calZ\times\calO;\calH_X)$ and $\calG$ norms, provided that the regularization parameter $\xi$ is selected adaptively with respect to the sample size $\tilde{n}$. 
\end{rem}

\subsection{Anisotropic Besov spaces}
\label{sec:bg_anb}
In this section, we introduce the definition of anisotropic Besov spaces \citep{leisner_nonlinear_wavelet_approximation_2003}, which is used to characterize the smoothness of $f_\ast$.
\begin{defi}[Modulus of smoothness]
Let $\calX = \prod_{i=1}^d \calX^{(i)} \subseteq \R^d$ be a subset with non-empty interior, $\nu$ be a product measure on $\calX$ with $\nu=\otimes_{i=1}^d \nu_i$, and $f: \calX \rightarrow \R$ be a function in $L^p(\nu)$ for some $p \in (0, \infty]$. 
The $r$-th modulus of smoothness of $f$ is defined by
\begin{talign}\label{eq:module_smoothness}
    \omega_{r, p}\left(f, \bt, \calX \right) = \sup _{0<\left|h_i\right| \leq t_i} \left\|\triangle_{\bh}^{r}f \right\|_{L^p\left(\nu\right)},
\end{talign}
where the $r$-th difference of $f$ in the direction $\bh$ at point $\bx$, denoted as $\triangle_{\bh}^{r} f \left( \bx \right)$, is defined through recursion: $\Delta_{\bh}^0 f(\bx) :=f(\bx)$ and $\Delta_{\bh}^r f(\bx) :=\Delta_{\bh}^{r-1} f(\bx+\bh)-\Delta_{\bh}^{r-1} f(\bx)$ if $\bx, \bx+\bh, \ldots, \bx+r \bh \in \calX$ and $0$ otherwise. 
\end{defi}
\begin{defi}[Anisotropic Besov space $B_{p, q}^{\bs}(\nu)$]
\label{defi:abs}
For $p \in[1, \infty), q \in[1, \infty]$ and $\bs=\left(s_1, \ldots, s_d\right) \in \R_{+}^d$, the anisotropic Besov space $B_{p, q}^{\bs}(\nu)$ is defined by
\begin{talign}\label{eq:ani_besov_space}
    B_{p, q}^{\bs}(\nu):=\left\{f \in L^p(\nu): \|f\|_{B_{p, q}^{\bs}(\nu)} := \|f\|_{L^p(\nu)} + |f|_{B_{p, q}^{\bs}\left(\nu\right)}<\infty\right\} ,
\end{talign}
where the Besov semi-norm $|f|_{B_{p, q}^{\bs}\left(\nu\right)}$ is defined as, 
\begin{talign}\label{eq:besov_semi_norm}
|f|_{B_{p, q}^{\bs}\left(\nu\right)} := \left[ \int_0^1\left[t^{-1} \omega_{r, p} \left(f, t^{1 / s_1}, \ldots, t^{1 / s_d}, \calX \right)\right]^q \;\frac{\mathrm{d}t}{t}\right]^{1 / q},
\end{talign}
for $r = \max\{ \lfloor s_1 \rfloor, \ldots, \lfloor s_d \rfloor \} + 1$. When $q = \infty$, we replace the integral by a supremum in Eq. \eqref{eq:besov_semi_norm}. When $\nu$ is the Lebesgue measure over $\calX$, we use the notation $B_{p, q}^{\bs}(\calX) := B_{p, q}^{\bs}(\nu)$.
\end{defi}
If $s_1 = \dots = s_d = s$, then the anisotropic Besov space recovers the standard isotropic Besov space \citep{devore1988interpolation, devore1993besov}. 
Since $f_\ast$ takes as input both the treatment $X$ and observed covariate $O$, it naturally exhibits different smoothness with respect to $X$ and $O$. 
Hence, as opposed to an isotropic Besov space which imposes uniform smoothness along all directions, an anisotropic Besov space captures such heterogeneous regularity. To simplify the exposition, we focus on anisotropic smoothness across $X$ and $O$, while assuming isotropic smoothness within $X$ and within $O$. 
In other words, we only consider $\bs=(s_x,\dots,s_x, s_o,\dots,s_o) \in \R_{\geq 0}^{d_x+d_0}$ and denote $B^{\bs}_{2,\infty}(\calX\times\calO)$ as $B^{s_x, s_o}_{2,\infty}(\calX\times\calO)$. Let $\mathbb{U}(B^{s_x,s_o}_{2,\infty}(\mathbb{R}^{d_x+d_o}))$ denote the unit ball of $B^{s_x,s_o}_{2,\infty}(\mathbb{R}^{d_x+d_o})$ with respect to the Besov norm. Let $C^{0}(\mathbb{R}^{d_x+d_o})$ denote the space of continuous functions $\mathbb{R}^{d_x+d_o}\to \mathbb{R}$. 
\begin{ass}
\label{ass:f_ast}
$f_{\ast}\in \mathfrak{S}$ where we define $\mathfrak{S} :=  \mathbb{U}(B^{s_x,s_o}_{2,\infty}(\mathbb{R}^{d_x+d_o}))\cap L^{\infty}(\mathbb{R}^{d_x+d_o}) \cap L^1(\mathbb{R}^{d_x+d_o})\cap C^{0}(\mathbb{R}^{d_x+d_o})$.  
\end{ass}
In particular, under \Cref{ass:f_ast}, $\|f_{\ast}\|_{L^2(\mathbb{R}^{d_x+d_o})} \leq \|f_{\ast}\|_{B^{s_x,s_o}_{2,\infty}(\mathbb{R}^{d_x+d_o})}\leq 1$. Moreover, by continuity, for all $\bo\in \mathbb{R}^{d_o}$ the slice function $f_{\ast}(\cdot, \bo)$ is well-defined. We assume $p=2$ in accordance with our $L^2(P_{XO})$-norm learning risk in Eq.~\eqref{eq:L2_risk}. 
We assume $q=\infty$ because $B^{s_x, s_o}_{2,\infty}(\calX\times\calO)$ is the largest anisotropic Besov space among all $B^{s_x, s_o}_{2,q}(\calX\times\calO)$ spaces \citep{triebel2011entropy}. 
To the best of our knowledge, we are the first to consider anisotropic smoothness in the NPIV literature.

We are now ready to state our choice of kernels $k_{X}, k_{O,2}$ in Stage II.
\begin{rem}[Choice of Stage II kernels $k_{X}, k_{O,2}$]\label{rem:kernel_stage_2_choice}
We choose $k_X$ and $k_{O,2}$ to be Gaussian kernels $k_{\gamma_x}$ and $k_{\gamma_o}$ with bandwidths $\gamma_x\in (0,1), \gamma_o\in (0,1)$.
Denote $\calH_{\gamma_x}$ and $\calH_{\gamma_o}$ as the associated Gaussian RKHSs; $\phi_{\gamma_x}$ and $\phi_{\gamma_o}$ as the associated feature maps. 
The tensor product RKHS $\calH_{\gamma_x,\gamma_o} := \calH_{\gamma_x} \otimes \calH_{\gamma_o}$ is the unique RKHS associated with the product kernel
\begin{talign}\label{eq:anis_gaussian_kernel}
    k_{\gamma_x}(\bx,\bx') \cdot k_{\gamma_o}(\bo,\bo') = \exp\left(-\sum_{j=1}^{d_x}\frac{(x_j - x_j')^2}{\gamma_x^2} - \sum_{j=1}^{d_o}\frac{(o_j - o_j')^2}{\gamma_o^2}\right).
\end{talign}
This kernel is called an anisotropic Gaussian kernel and its associated RKHS $\calH_{\gamma_x,\gamma_o}$ is the corresponding anisotropic Gaussian RKHS. 
\citet{hang2021optimal} proves that kernel ridge regression with anisotropic Gaussian kernel in the form of Eq.~\eqref{eq:anis_gaussian_kernel} is minimax optimal for anisotropic Besov space target functions, provided that both the regularization parameter and the kernel lengthscale $\gamma_x, \gamma_o$ are \emph{adaptive} to the number of samples $n$. 
Such adaptivity will also be evident in our setting (see \Cref{thm:upper_rate}).
\citet{singh2019kernel} adopt the median heuristic for selecting the kernel lengthscale in KIV algorithm, a widely used practical choice. Unlike ours, however, the theoretical relationship between their heuristic and the underlying smoothness of the target function remains unclear.
\end{rem}

\begin{rem}[Why use different kernels in Stage I and Stage II]\label{rem:choice_kernel}
We briefly explain the rationale for selecting different types of kernels for Stage I and Stage II. 
For Stage I, the regression target is the conditional mean embedding $F_\ast$. By \Cref{assn: technical}, we have shown that $F_{\ast}$ belongs to the mixed Sobolev space $MW^{m_z,m_o}(\calZ\times\calO;\calH)$, which can be learned at the minimax optimal rate using a tensor-product Sobolev RKHS (see \Cref{prop:cme_rate}). 
On the other hand, the Stage II regression target $f_\ast$ belongs to an anisotropic Besov space (\Cref{ass:f_ast}).
\citet{hang2021optimal} has proved that learning an anisotropic function in a nonparametric regression setting is minimax optimal via an anisotropic Gaussian RKHS. 
We have followed their approach with additional refinements to our setting, that reveals the interplay between the effect of $T$ and the anisotropic smoothness of $f_{\ast}$. See \Cref{sec:theory} for the details. 
\end{rem}

\section{Related work}\label{sec:related_work}
Early NPIV literature focuses on series estimators \citep{newey_2003, blundell2007semi, chen2007large, horowitz2011applied} and methods based on kernel density estimation~\citep{HallHorowitz2005, darolles2011nonparametric, florens2011identification}.
These works established minimax optimal convergence rates under various ill-posedness and smoothness conditions \citep{HallHorowitz2005, chen2011rate, Chen_2018}. 
Recent NPIV algorithms leverage modern machine learning techniques, including RKHSs \citep{singh2019kernel,zhang2023instrumental, meunier2024nonparametricinstrumentalregressionkernel} and neural networks \citep{hartford2017deep, bennett2019deep, xu2020learning, petrulionyte2024functional, kim2025optimality,sun2025spectral,meunier2025demystifying}. 
These modern methods mainly fall into two categories: two-stage estimation and min-max optimization. Min-max approaches \citep{bennett2019deep, dikkala2020minimax, liao2020provably, bennett2023minimax, zhang2023instrumental, wang2022spectral} formulate NPIV as a saddle point optimization problem, which can be unstable
and may fail to converge, especially when deep neural networks are used as function classes. 
In contrast, two-stage methods—such as the KIV-O algorithm studied in this manuscript (\Cref{sec:algorithm})—first estimate the conditional expectation operator $T$, and then perform a second-stage regression using the learned operator \citep{hartford2017deep, singh2019kernel, xu2020learning, li2024regularized, kim2025optimality}.
One recent paper~\citep{kankanala2025generalized} employs a sieve estimator in the first stage and a Gaussian process (a Bayesian analogue of an RKHS) estimator in the second stage. 

In the introduction, we outlined two challenges for the theoretical analysis of NPIV-O. The first challenge concerns the fact that $T$ is an identity operator restricted to the infinite dimensional function space $\calF_1$ (defined in Eq. \eqref{eq:calF_1_espace}). This has the following consequences. The $L^2$-\emph{stability condition} imposed in the NPIV literature (cf. \citet[Assumption 6]{blundell2007semi}, \citet[Assumption 5.2(ii)]{chen2012estimation} and \citet[Assumption 4.2]{Chen_2018}) fails to hold except for the degenerate case where the sieve measure of ill-posedness is $1$ (i.e. $Z=X$, see also \Cref{rem:partial_smoothing}). The \emph{link condition} imposed in the optimal rate literature for NPIV (cf. \citet{HallHorowitz2005}, \citet[Assumption 2.2]{chen2011rate}, \citet[Condition LB]{Chen_2018}) implies that $\|T f\|_{L^2 (P_{Z O})} \leq \|B^r f\|_{L^2(P_{X O})}$ for some known compact operator $B$, where a larger $r$ corresponds to a more ill-posed model. However, for any $f \in \mathcal{F}_1$ defined in Eq.~\eqref{eq:calF_1_espace}, we have $
\|T f\|_{L^2(P_{Z O})}=\|f\|_{L^2(P_{X O})}$ so the link condition only holds with $r = 0$.

The second challenge lies in deriving a unified analysis where $f_{\ast}$ lies in an anisotropic Besov space. 
Several prior works on NPIV-O 
and nonparametric proxy methods 
(see \Cref{sec:pcl}) \citet{singh2019kernel, mastouri2021proximal, singh2023kernel, bozkurt2025densityratiobasedproxycausal,HallHorowitz2005,bozkurt2025density} choose instead to assume the generalized source condition $f_{\ast}\in \mathcal{R}((T^{\ast}T)^{\beta})$ for some $\beta\geq 0$. However, as we have critiqued in the introduction, such an approach does not shed light on the separate contribution of the intrinsic smoothness of $f_{\ast}$ and the smoothing effect of $T$. It also suffers from a lack of interpretability since $T$ is a priori unknown. We also mention that \citet{HallHorowitz2005} derived optimal rates for a kernel density based estimator for NPIV-O, where the smoothness of $f_{\ast}(\cdot, \bo)$ is characterized via a generalized source condition with respect to the partial conditional expectation operator $T_{\bo}$ \citep[Section 4.3]{HallHorowitz2005} for $\bo\in\calO$. Such an assumption suffers from similar drawbacks to those outlined above, and it is moreover unclear how $f_{\ast}$'s smoothness in the direction of $O$ impacts learning rates.

To the best of our knowledge, our paper is the first theoretical analysis that simultaneously addresses the anisotropic smoothness of both $f_\ast$ and the operator $T$ in the $X$ and $O$ directions. We address both challenges by (i) introducing a novel Fourier-based measure of partial smoothing of $T$, and (ii) employing Gaussian kernel lengthscales that adapt to the anisotropic smoothness of $f_\ast$.

\subsection{Proximal causal learning (PCL)}
\label{sec:pcl}
The two challenges mentioned above also arise in a recent popular framework called proximal causal learning (PCL), which has gained considerable interest as a framework to identify and estimate causal effects from observational data, where the analyst only has access to imperfect proxies of the true underlying confounding mechanism without being able to observe the confounders directly~\citep{miao2018identifying, tchetgen2024introduction}. Our contributions to NPIV-O can directly be extended to this context. In the context of PCL, the (heterogeneous) dose response curve $f^{\ast}$ can be identified either via the \emph{outcome bridge function}~\citep{miao2018identifying, deaner2018proxy, mastouri2021proximal, xu2021deep, kallus2021causal, singh2023kernel}, which generalizes outcome regression, or via the \emph{treatment bridge function}~\citep{cui2024semiparametric, kallus2021causal, bozkurt2025densityratiobasedproxycausal, bozkurt2025density}, which generalizes inverse propensity weighting estimators. Analogous to the modern NPIV literature, the nonparametric estimators for bridge functions fall under the 2SLS approach \citep{deaner2018proxy, mastouri2021proximal, singh2023kernel, bozkurt2025densityratiobasedproxycausal, bozkurt2025density}, min-max optimization approach with either RKHS or deep neural networks as function classes~\citep{mastouri2021proximal, ghassami2022minimax, kallus2021causal}, or via spectral methods~\citep{sun2025spectral}. Notably, both outcome bridge function and treatment bridge function are identified via conditional moment constraints of the same form as NPIV-O (see Eq.~\eqref{eq:inverse_problem}), thus our theory in NPIV-O could be extended to 
the estimation of bridge functions in PCL. 

\subsection{Kernel ridge regression (KRR)} Our theoretical analysis of KIV-O builds on and extends existing theory in kernel ridge regression (KRR). The literature on KRR primarily follows two methodological lines: one based on \emph{empirical process}~\citep{steinwart2008support, steinwart2009optimal, eberts2013optimal, hang2021optimal, hamm2021adaptive} and one based on \emph{integral operator} techniques~\citep{de2005learning, smale2005shannon, smale2007learning, blanchard2018optimal, lin2020optimal, fischer2020sobolev, zhang2023optimality, zhang2024optimality}. 
In the context of learning an anisotropic Besov space function $f_\ast$ using KRR, the only available convergence rate is provided by~\citet{hang2021optimal}, which builds on an oracle inequality derived using empirical process techniques and hence necessitates a clipping operation on the KRR estimator. 
In our work, we are the first to remove the clipping step by leveraging integral operator techniques to directly control the finite sample estimation error.
Moreover, our analysis leverages state of the art results on optimal rates for vector-valued kernel ridge regression~\citep{JMLR:v25:23-1663,meunier2024optimalratesvectorvaluedspectral} to bound the statistical error arising from estimating a conditional mean embedding. 
\section{Theory}\label{sec:theory}
This section presents our main theoretical results on the non-asymptotic convergence rate of the learning risk defined in Eq.~\eqref{eq:L2_risk}.
\Cref{sec:ill_posedness} presents our assumptions on the conditional expectation operator $T$.
\Cref{sec:upper} presents an upper bound. 
\Cref{sec:lower_main} presents a minimax lower bound.

\subsection{Partial smoothing effect of $T$}\label{sec:ill_posedness}
The challenge of estimating $f_\ast$ via the inverse problem
$\E[Y \mid Z, O]=(T f_*)(Z, O)$ arises from its ill-posed nature: a small error in estimating $\E[Y \mid Z, O]$ may lead to a large error in estimating $f_\ast$. 
To address this challenge, we make the following assumptions. 
The first assumption enables unique identification of $f_\ast$.
\begin{ass}[$L^{\infty}$-completeness]\label{ass:T_injective}
The conditional distribution $P_{X\mid Z,O}$ satisfies that, for every bounded measurable function $f:\calX\times\calO\to\R$, if $\E[f(X,O)\mid Z, O] = 0$ holds $P_{ZO}$-almost surely, then $f(X,O)=0$ holds $P_{XO}$-almost surely.
\end{ass}
\Cref{ass:T_injective}, known as bounded completeness or $L^{\infty}$-completeness~\citep{d2011completeness, blundell2007semi}, is weaker than the $L^2$-completeness condition, which assumes that $T: L^2(P_{XO})\to L^2(P_{ZO})$ is injective. The latter is standard in the NPIV literature~\citep{newey_2003, HallHorowitz2005, carrasco2007linear,darolles2011nonparametric,  andrews2017examples, chen2014local, Chen_2018, chen2024adaptiveestimationuniformconfidence}.
Although we do not assume that the outcome $Y$ is bounded (see \Cref{ass:subgaussian}), the target heterogenous response curve $f_\ast$ is bounded (assumed in \Cref{ass:f_ast}), hence it suffices to impose the weaker $L^\infty$-completeness identification.
We refer the reader to \citet{andrews2017examples, d2011completeness} for sufficient conditions on the conditional distribution $P_{X\mid Z,O}$ such that bounded completeness holds.

Beyond identification, to establish a non-asymptotic rate of convergence for $f_\ast$, existing work on NPIV imposes additional assumptions on the smoothing properties of $T$ which are not compatible with the partial identity structure of $T$ imposed by the common variable $O$~\citep{blundell2007semi,Chen_2018,chen2011rate}. 
In contrast, as highlighted in \Cref{sec:intro}, with the existence of observed covariates $O$, our $T$ exhibits characteristics of a compact operator in the $X$ direction and acts as an identity operator in the $O$ direction. We thus propose a novel framework to characterize the \emph{partial} smoothing properties of $T$. 

We describe this partial smoothing in terms of the Fourier representation of a function $f$ on which $T$ acts. For $f\in L^1(\mathbb{R}^{d_x})$, its Fourier transform is defined as a Lebesgue integral \citep[9.1]{rudin1987real}: 
$\hat{f}(\cdot) = \int_{\R^{d_x}}f(\bx)\exp(-i\langle \bx,\cdot\rangle)\;\mathrm{d}\bx$.
One can extend the Fourier transform to $L^2(\R^{d_x})$ by defining it as a \emph{unitary} operator on $L^2(\R^{d_x})$~\citep[Theorem 9.13]{rudin1987real}. 
We use $\calF$ to denote this operator, and let $\calF^{-1}$ denote its inverse. 
In particular, $\calF^{-1}[\one[A]]$ is well-defined, where $\one[A]$ denotes the indicator function of a compact set $A\subset\R^{d_x}$. 
For any scalar $\gamma\in (0,1)$, we define the following two sets of functions:
\begin{talign}\label{eq:HF_LF}
\begin{aligned}
\mathrm{LF}(\gamma) := \{ f : \mathbb{R}^{d_x+d_o} \to \mathbb{R} \ \Big| \ 
& \forall \bo \in \mathcal{O}, \; [f(\cdot, \bo)] \in L^2(\mathbb{R}^{d_x}), \\
& \;\mathrm{supp}\big(\mathcal{F}[f(\cdot, \bo)]\big) 
\subseteq \{ \bomega_x \in \mathbb{R}^{d_x} : \|\bomega_x\|_2 \leq \gamma^{-1} \} \}.\\
\mathrm{HF}(\gamma) := \{ f : \mathbb{R}^{d_x+d_o} \to \mathbb{R} \ \Big| \ 
& \forall \bo \in \mathcal{O}, \; [f(\cdot, \bo)] \in L^2(\mathbb{R}^{d_x}), \\
& \;\mathrm{supp}\big(\mathcal{F}[f(\cdot, \bo)]\big) 
\subseteq \{ \bomega_x \in \mathbb{R}^{d_x} : \|\bomega_x\|_2 \geq \gamma^{-1} \} \}.
\end{aligned}
\end{talign}
where $\mathrm{supp}$ for an element of $L^2(\mathbb{R}^{d_x})$ is defined in \Cref{defi:fn_dist} and \Cref{defi:supp_dist} in the Supplementary. The set $\mathrm{LF}(\gamma)$ (respectively, $\mathrm{HF}(\gamma)$) consists of functions such that for every $\bo \in \mathcal{O}$, the slice function $f(\cdot, \bo)$ belongs to $L^1(\mathbb{R}^{d_x})$ and its Fourier transform is supported inside (respectively, outside) the centered ball of radius $\gamma^{-1}$ in the Fourier domain. See \Cref{fig:fourier_circle} for an illustration.

\begin{ass}[Fourier measure of partial ill-posedness of $T$]
\label{ass:T_frequency_ill_posedness}
There exists a constant $c_0>0$ and a parameter $\eta_0\in [0,\infty)$ depending only on $T$, such that for all $\gamma \in (0,1)$ and all functions $f \in \mathrm{LF}(\gamma) \cap L^\infty(P_{XO})$, the following inequality is satisfied:
    \begin{align*}
    \|Tf\|_{L^2(P_{ZO})} \geq c_0\gamma^{d_x\eta_0} \|f\|_{L^2(P_{XO})}.
    \end{align*}
In particular, $c_0$ does not depend on $\gamma$. 
\end{ass}
\begin{ass}[Fourier measure of partial contractivity of $T$]
\label{ass:T_contractivity}
There exists a constant $c_1>0$ and a parameter $\eta_1\in [0,\infty)$ depending only on $T$, such that for all $\gamma \in (0,1)$ and all functions $f \in \mathrm{HF}(\gamma) \cap L^\infty(P_{XO})$, the following inequality is satisfied:
    \begin{align*}
    \|Tf\|_{L^2(P_{ZO})} \leq c_1\gamma^{d_x\eta_1} \|f\|_{L^2(P_{XO})} .
    \end{align*}
In particular, $c_1$ does not depend on $\gamma$. 
\end{ass}

Assumptions~\ref{ass:T_frequency_ill_posedness} and \ref{ass:T_contractivity} are assumptions about the conditional distribution $P(X\mid Z,O)$. In \Cref{sec:appendix_examples} in the Supplement, for any positive integer $k\geq 1$, we construct a distribution $P_k(X,Z,O)$ such that, for the conditional expectation operator $T$ defined by $P_k(X\mid Z,O)$, a weaker version of Assumptions~ \ref{ass:T_contractivity} and \ref{ass:T_frequency_ill_posedness} is satisfied with $\eta_0 = \eta_1 = k$ (where we restrict to considering functions $f(\bx,\bo) = g(\bx)h(\bo)$, and impose a further technical restriction for \Cref{ass:T_frequency_ill_posedness}).
If Assumption \ref{ass:T_frequency_ill_posedness} and \ref{ass:T_contractivity} hold simultaneously, and $P_{XO}$ is absolutely continuous with respect to the Lebesgue measure on $\calX\times\calO$, then $\eta_0\geq \eta_1$ and $c_0\leq c_1$ (\Cref{lem:eta_0_eta_1_appendix} in \Cref{sec:appendix_examples}).
In the remainder of the manuscript, we assume that the distribution $P_{ZXOY}$ is fixed and we set the constants $c_0 = c_1 = 1$ for notational simplicity. \Cref{ass:T_frequency_ill_posedness} and \Cref{ass:T_contractivity} characterize the \emph{mildly ill-posed} regime in the NPIV literature.

\begin{figure}[t]
\centering
\vspace{-20pt}
    \includegraphics[width=0.4\linewidth]{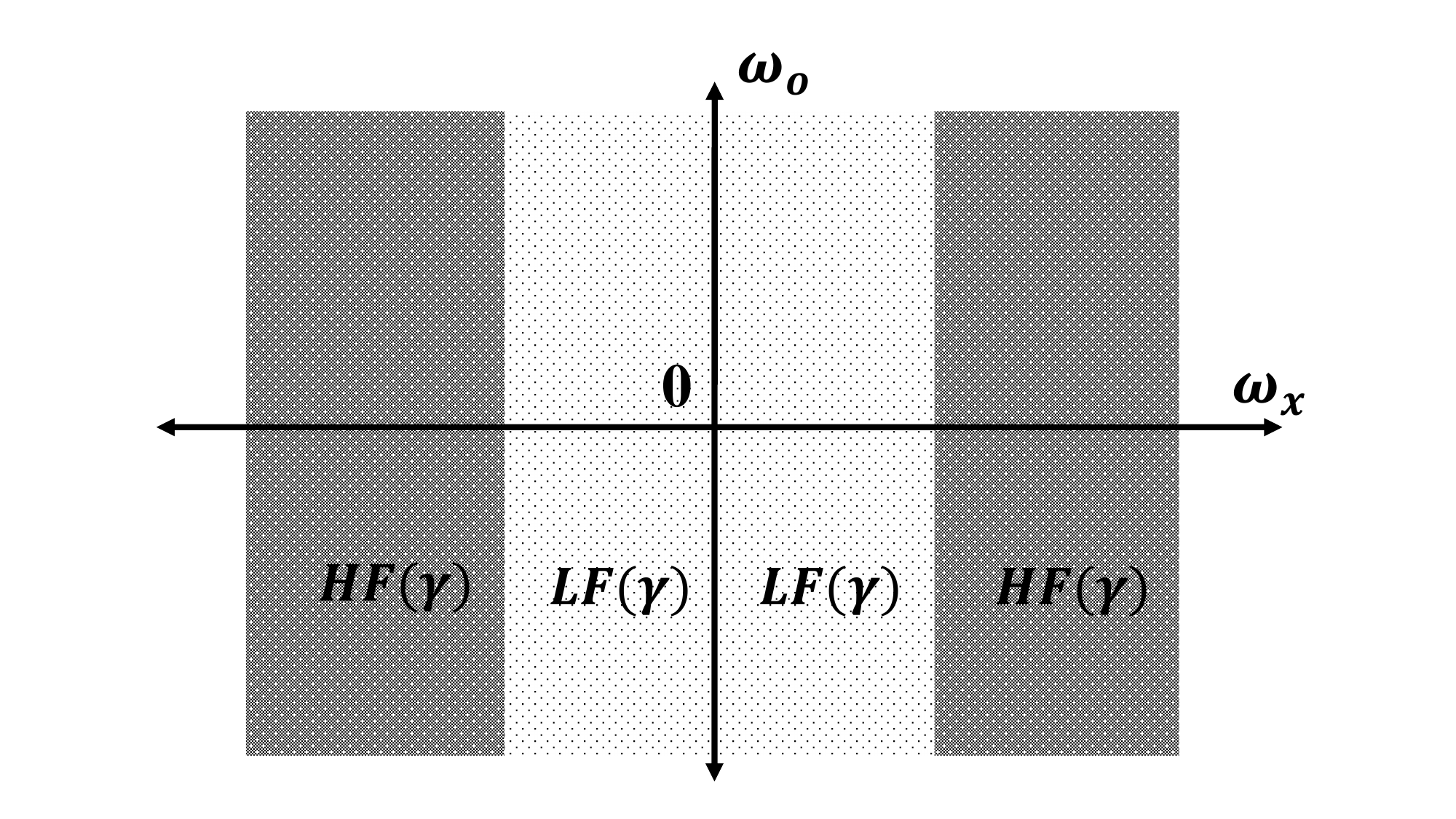}
    \vspace{-10pt}
    \caption{Illustration for LF($\gamma$) and HF($\gamma$).}
    \label{fig:fourier_circle}
    \vspace{-20pt}
\end{figure} 
\Cref{ass:T_contractivity} quantifies the \emph{partial smoothing} effect of $T$ on the high-frequency components of a function $f$ with respect to $X$; while  \Cref{ass:T_frequency_ill_posedness} captures the \emph{partial anti-smoothing} behaviour of $T$ on the low-frequency components of a function $f$ with respect to $X$. When the treatment $X$ is exogenous and we take $X=Z$, then $T$ is an identity mapping so $\eta_0=\eta_1=0$ and NPIV-O reduces to non-parametric regression from $(X,O)$ to $Y$.

To motivate the partial smoothing effect of $T$, notice that the bounded self-adjoint operator $T^{\ast}T: L^2(P_{XO}) \to L^2(P_{XO})$ acts on $f\in L^2(P_{XO})$ as follows:
\begin{align}\label{eq:TastT}
((T^{\ast}T) f)(\bx',\bo) = \int_{\calX}f(\bx,\bo)L(\bx,\bx';\bo)\;\mathrm{d}\bx, 
\end{align}
where $L(\bx,\bx';\bo):= \int_{\calZ}p(\bx\mid \bz,\bo)p(\bz\mid\bx',\bo)\;\mathrm{d}\bz$. Consider two subsets of $L^2(P_{XO})$. 
\begin{align*}
    \mathfrak{G}_{X} &= \left\{ g \in L^2(P_{XO}) \;\middle|\; \exists \Tilde{g} \in L^2(P_X) \text{ such that } \forall \bx\in\calX, \bo\in\calO, \; g(\bx, \bo) = \Tilde{g}(\bx) \right\}, \\
    \mathfrak{G}_{O} &= \left\{ g \in L^2(P_{XO}) \;\middle|\; \exists \Tilde{g} \in L^2(P_O) \text{ such that } \forall \bx\in\calX, \bo\in\calO, \; g(\bx, \bo) = \Tilde{g}(\bo) \right\} .
\end{align*}
Under mild conditions on the conditional distribution $p(\bx\mid \bz, \bo)$~\citep[Assumption A.1]{darolles2011nonparametric}, $T^{\ast}T|_{\mathfrak{G}_{X}}$ 
($T^{\ast}T$ restricted to $\mathfrak{G}_{X}$) is compact and its smoothing effect can be quantified through its eigenvalue decay; while in contrast, $T^{\ast}T|_{\mathfrak{G}_{O}}$ is an identity operator: for $g\in \mathfrak{G}_{O}$
\begin{align}\label{eq:TTast_mapping}
    ((T^{\ast}T) g)(\bx, \bo) = \int_{\calX}g(\bx,\bo)L(\bx,\bx';\bo)\;\mathrm{d}\bx = \tilde{g}(\bo) \int_{\calX}L(\bx,\bx';\bo)\;\mathrm{d}\bx = g(\bx, \bo).
\end{align}
Therefore, when we incorporate observed covariates $O$, the conditional expectation operator $T$ acts as a compact operator in the $X$ direction and as an identity operator in the $O$ direction.  As a result, we propose to characterize its \emph{partial} smoothing properties through measure of \emph{partial} contractivity (\Cref{ass:T_contractivity}) and measure of \emph{partial} ill-posedness (\Cref{ass:T_frequency_ill_posedness}).

In the literature on NPIV, conditions similar to \Cref{ass:T_contractivity} and \Cref{ass:T_frequency_ill_posedness} have been employed to quantify the smoothing effect of $T$. 
For instance, \citet{chen2011rate, meunier2024nonparametricinstrumentalregressionkernel} use the so-called link condition and reverse link condition which relate the smoothness of the hypothesis space with that of $\calR(T^\ast T)$; \citet{Chen_2018, blundell2007semi, chen2024adaptiveestimationuniformconfidence} employ the sieve measure of ill-posedness and stability conditions, which quantify the smoothing effect of $T$ on functions in the hypothesis space spanned by a sieve basis. 
Our \Cref{ass:T_contractivity} and \Cref{ass:T_frequency_ill_posedness} share strong resemblance with the latter. To see why, recall the definition of Gaussian RKHS $\calH_{\gamma_x}$ with length-scale $\gamma_x$ through Fourier transforms~\citep[Theorem 10.12]{wendland2004scattered}:
\begin{talign*}
    \calH_{\gamma_x} = \left\{f : \R^{d_x} \to \R \mid \int_{\R^{d_x}} \left|\calF[f](\bomega_x)\right|^2\; \exp\left(\frac{1}{4} \gamma_x^{2} \|\bomega_x\|_2^2\right) \; \mathrm{d} \bomega_x < \infty \right\} ,
\end{talign*}
where we can see that for $f\in\steve{\calH_{\gamma_x}}$, the bulk of its Fourier spectrum would belong to the ball $\{\bomega_x: \|\bomega_x\|_{2} \leq \gamma_x^{-1}\}$ with the remaining spectrum decaying exponentially as $\|\bomega_x\|_{2} \to \infty$. 
We formulate our \Cref{ass:T_contractivity} and \Cref{ass:T_frequency_ill_posedness} with Fourier transforms rather than Gaussian RKHSs for potential applications beyond Gaussian RKHSs. 
A closely related work is \citet{kankanala2025generalized}, which employs a \emph{local} sieve measure of ill-posedness for functions in the RKHS. 
Unfortunately, these conditions, including ours, are hard to verify in practice as $T$ is unknown.

\begin{rem}[Connection with sieve measure of ill-posedness]
\label{rem:connection_smoip}
In this remark, we connect \Cref{ass:T_frequency_ill_posedness} to the sieve measure of ill-posedness condition employed in the analysis of sieve 2SLS~\citep{blundell2007semi, Chen_2018, chen2024adaptiveestimationuniformconfidence,kim2025optimality}. For this remark, we omit observed covariates $O$, and take $T : L^2(P_{X})\to L^2(P_{Z})$. The sieve measure of ill-posedness is defined as
\begin{align}\label{eq:sieve_moip}
    \tau_J^{\mathrm{sieve}} := \sup_{0\neq f\in \Psi_J} \frac{\|f\|_{L^2(P_{X})}}{\|Tf\|_{L^2(P_{Z})}},
\end{align}
where $\Psi_J$ denotes the $J$th sieve space for the treatment variable \citep[Section 3]{Chen_2018}. For this remark, we let $\Psi_J$ be the linear span of cardinal B-splines of order $\mathfrak{m}$ up to resolution $\mathfrak{K}$ with $J = (2^{\mathfrak{K}} + \mathfrak{m} + 1)^{d_x} \asymp 2^{\mathfrak{K}d_x}$ \citep[Section 5]{devore1993constructive}. We note that the parameter $\gamma^{-1}$, where $\gamma$ occurs in the definition of the function space $\mathrm{LF}(\gamma)$ in Eq. \eqref{eq:HF_LF}, plays a role analogous to the resolution level $\mathfrak{K}$ for cardinal B-splines. Indeed, $\mathrm{LF}(\gamma)$ for smaller values of $\gamma$ (or $\Psi_J$ for larger values of $J$) correspond to a class of less smooth functions. The above observation and the form of Eq. \eqref{eq:sieve_moip} thus suggests an analogous definition:
\begin{align}\label{eq:kernel_moip}
    \tau_{\gamma}^{\mathrm{Fourier}} := \sup_{0\neq f\in \mathrm{LF}(\gamma)\cap L^{\infty}(P_{XO})} \frac{\|f\|_{L^2(P_{X})}}{\|Tf\|_{L^2(P_{Z})}}. 
\end{align}
We can thus restate \Cref{ass:T_frequency_ill_posedness} as: there exists a constant $c_0>0$ and a parameter $\eta_0\in [0,\infty)$ depending only on $T$, such that for all $\gamma\in (0,1)$, the following inequality is satisfied: $\tau_{\gamma}^{\mathrm{Fourier}}\leq c_0^{-1}\gamma^{-d_x\eta_0}$. 
In the sieve NPIV literature \citep{blundell2007semi,chen2024adaptiveestimationuniformconfidence, Chen_2018}, an NPIV model is said to be \emph{mildly ill-posed} if $\tau_{J}^{\mathrm{sieve}} = O(J^{\eta}) = O(2^{\eta\mathfrak{K}d_x})$. By the analogy between $\tau_{\gamma}^{\mathrm{Fourier}}$ and $\tau_{J}^{\mathrm{sieve}}$, we see that our \Cref{ass:T_frequency_ill_posedness} characterizes a \emph{mildly ill-posed} regime.
\end{rem}

\begin{rem}[Existing treatment of partial smoothing of $T$ in NPIV-O]\label{rem:partial_smoothing}
Existing work that concerns NPIV-O in the literature circumvents the partial identity structure of $T$ either by imposing additional structural assumptions~\citep{blundell2007semi, syrgkanis2019machine} or by stratifying the problem on $O$~\citep{horowitz2011applied} thereby reducing NPIV-O to NPIV, which is statistically inefficient and scales poorly with the dimension of $O$. 
Instead of $T$, \citet[Section 3.3]{Chen_2018} consider the compactness and the smoothing effect of the \emph{partial} conditional expectation operator $T_{\bo}: L^2(P_{X\mid O = \bo})\to L^2(P_{Z\mid O=\bo})$ for each $\bo\in \calO $. \citet[Section 3.3]{Chen_2018} proves that, if $\Psi_J$ (defined in the above remark) equals the span of the first $J$ eigenvectors of $T_{\bo}^\ast T_{\bo}$ for any $\bo\in\calO$, then
\begin{align}
\label{eq:chen_smoip_observed}
    \tau_J^{\mathrm{sieve}} \asymp \mathbb{E}_{\bo\sim O}[\mu_{J, \bo}^2]^{-\frac{1}{2}},
\end{align}
where $\mu_{J,\bo}$ is the $J$th singular value of $T_{\bo}$ arranged in non-increasing order. \citet{Chen_2018} then claims that the convergence rates derived for the standard NPIV can be extended to NPIV-O. However, we identify an essential oversight: \citet[Assumption 4 (ii)]{Chen_2018}, which is required for their $L^2$-norm upper bound, cannot hold in the NPIV-O setting. This assumption states that for $\Psi_J$, $f_{\ast}$ satisfies the following inequality:
\begin{align*}
    \tau_J^{\mathrm{sieve}}\|T(f_{\ast} - \Pi_J f_{\ast})\|_{L^2(P_{ZO})}\leq \|f_{\ast} - \Pi_{J}f_{\ast}\|_{L^2(P_{XO})},
\end{align*}
where $\Pi_J: L^2(P_{XO})\to \Psi_J$ denotes the $L^2$-orthogonal projection from $L^2(P_{XO})$ onto $\Psi_J$. 
In the NPIV-O setting, if $f_{\ast}$ depends only on $\bo$ and $P_{XO}$ is the Lebesgue measure, then the projection $\Pi_J f_{\ast}$ also depends only on $\bo$. 
This implies $(T^\ast T)(f_{\ast} - \Pi_J f_{\ast}) = f_{\ast} - \Pi_J f_{\ast}$ as per Eq.~\eqref{eq:TTast_mapping}, forcing ${\tau_J^{\mathrm{sieve}}}^{-1} = O(1)$, causing a contradiction with the measure of ill-posedness condition that $\tau_J^{\mathrm{sieve}} = \calO(J^{\eta_0})$ unless $\eta_0 = 0$. 
\citet{HallHorowitz2005} addresses NPIV-O by assuming smoothness of $f_\ast(\cdot, \bo)$ relative to $\mathcal{R}(T_\bo^\ast T_\bo)$ for each $\bo\in\calO$, which is hard to interpret because $T_\bo$ is
unknown in NPIV-O.
\end{rem}

\subsection{Upper Bound for KIV-O}\label{sec:upper}
To obtain the upper learning rate for KIV-O in \Cref{sec:algorithm}, we need to impose the following assumptions about the data-generating distribution. 
\begin{ass}\label{ass:equiv_leb}
1. The joint probability measures $P_{ZO}$ and $P_{XO}$ admit probability density functions $p_{ZO}$ and $p_{XO}$. There exists a universal constant $a>0$ such that $a^{-1} \geq p_{ZO}(\bz,\bo)\geq a$ for all $(\bz,\bo)\in [0,1]^{d_z+d_o}$ and $a^{-1}\geq p_{XO}(\bx,\bo)$ for all $(\bx,\bo)\in [0,1]^{d_x+d_o}$ and $p_{XO}(\bx,\bo)\geq a$ for all $(\bx,\bo)\in [1/4,3/4]^{d_x+d_o}$.
2. \Cref{assn: technical} holds with the following $m_z, m_o \in \N^+$:
\begin{talign}\label{def:m_z,m_o}
    \begin{aligned}
    m_o := \left\lceil \frac{d_o}{2}\frac{1+2\left(\frac{s_x}{d_x}+\eta_1\right) + \frac{d_o}{s_o}\left(\frac{s_x}{d_x}+\eta_1\right)}{1 + 2\left(\frac{s_x}{d_x}+\eta_1\right)} \right\rceil + 1, \quad m_z := \left\lceil \frac{d_z}{d_o}m_o\right\rceil .
    \end{aligned}
\end{talign}
\end{ass}
The requirement that $p_{XO}(\bx,\bo)\geq a$ for all $(\bx,\bo)\in [1/4,3/4]^{d_x+d_o}$ is a mild assumption to ensure that the $\calH_{\gamma_x, \gamma_o}$-norm of the kernel mean embedding of $P_{XO}$ is bounded away from zero. This plays a role in the control of the estimation error for Stage II regression (see \Cref{prop: projected_est_error} in the Supplement). The choice of $1/4,3/4$ is arbitrary and can be replaced by any fixed unequal values in $(0,1)$. 
The constraint on $m_o$ depends on $(s_x, s_o)$ because the embedding norm of the RKHS $\calH_{FO}$ (defined in \Cref{sec:defi_H_FO}) into the mixed-smoothness Sobolev space scales as $\gamma_o^{-m_o}$, where $\gamma_o$ itself depends on both $s_x$ and $s_o$ in \Cref{thm:upper_rate}; this embedding norm must be controlled, a requirement referred to as the \textsc{(EMB)} condition in \citet{fischer2020sobolev}. See  Eq.~\eqref{eq:n_large_than_A} for details.

\begin{ass}
\label{ass:subgaussian}
For all $(\bz,\bo)\in \calZ\times \calO$, the residual $\upsilon:= Y - (Tf_{\ast})(Z,O)$ is subgaussian conditioned on $Z=\bz,O=\bo$ with subgaussian norm at most $\sigma$. 
\end{ass}
In the NPIV literature, existing work assumes a moment condition on the residual $\upsilon$~\citep{blundell2007semi,HallHorowitz2005,Chen_2018, chen2011rate}, which is weaker than our \Cref{ass:subgaussian}. However, their corresponding high probability upper bounds only guarantee polynomially decaying tails, as a result of Chebyshev’s inequality (see, for example, the proof of Lemma F.9 in \citet{Chen_2018}, p.40).
In contrast, our upper bound holds in high probability with \emph{subexponential} tails. Our sharper guarantee is a consequence of our applying Bernstein concentration inequality and the more advanced techniques in the analysis of kernel ridge regression~\citep{fischer2020sobolev, eberts2013optimal, hang2021optimal}.   

Now we are ready to state the upper learning rate of $\|[\hat{f}_{\lambda}] - f_{\ast}\|_{L^2(P_{XO})}$. We remind the reader that $\tilde{n}, n$ denote respectively the number of Stage 1 and Stage 2 samples. We let Stage I kernels $k_{O, 1}, k_{Z}$ be Mat\'{e}rn kernels whose associated RKHSs $\calH_{O,1}$ and $\calH_Z$ are respectively norm equivalent to $W_2^{t_o}(\calO)$ and $W_2^{t_z}(\calZ)$, for $t_o, t_z$ to be specified. We let Stage II kernels $k_X, k
_{O,2}$ be Gaussian kernels with respective lengthscales $\gamma_x, \gamma_o$.  
\begin{thm}[Upper learning rate for KIV-O]\label{thm:upper_rate}
Suppose Assumptions~\ref{assn: technical}, \ref{ass:equiv_leb} hold with parameters $m_z, m_o \in \mathbb{N}^{+}$, Assumption~\ref{ass:f_ast} holds with parameters $s_x, s_o > 0$, Assumptions~\ref{ass:T_frequency_ill_posedness}, \ref{ass:T_contractivity} hold with parameters $\eta_0, \eta_1$. We further suppose Assumptions~\ref{ass:T_injective} and \ref{ass:subgaussian} hold. We assume $t_o, t_z$ satisfy $2t_o \geq m_o > t_o > d_o/2$, $2 t_z \geq m_z > t_z > d_z/2$. We let
\begin{align}\label{eq:lengthscale_main}
    \gamma_x = n^{-\frac{\frac{1}{d_x}}{1+2(\frac{s_x}{d_x}+\eta_1)+\frac{d_o}{s_o}(\frac{s_x}{d_x}+\eta_1)}}, \quad \gamma_o = n^{-\frac{\frac{1}{s_o}(\frac{s_x}{d_x}+\eta_1)}{1+2(\frac{s_x}{d_x}+\eta_1)+\frac{d_o}{s_o}(\frac{s_x}{d_x}+\eta_1)}}. 
\end{align}  Define $\meffective := (m_z t_z^{-1}) \wedge (m_o t_o^{-1})$ and $\deffective := (d_z t_z^{-1}) \vee (d_o t_o^{-1})$. 
Let Stage I regularization parameter $\xi$ be given by $\stageonesamples^{-\frac{1}{\meffective + \deffective + \zeta}}$ for any $\zeta>0$; and Stage II regularization parameter $\lambda$ be given by $n^{-1}$. Suppose that $n\geq 1$ is sufficiently large, and $\stageonesamples \geq 1$ satisfies
\begin{align}\label{eq:m_sufficient_large}
\begin{aligned}
    \stageonesamples \gtrsim n^{\frac{\meffective + \deffective/2 + \zeta}{\meffective-1}} \vee n^{2 \frac{\meffective + \deffective/2 + \zeta}{\meffective}}.
\end{aligned}
\end{align}
Then with $P^{n+\stageonesamples}$-probability at least $1-  40e^{-\tau}$, we have
\begin{align}
\label{eq:upper_rate}
    \left\|\left[\hat{f}_{\lambda}\right] - f_{\ast}\right\|_{L^2(P_{XO})} \lesssim \tau n^{-\frac{\frac{s_x}{d_x}+\eta_1-\eta_0}{1+2(\frac{s_x}{d_x}+\eta_1)+\frac{d_o}{s_o}(\frac{s_x}{d_x}+\eta_1)}} \cdot (\log n)^{\frac{d_x+d_o+1+d_x\eta_0}{2}} .
\end{align}
\end{thm}

\begin{rem}
\label{rem: stage_1_lower_bound}
    In \Cref{thm:upper_rate}, we present the regime where the Stage I sample size $\tilde{n}$ is sufficiently large relative to the Stage II sample size (see Eq.~\eqref{eq:m_sufficient_large}). 
    This is the appropriate regime where we can study rate-optimality because we present a minimax lower bound in \Cref{thm:lower_rate} with respect to the class of estimators that only utilize Stage II samples. 
    It is nevertheless possible to derive upper bounds with respect to both $n$ and $\tilde{n}$ without restrictions on the relative size of $n, \tilde{n}$. It remains a challenging open problem, however, to establish rate optimality for estimators utilizing a split dataset of the form $\{(\tilde{\bz}_i, \tilde{\bo}_i, \tilde{\bx}_i)\}_{i=1}^{\tilde{n}}$ and $\{(\bz_i, \bo_i, y_i)\}_{i=1}^{n}$, even for the standard NPIV setting \citep{chen2011rate, Chen_2018, meunier2024nonparametricinstrumentalregressionkernel}. 
\end{rem}

\begin{rem}[Interpolation between NPIV and non-parametric regression]\label{rem:interpolate}
\label{rem:ub_main_rem}
Our derived upper rate interpolates between the known optimal $L^2$-rates for NPIV without observed covariates and anisotropic kernel ridge regression.  \\
1. When $\eta_0 = \eta_1 = 0$ and $X=Z$, i.e. $T$ is the identity mapping, our setting reduces to nonparametric regression where the target function lies in the anisotropic Besov space $B_{2, \infty}^{s_x, s_o}(\calX\times\calO)$. The upper rate simplifies to $\tilde{\calO}_P(n^{-\frac{1}{2\tilde{s}+1}})$ with $\tilde{s} = (d_o/s_o + d_x/s_x)^{-1}$ being the intrinsic smoothness, which matches the known optimal learning rate of regression with an anisotropic Besov target function~\citep{hoffman2002random,hang2021optimal}. \\
2. When $d_o = 0$, our setting reduces to NPIV without observed covariates where the target function $f_\ast$ lies in an isotropic Besov space $B_{2, \infty}^{s_x}(\calX)$. We take $\eta_1 = \eta_0 = \eta$ following \citet{Chen_2018, chen2011rate} which employs a single parameter to characterize both the ill-posedness and contractivity, then our upper learning rate simplifies to $\tilde{\calO}_P(n^{-\frac{s_x}{d_x+2(s_x+\eta d_x)}})$, which matches the known optimal rate in NPIV regression~\citep[Corollary 3.1]{Chen_2018}.
\end{rem}

\subsubsection{Proof sketch}
The proof of \Cref{thm:upper_rate} is given in \Cref{sec:proof_upper} in the Supplement. Here we give an outline of our proof to facilitate a deeper understanding of both the assumptions and the results. Define $\bar{f}_{\lambda}$ as the oracle estimator for Stage II with access to the true conditional mean embedding $F_\ast$ and recall $\hat{f}_{\lambda}$ for comparison:
\begin{align}
    \bar{f}_{\lambda} &:= \underset{f\in \calH_{\gamma_x,\gamma_o}}{\arg\min} \lambda \|f\|^2_{\calH_{\gamma_x,\gamma_o}} + \frac{1}{n}\sum\limits_{i=1}^{n}( y_i - \langle f, F_\ast(\bz_i,\bo_i) \otimes \phi_{O, \gamma_o}(\bo_i) \rangle_{\calH_{\gamma_x,\gamma_o}} )^2 . \label{eq:bar_f_lambda_main} \\
    \hat{f}_{\lambda} &:= \underset{f\in \calH_{\gamma_x,\gamma_o}}{\arg\min} \lambda \|f\|^2_{\calH_{\gamma_x,\gamma_o}} + \frac{1}{n}\sum\limits_{i=1}^{n}( y_i - \langle f, \hat{F}_{\xi}(\bz_i,\bo_i) \otimes \phi_{O, \gamma_o}(\bo_i) \rangle_{\calH_{\gamma_x,\gamma_o}} )^2 \label{eq:hat_f_lambda_main} .
\end{align}
The proof can be summarized in 3 steps.
We upper bound $\|T[\hat{f}_\lambda]-T[\bar{f}_\lambda]\|_{L^2\left(P_{Z O}\right)}$ in \textit{Step 1} and upper bound $\|T[\bar{f}_\lambda]-T f_*\|_{L^2(P_{Z O})}$ in \textit{Step 2}, which induce an upper bound on $\|T[\hat{f}_\lambda]-T f_*\|_{L^2(P_{Z O})}$ via a triangular inequality. In \textit{Step 3}, we apply the partial measure of ill-posedness to obtain an upper bound on $\|[\hat{f}_\lambda]- f_*\|_{L^2(P_{XO})}$. 
We highlight our technical contributions in each step with \Cref{rem:tensor_krr} and \Cref{rem:no_clipping}.

\textit{Step 1} We upper bound $\|T[\hat{f}_\lambda]-T[\bar{f}_\lambda]\|_{L^2\left(P_{Z O}\right)}$. By their definition in Eq.~\eqref{eq:hat_f_lambda_main} and Eq.~\eqref{eq:bar_f_lambda_main}, the discrepancy between $\hat{f}_\lambda$ and $\bar{f}_\lambda$ arises solely from the difference between $\hat{F}_{\xi}$ and $F_\ast$; hence it corresponds to Stage I error. First, we prove in \Cref{prop:T_hat_f_bar_f} that $\|T[\hat{f}_\lambda]-T[\bar{f}_\lambda]\|_{L^2\left(P_{Z O}\right)}$ can be upper bounded by an expression involving $\|\hat{F}_{\xi}-F_*\|_{\calG}$ and $\|F_*-\hat{F}_{\xi}\|_{L^2(\calZ\times\calO; \mathcal{H}_{X,\gamma_x})}$, where we recall that $\calG$ denotes the unique vector-valued RKHS induced by
the operator-valued kernel $K$ defined in Eq. \eqref{eq:op_val_kernel}. 
To obtain high-probability upper bounds on both of these quantities, we adapt the existing optimal learning rates on CME from \citet{lietal2022optimal} to our setting with a tensor product RKHS. 
The fact that we require upper learning rate for $\|F_*-\hat{F}_{\xi}\|_{\calG}$ imposes the conditions $m_o>t_o$ and $m_z>t_z$, so the CME $F_\ast$ lies in a smoother space than RKHS $\calG$. Specifically, $F_{\ast}$ belongs to the \emph{power space} $[\calG]^{\meffective}$ (See Eq.~\eqref{eq:F_ast_power_space}).
In addition, the constraints $2 t_o>m_o$ and $2 t_z>m_z$ reflect the saturation effect inherent in Tikhonov regularization~\citep{bauer2007regularization, lu2024saturation, meunier2024optimalratesvectorvaluedspectral}. 
These constraints can be removed to allow for greater smoothness of $F_\ast$ by employing spectral regularization~\citep{meunier2024optimalratesvectorvaluedspectral}. 
The appearance of an arbitrarily small $\zeta>0$ in Stage I regularization parameter $\xi$ reflects the fact that, after reordering, the eigenvalues of the tensor product operator exhibit slower decay than those of the individual components, owing to an extra logarithmic term~\citep{krieg2018tensor}.

\begin{rem}[Tensor product kernel ridge regression]
\label{rem:tensor_krr}
    In the above step, we generalize the upper learning rate for vector-valued kernel ridge regression to the setting of tensor product kernels (See \Cref{prop:cme_rate}). Although tensor product kernels have been widely used in kernel-based hypothesis tests~\citep{gretton2007kernel,GreGyo10,SejGreBer13,gretton2015characteristic,ZhaFilGreSej18,AlbLaurMarMey22,szabo2018characteristic},  kernel independent component analysis \citep{BacJor02,SheJegGre09},  and feature selection \citep{SonSmoGreBedetal12,LiPogSutGre21}, they have  been less well studied  in kernel ridge regression, with the exception of \citet{hang2021optimal} for Gaussian kernels. Our analysis is also applicable to real-valued kernel ridge regression with tensor product kernels.
\end{rem}
\textit{Step 2} We upper bound $\|T[\bar{f}_\lambda]-T f_*\|_{L^2(P_{Z O})}$. We follow the approach in \citet{blanchard2018optimal, meunier2024nonparametricinstrumentalregressionkernel}, where it is observed (in the standard NPIV case) that this term corresponds to the learning risk of a kernel ridge regression problem with an appropriately defined RKHS $\calH_{FO} \subseteq \{\calZ\times \calO\to \mathbb{R}\}$, namely the RKHS induced by the feature map $(\bz,\bo)\mapsto F_{\ast}(\bz,\bo)\otimes \phi_{O,2}(\bo)$. We refer the reader to \Cref{sec:defi_H_FO} in the Supplement for the definition of $\calH_{FO}$. Our construction is adapted from \citet[Appendix E.1.2]{meunier2024nonparametricinstrumentalregressionkernel}. However, unlike \citet{blanchard2018optimal, meunier2024nonparametricinstrumentalregressionkernel} who only consider fixed RKHSs, we employ tensor product Gaussian RKHS $\calH_{\gamma_x, \gamma_o}$ with \emph{adaptive} length-scales $\gamma_x, \gamma_o$ (as in Eq.~\eqref{eq:lengthscale_main}) to capture the anisotropic smoothness of $f_\ast \in B^{s_x, s_o}_{2,\infty}(\calX\times \calO)$~\citep{hang2021optimal}. 
When $\eta_1 = 0$, our choice of length-scales $\gamma_x, \gamma_o$ coincides with that of kernel ridge regression in \citet{hang2021optimal}. 
The logarithmic factor \smash{$(\log n)^{\frac{d_x+d_o+1}{2}}$} in Eq.~\eqref{eq:upper_rate} arises from the entropy numbers of the Gaussian RKHS in this step (see \Cref{sec:capacity} in the Supplement and \citet[Proposition 1]{hang2021optimal}). 
\begin{rem}[Gaussian kernel ridge regression with Besov space target functions]
\label{rem:no_clipping}
    To establish learning rates for kernel ridge regression, there are two main techniques in the literature: the empirical process technique~\citep{steinwart2008support,steinwart2009optimal} and the integral operator technique~\citep{fischer2020sobolev,Lin_2020,smale2007learning,caponnetto2007optimal,blanchard2018optimal}. Previous works on Gaussian kernel ridge regression with Besov space targets~\citep{eberts2013optimal, hang2021optimal,hamm2021adaptive}  rely on an oracle inequality proved via empirical process techniques~\citep[Theorem 7.23]{steinwart2008support}, which necessitates a clipping operation on the estimator. 
    On the other hand, the integral operator technique avoids the clipping operation, but it requires the target $f_\ast$ in a power space of the RKHS---a condition known as the source condition~\citep[SRC]{fischer2020sobolev}---which does not hold for Besov space targets and Gaussian RKHSs. 

    In our proof in step 2, we prove an upper learning rate of Gaussian kernel ridge regression with Besov space target functions \emph{without} clipping the estimator.
    Specifically, we combine the two techniques above, in that we bound the approximation error with \citet{hang2021optimal}, while we bound the estimation error with the integral operator technique~\citep{fischer2020sobolev}.
    To be precise, define 
    \begin{align}
    \label{eq:f_lambda_main}
        f_{\lambda} := \argmin_{f\in \calH_{\gamma_x,\gamma_o}}\lambda\|f\|^2_{\calH_{\gamma_x,\gamma_o}} + \|T([f] - f_{\ast})\|^2_{L^2(P_{ZO})}. 
    \end{align}
    The learning risk $ \|T([\bar{f}_{\lambda}] - f_\ast) \|_{L^2(P_{ZO})}$ can be decomposed into an estimation error term $\|T([\bar{f}_{\lambda}] - [f_{\lambda}])\|_{L^2(P_{ZO})}$ and an approximation error term $\|T([f_{\lambda}] - f_{\ast})\|_{L^2(P_{ZO})}$. 
    We upper bound the estimation error with \citet[Theorem 16]{fischer2020sobolev}, once we prove that the RKHS $\calH_{FO}$ satisfies an \emph{embedding property} (see \citet[EMB]{fischer2020sobolev}), which avoids the clipping operation.
    We upper bound the approximation error with \citet[Theorem 4]{hang2021optimal}, which avoids the source condition. 
\end{rem}

\textit{Step 3} We combine the above two terms $\|T[\hat{f}_{\lambda}] - T[\bar{f}_{\lambda}]\|_{L^2(P_{ZO})}$ and $\|T[\bar{f}_{\lambda}] - Tf_{\ast}\|_{L^2(P_{ZO})}$ through a triangle inequality, which gives an upper bound on the projected risk $\|T[\hat{f}_{\lambda}] - Tf_{\ast}\|_{L^2(P_{ZO})}$.
\vspace{-10pt}
\begin{align}\label{eq:projected_risk_main}
    \|T[\hat{f}_{\lambda}] - Tf_{\ast}\|_{L^2(P_{ZO})} \lesssim (\log n)^{\frac{d_x+d_o+1}{2}} n^{-\frac{\frac{s_x}{d_x}+\eta_1}{1+2(\frac{s_x}{d_x}+\eta_1)+\frac{d_o}{s_o}(\frac{s_x}{d_x}+\eta_1)}} .
\end{align}
To bound the unprojected risk $\|[\hat{f}_{\lambda}] - f_{\ast}\|_{L^2(P_{ZO})}$, it seems all that remains is to apply the Fourier measure of partial ill-posedness in \Cref{ass:T_frequency_ill_posedness} to remove $T$. 
Unfortunately, however, \Cref{ass:T_frequency_ill_posedness} only holds for functions in $\mathrm{LF}(\gamma)$ (Eq.~\eqref{eq:HF_LF}) with low partial Fourier frequency in $\calX$. Notice that for a function $f \in \calH_{\gamma_x,\gamma_o}$, its partial Fourier spectrum $|\calF[f(\cdot,\bo)](\bomega_x)|$ decays exponentially fast as $\|\bomega_x\|_2\to\infty$~\citep[Theorem 10.12]{wendland2004scattered}
\begin{talign}
\label{eq:fourier_f_tail_main}
    (\forall\bo\in \calO), \quad\int_{\mathbb{R}^{d_x}}|\calF[f(\cdot,\bo)](\bomega_x)|^2\exp\left(\frac{1}{4} \gamma_x^2\|\bomega_x\|^2_{2} \right) \;\mathrm{d} \bomega_x < \infty .
\end{talign}
This exponential decay implies that $\hat{f}_\lambda(\cdot, \bo)$ satisfies the conditions of \Cref{ass:T_frequency_ill_posedness} up to some logarithmic factors as reflected by \smash{$(\log n)^{\frac{d_x\eta_0}{2}}$} in Eq.~\eqref{eq:upper_rate}.
For $f_{\ast}$, we find an auxiliary function $f_{\mathrm{aux}}\in \calH_{\gamma_x,\gamma_o}$ that is close to $f_{\ast}$ in $L^2(P_{XO})$-norm and agrees with $f_{\ast}$ at low frequencies. To be precise, we require that $\mathrm{supp}(\calF[(f_{\mathrm{aux}} - f_{\ast})(\cdot,\bo)]) \subseteq \{\bomega_x:\|\bomega_x\|_
{2} \geq \gamma_x^{-1} \}$ for any $\bo\in \calO$, and we refer the reader to Eq.~\eqref{eq:f_aux} for the exact definition of $f_{\mathrm{aux}}$. Hence, 
\begin{talign*}
    \|[\hat{f}_{\lambda}] - f_{\ast}\|_{L^2(P_{XO})} 
    \lesssim \|[\hat{f}_{\lambda}] - f_{\mathrm{aux}} \|_{L^2(P_{XO})} 
    \lesssim \gamma_x^{-\eta_0d_x}(\log n)^{\frac{d_x\eta_0}{2}} \|T[\hat{f}_{\lambda}] - T f_{\mathrm{aux}} \|_{L^2(P_{ZO})} .
\end{talign*}
In the last step, we utilize the Fourier measure of partial ill-posedness in \Cref{ass:T_frequency_ill_posedness}, and the fact that $\hat{f}_{\lambda}, f_{\mathrm{aux}} \in \calH_{\gamma_x,\gamma_o}$.
The above equation is a sketch where formal derivations can be found at the beginning of \Cref{sec:proof_upper}, particularly Eq. \eqref{eq:eq_to_refer_main_pf_sketch} and Eq. \eqref{eq:eq_to_refer_main_pf_sketch_final}. 
Combining the above relation and the choice of $\gamma_x$ in Eq.~\eqref{eq:lengthscale_main} concludes the proof of \Cref{thm:upper_rate}.

\begin{rem}[Extension to more anisotropy]
\label{rem:extension_anisotropy}
For simplicity of presentation, we focus on the case where anisotropic smoothness exists across $(X,O)$ but we assume no anisotropy within $X$ and $O$.
The KIV-O algorithm with adaptive length-scales and its associated learning rate can both be easily extended to the fully anisotropic setting, where $f_\ast \in B_{2, \infty}^{\bs} (\calX\times\calO)$ with $\bs=[s_1, \ldots, s_{d_x}, s_{d_x+1}, \ldots, s_{d_x+d_o}]\in\R^{d_x+d_o}$, a generalization of previous results from anisotropic nonparametric regression~\citep{hang2021optimal, suzuki2021deep,hoffman2002random} to anisotropic NPIV-O. This is particularly relevant in applied work, where the observed covariates $O$ are often of high dimensionality, because practitioners tend to adjust for as many observed covariates as possible to mitigate unobserved confounding. In such cases, our KIV-O algorithm with adaptive length-scales adapts to the intrinsic smoothness with respect to $O$. This mitigates the slow rates typically caused by a high ambient dimension when the intrinsic smoothness is high, as our learning rates avoid being limited by the worst-case smoothness across all dimensions. 
\end{rem}

\subsection{Minimax lower bound for NPIV-O}
\label{sec:lower_main}
In \Cref{thm:lower_rate}, we prove a minimax lower bound for the NPIV-O problem. We call \emph{admissible} a distribution $P_{ZXOY}$ over $(Z, X, O, Y)$ satisfying \Cref{assn: technical}, \Cref{ass:T_injective}, \Cref{ass:equiv_leb} and \Cref{ass:subgaussian}, inducing a model of the form
\begin{align*}
    Y = f_{\ast}(X,O) + \epsilon,\quad \E[\epsilon|Z,O] = 0 \;,
\end{align*}
where $f_{\ast}$ satisfies \Cref{ass:f_ast} and the conditional expectation operator $T : L^2(P_{XO})\to L^2(P_{ZO})$ satisfies \Cref{ass:T_contractivity}. For an \emph{admissible} distribution $P_{ZXOY}$, consider $\mathcal{D} = (\bz_i, \bx_i, \bo_i, y_i)_{i=1}^N, N \geq 1$ sampled i.i.d from $P_{ZXOY}$ and consider $\mathcal{D}_1 = (\tilde{\bz}_i, \tilde{\bo}_i, \tilde{\bx}_i)_{i=1}^{\tilde{n}}$ and $\mathcal{D}_2 = (\bz_i, \bo_i, y_i)_{i=1}^n$ with $\tilde{n},n \leq N$. 

\begin{thm}[Minimax lower bound]\label{thm:lower_rate} 
There exists an \emph{admissible} distribution $P_{ZXOY}$ such that for all learning methods $(\mathcal{D}_1,\mathcal{D}_2) \mapsto \hat{f}_{(\mathcal{D}_1,\mathcal{D}_2)}$, for all $\tau>0$, and sufficiently large $n\geq 1$, the following minimax lower bound holds with $P^{n}$-probability at least $1- C_1 \tau^2$ and $P^{\tilde{n}}$-almost surely ,
\begin{align}
\label{eq:minimax_lower_rate}
    \left\| \hat{f}_{(\mathcal{D}_1,\mathcal{D}_2)} - f_\ast \right\|_{L^2(P_{XO})} \geq C_0 \tau^2 n^{-\frac{\frac{s_x}{d_x}}{1+2(\frac{s_x}{d_x}+\eta_1)+\frac{d_o}{s_o}\frac{s_x}{d_x}}} (\log n)^{-2s_x - d_x} .
\end{align}
\end{thm}

We remind the reader here of our upper bound in Eq. \eqref{eq:upper_rate}: with $P^{n+\stageonesamples}$-probability at least $1- 40e^{-\tau}$, we have
\begin{align*}
    \left\|\left[\hat{f}_{\lambda}\right] - f_{\ast}\right\|_{L^2(P_{XO})} \lesssim \tau n^{-\frac{\frac{s_x}{d_x}+\eta_1-\eta_0}{1+2(\frac{s_x}{d_x}+\eta_1)+\frac{d_o}{s_o}(\frac{s_x}{d_x}+\eta_1)}} \cdot (\log n)^{\frac{d_x+d_o+1+d_x\eta_0}{2}} .
\end{align*}
There remains a gap between the upper and minimax lower bounds even if we take $\eta_1 = \eta_0$ and ignore the logarithmic terms. We dedicate \Cref{sec:challenge_optimal} to its discussion. 

\begin{rem}[Interpolation between NPIV and nonparametric regression] \label{rem:lb_main_rem}
    Similar to the \Cref{rem:interpolate} of the upper learning rate, 
    our minimax lower learning rate also interpolates between the known optimal $L^2$-rates for NPIV without observed covariates ($d_o = 0$) and anisotropic
    non-parametric regression ($\eta_1 = \eta_0 = 0$). 
\end{rem}

\begin{rem}[Comparison to existing $L^2$-minimax lower bounds for NPIV regression]
The work \citet{chen2011rate} established minimax rate-optimality in $L^2$-norm for NPIV under an \emph{approximation condition} and \emph{link condition}. For a conveniently chosen compact operator $B$, the approximation condition characterizes the smoothness of the structural function by the optimal $L^2$-rate approximation by the eigenvectors of $B$, and the link condition relates the mapping properties of the conditional expectation operator $T$ to the Hilbert scale generated by $B$. This framework subsumes an earlier minimax convergence rate result in \citet{HallHorowitz2005}, and was subsequently instantiated in \citet{Chen_2018} for a $B$ constructed via a wavelet basis, and generalized to kernel instrumental variables in  \citet{meunier2024nonparametricinstrumentalregressionkernel} with a Hilbert scale given by the covariance operator of the RKHS. A crucial limitation of this framework is the link condition can only hold if the singular values of the conditional expectation operator are decaying to zero. Thus the minimax framework given by the link condition is applicable only when $T$ is compact, and is not suitable for our observed covariates setting. 
\end{rem}

We give an outline of our proof for \Cref{thm:lower_rate} to facilitate a deeper
understanding of both the assumptions and the results. The full proof is in \Cref{sec:NPIR_lower_appendix} in the Supplementary. The primary strategy in deriving minimax lower bounds is to construct a family of distributions that are similar enough so that they are statistically indistinguishable, but for which the target function of interest is maximally separated. 
This implies no estimator can have error uniformly smaller than this separation. 

Following prior work in the literature~\citep{chen2011rate,meunier2024nonparametricinstrumentalregressionkernel}, we observe that NPIV-O in Eq.~\eqref{eq:npiv_o_intro} is statistically more challenging than the reduced form non-parametric indirect regression with observed covariates (NPIR-O) with a \emph{known} operator $T:L^2(P_{XO})\to L^2(P_{ZO})$. The NPIR-O model is defined below
\begin{align}\label{eq:npir}
    Y= (Tf_*)(Z, O) + \upsilon,
\end{align}
where $\upsilon$ is a random variable such that $\mathbb{E}[\upsilon\mid Z,O] = 0$, and it satisfies \Cref{ass:subgaussian}, and where $T$ satisfies \Cref{ass:T_injective}, \Cref{ass:T_contractivity} and the associated conditional distribution $P_{X\mid Z,O}$ satisfies \Cref{assn: technical}. We refer the reader to \Cref{sec:npir_reduction} in the Supplementary Material for a detailed definition of the NPIR-O model class. We formally prove in \Cref{lem:npir_o_more_informative} in the Supplementary that it suffices to construct a minimax lower bound for the NPIR-O model. 
To this end, we adapt Theorem 20 of \cite{fischer2020sobolev}, which is itself an adaptation of Proposition 2.3 of \citet{Tsybakov}. Recall that the \emph{Kullback-Leibler divergence} of two probability measures $P_1, P_2$ on some measurable space $(\Omega,\calA)$ is given by $\mathrm{KL}(P_1,P_2) := \int_{\Omega}\log (\frac{\mathrm{d}P_1}{\mathrm{d}P_2} )\;\mathrm{d}P_1$ 
if $P_1$ is absolutely continuous with respect to $P_2$, and $+\infty$ otherwise. 

To apply Theorem 20 of \cite{fischer2020sobolev} on the measurable space $\Omega = (\calZ\times\calO\times\mathbb{R})^{n}$, we construct a family of probability measures $P_0, P_1, \dots, P_M$ over $(\calZ\times\calO\times\mathbb{R})$ that share the same marginal distribution over $(Z,O)$ but different conditional distributions $P_{Y\mid Z, O}$. Our strategy follows the construction in \citet{chen2011rate} and can be explained as follows. We fix a marginal probability measure $P_{ZO}$ over $\calZ\times \calO$. We then propose a family of conditional probability measures $P_{i; Y\mid Z,O}$ indexed by $f_i \in \mathfrak{F}$ for $0\leq i\leq M$. Since $f_i$'s are functions on $\calX\times \calO$, we fix a marginal probability measure $P_X$ on $\calX$ under the constraint that $X$ is independent of $O$. Then, we define a smooth copula to parametrize the dependence between $P_{X}$ and $P_{Z}$, which fully specifies $P_{X\mid Z,O}$. This induces a fixed conditional expectation operator $T : L^2(P_{XO})\to L^2(P_{ZO})$. We then specify $P_{i; Y\mid Z,O}$ for $0\leq i\leq M$ via the following equation:
\begin{align*}
    Y\mid Z = \bz, O = \bo \sim \mathcal{N}((Tf_{i})(\bz, \bo), \sigma^2),
\end{align*}
for a fixed $\sigma>0$, and for any $(\bz,\bo)\in \calZ\times \calO$. $P_i$ denotes the probability measure on $\calZ\times \calO\times \mathbb{R}$ by coupling $P_{i;Y\mid Z,O}$ and $P_{ZO}$. In the above construction, we require $P_{X\mid Z,O}$ to satisfy \Cref{assn: technical}, and all $f_i$'s to satisfy \Cref{ass:f_ast}, namely $f_i\in \mathfrak{S}$. 

Concretely, we fix $P_{ZO}$ to be the Lebesgue measure over $\calZ\times \calO$. Without loss of generality, we take $\calX = [-0.5, 0.5]^{d_x}$ rather than $[0,1]^{d_x}$.  By working with a symmetric domain, we exploit the fact that even functions have real-valued Fourier transforms, which simplify our subsequent calculations. We fix $P_{X}$ via the density function
\begin{talign}\label{eq:p_X_bump_main}
    p_X(\bx)\propto \prod_{i=1}^{d_x} g_i(x_i), \quad g_i(x_i) := \exp\left(-\frac{2}{1 - 4x_i^2}\right) \one_{x_i \in [-0.5, 0.5]},
\end{talign}
where each $g_i$ is a smooth, compactly supported \emph{bump function}. We fix an arbitrary smooth copula~\citep{nelsen2006introduction}, such that $P_{X\mid Z,O}$ satisfies \Cref{assn: technical} and $T$ satisfies \Cref{ass:T_contractivity} with parameter $\eta_1>0$. 

The main technical challenge of the minimax lower bound is the construction of the function class $\mathfrak{F}$ whose each element $f_i$ induces a conditional distribution $P_{X\mid Z, O}$.
Our goal is to design $\mathfrak{F}$ so that it satisfies three desirable properties: 1) we want $\mathfrak{F}\in B_{2,\infty}^{s_x,s_o}(\R^{d_x+d_o})$ as per \Cref{ass:f_ast}. 
2) we want the size of $\mathfrak{F}$ to be large to enable the application of \citet[Theorem 20]{fischer2020sobolev}. 
3) we want all elements of $\mathfrak{F}$ to exhibit high partial Fourier frequency such that after applying $T$, the resulting functions become difficult to distinguish. 
This increases the intrinsic difficulty of the problem and yields a tighter lower bound. 
The construction of such a function class satisfying the first two properties is standard and can be achieved using anisotropic B-splines~\citep{ibragimov1984asymptotic, suzuki2021deep,schmidt2020nonparametric}. 
However, the third property poses a challenge: the Fourier transform of B-splines has full support. 
Consequently, conventional B-splines cannot enforce the desired partial high-frequency restriction. To address this limitation, we convolve the B-splines with a high frequency bandpass filter in the $X$-direction. 
Since convolution acts as a smoothing operator, the function class constructed with such modified B-splines remains in the Besov space. 

To start with, we introduce the definition of anisotropic B-splines~\citep{leisner_nonlinear_wavelet_approximation_2003}.

\begin{defi}[Anisotropic B-spline \citep{leisner_nonlinear_wavelet_approximation_2003}]
\label{defi:ani_b_spliine_main}
The cardinal B-spline of order $\mathfrak{m}\in \mathbb{N}$ is defined as the repeated $\mathfrak{m}$ times convolution $\iota_{\mathfrak{m}} = \iota_{0}\ast \dots \ast \iota_0$ where $\iota_0$ is the indicator function $\one_{[0,1]}$. Let $d$ denote the ambient dimension, and let $\bs = (s_1,\dots, s_d)$ denote a smoothness vector. Define the notation $\underline{s} = \min(s_1,\dots,s_{d})$ and $s_i^\prime := \frac{\underline{s}}{s_i}$ for $1\leq i\leq d$.  The (anisotropic) B-spline  of order $\mathfrak{m}$ with resolution $\mathfrak{K} \in \N_+$ and location vector $ \bl \in  \prod_{i=1}^{d}\{-\mathfrak{m},-\mathfrak{m}+1,\dots, 2^{\left\lfloor \mathfrak{K} s_i^\prime \right\rfloor} \}$ is defined as $M_{\mathfrak{K},\bl}(\bx) = \prod_{i=1}^{d} \iota_{\mathfrak{m}} (2^{\lfloor \mathfrak{K} s_i^\prime \rfloor}x_i - \ell_i)$.
\end{defi}
\noindent

Let $\mathfrak{m} > s_x \vee s_o$ be a non-negative even integer. 
The resolution $\mathfrak{K} = \mathfrak{K}(n)$ is defined later in Eq.~\eqref{eq:choice_k_lb}, such that $\mathfrak{K} \to \infty$ as $n\to\infty$. 
The B-spline basis in the $O$-direction follows the standard construction. Define the set of location vectors $\calL_O$ and the associated B-splines
\begin{talign}\label{eq:main_calL_defn}
   \calL_O := \left\{ 2 \mathfrak{m} \N\cap \left\{0, \dots, 2^{\left\lfloor \mathfrak{K}\frac{\underline{s}}{s_o}\right\rfloor}\right\}\right\}^{d_o} , \quad M_{\mathfrak{K},\bl_o}(\bo) = \prod_{j=1}^{d_o}\iota_{\mathfrak{m}}\left(2^{\left\lfloor \frac{\mathfrak{K}\underline{s}}{s_o}\right\rfloor} o_j - \ell_{o,j}\right).
\end{talign}
Note that here we enforce all location vectors to be a multiple of $2\mathfrak{m}$ such that the $\{M_{\mathfrak{K},\bl_o}\}_{\bl_o\in\calL_O}$ have disjoint supports, which will simplify the calculations later on. 
In contrast, the basis functions in the $X$-direction are constructed by convolving a standard B-spline $M_{0,-\frac{\mathfrak{m}}{2}}$ on $X$ with the inverse Fourier transform of the indicator function $\one_{\bl_x}$, for $\bl_x \in \calL_X$:  
\begin{talign}\label{eq:main_Omega_defi}
    \calL_X &:= \left\{0,1,\dots, \left\lfloor \frac{0.8\pi}{\zeta}2^{\mathfrak{K}\frac{\underline{s}}{s_x}}\right\rfloor\right\}^{d_x} , \quad \Omega_{\mathfrak{K}\bl_x}(\bx) := \left(M_{0, -\frac{\mathfrak{m}}{2}}\ast \calF^{-1}[\one_{\bl_x}]\right)\left(2^{\frac{\mathfrak{K}\underline{s}}{s_x}}\bx\right) . 
\end{talign}
where $\one_{\bl_x}$ is the indicator function over the following hyper-rectangle: 
\begin{align}\label{eq:mask_main}
    I_{\bl_x} := \bigtimes_{j=1}^{d_x}\left[1.1\pi + \zeta \ell_{x,j}2^{-\mathfrak{K}\frac{\underline{s}}{s_x}}, 1.1\pi + (\zeta \ell_{x,j} + 1)2^{-\mathfrak{K}\frac{\underline{s}}{s_x}}\right]. 
\end{align}
$\zeta>0$ is a width hyperparameter that determines the spacing of different hyper-rectangles; its value will be specified later.
For sufficiently large $n\geq 1$, since $\ell_{x,j} \leq \frac{0.8\pi}{\zeta}2^{\frac{\mathfrak{K}\underline{s}}{s_x}}$, we know that $ 1.1\pi + (\zeta \ell_{x,j} + 1)2^{-\mathfrak{K}\frac{\underline{s}}{s_x}} \leq 1.9\pi + 2^{-\mathfrak{K}\frac{\underline{s}}{s_x}} \leq 1.95\pi$. Thus we have $I_{\bl_x}\subseteq [1.1\pi, 1.95\pi]^{d_x}$. 

The main consequence of this construction is revealed via the convolution theorem~\citep[Theorem 9.2]{rudin1987real}, which gives
\begin{align}
\label{eq:f_x_omega_klx}
    \calF[\Omega_{\mathfrak{K}\bl_x}](\bomega_x) = \calF[M_{\mathfrak{K},-\frac{\mathfrak{m}}{2}}](\bomega_x) \cdot \one_{\bl_x}(2^{-\frac{\mathfrak{K}\underline{s}}{s_x}}\bomega_x).
\end{align}
As a result, we have that $\mathrm{supp}(\calF[\Omega_{\mathfrak{K}\bl_x}])= 2^{\frac{\mathfrak{K}\underline{s}}{s_x}} \cdot I_{\bl_x} \subseteq [\pi 2^{\frac{\mathfrak{K}\underline{s}}{s_x}}, 2\pi 2^{\frac{\mathfrak{K}\underline{s}}{s_x}}]^{d_x}$. Since $\mathfrak{K} \to \infty$ as $n \to \infty$, the construction guarantees that $\Omega_{\mathfrak{K}\bl_x}$ only has high-frequency spectrum, thereby fulfilling the third desirable property for $\mathfrak{F}$.

\begin{figure}
\vspace{-10pt}
    \centering
    \includegraphics[width=0.65\linewidth]{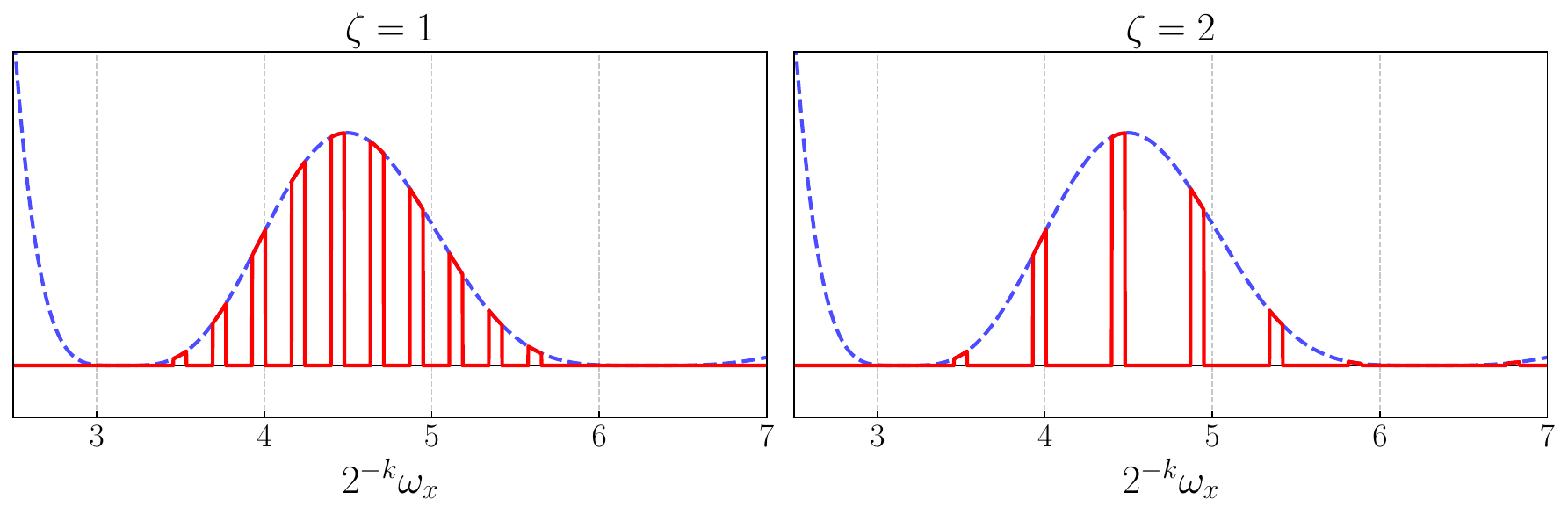}
    \vspace{-5pt}
    \caption{
    This figure illustrates the partial Fourier transform of $f_\bv$ in Eq.~\eqref{eq:mathfrak_F}, where we take $d_x = 1$, $\mathfrak{K} = 2$, $\mathfrak{m} = 4$, $s_x = \underline{s} \leq s_o$.  
    The dashed blue line represents the Fourier transform of the B-spline $M_{\mathfrak{K},-2}$. 
    The red line shows the result of applying frequency masks to this B-spline: $\sum_{\bl_x}\beta_{\bv(\bl_x,\bl_o)} \calF[M_{\mathfrak{K},-2}](\omega_x) \cdot \one_{\bl_x}(2^{-\mathfrak{K}}\omega_x)$. The Fourier transform of $f_\bv(\cdot, \bo)$ is equal to this result up to scaling factors dependent only on $\bo$. 
    }
    \label{fig:masks_visualisation}
    \vspace{-15pt}
\end{figure}

Define the basis $\Omega_{\mathfrak{K}(\bl_x,\bl_o)}(\bx,\bo) := \Omega_{\mathfrak{K}\bl_x}(\bx) \cdot M_{\mathfrak{K},\bl_o}(\bo)$ and the set of all location vectors $\calL=\calL_X\times\calL_O$. 
We have $|\calL| \asymp \zeta^{-d_x}2^{\mathfrak{K}\underline{s}(\frac{d_x}{s_x} + \frac{d_o}{s_o})}$. 
We construct a function class $\mathfrak{F}$ as follows:
\begin{talign}\label{eq:mathfrak_F}
    \mathfrak{F}:= \left\{f_{\bv}: f_{\bv}(\bx,\bo) =  2^{-\mathfrak{K}\underline{s}\left(1 - \frac{d_x}{2s_x}\right)} \sum_{(\bl_x,\bl_o)\in \calL} \beta_{\bv(\bl_x,\bl_o)}\Omega_{\mathfrak{K}(\bl_x,\bl_o)}(\bx,\bo)\,\middle\vert\,  \bv\in \{0,1\}^{|\calL|}\right\},
\end{talign}
where $\beta_{\bv(\bl_x,\bl_o)}\in\{0,1\}$ is the value assigned by $\bv\in \{0,1\}^{|\calL|}$ to the location vector $(\bl_x,\bl_o)\in\calL$. 
Specifically, for each $f_\bv\in\mathfrak{F}$ and for a fixed $\bo\in\calO$, we can see that $f_\bv(\cdot, \bo)$ is constructed by applying a sum of \emph{frequency masks} indexed by $\bl_x\in\calL_X$ to the original B-spline $M_{\mathfrak{K},-\frac{\mathfrak{m}}{2}}$,
\begin{align*}
    \calF[f_{\bv}(\cdot, \bo)](\bomega_x) = \calF\left[M_{\mathfrak{K},-\frac{\mathfrak{m}}{2}}\right](\bomega_x) \cdot \Big(  \sum_{\bl_o\in \calL_O} \underbrace{ \sum_{\bl_x \in \calL_X} \beta_{\bv(\bl_x,\bl_o)} \one_{\bl_x}\left(2^{-\frac{\mathfrak{K}\underline{s}}{s_x}}\bomega_x\right)}_{\text{frequency masks}} M_{\mathfrak{K},\bl_o}(\bo) \Big).
\end{align*}
\Cref{fig:masks_visualisation} is provided to illustrate this effect of frequency masking. 

We prove in Eq. \eqref{eq:f_bv_besov_norm_estimate} in the Supplementary that $\mathfrak{F}\subseteq \mathbb{U}(B^{s_x,s_o}_{2,\infty}(\mathbb{R}^{d_x+d_o})) \cap L^\infty(\R^{d_x+d_o}) \cap L^1(\R^{d_x+d_o}) \cap C^0(\R^{d_x+d_o}) \subseteq \mathfrak{S}$, satisfying \Cref{ass:f_ast}. 
For each $f_\bv \in \mathfrak{F}$, we construct the conditional distribution $P_\bv(\cdot\mid \bz,\bo) := \calN((Tf_\bv)(\bz,\bo), \sigma^2)$ as the normal distribution on $\mathbb{R}$ with mean $(Tf_\bv)(\bz,\bo)$ and variance $\sigma^2$, such that $P_{\bv}$ satisfies \Cref{ass:subgaussian}. Together with the marginal distribution $P_{ZO}$ over $\calZ\times\calO$, we have constructed a family of distributions $\{
P_\bv, \bv \in \{0,1\}^{|\calL|}\}$ over $\calZ\times\calO\times\R$ with each $P_\bv$ associated to a $f_\bv \in \mathfrak{F}$. 
Next, by the Gilbert-Varshamov Bound~\citep[Lemma 2.9]{Tsybakov}, we prove in Eq.~\eqref{eq:f_l2_separate} and Eq.~\eqref{eq:kl_upper_bound_small} in the Supplementary that there exists a subset $V_\mathfrak{K} \subseteq \{0,1\}^{|\calL|}$ with 
$|V_\mathfrak{K}|\geq 2^{\frac{\calL}{8}}$ such that for any $f_\bv, f_{\bv^\prime} \in \mathfrak{F}_{\mathrm{pruned}} = \{f_{\bv}\mid \bv\in V_\mathfrak{K}\}$, there is 
\begin{align*}
    \|f_{\bv} - f_{\bv'}\|^2_{L^2(P_{XO})} \geq 2^{-2\mathfrak{K}\underline{s}} \zeta^{-d_x}, \quad \mathrm{KL} \left(P_{\boldsymbol{v}}^{\otimes n}, P_{\boldsymbol{0}}^{\otimes n}\right) \leq n 2^{-2\mathfrak{K}\underline{s}(\frac{d_x}{s_x}\eta_1 + 1)} . 
\end{align*}
Therefore, we have established that the function class $\mathfrak{F}_{\mathrm{pruned}}$ satisfies all the three properties as desired, namely $\mathfrak{F}_{\mathrm{pruned}}\subset \mathfrak{G}$, functions $f_\bv \in \mathfrak{F}_{\mathrm{pruned}}$ are well-separated in $L^2(P_{XO})$ yet statistically indistinguishable. 
Finally, we take $\zeta = (\log n)^2$, we apply \citet[Theorem 20]{fischer2020sobolev} and the general reduction scheme in \citet[Section 2.2]{Tsybakov} to obtain the desired minimax lower bound. 

\section{On the gap between the upper bound and minimax lower bound}
\label{sec:challenge_optimal}

As explained in \Cref{sec:lower_main}, in the general setting where $\frac{d_x}{s_x}>0$ and $\frac{d_o}{s_o}>0$, a gap arises between the upper bound given in \Cref{thm:upper_rate} and the minimax lower bound given in \Cref{thm:lower_rate}. By setting $\eta_0 = \eta_1 = \eta$ in \Cref{ass:T_contractivity} and \Cref{ass:T_frequency_ill_posedness},  which gives a precise characterization of the \emph{partial} smoothing effect of $T$, we obtain the following upper and lower bounds (ignoring the logarithmic terms for simplicity): 
\begin{align*}
    \text{Lower Bound:} \quad n^{-\frac{\frac{s_x}{d_x}}{1 + 2(\frac{s_x}{d_x}+\eta) + \frac{d_o}{s_o}\frac{s_x}{d_x}}} , \quad\quad \text{Upper Bound:} \quad n^{-\frac{\frac{s_x}{d_x}}{1 + 2(\frac{s_x}{d_x}+\eta) + \frac{d_o}{s_o}\frac{s_x}{d_x} + \frac{d_o}{s_o}\eta}} 
\end{align*}
Note that the denominator in the exponent of the upper bound contains an extra $\frac{d_o}{s_o}\eta$ term, which vanishes when $T$ has no \emph{partial} smoothing effect ($\eta=0$) or when the ratio $\frac{d_o}{s_o}$ is zero. 

The gap between the upper bound and the minimax lower bound arises due to the existence of observed covariates $O$ in the analysis of the approximation error term in the upper bound, 
\begin{align*}
    \|[f_{\lambda}] - f_{\ast}\|_{L^2(P_{XO})},
\end{align*}
where $f_{\lambda}$ is defined in Eq. \eqref{eq:f_lambda_main}. To provide an in-depth discussion, we begin with a warm-up in \Cref{sec:analysis_npiv} by considering the standard KIV setting without observed covariates, where our upper and minimax lower bounds match. We then move on to KIV-O in \Cref{sec:analysis_npivo} and demonstrate how the presence of observed covariates $O$ causes a gap to emerge.

\subsection{Analysis of the approximation error in KIV}\label{sec:analysis_npiv} 
In NPIV, the conditional expectation operator is $T: L^2(P_{X}) \to L^2(P_{Z})$, $f \mapsto \E[f(X)\mid Z]$.
Hence $f_{\lambda}$ defined in Eq.~\eqref{eq:f_lambda_main} reduces to the following
\begin{talign*}
    f_{\lambda} := \argmin_{f\in \calH_{X,\gamma_x}}\lambda\|f\|^2_{\calH_{X,\gamma_x}} + \|T([f]-f_{\ast})\|^2_{L^2(P_{Z})}.
\end{talign*}
We first employ the fact that $f_{\lambda}$ is the minimizer, hence for some $f_{\mathrm{aux}} \in \calH_{X,\gamma_x}$,
\begin{align*}
    \|T([f_{\lambda}] - f_{\ast})\|^2_{L^2(P_{Z})} \leq \lambda \|f_{\mathrm{aux}}\|^2_{\calH_{X,\gamma_x}} + \|T([f_{\mathrm{aux}}] - f_{\ast})\|^2_{L^2(P_{Z})} .
\end{align*}
Here, $f_{\mathrm{aux}}$ is defined as
$f_{\mathrm{aux}} := f_{\ast, \mathrm{low}} + K_{\gamma_x}\ast f_{\ast, \mathrm{high}}$ , where $K_{\gamma_x}$ is defined, following \citet{hang2021optimal,eberts2013optimal}, as
\begin{talign}\label{eq:K_function_main}
    K_{\gamma_x}(\bx) := \sum_{j=1}^r\binom{r}{j}(-1)^{1-j} \frac{1}{(j\gamma_x)^{d}} \left(\frac{2}{\pi}\right)^{\frac{d}{2}}  \exp\left( -2 \sum_{i=1}^{d} \frac{x_i^2}{(j\gamma_x)^2} \right) ,
\end{talign}
and $f_{\ast, \mathrm{low}}, f_{\ast ,\mathrm{high}}$ are defined such that $f_{\ast, \mathrm{low}}$ (resp. $f_{\ast, \mathrm{high}}$) corresponds to the low-frequency (resp. high-frequency) component of $f_{\ast} = f_{\ast, \mathrm{low}} + f_{\ast ,\mathrm{high}}$:
\begin{align*}
    \calF[f_{\ast, \mathrm{low}}](\bomega_x) &= \calF[f_{\ast}](\bomega_x)\cdot \one[\bomega_x: \|\bomega_x\|_{2}\leq \gamma_x^{-1}]\\
    \calF[f_{\ast, \mathrm{high}}](\bomega_x) &= \calF[f_{\ast}](\bomega_x)\cdot \one[\bomega_x: \|\bomega_x\|_{2}\geq \gamma_x^{-1}]. 
\end{align*}
To see why $f_{\mathrm{aux}}\in\calH_{X,\gamma_x}$, note that 
$\|f_{\ast,\mathrm{low}}\|^2_{\calH_{X,\gamma_x}} \lesssim \int_{\mathbb{R}^{d_x}}|\calF[f_{\ast}](\bomega_x)|^2\;\mathrm{d}\bomega_x = \|f_{\ast}\|^2_{L^2(\mathbb{R}^{d_x})} < \infty$ and $K_{\gamma_x}\ast f_{\ast, \mathrm{high}} \in \calH_{X,\gamma_x}$ proved by \citet{eberts2013optimal}.
Then we have
\begin{align}
    \|T([f_{\lambda}] - f_{\ast})\|^2_{L^2(P_{Z})} &\leq \lambda \|f_{\mathrm{aux}}\|^2_{\calH_{X,\gamma_x}} + \|T(f_{\ast,\mathrm{high}} - K_{\gamma_x}\ast f_{\ast, \mathrm{high}})\|^2
    _{L^2(P
    _{Z})} \nonumber \\
    &\stackrel{(\ast)}{\leq} \lambda \|f_{\mathrm{aux}}\|^2_{\calH_{X,\gamma_x}} + \gamma_x^{2d_x\eta}\|f_{\ast,\mathrm{high}} - K_{\gamma_x}\ast f_{\ast, \mathrm{high}}\|^2
    _{L^2(P_{X})} \label{eq:key_step} \\
    &\stackrel{(\ast\ast)}\lesssim \lambda \gamma_x^{-d_x} + \gamma_x^{2d_x\eta + 2s_x} \nonumber .
\end{align}
where $(\ast)$ follows by \Cref{ass:T_contractivity}, and $(\ast\ast)$ follows by the same derivations in \citet[Theorem 2.2, Theorem 2.3]{eberts2013optimal}. 
Eq.~\eqref{eq:key_step} is the key step which would require significant modifications in the setting of NPIV-O due to the existence of observed covariates. Finally, we have
\begin{align*}
    &\quad \|[f_{\lambda}] - f_{\ast}\|^2_{L^2(P_{X})}\nonumber\\
    &\leq \|[f_{\lambda}] - [f_{\mathrm{aux}}]\|^2_{L^2(P_{X})} + \|[f_{\mathrm{aux}}] - f_{\ast}\|^2_{L^2(P_{X})}\nonumber\\
    &\stackrel{(a)}{\leq} \gamma_{x}^{-2d_x\eta} \|T[f_{\lambda}] - T [f_{\mathrm{aux}}]\|^2_{L^2(P_{Z})} + \|[f_{\mathrm{aux}}] - f_{\ast}\|^2_{L^2(P_{X})}\nonumber\\
    &\leq \gamma_{x}^{-2d_x\eta}\left(\|T([f_{\lambda}] - f_{\ast})\|^2_{L^2(P_{Z})} + \|T([f_{\mathrm{aux}}] - f_{\ast})\|^2_{L^2(P_{Z})}\right) + \|[f_{\mathrm{aux}}] - f_{\ast}\|^2_{L^2(P_{X})}\nonumber\\
    &\stackrel{(b)}{\leq} \gamma_{x}^{-2d_x\eta}\left(\|T([f_{\lambda}] - f_{\ast})\|^2_{L^2(P_{Z})} + \gamma_x^{2d_x\eta}\|[f_{\mathrm{aux}}] - f_{\ast}\|^2_{L^2(P_{X})}\right) + \|[f_{\mathrm{aux}}] - f_{\ast}\|^2_{L^2(P_{X})}\nonumber\\
    &\lesssim \gamma_x^{-2d_x\eta} \left(\lambda \gamma_x^{-d_x} + \gamma_x^{2d_x\eta + 2s_x} \right) + \steve{2} \|[f_{\mathrm{aux}}] - f_{\ast}\|^2_{L^2(P_{X})} \nonumber \stackrel{(c)}{\lesssim} n^{-\frac{2s_x}{d_x + 2s_x + 2\eta d_x}} .
\end{align*}
In the above derivations, $(a)$ follows by an application of \Cref{ass:T_frequency_ill_posedness} and \Cref{lem:l2toprojected_illp} since $\|f_\lambda\|^2 \leq \lambda^{-1}\|Tf_\ast\|_{L^2(P_Z)}^2 \lesssim n$ by the optimality of $f_\lambda$ and $\|f_{\text{aux}}\|\leq n$ by the same derivations as in Eq.~\eqref{eq:f_aux_norm} in the Supplement. The second term of Eq.~\eqref{eq:term_lemma_D_18} in the Supplement is subsumed by the following choice of $\gamma_x$. 
Furthermore, $(b)$ follows from \Cref{ass:T_contractivity}, and $(c)$ follows from using Eq.~\eqref{eq:f_f_aux_diff} in the Supplement and choosing $
\lambda = n^{-1}$, $\gamma_x = n^{-\frac{1}{d_x + 2s_x + 2\eta d_x}}$.
The upper bound on the approximation error above is \emph{optimal} in the sense that \steve{it matches the minimax lower bound of NPIV with Besov targets (see e.g \citep{Chen_2018} and our Theorem 4.2 with $d_o=0$)}. 

\subsection{Analysis of the approximation error in KIV-O} \label{sec:analysis_npivo} 
We now proceed to our NPIV-O setting, and demonstrate how the existence of observed covariates $O$ fundamentally changes the problem. As above, to upper bound the approximation error $\|[f_{\lambda}] - f_{\ast}\|_{L^2(P_{XO})}$, we first need to upper bound the projected approximation error $\|T f_{\lambda} - T f_{\ast}\|_{L^2(P_{ZO})}$. To this end, we employ the fact that $f_{\lambda}$ is the minimizer in Eq.~\eqref{eq:f_lambda_main}, hence for some $f_{\mathrm{aux}} \in \calH_{\gamma_x,\gamma_o}$,
\begin{align*}
    \|T([f_{\lambda}] - f_{\ast})\|^2_{L^2(P_{ZO})} \leq \lambda \|f_{\mathrm{aux}}\|^2_{\calH_{\gamma_x,\gamma_o}} + \|T([f_{\mathrm{aux}}] - f_{\ast})\|^2_{L^2(P_{ZO})} .
\end{align*}
We construct such an auxiliary function $f_{\mathrm{aux}}$. For any $\bo\in \calO$, we construct $f_{\ast,\mathrm{low}}(\cdot,\bo)$ and $f_{\ast, \mathrm{high}}(\cdot,\bo)$ such that 
\begin{align}
    \calF[f_{\ast,\mathrm{low}}(\cdot, \bo)](\bomega_x) &= \calF[f_{\ast}(\cdot,\bo)](\bomega_x)\cdot \one[\bomega_x:\|\bomega_x\|_2\leq \gamma_x^{-1}]\label{eq:f_ast_low}\\
    \calF[f_{\ast,\mathrm{high}}(\cdot, \bo)](\bomega_x) &= \calF[f_{\ast}(\cdot,\bo)](\bomega_x)\cdot \one[\bomega_x:\|\bomega_x\|_2\geq \gamma_x^{-1}]\label{eq:f_ast_high}.
\end{align}
Note that $f_{\ast, \mathrm{high}} = f_{\ast} - f_{\ast, \mathrm{low}}$. Despite $ f_{\ast,\mathrm{low}}(\cdot, \bo) \in \calH_{X,\gamma_x}$ for any $\bo\in\calO$, a crucial difference with NPIV is that, $f_{\ast,\mathrm{low}} \notin \calH_{\gamma_x,\gamma_o}$ since $f_{\ast, \mathrm{low}}$ is only constructed by a cut-off with respect to the partial Fourier spectrum of on $X$. Hence, to construct $f_{\mathrm{aux}}$ we convolve $f_{\ast, \mathrm{low}}$ again with $K_{\gamma_o}$:
\begin{align}\label{eq:f_aux}
    f_{\mathrm{aux}} = f_{\ast,\mathrm{low}} \ast K_{\gamma_o} + f_{\ast,\mathrm{high}}\ast K_{\gamma_x,\gamma_o}\in\calH_{\gamma_x,\gamma_o}.
\end{align}
Here, $\;K_{\gamma_x,\gamma_o}(\bx, \bo) := K_{\gamma_x}(\bx)\cdot K_{\gamma_o}(\bo)$. Proceeding from the above, we have
\begin{align}
    &\quad \|T([f_{\lambda}] - f_{\ast})\|^2_{L^2(P_{ZO})} \leq \lambda \|f_{\mathrm{aux}}\|^2_{\calH_{\gamma_x,\gamma_o}} \nonumber \\
    &\qquad\qquad + \|T(f_{\ast, \mathrm{low}} - f_{\ast,\mathrm{low}} \ast K_{\gamma_o})\|_{L^2(P_{ZO})}^2 + \|T(f_{\ast, \mathrm{high}} - f_{\ast,\mathrm{high}} \ast K_{\gamma_x,\gamma_o})\|_{L^2(P_{ZO})}^2 \label{eq:key_term_npivo}  \\
    &\leq \lambda \|f_{\mathrm{aux}}\|^2_{\calH_{\gamma_x,\gamma_o}} + \|f_{\ast, \mathrm{low}} - f_{\ast,\mathrm{low}} \ast K_{\gamma_o}\|_{L^2(P_{XO})}^2 \nonumber \\
    &\qquad\qquad + \gamma_x^{2d_x\eta}\|f_{\ast, \mathrm{high}} - f_{\ast,\mathrm{high}} \ast K_{\gamma_x,\gamma_o}\|_{L^2(P_{XO})}^2. \nonumber 
\end{align}
The last inequality follows by Jensen's inequality for the second term and \Cref{ass:T_contractivity} for the third term. 
We would like to highlight here that the second term in Eq.~\eqref{eq:key_term_npivo} is absent in Eq.~\eqref{eq:key_step} for the NPIV setting when there are no observed covariates $O$. Unlike the high frequency term in Eq.~\eqref{eq:key_term_npivo}, we cannot employ the Fourier measure of \emph{partial} contractivity in \Cref{ass:T_contractivity} because $f_{\ast, \mathrm{low}} - f_{\ast,\mathrm{low}} \ast K_{\gamma_o}$ only contains low frequency spectrum on $X$ by construction.
From the above derivations, we can also deduce 
\begin{align*}
    \|T([f_{\mathrm{aux}}] - f_{\ast})\|^2_{L^2(P_{ZO})} &\lesssim \|f_{\ast, \mathrm{low}} - f_{\ast,\mathrm{low}} \ast K_{\gamma_o}\|_{L^2(P_{XO})}^2 \\
    &\qquad + \gamma_x^{2d_x\eta}\|f_{\ast, \mathrm{high}} - f_{\ast,\mathrm{high}} \ast K_{\gamma_x,\gamma_o}\|_{L^2(P_{XO})}^2 .
\end{align*}
To upper bound the approximation error $\|[f_{\lambda}] - f_{\ast}\|_{L^2(P_{XO})}$, we notice that 
\begin{align*}
    &\quad \|[f_{\lambda}] - f_{\ast}\|_{L^2(P_{XO})}^2 \\
    &\leq \|[f_{\lambda}] - [f_{\mathrm{aux}}]\|^2_{L^2(P_{XO})} + \|[f_{\mathrm{aux}}] - f_{\ast}\|^2_{L^2(P_{XO})} \\
    &\leq \gamma_x^{-2d_x\eta} \|T([f_{\lambda}] - [f_{\mathrm{aux}}])\|_{L^2(P
    _{ZO})}^2 + \|[f_{\mathrm{aux}}] - f_{\ast}\|_{L^2(P_{XO})}^2 \\
    &\lesssim \gamma_{x}^{-2d_x\eta} \|T([f_{\lambda}] - f_{\ast})\|^2_{L^2(P_{ZO})} + \gamma_{x}^{-2d_x\eta} \|T [f_{\mathrm{aux}}] - T f_{\ast}\|^2_{L^2(P_{ZO})} + \|[f_{\mathrm{aux}}] - f_{\ast}\|^2_{L^2(P_{XO})} \\
    &\lesssim \gamma_{x}^{-2d_x\eta} \Big( \lambda \|f_{\mathrm{aux}}\|^2_{\calH_{\gamma_x,\gamma_o}} + \|f_{\ast, \mathrm{low}} - f_{\ast,\mathrm{low}} \ast K_{\gamma_o}\|_{L^2(P_{XO})}^2 \\
    &\qquad\qquad + \gamma_x^{2d_x\eta}\|f_{\ast, \mathrm{high}} - f_{\ast,\mathrm{high}} \ast K_{\gamma_x,\gamma_o}\|_{L^2(P_{XO})}^2 \Big) .
\end{align*}
The last inequality holds by plugging into the upper bound on $\|T([f_{\lambda}] - f_{\mathrm{aux}})\|_{L^2(P_{ZO})}^2$ and $\|T [f_{\mathrm{aux}}] - T f_{\ast}\|^2_{L^2(P_{ZO})}$ derived above.
Notice that the term $\|f_{\ast, \mathrm{low}} - f_{\ast,\mathrm{low}} \ast K_{\gamma_o}\|_{L^2(P_{XO})}$ is unnecessarily inflated by the Fourier measure of \emph{partial} ill-posedness $\gamma_x^{-d_x\eta}\gg 1$ (\Cref{ass:T_frequency_ill_posedness}) even if there is no smoothing effect of $T$ acting on $O$. We obtain
\begin{align*}
    \|[f_{\lambda}] - f_{\ast}\|_{L^2(P_{XO})}^2 &\lesssim \gamma_{x}^{-2d_x\eta} \left( \lambda \gamma_x^{-d_x} \gamma_o^{-d_o} +  \gamma_o^{2s_o} + \gamma_x^{2s_x + 2d_x\eta} \right) + \gamma_o^{2s_o} + \gamma_x^{2s_x} \\
    &\lesssim n^{-\frac{\frac{s_x}{d_x}}{1 + 2(\frac{s_x}{d_x}+\eta) + \frac{d_o}{s_o}(\frac{s_x}{d_x}+\eta)}}.  
\end{align*}
The last step holds by choosing $\lambda, \gamma_x, \gamma_o$ as in \Cref{thm:upper_rate}. The upper bound on the approximation error does not match the minimax lower bound in \Cref{thm:lower_rate}.

We conclude this section by offering a more practical perspective on why the gap emerges. 
Hyperparameters including the kernel lengthscales are selected via cross-validation in practice. 
For the KIV-O estimator $\hat{f}_{\lambda}$ computed from Stage II samples $\mathcal{D}_2=\{(\mathbf{z}_i, \mathbf{o}_i, y_i)\}_{i=1}^n$, the kernel lengthscales $\gamma_x, \gamma_o$ are selected by minimizing the following cross-validation criterion. 
\begin{talign}\label{eq:cv_projected}
    \mathrm{CV}(\gamma_x,\gamma_o,\lambda) = \frac{1}{n}\sum_{i=1}^{n}(y_i - \langle \hat{f}_{-i, \gamma_x,\gamma_o,\lambda}, \hat{F}_{\xi}(\bz_i,\bo_i)\rangle_{\calH_{\gamma_x,\gamma_o}})^2,
\end{talign}
where $\hat{f}_{-i, \gamma_x,\gamma_o,\lambda}$ denotes the version of $\hat{f}_{\gamma_x,\gamma_o,\lambda}$ computed from $n-1$ samples from $\mathcal{D}_2$ excluding its $i$-th sample. 
This cross-validation criterion $\mathrm{CV}(\gamma_x,\gamma_o,\lambda)$ is a consistent estimator for the \emph{projected} risk $\|T([\hat{f}_\lambda] - f_*)\|_{L^2(P_{ZO})}$ and has been used in many 2SLS approaches~\citep{hartford2017deep, xu2021deep, mastouri2021proximal, xu2025kernelsingleproxycontrol}. 

Similarly, the choice of lengthscales $\gamma_x,\gamma_o, \lambda$ in our upper bound analysis is also obtained by minimizing the projected risk $\|T([\hat{f}_\lambda] - f_*)\|_{L^2(P_{X O})}$ which results 
$\gamma_x^{s_x+d_x\eta} = \gamma_o^{s_o}$. 
This can be contrasted with the choice of kernel lengthscales in anisotropic kernel ridge regression which imposes a different balance condition $\gamma_x^{s_x} = \gamma_o^{s_o}$~\citep{hang2021optimal}. The extra $d_x\eta$ in the exponent of $\gamma_x$ arises due to the partial smoothing effect of $T$, which maps a function that is $(s_x, s_o)$-smooth on $\mathcal{X} \times \mathcal{O}$ to a function that is $(s_x + d_x\eta, s_o)$-smooth on $\mathcal{Z} \times \mathcal{O}$. 
On the other hand, our minimax lower bound requires constructing B-splines in Eq.~\eqref{eq:Bspline_main} in the Supplement with resolution $\mathfrak{K}$ and $J_x := 2^{\lfloor\frac{\mathfrak{K} s}{s_x}\rfloor}$ and $J_o := 2^{\lfloor\frac{\mathfrak{K} s}{s_o}\rfloor}$. $J_x, J_o$ play a role analogous to $\gamma_x, \gamma_o$ (see \Cref{rem:connection_smoip}), but they satisfy the balance condition $J_x^{s_x} \asymp J_o^{s_o}$. 
This discrepancy between the kernel lengthscales balance condition and B-spline lengthscales balance condition gives rise to the gap between our upper and lower bound.
Simply enforcing $\gamma_x^{s_x} = \gamma_o^{s_o}$ in our upper bound analysis would result in a slower rate.

\section{Conclusion}

We study nonparametric instrumental variable regression with observed covariates (NPIV-O), a setting that generalizes NPIV by incorporating observed covariates to enable heterogeneous treatment effect estimation. 
The conditional expectation operator $T$ behaves as a partial identity operator, which makes NPIV-O a hybrid of NPIV and NPR. 
We prove an upper bound for kernel 2SLS and the first minimax lower bound. 
Our upper and lower bounds interpolate between the known optimal rates for NPIV and NPR, and adapt to the anisotropic smoothness of $f_\ast$. 
Our analysis reveals a gap between the upper and lower bounds in the general setting, and closing this gap remains an open direction for NPIV-O.

\bibliographystyle{plainnat}
\bibliography{causal_bib}

\begin{appendix}
\section{Examples for the partial smoothing effect of T}
\label{sec:appendix_examples}
In this section, we extend \Cref{ass:T_frequency_ill_posedness} and \Cref{ass:T_contractivity} from the main text—on the Fourier measure of partial ill-posedness and on the Fourier measure of partial contractivity of the operator $T$, respectively—to the space of distributions. This allows us to apply these assumptions to periodic functions and to construct an explicit example verifying \Cref{ass:T_contractivity}. First, we prove the following Lemma relating \Cref{ass:T_frequency_ill_posedness} and \Cref{ass:T_contractivity}. 
\begin{lem}\label{lem:eta_0_eta_1_appendix}
If Assumption \ref{ass:T_frequency_ill_posedness} and \ref{ass:T_contractivity} hold simultaneously, and $P_{XO}$ is absolutely continuous with respect to the Lebesgue measure on $\calX\times\calO$, then $\eta_0\geq \eta_1$.
\end{lem}
\begin{proof}
Let $\gamma\in (0,1)$ be arbitrary. Let $\epsilon  > 0$ be arbitrary. We define
\begin{align*}
    f_{\gamma, \epsilon}(\bx,\bo) &:= \mathcal{F}^{-1}\left[\one[\|\bomega\|_2 \in [\gamma^{-1}, \gamma^{-1}+\epsilon]]\right](\bx). 
\end{align*}
We have $f_{\gamma,\epsilon}\in \mathrm{LF}\left((\gamma^{-1} + \epsilon)^{-1}\right)\cap \mathrm{HF}(\gamma)\cap L^{\infty}(P_{XO})$. Since $P_{XO}$ admits a density function on $\calX\times \calO$, and $f_{\gamma,\epsilon}$ only vanishes on sets of null Lebesgue measure, we have $\|f_{\gamma,\epsilon}\|_{L^2(P_{XO})}\neq 0$. As imposed by \Cref{ass:T_contractivity} and \Cref{ass:T_frequency_ill_posedness}, we thus have
\begin{align*}
    c_0^{-1} \left(\gamma^{-1} + \epsilon\right)^{-d_x\eta_0}\|f_{\gamma,\epsilon}\|_{L^2(P_{XO})}
\leq \|Tf_{\gamma,\epsilon}\|_{L^2(P_{ZO})} \leq c_1 \gamma^{d_x\eta_1}\|f_{\gamma,\epsilon}\|_{L^2(P_{XO})}.
\end{align*}
Since $\|f_{\gamma,\epsilon}\|_{L^2(P_{XO})}\neq 0$, we have $(\forall \gamma\in (0,1))\;(\forall \epsilon >0)\; c_0^{-1}\left(\gamma^{-1} + \epsilon\right)^{-d_x\eta_0} \leq c_1 \gamma^{d_x\eta_1}$. For a fixed $\gamma$, taking the limit $\epsilon \to 0$, we have by continuity
\begin{align}
\label{eq:lem_eta_key}
   (\forall \gamma\in (0,1))\; \frac{1}{c_0 c_1}\leq \gamma^{d_x(\eta_1-\eta_0)}.
\end{align}
Taking the limit $\gamma\to 1$ in Eq.~\eqref{eq:lem_eta_key}, we find $c_0c_1 \geq 1$. Then since Eq.~\eqref{eq:lem_eta_key} holds for all $\gamma\in (0,1)$, we find that $\eta_1\leq \eta_0$. 
\end{proof}

\begin{defi}[Distribution and distribution of a function]
\label{defi:fn_dist}
Let $\Omega\subseteq \mathbb{R}^{d}$ be a domain. A distribution $u\in \mathcal{D}'(\Omega)$ is a continuous linear functional on the space of test functions $\mathcal{D}(\Omega)$, where $\mathcal{D}(\Omega)$ is the set $C^{\infty}_c(\Omega)$ endowed with the \emph{canonical limit of Fréchet topology}. For a locally integrable function $f\in L^1_{\mathrm{loc}}(\mathbb{R}^{d})$, for all $\phi\in \mathcal{D}(\Omega)$, we define the distribution $T_f$ by
\begin{align*}
    T_f\phi := \int_{\Omega}f(\bx)\phi(\bx)\;\mathrm{d}\bx . 
\end{align*}
\end{defi}

\begin{defi}[Support of a distribution]
\label{defi:supp_dist}
    A distribution $u\in \mathcal{D}'(\Omega)$ is supported in the closed set $K\subset \Omega$ if $u[\phi] = 0$ $\forall \phi\in C^{\infty}_c(\Omega\setminus K)$. The \emph{support} of $u$, $\mathrm{supp}\; u$ is the set 
    \begin{align*}
        \mathrm{supp}\; u = \cap \{K: u\text{ is supported in }K\}.
    \end{align*}
\end{defi}

\begin{defi}[Tempered distribution]
    We define the \emph{Schwartz space} $\mathcal{S}$ via
    \begin{align*}
        \mathcal{S} = \left\{\phi\in C^{\infty}(\mathbb{R}^{d})\mid \forall \bm{\alpha}, \forall N\in \mathbb{N},\;\sup_{\bx\in \mathbb{R}^{d}}\left|(1+|\bx|)^{N}D^{\bm{\alpha}}\phi(\bx)\right| < \infty\right\}.
    \end{align*}
    We say that a sequence $\{\phi_j\}^{\infty}_{j=1}\subset \mathcal{S}$ tends to zero iff
    \begin{align*}
        \sup_{\bx\in\mathbb{R}^{d}}|(1+|\bx|)^{N}D^{\bm{\alpha}}\phi_j(\bx)|\to 0 
    \end{align*}
    for all $N\in \mathbb{N}$ and all multi-indices $\bm{\alpha}$. This endows $\mathcal{S}$ with a topology. We define the space of \emph{tempered distributions} to be the continuous dual space of $\mathcal{S}$, denoted as $\mathcal{S}'$.
\end{defi}

\begin{defi}[Periodic distribution]
    We define the translation of a distribution $u\in \mathcal{D}'(\mathbb{R}^{d})$ via $\tau_{\bz}u[\phi] = u[\tau_{-\bz}\phi]$ 
    for all $\phi\in \mathcal{D}(\mathbb{R}^{d})$, where we use $\tau_{\bz}\phi(\bx):= \phi(\bx-\bz)$. 
    We say that a distribution $u\in \mathcal{D}'(\mathbb{R}^{d})$ is periodic if for each $\bg\in \mathbb{Z}^{d}$ we have $\tau_{\bg} u = u$.
    Clearly, if $f\in L^1_{\mathrm{loc}}(\mathbb{R}^{d})$ is periodic, then $T_f$ is a periodic distribution.
\end{defi}

\begin{defi}[Fourier transform of $L^1$ functions]
    For $f\in L^1(\mathbb{R}^{d})$, we define the Fourier transform $\calF[f] = \hat{f} : \mathbb{R}^{d}\to \mathbb{C}$ by 
    \begin{align*}
        \calF[f](\boldsymbol{\xi}) = \hat{f}(\boldsymbol{\xi}) := \int_{\mathbb{R}^{d}} f(\bx)\exp(-i\langle \bx, \boldsymbol{\xi}\rangle)\;\mathrm{d}\bx.
    \end{align*}
\end{defi}

\begin{defi}[Fourier transform of tempered distributions]
    For a distribution $u\in \mathcal{S}'$, we define the Fourier transform of $u$, written $\hat{u}\in \mathcal{S}'$, to be the distribution satisfying:
    \begin{align*}
        \hat{u}[\phi] = u[\hat{\phi}], \;\forall \phi\in\mathcal{S}, 
    \end{align*}
    which is well defined since the Fourier transform maps $\mathcal{S}$ to $\mathcal{S}$ continuously.
\end{defi}
We note that, for all $\phi\in \mathcal{S}$, by the Fourier inversion theorem \citep[9.11]{rudin1987real}, 
\begin{align*}
    T_{1}[\hat{\phi}] = \int_{\mathbb{R}^{d}} \hat{\phi}(\bx)\;\mathrm{d}\bx = (2\pi)^{d}\phi(\boldsymbol{0}) = (2\pi)^{d}\delta_{\boldsymbol{0}}[\phi]. 
\end{align*}
Hence, $\forall \bx\in \mathbb{R}^{d}$, we define $e_{\boldsymbol{\xi}}(\bx) = \exp(i2\pi \langle \bx, \boldsymbol{\xi}\rangle)$. 
So $\forall \phi\in \mathcal{S}$, we have
\begin{align}
\label{eq:fourier_transform_trig_poly}
    T_{e_{\boldsymbol{\xi}}}[\hat{\phi}] = \int_{\mathbb{R}^{d}} \exp(i2 \pi \langle \bx, \boldsymbol{\xi}\rangle)\hat{\phi}(\bx)\;\mathrm{d}\bx 
    = (2\pi)^{d}\phi(2\pi \boldsymbol{\xi}).
\end{align}
i.e. $\calF[T_{e_{\boldsymbol{\xi}}}] = (2\pi)^{d}\delta_{2\pi \boldsymbol{\xi}}$, the Dirac delta distribution at $2\pi \boldsymbol{\xi}$. 
We can now make sense of the Fourier transform of a periodic function. The following Lemma is from \citet[8.5]{friedlander1998introduction}:
\begin{lem}[Periodic distributions are tempered]
\label{lem:periodic_distributions_tempered}
    Let $u\in \mathcal{D}'(\mathbb{R}^{d})$ be a periodic distribution. Then $u$ is in fact a tempered distribution, i.e. $u\in \mathcal{S}'$.
\end{lem}
\begin{prop}
    Suppose $u\in \mathcal{D}'(\mathbb{R}^{d})$ is a periodic distribution. Then there exist constants $c_{\boldsymbol{\xi}}\in \mathbb{C}$ such that $u$ can be represented as a (generalized) Fourier series, 
    \begin{align*}
        u = \sum_{\boldsymbol{\xi}\in \mathbb{Z}^d}c_{\boldsymbol{\xi}} \,T_{e_{\boldsymbol{\xi}}},
    \end{align*}
    with $c_{\boldsymbol{\xi}}$ satisfying the bound $|c_{\boldsymbol{\xi}}|\leq K(1+|\boldsymbol{\xi}|)^N$ for some $K>0$ and $N\in \mathbb{Z}$.
\end{prop}
We can now make sense of the following definitions: for any scalar $\gamma\in (0,1)$, we define the following two sets of functions which are generalization of the $\mathrm{LF}(\gamma)$ and $\mathrm{HF}(\gamma)$ defined in the main text to \emph{distributions}:
\begin{align} 
\label{eq:distribution_lf_hf_set_defn}
\begin{aligned}
\mathrm{LF}(\gamma) := \{ f\in L^1_{\mathrm{loc}}(\mathbb{R}^{d_x+d_o}) \ | \
& \forall \bo \in \mathcal{O}, \; T_{f(\cdot, \bo)} \in \mathcal
{S}'(\mathbb{R}^{d_x}), \\
& \;\mathrm{supp}\big(\calF[f(\cdot, \bo)]\big) 
\subseteq \left\{ \bomega_x \in \mathbb{R}^{d_x} : \|\bomega_x\|_2 \leq \gamma^{-1} \right\}
\}.\\
\mathrm{HF}(\gamma) := \{ f\in L^1_{\mathrm{loc}}(\mathbb{R}^{d_x+d_o}) \ | \
& \forall \bo \in \mathcal{O}, \; T_{f(\cdot, \bo)} \in \mathcal{S}'(\mathbb{R}^{d_x}), \\
& \;\mathrm{supp}\big(\calF[f(\cdot, \bo)]\big) 
\subseteq \left\{ \bomega_x \in \mathbb{R}^{d_x} : \|\bomega_x\|_2 \geq \gamma^{-1} \right\}
\}.
\end{aligned}
\end{align}
We now recall the statements of Assumption 4.2 and 4.3 from the main text with the above generalization of $\mathrm{LF}(\gamma)$ and $\mathrm{HF}(\gamma)$.
\begin{ass}[Fourier measure of partial ill-posedness of $T$]
\label{ass:T_frequency_ill_posedness_schwartz}
There exists a constant $c_0>0$ and a parameter $\eta_0\in [0,\infty)$ only depending on $T$, such that for all $\gamma \in (0,1)$ and all functions $f \in \mathrm{LF}(\gamma) \cap L^\infty(P_{XO})$, the following inequality is satisfied:
    \begin{align*}
        \|f\|_{L^2(P_{XO})} \leq c_0\gamma^{-d_x\eta_0} \|Tf\|_{L^2(P_{ZO})} .
    \end{align*}
    In particular, $c_0$ does not depend on $\gamma$. 
\end{ass}
\begin{ass}[Fourier measure of partial contractivity of $T$]
\label{ass:T_contractivity_schwartz}
There exists a constant $c_1>0$ and a parameter $\eta_1\in [0,\infty)$ only depending on $T$, such that for all $\gamma \in (0,1)$ and all functions $f \in \mathrm{HF}(\gamma) \cap L^\infty(P_{XO})$, the following inequality is satisfied:
    \begin{align*}
        \|Tf\|_{L^2(P_{ZO})} \leq c_1\gamma^{d_x\eta_1}\|f\|_{L^2(P_{XO})} .
    \end{align*}
    In particular, $c_1$ does not depend on $\gamma$. 
\end{ass}

Let $(S^{1}, +)$ denote the unit circle (equipped with a group structure via addition), and let $\mathrm{d}x$ denote the Haar measure on $S^{1}$, which coincides with the pushforward of the Lebesgue measure under the quotient map $S^{1}\cong \frac{\mathbb{R}}{\mathbb{Z}}$.   We make use of the obvious identification between functions on $(S^1)^{d}$ and $1$-periodic functions on $\mathbb{R}^{d}$, for any $d\geq 1$. 
\begin{defi}[Fourier series \citep{Katznelson_2004}]
    Let $f\in L^1(S^{1})$. We define the $n$th Fourier coefficient of $f$ by $\calF[f][n] = \int_{S^1} f(x)e^{-i2\pi n x}\;\mathrm{d}x$.
    The Fourier series of $f\in L^1(S^1)$ is the trigonometric series $f(x) = \sum_{n=-\infty}^{\infty}\calF[f][n]e^{i2\pi nx}$.
\end{defi}
We make use of the group structure of $S^{1}$ and the \emph{translation invariance} of the measure $\mathrm{d}x$ on $S^{1}$ to define the convolution operation in $L^1(S^{1})$, following \citep[Section 1.7]{Katznelson_2004}.
\begin{prop}
\label{prop:convolution_theorem}
    Let $f,g \in  L^1(S^{1})$. For almost all $t\in S^1$, the function $f(t - \tau)g(\tau)$ is $L^1(S^{1})$-integrable as a function of $\tau$, and if we define the convolution
    \begin{align*}
        h(t) = \int_{S^1} f(t-\tau)g(\tau)\;\mathrm{d}\tau
    \end{align*}
    then $h\in L^1(S^1)$ with $\|h\|_{L^1(S^{1})}\leq \|f\|_{L^1(S^{1})}\|g\|_{L^1(S^{1})}$. Moreover, $(\forall n\in \mathbb{Z})$, we have
    \begin{align*}
        \calF[h][n] = \calF[f][n]\calF[g][n].
    \end{align*}
\end{prop}

For any $\gamma>0$, we now exhibit a class of functions $f\in \mathrm{HF}(\gamma)$ and a distribution $p(x,z,o)$ satisfying the statement in \Cref{ass:T_contractivity_schwartz}, with the help of \Cref{prop:convolution_theorem}. For simplicity, we assume that $d_x = d_z = d_o = 1$.  Fix a scalar $\gamma \in(0,1)$. Let $g \in L^1\left(S^1\right)$ be a function whose Fourier coefficients vanish on low frequencies, in the sense that
\begin{align}
\label{eq:schwartz_assumption_calf_g}
    \mathcal{F}[g][n] = 0 \quad \text{for all } n \in \mathbb{Z} \text{ such that } |n| \leq (2\pi \gamma)^{-1}.
\end{align}
An example of such a function is $g(x) = \exp(i2\pi m x)$, where $m\in \mathbb{Z}$ and $m > (2\pi\gamma)^{-1}$. Let $h\in L^1_{\mathrm{loc}}(\mathbb{R}^{d_o})$ be such that it does not vanish identically. 
We then define $f:\mathbb{R}^{2} \to \mathbb{C}$:
\begin{align*}
f(x + t, o) := g(x)h(o) \quad \text{for all } x \in [0,1), t \in \mathbb{Z}, o \in \mathbb{R}.
\end{align*}
Then it follows that $f\in L^1_{\mathrm{loc}}(\mathbb{R}^{d_x+d_o})$ since $g\in L^1(S^{1})$ and $h\in L^1_{\mathrm{loc}}(\mathbb{R})$. It follows from \Cref{lem:periodic_distributions_tempered} that $(\forall o\in \mathbb{R})$, $T_{f(\cdot, o)}\in \mathcal{S}'(\mathbb{R})$ for any $o\in\R$. Moreover, for every $o\in\R$, we observe that: i) $h(o) = 0$, then $\mathrm{supp}(\calF_x[T_{f(\cdot, o)}]) =  \emptyset$, ii) $h(o)\neq 0$, then
\begin{align*}
    \mathrm{supp}(\calF[T_{f(\cdot, o)}]) &= \mathrm{supp}(\calF[T_g])\\
    &= \mathrm{supp}\left(\calF\left[\sum_{n = -\infty}^{\infty}\calF[g][n]\,T_{e^{i2\pi n\cdot }}\right]\right)\\
    &\stackrel{(a)}{=} \mathrm{supp}\left(\sum_{n=-\infty}^{\infty}\calF[g][n](2\pi)\delta_{2\pi n}\right)\\
    &\stackrel{(b)}{=} \mathrm{supp}\left(\sum_{|n|> (2\pi\gamma)^{-1}}\calF[g][n](2\pi)\delta_{2\pi n}\right)\\
    &\subseteq \left\{\omega_x\in \mathbb{R} : |\omega_x| > \frac{1}{\gamma}\right\}.
\end{align*}
Here, step (a) follows from \eqref{eq:fourier_transform_trig_poly} and the use of distributional support as defined in \Cref{defi:supp_dist}, and step (b) follows from Eq.~\eqref{eq:schwartz_assumption_calf_g}. Thus, from both cases, we conclude that $f\in \mathrm{HF}(\gamma)$ defined in Eq. \eqref{eq:distribution_lf_hf_set_defn}. 

Fix an integer $k\geq 1$. 
We consider a probability space $\Omega = S^1 \times  S^1 \times \underbrace{S^1 \times \dots \times S^1}_{k}$, where each copy of $S^1$ is equipped with its Borel $\sigma$-algebra and the normalized Haar measure, and $\Omega$ is equipped with the product Borel $\sigma$-algebra and the product measure. We define the following mappings, where $1\leq i\leq k$:
\begin{align*}
    \pi_Z  : \Omega \to S^1, \quad&\pi_Z(z,o, u_1,\dots, u_k) = z \\
    \pi_{O} : \Omega \to S^1, \quad&\pi_{O}(z, o, u_1, \dots, u_k) = o \\
    \pi_{U_i} : \Omega \to S^1, \quad&\pi_{U_i}(z,o,u_1,\dots,u_k) = u_i,
\end{align*}
Since $\pi_Z, \pi_O, \pi_{U_i}$'s are measurable, they are valid random variables, and we also denote them by $Z,O, U_1,\dots, U_k$. We then write $W = 0.1(U_1 + \dots + U_k)$ and $X = Z + W$, where addition is understood as a group operation on $S^1$. By construction, $X$ is independent of $O$. The random tuple $(X,Z,O)$ is as considered in the NPIV-O set-up in the main text. 

We are now going to show that there exists some constant $c_1 > 0$ (which does not depend on $\gamma$), such that 
\begin{align*}
    \|Tf\|_{L^2(P_{ZO})}\leq c_1 \gamma^{k}\|f\|_{L^2(P_{XO})}.
\end{align*}
We observe that 
\begin{align}
\label{eq:schwartz_f_g_l2_norm_equiv}
    \|f\|_{L^2(P_{XO})}^2 &= \int_{\mathbb{R}^{2}}|f(x,o)|^2p(x,o)\;\mathrm{d}x\;\mathrm{d}o \nonumber \\
    &= \int_{S^1\times S^1}|f(x,o)|^2p(x,o)\;\mathrm{d}x\;\mathrm{d}o = \|g\|_{L^2(P_{X})}^2\|h\|^2_{L^2(P_{O})}. 
\end{align}
We also observe that 
\begin{align}
    \|Tf\|_{L^2(P_{ZO})}^2 &= \langle f, T^{\ast}Tf\rangle_{L^2(P_{XO})}\nonumber\\
    &= \int_{S^1\times S^1}\overline{f(x,o)}(T^{\ast}Tf)(x,o)p(x,o)\;\mathrm{d}x\;\mathrm{d}o.\label{eq:schwartz_expansion_tf_l2_sq}
\end{align}
We also observe that for all $x \in S^1$ and $o \in S^1$,
\begin{align*}
    (T^{\ast}Tf)(x,o) &= \int_{S^1}\int_{S^1} f(x',o)p(x'\mid z,o) \;\mathrm{d}x'\;p(z\mid x,o)\;\mathrm{d}z\\
    &\stackrel{(\ast)}{=} \int_{S^1}f(x',o)\left(\int_{S^1}p(x'\mid z,o)p(z\mid x,o)\;\mathrm{d}z\right)\;\mathrm{d}x'\\
    &= \int_{S^1}f(x',o)L(x,x',o)\;\mathrm{d}x',
\end{align*}
where $(\ast)$ follows by Fubini's theorem. $L$ is defined as, for all $x, x^\prime \in S^1$ and $o \in S^1$,
\begin{align*}
    L(x,x',o) := \int_{S^1}p(x'\mid z,o)p(z\mid x,o)\;\mathrm{d}z = \int_{S^1}p(x'\mid z)p(z\mid x)\;\mathrm{d}z.
\end{align*}
The last step holds because $X$ is independent of $O$. 
As $L$ is not dependent on $o$, we write $L(x,x',o) = L(x,x')$. Hence $(\forall x\in S^1,o \in S^{1})$
\begin{align*}
    (T^{\ast}Tf)(x,o) = \left(\int_{S^1}g(x')L(x,x')\;\mathrm{d}x'\right)h(o).
\end{align*}
Hence continuing from Eq. \eqref{eq:schwartz_expansion_tf_l2_sq}, we find
\begin{align}
    \|Tf\|_{L^2(P_{ZO})}^2 &= \int_{S^1\times S^1}\overline{g(x)}\overline{h(o)}(T^{\ast}Tf)(x,o)p(x,o)\;\mathrm{d}x\;\mathrm{d}o\nonumber\\
    &=  \left(\int_{S^1}\overline{g(x)}\left(\int_{S^1}g(x')L(x,x')p(x)\;\mathrm{d}x'\right)\;\mathrm{d}x\right)\|h\|^2_{L^2(P_{O})}\label{eq:schwartz_tf_intm}.
\end{align}
We also notice that $(\forall x\in S^1, x'\in S^1)$, the following hold via Bayes' rule and the fact that $p(z)$ is the Haar measure on $S^1$:
\begin{align*}
    L(x,x')p(x) &= \int_{z\in S^1}p(x'\mid z)p(z\mid x)p(x)\;\mathrm{d}z\\
    &= \int_{z\in S^1} p(x'\mid z)p(x\mid z)p(z)\;\mathrm{d}z\\
    &= \int_{z\in S^1}p(x'\mid z)p(x\mid z)\;\mathrm{d}z\\
    &= \int_{z\in S^1}p_W(x'-z)p_{W}(x-z)\;\mathrm{d}z\\
    &= \int_{z\in S^1}p_{W}(z)p_W(z + (x-x'))\;\mathrm{d}z\\
    &=: \mathfrak{L}(x-x').
\end{align*}
where the second last step follows from the change of variable $z\leftarrow x'-z$, and in the last step we use the fact that $L(x,x')p(x)$ is translation-invariant. As throughout the calculations, the difference $x-x'$ denotes the group operation on $S^{1}$ rather than the usual difference on $\mathbb{R}$. We calculate the Fourier coefficients of $\mathfrak{L}$ as follows:
\begin{align*}
    \calF[\mathfrak{L}][n] &= \int_{w\in S^1}\mathfrak{L}(w)e^{-i2\pi nw}\;\mathrm{d}w = \int_{w\in S^1}\int_{z\in S^1}p_{W}(z)p_W(z + w)\;\mathrm{d}z\;e^{-i2\pi nw}\;\mathrm{d}w\\
    &=\left|\int_{z\in S^1}p_{W}(z)e^{2\pi in z}\;\mathrm{d}z\right|^2 = |\calF[p_{W}][n]|^2.
\end{align*}
Since $W = 0.1(U_1 + \dots + U_k)$, we have $p_{W}$ is the $k$-times convolution of the probability density function of $0.1U_1$. By the convolution Theorem, we find
\begin{align}
\label{eq: schwartz_fourier_nonabs_val}
    \calF[\mathfrak{L}][n] &= \left|10\int_{0}^{0.1}e^{-2\pi i nz}\;\mathrm{d}z\right|^{2k} = \left|\frac{10}{2\pi i n}\left(1 - e^{-0.2\pi i n}\right)\right|^{2k}  = \left(\frac{10}{\pi n}\right)^{2k}\sin^{2k}(0.1\pi n).
\end{align}
Hence
\begin{align}
\label{eq:schwartz_fcoeff_mf_L}
    \calF[\mathfrak{L}][n]\leq \left(\frac{10}{\pi n}\right)^{2k}.
\end{align}
We have
\begin{align}
    \frac{\|Tf\|^2_{L^2(P_{ZO})}}{\|h\|^2_{L^2(P_{O})}}
    &= \int_{S^1}\overline{g(x)}\left(\int_{S^1}g(x')\mathfrak{L}(x-x')\;\mathrm{d}x'\right)\;\mathrm{d}x\nonumber\\
    &\stackrel{(a)}{=} \sum_{n = -\infty}^{\infty}\overline{\calF[g][n]}\calF\left[\int_{S^1}g(x')\mathfrak{L}(\cdot -x')\;\mathrm{d}x'\right][n]\nonumber\\
    &\stackrel{(b)}{=} \sum_{n = -\infty}^{\infty}\overline{\calF[g][n]}\calF[g][n]\calF[\mathfrak{L}][n] = \sum_{n=-\infty}^{\infty}|\calF[g][n]|^2\cdot \calF[\mathfrak{L}][n]\label{eq: example_ratio}\\
    &\stackrel{(c)}{=} \sum_{n\in \mathbb{Z}, |n|> (2\pi\gamma)^{-1}}|\calF[g][n]|^2\cdot \calF[\mathfrak{L}][n] \stackrel{(d)}{\leq} (20\gamma)^{2k}\sum_{n\in \mathbb{Z}, |n|> (2\pi\gamma)^{-1}}|\calF[g][n]|^2\nonumber\\
    &= (20\gamma)^{2k}\sum_{n=-\infty}^{\infty}|\calF[g][n]|^2 \stackrel{(e)}{=} (20\gamma)^{2k}\|g\|^2_{L^2(S^1)}\nonumber\\
    &\stackrel{(f)}{=} (20\gamma)^{2k}\|g\|^2_{L^2(P_{X})}. \nonumber
\end{align}
In the above derivations, step $(a)$ follows from Parseval's theorem \citep[5.4]{Katznelson_2004} and the fact that $\{e^{i2\pi nx}\}_{n\in \mathbb{Z}}$ forms an orthonormal system in $L^2(S^{1})$, step $(b)$ follows from \Cref{prop:convolution_theorem}, step $(c)$ follows Eq. \eqref{eq:schwartz_assumption_calf_g}, step $(d)$ follows from Eq. \eqref{eq:schwartz_fcoeff_mf_L}, step $(e)$ follows again from Parseval's Theorem, and finally step $(f)$ follows from the fact that $P_{X}$ is the Haar measure on $S^1$.  Continuing from Eq. \eqref{eq:schwartz_tf_intm}, we thus have
\begin{align*}
    \|Tf\|^2_{L^2(P_{ZO})}\leq (20\gamma)^{2k}\|g\|^2_{L^2(P_{X})}\|h\|^2_{L^2(P_{O})} = (20\gamma)^{2k}\|f\|^2_{L^2(P_{XO})},
\end{align*}
where the last equality follows from Eq. \eqref{eq:schwartz_f_g_l2_norm_equiv}. Therefore, we have proved that \Cref{ass:T_contractivity_schwartz} is satisfied with $\eta_1 = k$ and $c_1 = 20^k$. For any $\gamma>0$, we now exhibit a class of functions $f'\in \mathrm{LF}(\gamma)$ and a distribution $p(x,z,o)$ satisfying the statement in \Cref{ass:T_frequency_ill_posedness_schwartz}. We let $p(x,z,o)$ be the probability distribution constructed above. Let $g'\in L^1(S^1)$ be a function whose Fourier coefficients vanish on high frequencies, in the sense that 
\begin{align*}
    \mathcal{F}[g'][n] = 0 \quad \text{for all } n \in \mathbb{Z} \text{ such that } |n| \geq (2\pi\gamma)^{-1}.
\end{align*}
We further assume that $\calF[g'][n]\neq 0$ only if $\sin^{2k}(0.1\pi n) \geq c$ for a fixed positive constant $c>0$. Let $h\in L^1_{\mathrm{loc}}(\mathbb{R}^{d_o})$ be such that it does not vanish identically. We then define $f': \mathbb{R}^2 \to \mathbb{C}$:
\begin{align*}
    f'(x+t,o):= g'(x)h(o)\quad \text{for all }x\in [0,1), t\in \mathbb{Z}, o\in \mathbb{R}.
\end{align*}
We can show that $f'\in \mathrm{LF}(\gamma)$ defined in Eq.~\eqref{eq:distribution_lf_hf_set_defn} via a similar argument as before. By Eq.~\eqref{eq: example_ratio}, we have
\begin{align*}
    \frac{\|Tf'\|^2_{L^2(P_{ZO})}}{\|h\|^2_{L^2(P_O)}} &= \sum_{n = -\infty}^{\infty}|\calF[g'][n]|^2\cdot \calF[\mathfrak{L}][n]\\
    &= \sum_{n\in \mathbb{Z}, |n| < (2\pi\gamma)^{-1}}|\calF[g'][n]|^2\cdot \calF[\mathfrak{L}][n]\\
    &\stackrel{(a)}{=} \sum_{n\in \mathbb{Z}, |n| < (2\pi\gamma)^{-1}}|\calF[g'][n]|^2 \left(\frac{10}{\pi n}\right)^{2k}\sin^{2k}(0.1\pi n)\\
    &\stackrel{(b)}{\geq} \sum_{n\in \mathbb{Z}, |n| < (2\pi\gamma)^{-1}}|\calF[g'][n]|^2 \left(\frac{10}{\pi n}\right)^{2k}c\\
    &\geq c(20\gamma)^{2k}\sum_{n\in \mathbb{Z}, |n| < (2\pi\gamma)^{-1}}|\calF[g'][n]|^2 = c(20\gamma)^{2k}\|g'\|^2_{L^2(S^1)},
\end{align*}
where step $(a)$ follows from Eq.~\eqref{eq: schwartz_fourier_nonabs_val}, and step $(b)$ follows from the assumption on the Fourier coefficient of $g$. We thus have
\begin{align*}
    \|Tf'\|^2_{L^2(P_{ZO})} \geq c(20\gamma)^{2k}\|f'\|^2_{L^2(P_{XO})}.
\end{align*}

\section{Explicit Solutions of KIV-O}
\label{sec:closed_form_kivo}

The following derivation is adapted from \citet[Section D]{meunier2024nonparametricinstrumentalregressionkernel}, which only covers the case of no observed covariates. We refer the reader to \citet[Section A.5.1]{singh2019kernel} for the original derivation of closed-form solution for KIV with no observed covariates, and to \citet{mastouri2021proximal, singh2023kernel, xu2025kernelsingleproxycontrol} for derivation of closed-form solutions for the RKHS two-stage proximal causal learning framework, which is mathematically equivalent to KIV with observed covariates, as discussed in \Cref{sec:pcl}. Whenever an operator or Gram matrix require distinguishing between Stage I or Stage II kernel on $O$, we denote this via $;1$ or $;2$ in the subscript. $\odot$  denotes Hadamard product. For a matrix $J\in \mathbb{R}^{m\times n}$, $J_{\cdot, j}$ denotes its $j$th column. \\
\textbf{Stage 1}\quad We follow the closed-form solution given in \citet{lietal2022optimal}.  We define
\begin{align*}
    \boldsymbol{\Phi}_{\tilde{Z}\tilde{O}; 1}: \mathcal{H}_{ZO} \rightarrow \mathbb{R}^{\tilde{n}},& \quad  \boldsymbol{\Phi}_{\tilde{Z}\tilde{O};1} = \left[\phi_Z\left(\tilde{\bz}_1\right) \otimes \phi_{O,1}\left(\tilde{\bo}_1\right), \ldots, \phi_Z\left(\tilde{\bz}_{\tilde{n}} \right) \otimes \phi_{O,1}(\tilde{\bo}_{\tilde{n}})\right]^*, \\
    \boldsymbol{\Phi}_{\tilde{X}}: \mathcal{H}_X \rightarrow \mathbb{R}^{\tilde{n}},& \quad \boldsymbol{\Phi}_{\tilde{X}}=\left[\phi_X\left(\tilde{\bx}_1\right), \ldots, \phi_X\left(\tilde{\bx}_{\tilde{n}} \right)\right]^* 
\end{align*}
We obtain the following estimator
\begin{align}
\label{eq:closed_form_stage_1_estimator}
    \hat{F}_{\xi}(\cdot, \cdot) = \hat{C}_{X\mid Z,O; \xi} \phi_{Z}(\cdot)\otimes \phi_{O,1}(\cdot),\quad 
    \hat{C}_{X\mid Z,O; \xi} = \bPhi_{\tilde{X}}^{\ast}\left(K_{\tilde{Z}\tilde{O};1} + \tilde{n}\xi\mathrm{Id}\right)^{-1}\bPhi_{\tilde{Z}\tilde{O}; 1},
\end{align}
where we introduce the Gram matrix
\begin{align*}
    K_{\tilde{Z}\tilde{O};1} = \bPhi_{\tilde{Z}\tilde{O};1}\bPhi_{\tilde{Z}\tilde{O};1}^{\ast}, \quad [K_{\tilde{Z}\tilde{O};1}]_{ij} = k_{Z}(\tilde{\bz}_i, \tilde{\bz}_j)k_{O,1}(\tilde{\bo}_i, \tilde{\bo}_j). 
\end{align*}
\textbf{Stage 2}\quad The Stage 2 solution can be written as 
\begin{align*}
    \hat{f}_{\lambda} = \left(\frac{1}{n}\bPhi_{\hat{F}O;2}^{\ast}\bPhi_{\hat{F}O;2} + \lambda\mathrm{Id}\right)^{-1}\frac{1}{n}\bPhi_{\hat{F}O;2}^{\ast}\mathbf{Y}
\end{align*}
where we define 
\begin{align*}
    \bPhi_{\hat{F}O;2} : \calH_{XO}\to \mathbb{R}^n \quad \bPhi_{\hat{F}O;2}  = \left[\hat{F}_{\xi}(\bz_1,\bo_1)\otimes \phi_{O,2}(\bo_1),\dots, \hat{F}_{\xi}(\bz_n,\bo_n)\otimes \phi_{O,2}(\bo_n)\right]^{\ast}.
\end{align*}
We then write this in a dual form
\begin{align*}
    \hat{f}_{\lambda} = \bPhi_{\hat{F}O;2}^{\ast}\left(K_{\hat{F}O;2} + n\lambda\mathrm{Id}\right)^{-1}\mathbf{Y}.
\end{align*}
where we introduce the Gram matrix
\begin{align*}
    K_{\hat{F}O;2} = \bPhi_{\hat{F}O;2}\bPhi_{\hat{F}O;2}^{\ast}, \quad [K_{\hat{F}O;2}]_{ij} = \langle \hat{F}_{\xi}(\bz_i,\bo_i), \hat{F}_{\xi}(\bz_j, \bo_j)\rangle_{\calH_{X}}k_{O,2}(\bo_i, \bo_j).
\end{align*}
By Eq.~\eqref{eq:closed_form_stage_1_estimator}, we obtain, for $1\leq j\leq n$,  
\begin{align}
\label{eq:closed_form_hat_F_xi}
    \hat{F}_{\xi}(\bz_j, \bo_j) = \bPhi_{\tilde{X}}^{\ast}\underbrace{\left(K_{\tilde{Z}\tilde{O},1} + \tilde{n}\xi\mathrm{Id}\right)^{-1}\left(K_{\tilde{Z}Z}\odot K_{\tilde{O}O;1}\right)_{: j}}_{=: J_{:j}} = \sum_{i=1}^{\tilde{n}}J_{ij}\phi_{X}(\tilde{\bx}_i),
\end{align}
where $J$ as defined column-wise is a $\tilde{n}\times n$ matrix, and we define the (cross) Gram matrices
\begin{align*}
    [K_{\tilde{Z}Z}]_{ij} = k_{Z}(\tilde{\bz}_i, \bz_j),\quad [K_{\tilde{O}O;1}]_{ij} = k_{O,1}(\tilde{\bo}_i, \bo_j),
\end{align*}
for $1\leq i\leq \tilde{n}$, $1\leq j\leq n$. Consequently, for $1\leq i,j\leq n$,  we have
\begin{align*}
    K_{\hat{F}O;2} = \left(J^{T}K_{\tilde{X}\tilde{X}}J\right)\odot K_{OO;2},
\end{align*}
where we define the Gram matrices
\begin{align*}
    [K_{\tilde{X}\tilde{X}}]_{ij} = k_{X}(\tilde{\bx}_{i}, \tilde{\bx}_{j}), \quad [K_{OO;2}]_{lm} = k_{O,2}(\bo_l, \bo_m),
\end{align*}
for $1\leq i,j\leq \tilde{n}$, $1\leq l, m\leq n$. For a new test point $(\bx,\bo)\in \calX\times \calO$, we have, for $1\leq j\leq n$, 
\begin{align*}
    \left\langle \phi_{X}(\bx)\otimes \phi_{O,2}(\bo), \hat{F}_{\xi}(\bz_j, \bo_j)\otimes \phi_{O,2}(\bo_j)\right\rangle_{\calH_{XO}} &= k_{O,2}(\bo, \bo_j) \left(\sum_{i=1}^{\tilde{n}}J_{ij}k_{X}(\tilde{\bx}_i, \bx)\right)\\
    &= \left(K_{O\bo,2}\odot (J^{T}K_{\tilde{X}\bx})\right)_j, 
\end{align*}
where we define $K_{\tilde{X}\bx} \in \mathbb{R}^{\tilde{n}\times 1}$ and $K_{O\bo,2}\in \mathbb{R}^{n\times 1}$ respectively as follows: 
\begin{align*}
[K_{\tilde{X}\bx}]_{i} = k_{X}(\tilde{\bx}_i, \bx),\quad [K_{O\bo, 2}]_{i} = k_{O,2}(\bo_i, \bo).     
\end{align*}
Thus we have
\begin{align*}
    \hat{f}_{\lambda}(\bx,\bo) 
    &= \left\langle \phi_{X}(\bx)\otimes \phi_{O,2}(\bo), \bPhi_{\hat{F}O;2}^{\ast}\left(K_{\hat{F}O;2} + n\lambda \mathrm{Id}\right)^{-1}\mathbf{Y}\right\rangle_{\calH_{XO}}\\
    &= \left(K_{O\bo, 2}^{T} \odot (K_{\tilde{X}\bx}^TJ)\right) \left(\left(J^{T}K_{\tilde{X}\tilde{X}}J\right)\odot K_{OO;2} + n\lambda \mathrm{Id}\right)^{-1}\mathbf{Y},
\end{align*}
where the last line follows by Eq.~\eqref{eq:closed_form_hat_F_xi}, 
Thus the derivation is concluded.

\section{RKHS $\calH_{FO}$}\label{sec:defi_H_FO}
We recall that the NPIV-O problem can be written as 
\begin{align}
\label{eq:xi_inverse}
    Y = (Tf_\ast)(Z, O) + \upsilon, \quad \E[\upsilon \mid Z, O] = 0,
\end{align}
where $\upsilon:= f_{\ast}(X,O) - (Tf_{\ast})(Z,O) + \epsilon$. \steve{As introduced in \citet[Appendix E.1.2]{meunier2024nonparametricinstrumentalregressionkernel}, we define an RKHS $\calH_{FO}$ induced by the statistical inverse problem Eq.~\eqref{eq:xi_inverse}, following \citet[Theorem 4.21]{steinwart2008support}. Our construction can be obtained from that of \citet{meunier2024nonparametricinstrumentalregressionkernel} with an appropriate feature map construction, as follows. } Recall the definition of $F_{\ast}: \calZ\times \calO\to \calH_{X,\gamma_x}$ that for any $\bz\in \calZ, \bo\in \calO$, $F_{\ast}(\bz,\bo) = \mathbb{E}[\phi_{X,\gamma_x}(X)\mid Z = \bz, O = \bo]$.

\begin{defi}\label{defi:H_F}
We define a reproducing kernel Hilbert space $\calH_{FO}$ as
\begin{align*}
    \calH_{FO} = \left\{f: \mathcal{Z}\times \calO\to \R\mid \exists w\in \calH_{\gamma_x,\gamma_o}, \; f(\bz,\bo) \equiv \langle w, F_{\ast}(\bz,\bo) \otimes \phi_{\gamma_o}(\bo)\rangle_{\calH_{\gamma_x,\gamma_o}}\right\},
\end{align*}
equipped with the norm
\begin{align*}
    \|f\|_{\calH_{FO}} := \inf\left\{\|w\|_{\calH_{\gamma_x,\gamma_o}}\mid f(\bz,\bo) \equiv \langle w, F_{\ast}(\bz,\bo)\otimes \phi_{\gamma_o}(\bo)\rangle_{\calH_{\gamma_x,\gamma_o}}\right\}.
\end{align*}
\end{defi}
We observe that $(\bz, \bo) \mapsto F_*\left(\bz, \bo \right) \otimes \phi_{\gamma_o}(\bo)$ is \emph{not} the canonical feature map of $\calH_{FO}$. To construct its canonical feature map, we define $V:\calH_{\gamma_x,\gamma_o}\to \calH_{FO}$ such that
\begin{equation*}
    (Vw)(\bz, \bo) \equiv \langle w, F_{\ast}(\bz,\bo)\otimes \phi_{\gamma_o}(\bo)\rangle_{\calH_{\gamma_x,\gamma_o}}, \quad w\in \calH_{\gamma_x,\gamma_o} .
\end{equation*}
From Theorem 4.21 of \cite{steinwart2008support}, $V$ is a metric surjection. 
By definition, for any $f\in \calH_{\gamma_x,\gamma_o}$, we have
\begin{align*}
    (Vf)(\bz,\bo) = \langle f, F_{\ast}(\bz,\bo)\otimes \phi_{\gamma_o}(\bo)\rangle_{\calH_{\gamma_x,\gamma_o}} = \E[f(X,O)\mid Z=\bz, O = \bo] = (T[f])(\bz,\bo).
\end{align*}
Furthermore, we know that $V$ is also a surjective partial isometry, i.e. $V$ is surjective and satisfies
\begin{align*}
    (\forall f\in \ker(V)^{\perp}),\quad\|f\|_{\calH_{\gamma_x,\gamma_o}} = \|Vf\|_{\calH_{FO}}.
\end{align*}
Equivalently, $(\forall r\in \calH_{FO})$,
\begin{align}
\label{eq:r_pseudoinv_norm}
    \|r\|_{\calH_{FO}} = \inf\{\|h\|_{\calH_{\gamma_x,\gamma_o}}: h\in \calH_{\gamma_x,\gamma_o},\;r = Vh\},
\end{align}
Thus, the canonical feature map of $\calH_{FO}$ is $(\bz,\bo)\mapsto V(F_{\ast}(\bz,\bo)\otimes \phi_{\gamma_o}(\bo))$, a fact which was also observed on \citet[Page 28]{meunier2024nonparametricinstrumentalregressionkernel}.

Recall the definition of $f_{\lambda}$ in Eq.~\eqref{eq:f_lambda_main}, 
\begin{align}
    \label{eq:f_lambda}
        f_{\lambda} := \argmin_{f\in \calH_{\gamma_x,\gamma_o}}\lambda\|f\|^2_{\calH_{\gamma_x,\gamma_o}} + \|T([f] - f_{\ast})\|^2_{L^2(P_{ZO})}. 
    \end{align}
and $\bar{f}_{\lambda}$ in Eq.~\eqref{eq:bar_f_lambda_main}, 
\begin{align}
\label{eq:bar_f_lambda}
    \bar{f}_{\lambda}:= \argmin_{f\in \calH_{\gamma_x,\gamma_o}} \lambda \|f\|^2_{\calH_{\gamma_x,\gamma_o}} + \frac{1}{n}\sum_{i=1}^{n}(y_i - \langle f, F_{\ast}(\bz_i, \bo_i)\otimes \phi_{\gamma_o}(\bo_i)\rangle_{\calH_{\gamma_x,\gamma_o}})^2.
\end{align}
We define the images of $f_{\lambda}$ and $\bar{f}_{\lambda}$ respectively under the metric surjection $V$ as follows:
\begin{align*}
    h_{\lambda} := Vf_{\lambda}, \quad \Bar{h}_{\lambda} := V\Bar{f}_{\lambda} \in \calH_{FO},\quad h_{\ast} := Tf_{\ast}.
\end{align*}
We observe by combining Eq. \eqref{eq:r_pseudoinv_norm} and Eq. \eqref{eq:bar_f_lambda} that $\overline{h}_{\lambda}$ is the solution to a standard kernel ridge regression problem with the RKHS $\calH_{FO}$:
\begin{align}\label{eq:bhl_bvfl}
    \Bar{h}_{\lambda} = \argmin_{h\in \calH_{FO}}\lambda \|h\|_{\calH_{FO}}^2 + \frac{1}{n}\sum_{i=1}^{n}(h(\bz_i, \bo_i) - y_i)^2 . 
\end{align}
Similarly, for $h_\lambda$ we have 
\begin{align}\label{eq:hat_bhl_bvfl}
    h_{\lambda} = \argmin_{h\in \calH_{FO}} \lambda \|h\|^2_{\calH_{FO}} + \|h_{\ast} - h\|^2_{L^2(P_{ZO})}.
\end{align} 
Finally, we have 
\begin{align}\label{eq:error_translate}
    \left\|T\left(\left[\Bar{f}_{\lambda} \right] - [f_{\lambda}]\right)\right\|_{L^2(P_{ZO})} = \left\| \left[V \Bar{f}_{\lambda}\right] - [V f_{\lambda}] \right\|_{L^2(P_{ZO})} = \left\|\left[\Bar{h}_{\lambda} \right] - \left[h_{\lambda}\right]\right\|_{L^2(P_{ZO})},
\end{align}
which means that the projected error $\|T([\Bar{f}_{\lambda}] -[f_{\lambda}]])\|_{L^2(P_{ZO})}$ has been translated to the generalization error $\left\|\left[\Bar{h}_{\lambda} \right] - [h_{\lambda}]\right\|_{L^2(P_{ZO})}$ of a standard kernel ridge regression (KRR) with this new hypothesis space $\calH_{FO}$, which allows us to apply techniques from the analysis of kernel ridge regression. To this end, we will analyse the capacity (\Cref{sec:capacity}) along with the embedding property (\Cref{sec:emb}) of $\calH_{FO}$, both of which are crucial properties for characterizing the generalization error of KRR~\citep{fischer2020sobolev}. 

\subsection{Capacity of $\calH_{FO}$}
\label{sec:capacity}
We have
\begin{align*}
    &\quad \sup_{\bo,\bo^\prime} \sup_{\bz,\bz^\prime} \left\langle V (F_{\ast}(\bz,\bo) \otimes \phi_{\gamma_o}(\bo)), V (F_{\ast}(\bz^\prime,\bo^\prime) \otimes \phi_{\gamma_o}(\bo^\prime) ) \right\rangle_{\calH_{FO}} \\
    &= \sup_{\bo,\bo^\prime} \sup_{\bz,\bz^\prime} \left\langle F_{\ast}(\bz,\bo) \otimes \phi_{\gamma_o}(\bo), F_{\ast}(\bz^\prime,\bo^\prime) \otimes \phi_{\gamma_o}(\bo^\prime) \right\rangle_{\calH_{\gamma_x,\gamma_o}} \\
    &= \sup_{\bo,\bo^\prime} k_{O, \gamma_o}(\bo, \bo^\prime) \sup_{\bz,\bz^\prime} \iint_{\calX\times\calX} k_{X, \gamma_x}(\bx, \bx^\prime)p(\bx \mid \bz,\bo) p(\bx^\prime \mid \bz^\prime,\bo^\prime)\;\mathrm{d}\bx\;\mathrm{d}\bx'\\
    &\leq \sup_{\bo,\bo^\prime}k_{O, \gamma_o}(\bo, \bo^\prime) \sup_{\bx,\bx^\prime} k_{X, \gamma_x}(\bx, \bx^\prime) \sup_{\bz,\bz^\prime} \iint_{\calX\times\calX} p(\bx \mid \bz,\bo) p(\bx^\prime \mid \bz^\prime,\bo^\prime)\;\mathrm{d}\bx\;\mathrm{d}\bx' \\
    &= \sup_{\bo,\bo^\prime} k_{O, \gamma_o}(\bo, \bo^\prime) \sup_{\bx,\bx^\prime}k_{X, \gamma_x}(\bx, \bx^\prime) . 
\end{align*}
Hence the RKHS $\calH_{FO}$ has a bounded kernel, therefore by \citet[Lemma 2.3]{steinwart2012mercer}, $\calH_{FO}$ is compactly embedded into $L^2(P_{ZO})$. We define the covariance operator $C_{FO}: \calH_{FO} \to \calH_{FO}$ 
\begin{align*}
    C_{FO} := V\mathbb{E}_{ZO}\left[(F_{\ast}(Z,O)\otimes \phi_{\gamma_o}(O))\otimes (F_{\ast}(Z,O)\otimes \phi_{\gamma_o}(O))\right]V^{\ast} 
\end{align*}
It is a self-adjoint compact operator by \citet[Lemma 2.2]{steinwart2012mercer}.
The spectral theorem for self-adjoint compact
operators~\cite[Theorems VI.16,
VI.17]{reed1980methods} yields, there exists countable $\mu_{FO,1}\geq \mu_{FO,2}\geq \dots \geq 0$,  $([e_{FO,i}])_{i\geq 1}$ an orthonormal system of $L^2(P_{ZO})$ and $\left(\sqrt{\mu_{FO,i}}e_{FO,i}\right)\subseteq \calH_{FO}$ an orthonormal system in $\calH_{FO}$, such that 
\begin{align}\label{eq:C_F_covariance}
    C_{FO} = \sum_{i\geq 1}\mu_{FO,i}\langle \cdot, \sqrt{\mu_{FO,i}}e_{FO,i}\rangle_{\calH_{FO}} \sqrt{\mu_{FO,i}}e_{FO,i}.
\end{align}
\begin{defi}[Effective dimension]
The effective dimension of $\calH_{FO}$, denoted as $\calN_{FO} : [0,\infty)\to [0,\infty)$ is  defined as 
    \begin{align*}
        \calN_{FO}(\lambda) = \mathrm{tr}\left((C_{FO}+\lambda)^{-1}C_{FO}\right) = \sum_{i\geq 1}\frac{\mu_{FO,i}}{\mu_{FO,i} + \lambda} .
    \end{align*}
\end{defi}
\begin{prop}
\label{prop:eff_dim}
Let $n\geq 10$ and $\lambda=n^{-1}$. Let $\gamma_x,\gamma_o\in (0,1]$ be the lengthscales for the RKHS $\calH_{\gamma_x,\gamma_o}$. Suppose that the distribution $P_{XO}$ satisfies Assumption~\ref{assn: technical}. Then we have,
    \begin{align*}
        \mathcal{N}_{FO}(\lambda) \leq C'(\log n)^{d_x+d_o+1} \left(\gamma_x^{d_x}\gamma_o^{d_o}\right)^{-1}
    \end{align*}
    for some constant $C'$ independent of $n, \gamma_x,\gamma_o$.
\end{prop}
\begin{proof}
From \Cref{cor:CFO_evalue_entropy_bound} and \Cref{cor:entropy_bounded_obscov} and setting $p = \frac{1}{\log n}$, $(\forall i\geq 1)$, we have 
\begin{align*}
    \mu_{FO,i} \leq (C \log n)^{2(d_x+d_o+1) \log n}(\gamma_x^{d_x}\gamma_o^{d_o})^{-2\log n} i^{-2\log n}.
\end{align*}
Next, we use \cite[Proposition 3]{caponnetto2007optimal} (with error corrected in \cite{sutherland2021fixingerrorcaponnettovito}) to obtain
\begin{align}
    \mathcal{N}_{FO}(\lambda) &\leq \frac{\pi/(2\log n)}{\sin(\pi/(2\log n))}\left( (C \log n)^{2(d_x+d_o+1) \log n}(\gamma_x^{d_x}\gamma_o^{d_o})^{-2\log n} \right)^{\frac{1}{2\log n}}\lambda^{-(2\log n)^{-1}}\nonumber\\
    &= \frac{\pi/(2\log n)}{\sin(\pi/(2\log n))} (C\log n)^{(d_x+d_o+1)} \left(\gamma_x^{d_x}\gamma_o^{d_o}\right)^{-1}\lambda^{-(2\log n)^{-1}}\nonumber\\
    &\stackrel{(i)}{\leq} 3C^{(d_x+d_o+1)} (\log n)^{(d_x+d_o+1)} 
    \left(\gamma_x^{d_x}\gamma_o^{d_o}\right)^{-1}n^{(2\log n)^{-1}}\nonumber\\
    &\stackrel{(ii)}{=} 3\sqrt{e} C^{(d_x+d_o+1)} (\log n)^{(d_x+d_o+1)} \left(\gamma_x^{d_x}\gamma_o^{d_o}\right)^{-1}\nonumber\\
    &= C'(\log n)^{d_x+d_o + 1} \left(\gamma_x^{d_x}\gamma_o^{d_o}\right)^{-1}\nonumber,
\end{align}
where $C' = \sqrt{e} 3C^{(d_x+d_o+1)}$. 
In the above chain of derivations, $(i)$ holds because $\frac{t}{\sin t} \leq 3$ for $t \leq \pi/2$ and $\pi/(2\log n) \leq \pi/2$ and $(ii)$ holds because $n^{\frac{1}{2\log n}} = \sqrt{e}$ for $n > 1$. 
\end{proof}

In the proof of \Cref{prop:eff_dim}, to bound the effective dimension, we need to bound the eigendecay of the compact self-adjoint operator $C_{FO}$. We first control the decay of the entropy numbers of the RKHS $\calH_{FO}$, which translates into a bound on the eigendecay of $C_{FO}$ as shown in \Cref{cor:CFO_evalue_entropy_bound}. In \Cref{lem:reduced_rkhs_entropy}, we show that the $i$th entropy number of $\calH_{FO}$ is bounded above by the $i$th entropy number of $\calH_{\gamma_x,\gamma_o}$. The entropy numbers of $\calH_{\gamma_x,\gamma_o}$ are well-understood by the results of \citet{hang2021optimal} (restated in \Cref{cor:entropy_bounded_obscov}), which completes the derivation. 

In this section, for real-valued Hilbert spaces $E,F$ and a bounded, linear, compact operator $S:E\to F$, $s_i(S)$ denotes the $i$th singular value of $S$, as defined in \citet[Eq. (A.25) Page 505]{steinwart2008support}; $e_i(S)$ denotes the $i$th entropy number of $S$, as defined in \citet[Definition A.5.26 Page 516]{steinwart2008support}; $a_i(S)$ denotes the $i$th approximation number of $S$, as defined in \citet[Eq. (A.29) Page 506]{steinwart2008support}. 

\begin{lem}
    \label{cor:entropy_bounded_obscov}
    Suppose that $P_{XO}$ satisfies \Cref{assn: technical} in the main text. Then, $(\forall i\geq 1)$, $(\gamma_x,\gamma_o\in (0,1])$, there exists a constant $C>0$ such that for any $p>0$,
    \begin{align*}
        e_i(\mathrm{id} : \calH_{\gamma_x,\gamma_o}\hookrightarrow L^{2}(P_{XO})) \leq (3C)^{\frac{1}{p}} \left(\frac{d_x+d_o+1}{ep}\right)^{\frac{d_x+d_o+1}{p}}\left(\gamma_x^{d_x}\gamma_o^{d_o}\right)^{-\frac{1}{p}}i^{-\frac{1}{p}}.
    \end{align*}
\end{lem}
\begin{proof}
    The corollary follows immediately from \citet[Proposition 1]{hang2021optimal} by setting $\bgamma=[\underbrace{\gamma_x, \ldots, \gamma_x}_{d_x}, \underbrace{\gamma_o, \ldots, \gamma_o}_{d_o}]^\top \in (0,1]^{d_x+d_o}$, and noting that $L^{\infty}(\calX\times \calO)$ continuously embeds into $L^2(P_{XO})$ with $\|L^{\infty}(\calX\times \calO)\hookrightarrow L^2(P_{XO})\|\leq 1$ and \citet[Eq. (A.38)]{steinwart2008support}. 
\end{proof}

\begin{lem}
\label{lem:reduced_rkhs_entropy}
    We have, $(\forall i\geq 1)$, 
    \begin{align*}
        e_i(\mathrm{id}: \calH_{FO} \hookrightarrow L^2(P_{ZO})) \leq e_i(\mathrm{id}:\calH_{\gamma_x,\gamma_o}\hookrightarrow L^2(P_{XO})).
    \end{align*}
\end{lem}
\begin{proof}
Fix $i\geq 1$. For a Hilbert space $\calH$, $B_{\calH}$ denotes the unit ball in $\calH$, and $B(x, r, \|\cdot \|_{\calH})$ denotes the ball in $\calH$ centred at $x\in\calH$ with radius $r$. Fix $\epsilon>0$ such that $\exists g_1, \dots, g_{2^{i-1}} \in B_{\calH_{\gamma_x,\gamma_o}}$ such that 
\begin{align*}
    B_{\calH_{\gamma_x,\gamma_o}} \subseteq \bigcup_{i=1}^{2^{i-1}}B\left(g_i, \epsilon, \|\cdot \|_{L^2(P_{XO})}\right).
\end{align*}
Fix $f\in B_{\calH_{FO}}$ and an arbitrary $\tilde{\epsilon}>0$. By \citet[Eq. (4.11)]{steinwart2008support}, there exists $g\in \calH_{\gamma_x,\gamma_o}$ such that 
\begin{align*}
    f = \langle g, F_{\ast}(\cdot, \cdot)\otimes \phi_{\gamma_o}(\cdot)\rangle_{\calH_{\gamma_x,\gamma_o}} = Vg,
\end{align*}
and $\|g\|_{\calH_{\gamma_x,\gamma_o}}\leq 1 + \tilde{\epsilon}$. By the preceding statement, $(\exists i\in \{1,\dots, 2^{i-1}\})$ such that 
\begin{align*}
    \|g - (1+\tilde{\epsilon})g_i\|_{L^2(P_{XO})}\leq \epsilon(1 + \tilde{\epsilon}).
\end{align*}
For $i\in \{1,\dots, 2^{i-1}\}$, define $f_i = (1+\tilde{\epsilon})Vg_i = (1+\tilde{\epsilon})\langle g_i, F_{\ast}(\cdot, \cdot)\otimes \phi_{\gamma_o}(\cdot)\rangle_{\calH_{\gamma_x,\gamma_o}}$. Then we have
\begin{align*}
    \|f - f_i\|^2_{L^2(P_{ZO})} &= \int_{\calZ\times \calO}\langle g - (1+\tilde{\epsilon})g_i, F_{\ast}(\bz,\bo)\otimes \phi_{\gamma_o}(\bo)\rangle_{\calH_{\gamma_x,\gamma_o}}^2p(\bz,\bo)\;\mathrm{d}\bz\;\mathrm{d}\bo\\
    &= \int_{\calZ\times \calO}\mathbb{E}[(g-(1+\tilde{\epsilon})g_i)(X,O)\mid Z = \bz,O = \bo]^2p(\bz,\bo)\;\mathrm{d}\bz\;\mathrm{d}\bo\\
    &\leq \|g - (1+\tilde{\epsilon}) g_i\|^2_{L^2(P_{XO})},
\end{align*}
where the inequality is deduced by Jensen's inequality. We find thus that 
\begin{align*}
    \calH_{B_{FO}}\subseteq \bigcup_{i=1}^{2^{i-1}}B(f_i, \epsilon(1+\tilde{\epsilon}), \|\cdot\|_{L^2(P_{ZO})}).
\end{align*}
Hence $e_i(\mathrm{id}: \calH_{FO}\hookrightarrow L^2(P_{ZO})) \leq \epsilon(1+\tilde{\epsilon})$. Since $\epsilon > e_i(\mathrm{id}:\calH_{\gamma_x,\gamma_o}\hookrightarrow L^2(P_{XO}))$ and $\tilde{\epsilon}>0$ are arbitrary, it follows that $e_i(\mathrm{id}:\calH_{FO}\hookrightarrow L^2(P_{ZO}))\leq e_i(\mathrm{id}:\calH_{\gamma_x,\gamma_o}\hookrightarrow L^2(P_{XO}))$.
\end{proof}

\begin{cor}
\label{cor:CFO_evalue_entropy_bound}
    Suppose that $P_{XO}$ satisfies \Cref{assn: technical} in the main text. We have that 
    \begin{align*}
        \mu_{FO,i}\leq 4e_i(\mathrm{id}:\calH_{\gamma_x,\gamma_o}\hookrightarrow L^2(P_{XO}))^2
    \end{align*}
\end{cor}

\begin{proof}
Let $\iota_{FO}: \calH_{FO}\to L^2(P_{ZO})$ denote the  embedding $\calH_{FO}\hookrightarrow L^2(P_{ZO})$. We have, $(\forall i\geq 1)$, the following chain of derivations
\begin{align*}
    \mu_{FO,i} &= \lambda_i(\iota_{FO}^{\ast}\iota_{FO}) \stackrel{(a)}{=} s_i(\iota_{FO})^2 \stackrel{(b)}{=} a_i(\iota_{FO})^2\\
    &\stackrel{(c)}{\leq} 4e_i(\iota_{FO})^2 \stackrel{(d)}{\leq} 4e_i(\mathrm{id}:\calH_{\gamma_x,\gamma_o}\hookrightarrow L^2(P_{XO}))^2.
\end{align*}
In the above derivations, $(a)$ follows from \citet[Eq. (A.25)]{steinwart2008support}, $(b)$ follows from the paragraph after \citet[Eq. (A.29)]{steinwart2008support}, $(c)$ follows from \citet[Eq. (A.44)]{steinwart2008support}, and $(d)$ follows from \Cref{lem:reduced_rkhs_entropy}.
\end{proof}

\subsection{Embedding property of $\calH_{FO}$}
\label{sec:emb}
\begin{defi}[Continuous embedding]
\label{defi:cont_embed}
    A Hilbert space $(X,\|\cdot\|)$ is said to continuously embed into Hilbert space $(Y,\|\cdot\|)$ if $X\subset Y$ and there exists a constant $C$ such that $\|x\|_{Y}\leq C\|x\|_{X} $ for all $x\in X$. 
    We denote this as $X\hookrightarrow Y$. The embedding norm $\|X\hookrightarrow Y\|$ is defined as the smallest constant $C$ for which the above inequality holds.
\end{defi}
\begin{defi}[Interpolation space]
\label{def:interpolation_space}
Assume that $X_0$ and $X_1$ are Hilbert spaces and $X_1\subseteq X_0$. For $\theta\in (0,1)$ and $x\in X_0$, we define the \emph{$K$-functional} $K(t,x;X_0,X_1)$ as follows
\begin{align}
\label{eq: k_functional_instance}
    K(t,x;X_0,X_1):= \inf_{y\in X_1}\left\{\|x-y\|_{X_0} + t\|y\|_{X_1}\right\} . 
\end{align}
For $\theta\in (0,1)$, we define \emph{interpolation space} \cite{Hytönen2016AnalysisBanachSpaces}[Definition C.3.1]
    \begin{align*}
        (X_0,X_1)_{\theta,2}:= \{x\in X_0 \mid \|x\|_{\theta,2}<\infty\},\quad 
        \|x\|_{\theta,2} := \left(\int_{0}^{\infty}\left(t^{-\theta}K(t,x;X_0,X_1)\right)^{2}\frac{\mathrm{d}t}{t}\right)^{\frac{1}{2}} .
    \end{align*}
\end{defi}
We have proved in \Cref{sec:capacity} that $\calH_{FO}$ can be compactly embedded into $L^2(P_{ZO})$. 
Since $P_{ZO}$ is equivalent to the Lebesgue measure over $\calZ\times\calO$ as per \Cref{ass:equiv_leb}, $\calH_{FO}$ can also be compactly embedded into $L^2(\calZ\times\calO)$. 
So we can define the $\theta$-power space of $\calH_{FO}$ following \cite[Eq. (36)]{steinwart2012mercer} 
\begin{equation*}
    \left([\calH_{FO}]^{\theta}_{P_{ZO}}, \|\cdot\|_{[\calH_{FO}]^{\theta}_{P_{ZO}}}\right) := \left\{\sum_{i\geq 1}a_i \mu_{FO,i}^{\theta/2}[e_{FO,i}]: \;(a_i)\in \ell_2(\mathbb{N})\right\}\subseteq L^2(P_{ZO}).
\end{equation*}
For $0<\theta<1$, the $\theta$-power space coincides with the interpolation space $[\calH_{FO}]^{\theta}_{L^2(P_{ZO})} \cong [L_2(P_{ZO}),[\calH_{FO}]_{P_{ZO}}]_{\theta, 2}$~\cite[Theorem 4.6]{steinwart2012mercer}.  We write $[\calH_{FO}]_{L^2(P_{ZO})} := [\calH_{FO}]_{L^2(P_{ZO})}^{1}$. Similarly, we write $[\calH_{FO}]_{L^2(\calZ\times \calO)} := [\calH_{FO}]_{L^2(\calZ\times \calO)}^{1}$. 

\begin{prop}
\label{prop:emb_main}
Suppose Assumption~\ref{assn: technical} holds, and let $m_o,m_z,\rho$ be as defined in the statement of Assumption~\ref{assn: technical}. Suppose $\theta\in (0,1)$ satisfies  $\frac{d_o}{2m_o\theta} < 1$ and $\frac{d_z}{2m_z\theta} < 1$. Then there exists a constant $C_{\theta}>0$ independent of $n$ such that 
\begin{align*}
    \left\|k_{FO}^{\theta}\right\|_{\infty} 
    := \left\|[\calH_{FO}]^{\theta}_{P_{ZO}}\hookrightarrow L^{\infty}(P_{ZO})\right\|\lesssim  C_{\theta} \rho^{\theta}\gamma_o^{-\theta m_o}.
\end{align*}
\end{prop}
\begin{proof}
By \Cref{lem:hfo_conts_embedding}, we have $[\calH_{FO}]_{L^2(\calZ\times \calO)} \hookrightarrow MW^{m_z,m_o}(\calZ\times \calO)$ and
\begin{align*}
    \| [\calH_{FO}]_{L^2(\calZ\times \calO)} \hookrightarrow MW^{m_z,m_o}(\calZ\times \calO)\| \lesssim \rho  \gamma_o^{-m_o}.
\end{align*}
By \Cref{lem:anisotropic_sobolev_embedding}, since $\frac{d_o}{2s_o}<1$ and $\frac{d_z}{2s_z}<1$, we have $MW^{\theta m_z, \theta m_o}(\calZ\times \calO)\hookrightarrow L^{\infty}(\calZ\times \calO)$, with embedding norm a universal constant independent of $n$.  By Lemma \ref{lem:embedintosobolevimplies}, we thus have $(L^2(\calZ\times \calO), [\calH_{FO}]_{L^2(\calZ\times \calO)})_{\theta,2}\hookrightarrow L^{\infty}(\calZ\times \calO)$ and 
\begin{align*}
    \| (L^2(\calZ\times \calO), [\calH_{FO}]_{L^2(\calZ\times\calO)})_{\theta,2}\hookrightarrow L^{\infty}(\calZ\times \calO)\| \lesssim \rho^{\theta}  \gamma_o^{-\theta m_o},
\end{align*}
Next, since we know by Assumption~\ref{assn: technical} that $P_{ZO}$ is equivalent to Lebesgue measure on $\calZ\times \calO$, \Cref{lem:interpolation_equiv_measure} shows that we have 
\begin{align}
    \left(L^2(P_{ZO}), [\calH_{FO}]_{L^2(P_{ZO})}\right)_{\theta,2}\hookrightarrow \left(L^2(\calZ\times \calO), [\calH_{FO}]_{L^2(\calZ\times \calO)}\right)_{\theta,2} \hookrightarrow L^{\infty}(P_{ZO})\nonumber,
\end{align}
and we have
\begin{align}
\label{eq: chain_embedding}
    \left\|\left(L^2(P_{ZO}), [\calH_{FO}]_{L^2(P_{ZO})}\right)_{\theta,2}  \hookrightarrow L^{\infty}(P_{ZO})\right\| \lesssim \rho^{\theta}\gamma_o^{-\theta m_o}
\end{align}
Finally, by \cite{steinwart2012mercer}[Theorem 4.6], $(L^2(P_{ZO}), [\calH_{FO}]_{L^2(P_{ZO})})_{\theta,2} \cong [\calH_{FO}]^{\theta}_{L^2(P_{ZO})}$ with the constant of equivalence depends only on $\theta$. Putting it back to Eq. \eqref{eq: chain_embedding} completes the proof of the proposition.
\end{proof}

\begin{lem}
\label{lem:hfo_conts_embedding}
Suppose Assumption~\ref{assn: technical} holds, and let $m_z,m_o,\rho$ be as defined in the statement of Assumption~\ref{assn: technical}. We have $[\calH_{FO}]_{L^2(\calZ\times \calO)}\hookrightarrow MW^{m_z,m_o}_2(\calZ\times\calO)$ with 
\begin{align*}
    \left\|[\calH_{FO}]_{L^2(\calZ\times \calO)}\hookrightarrow MW^{m_z,m_o}_2(\calZ\times \calO)\right\| \lesssim \rho \gamma_o^{-m_o} .
\end{align*}
\end{lem}
\begin{proof}
Let $\mathrm{id}: \calH_{FO}\hookrightarrow L^2(\calZ\times \calO)$ denote the canonical inclusion map. Since we have $(\ker \mathrm{id})^{\perp} \cong [\calH_{FO}]_{L^2(\calZ\times \calO)}$, we may represent an arbitrary element of $[\calH_{FO}]_{L^2(\calZ\times \calO)}$ as $[f]$, where $f\in (\ker \mathrm{id})^{\perp}$. Moreover, we have $\|[f]\|_{L^2(\calZ\times \calO)} = \|f\|_{\calH_{FO}}$. We fix $f$ for the remainder of the proof. 

Let $\balpha \in \mathbb{N}^{d_z}$ with $|\balpha| \leq m_z$ and $\bbeta\in \mathbb{N}^{d_o}$ with $|\bbeta| \leq m_o$, we have for any $f\in\calH_{FO}$,
\begin{small}
\begin{align*}
    &\left\|\partial^{\balpha}_{\bz}\partial^{\bbeta}_{\bo}f\right\|^2_{L^2(\calZ\times\calO;\R)} = \int_{\calZ}\int_{\calO} \left(\partial_{\bz}^{\balpha}\partial^{\bbeta}_{\bo}f(\bz,\bo)\right)^2\;\mathrm{d}\bz\;\mathrm{d}\bo\\
    \stackrel{(a)}{=}& \int_{\calZ}\int_{\calO}\left(\partial_{\bz}^{\balpha}\partial_{\bo}^{\bbeta}\int_{\calX} w(\bx,\bo) p(\bx\mid \bz,\bo)\;\mathrm{d}\bx\right)^2\;\mathrm{d}\bz\;\mathrm{d}\bo\\
    \stackrel{(b)}{=}& \int_{\calZ}\int_{\calO}\left(\sum_{\bbeta_1+\bbeta_2 = \bbeta}\binom{\bbeta}{\bbeta_1}\int_{\calX} \partial^{\bbeta_1}_{\bo} w(\bx,\bo) \partial^{\balpha}_{\bz}\partial^{\bbeta_2}_{\bo}p(\bx\mid \bz,\bo)\;\mathrm{d}\bx\right)^2\;\mathrm{d}\bz\;\mathrm{d}\bo\\
    \stackrel{(c)}{\leq}& \int_{\calZ}\int_{\calO}\left(\sum_{\bbeta_1+\bbeta_2 = \bbeta}\binom{\bbeta}{\bbeta_1}\left(\int_{\calX} \left(\partial^{\bbeta_1}_{\bo} w(\bx,\bo) \right)^2\;\mathrm{d}\bx\right)^{\frac{1}{2}} \left(\int_{\calX} \left(\partial^{\balpha}_{\bz}\partial^{\bbeta_2}_{\bo}p(\bx\mid \bz,\bo)\right)^2 \;\mathrm{d}\bx\right)^{\frac{1}{2}}\right)^2\;\mathrm{d}\bz\;\mathrm{d}\bo\\
    \leq & \left(\sum_{\bbeta_1+\bbeta_2 = \beta}\binom{\bbeta}{\bbeta_1}\right)^2 \left(\max_{\bbeta_2\leq \beta}\sup_{\bx,\bz,\bo}\left|\partial^{\balpha}_{\bz}\partial^{\bbeta_2}_{\bo}p(\bx\mid \bz,\bo)\right|^2\right) \int_{\calZ}\int_{\calO} \max_{\bbeta_1\leq \beta} \int_{\calX}\left(\partial^{\bbeta_1}_{\bo}w(\bx,\bo)\right)^2\;\mathrm{d}\bx \;\mathrm{d}\bz\;\mathrm{d}\bo\\
    =& \left(\sum_{\bbeta_1+\bbeta_2 = \beta}\binom{\bbeta}{\bbeta_1}\right)^2 \left(\max_{\bbeta_2\leq \beta}\sup_{\bx,\bz,\bo}\left|\partial^{\balpha}_{\bz}\partial^{\bbeta_2}_{\bo}p(\bx\mid \bz,\bo)\right|^2\right) \max_{\bbeta_1\leq \beta} \int_{\calO}\int_{\calX}\left(\partial^{\bbeta_1}_{\bo}w(\bx,\bo)\right)^2\;\mathrm{d}\bx \;\mathrm{d}\bo\\
    \stackrel{(d)}{=}& c_{m_o, d_o}\left(\max_{\bbeta_2\leq \beta}\sup_{\bx,\bz,\bo}\left|\partial^{\balpha}_{\bz}\partial^{\bbeta_2}_{\bo}p(\bx\mid \bz,\bo)\right|^2\right)  \gamma_o^{-2|\beta|} \|w\|_{\calH_{\gamma_x,\gamma_o}}^2\\
    \leq & c_{m_o,d_o} \left(\max_{\balpha\leq m_z}\max_{\bbeta\leq m_o}\sup_{\bx,\bz,\bo}\left|\partial^{\balpha}_{\bz}\partial^{\bbeta}_{\bo}p(\bx\mid \bz,\bo)\right|^2\right)  \gamma_o^{-2m_o} \|w\|_{\calH_{\gamma_x,\gamma_o}}^2 ,
\end{align*}
\end{small}
where $(a)$ follows from the fact that, for any $f\in\calH_{FO}$, there exists $w\in\calH_{\gamma_x,\gamma_o}$ such that $f(\bz,\bo) = \langle w, F_{\ast}(\bz,\bo)\otimes\phi_{\gamma_o}(\bo)\rangle_{\calH_{\gamma_x,\gamma_o}} = \int_{\calX}w(\bx,\bo)p(\bx\mid \bz,\bo)\;\mathrm{d}\bx$ and $\|f\|_{\calH_{FO}} = \|w\|_{\calH_{\gamma_x,\gamma_o}}$.
$(b)$ follows from generalized Leibniz's rule and differentiation under the integral sign \cite{klenke2013probability}[Theorem 6.28], $(c)$ follows from a Cauchy-Schwarz inequality.
$(d)$ follows from the following arguments.

From \citet[Theorem 4.21]{steinwart2008support}, we know that 
for any $w\in \calH_{\gamma_x,\gamma_o}$, there exists $g\in L^2(\R^{d_x+d_o})$ such that 
\begin{align*}
    w(\bx,\bo) = \langle g, \Phi_{\gamma_x}(\bx)\Phi_{\gamma_o}(\bo)\rangle_{L^2(\R^{d_x+d_o})},\quad \|w\|_{\calH_{\gamma_x,\gamma_o}} = \|g\|_{L^2(\R^{d_x+d_o})},
\end{align*}
with $\Phi_{\gamma_x}: \calX\to L^2(\R^{d_x}), \Phi_{\gamma_o}: \calO\to L^2(\R^{d_o})$ defined in \citet[Lemma 4.45]{steinwart2008support}.
We have
\begin{align*}
    &\int_{\calO} \int_{\calX}\left(\partial^{\bbeta_1}_{\bo}w(\bx,\bo)\right)^2\;\mathrm{d}\bx \;\mathrm{d}\bo\\
    =& \int_{\calO}\int_{\calX}\left(\partial_{\bo}^{\bbeta_1}\int_{\R^{d_x+d_o}} g(\bx',\bo')\Phi_{\gamma_x}(\bx)(\bx')\Phi_{\gamma_o}(\bo)(\bo')\;\mathrm{d}\bx'\;\mathrm{d}\bo'\right)^{2}\;\mathrm{d}\bx\;\mathrm{d}\bo\\
    \stackrel{(i)}{=}& \int_{\calO}\int_{\calX}\left(\int_{\R^{d_x+d_o}} g(\bx',\bo')\Phi_{\gamma_x}(\bx)(\bx')\partial_{\bo}^{\bbeta_1}(\Phi_{\gamma_o}(\bo)(\bo'))\;\mathrm{d}\bx'\;\mathrm{d}\bo'\right)^{2}\;\mathrm{d}\bx\;\mathrm{d}\bo\\
    \leq &  \|g\|^2_{L^2(\R^{d_x+d_o})}\int_{\calO}\int_{\calX}\int_{\R^{d_x+d_o}} \left(\Phi_{\gamma_x}(\bx)(\bx')\partial_{\bo}^{\bbeta_1}(\Phi_{\gamma_o}(\bo)(\bo'))\right)^{2}\;\mathrm{d}\bx'\;\mathrm{d}\bo'\;\mathrm{d}\bx\;\mathrm{d}\bo\\
    = & \|g\|^2_{L^2(\R^{d_x+d_o})}\int_{\calO}\int_{\calX}\left(\int_{\R^{d_x}}(\Phi_{\gamma_x}(\bx)(\bx'))^2\;\mathrm{d}\bx'\right) \left(\int_{\R^{d_o}}\left(\partial_{\bo}^{\bbeta_1}(\Phi_{\gamma_o}(\bo)(\bo'))\right)^{2}\;\mathrm{d}\bo'\right)\;\mathrm{d}\bx\;\mathrm{d}\bo\\
    \stackrel{(ii)}{=} & \|g\|^2_{L^2(\R^{d_x+d_o})}\int_{\calO} \int_{\R^{d_o}}\left(\partial_{\bo}^{\bbeta_1}(\Phi_{\gamma_o}(\bo)(\bo'))\right)^{2}\;\mathrm{d}\bo'\;\mathrm{d}\bo\\
    \stackrel{(iii)}{\leq} & \|g\|^2_{L^2(\R^{d_x+d_o})}c_{\bbeta_1, d_o}\gamma_o^{-2|\bbeta_1|} \\
    = & \|w\|^2_{\calH_{\gamma_x,\gamma_o}} c_{\bbeta_1, d_o}\gamma_o^{-2|\bbeta_1|},
\end{align*}
where we're allowed to exchange differentiation and integration in $(i)$ using the differentiation lemma \cite{klenke2013probability}[Theorem 6.28], $(ii)$ follows from the fact that
\begin{align*}
    \int_{\R^{d_x}}(\Phi_{\gamma_x}(\bx)(\bx'))^2\;\mathrm{d}\bx' = \int_{\R^{d_x}} \left(\frac{2^{d_x/2}}{\pi^{d_x/4}\gamma_x^{d_x/2}}\exp\left(-2\frac{\|\bx - \bx'\|^2_{2}}{\gamma_x^2}\right)\right)^2\;\mathrm{d}\bx' = 1,
\end{align*}
and $(iii)$ follows from the proof of \cite{steinwart2008support}[Theorem 4.48]. Here $c_{\bbeta_1,d_o}$ is a constant only depending on $\bbeta_1,d_o$. Hence we've shown
\begin{align*}
    &\|f\|^2_{MW^{m_z,m_o}_2(\calZ\times \calO;\R)}\\
    = & \sum_{\balpha\leq m_z, \bbeta\leq m_o}\left\|\partial^{\balpha}_{\bz}\partial^{\bbeta}_{\bo}f\right\|^2_{L^2(\calZ\times \calO;\R)}\\
    \lesssim & \sum_{\balpha \leq m_z, \bbeta\leq m_o} \left(\max_{\balpha\leq m_z}\max_{\bbeta\leq m_o}\sup_{\bx,\bz,\bo}\left|\partial^{\balpha}_{\bz}\partial^{\bbeta}_{\bo}p(\bx\mid \bz,\bo)\right|^2\right)  \gamma_o^{-2m_o} \|w\|_{\calH_{\gamma_x,\gamma_o}}^2\\
    = & \left(\sum_{\alpha \leq m_z, \bbeta\leq m_o} 1\right) \left(\max_{\balpha\leq m_z}\max_{\bbeta\leq m_o}\sup_{\bx,\bz,\bo}\left|\partial^{\balpha}_{\bz}\partial^{\bbeta}_{\bo}p(\bx\mid \bz,\bo)\right|^2\right)  \gamma_o^{-2m_o} \|f\|_{\calH_{FO}}^2 \\
    \lesssim &\rho^2 \gamma_o^{-2m_o} \|f\|_{\calH_{FO}}^2 \\
    = & \rho^2 \gamma_o^{-2m_o} \|[f]\|_{[\calH_{FO}]_{L^2(\calZ\times \calO)}}^2 .
\end{align*}
The second last inequality holds by Assumption \ref{assn: technical}.
This concludes the proof. 
\end{proof}

\subsection{Auxiliary results for \Cref{sec:defi_H_FO}}

\begin{lem}
\label{lem:interpolation_equiv_measure}
    Let $\nu_1,\nu_2$ be two measures on $\calZ$ which are equivalent, i.e $0 < c^\prime \leq \frac{\mathrm{d}\nu_2}{\mathrm{d}\nu_1} \leq c < \infty$.
    Let $\mathcal{H}\subseteq L^2(\nu_1)$ be a Hilbert space. Then $(L^2(\nu_1),\calH)_{\theta,2} \hookrightarrow (L^2(\nu_2), \calH)_{\theta,2}$ with embedding norm $\| (L^2(\nu_1),\calH)_{\theta,2} \hookrightarrow (L^2(\nu_2), \calH)_{\theta,2} \| \leq c^{1-\theta}$. 
\end{lem}
\begin{proof}
For any $f \in (L^2(\nu_1),\calH)_{\theta,2}$, we have
\begin{align*}
   \|f\|_{\left(L^2(\nu_2), \calH\right)_{\theta,2}}&= \left(\int_{0}^{\infty}\left(t^{-\theta}K(t,x;L^2(\nu_2),\calH)\right)^{2}\frac{\mathrm{d}t}{t}\right)^{\frac{1}{2}}\\  &= \left(\int_{0}^{\infty}\left(t^{-\theta}\inf_{y\in \calH}\left\{\|x-y\|_{L^2(\nu_2)} + t\|y\|_{\calH}\right\}\right)^{2}\frac{\mathrm{d}t}{t}\right)^{\frac{1}{2}} 
    \\
    &\stackrel{(a)}{\leq} \left(\int_{0}^{\infty}\left(ct^{-\theta}\inf_{y\in \calH}\left\{\|x-y\|_{L^2(\nu_1)} + \frac{t}{c}\|y\|_{\calH}\right\}\right)^{2}\frac{\mathrm{d}t}{t}\right)^{\frac{1}{2}}\\
    &\stackrel{(b)}{=} c^{1-\theta}\left(\int_{0}^{\infty}\left(t^{-\theta}\inf_{y\in \calH}\left\{\|x-y\|_{L^2(\nu_1)} + t\|y\|_{\calH}\right\}\right)^{2}\frac{\mathrm{d}t}{t}\right)^{\frac{1}{2}}\\
    &= c^{1-\theta}\|f\|_{\left(L^2(\nu_1), \calH\right)_{\theta,2}}
\end{align*}
In the above derivations, we use $\|x-y\|_{L^2(\nu_2)}\leq c\|x-y\|_{L^2(\nu_1)}$ in $(a)$, and the change of variables $t\mapsto ct$ in $(b)$. 
\end{proof}

\begin{lem}
\label{lem:embedintosobolevimplies}
Let $\calZ = [0,1]^{d}$. Let $\calH, W \subseteq L^2(\calZ)$ be two Hilbert spaces. Suppose that $\calH \hookrightarrow W$ with embedding norm $\| \calH \hookrightarrow W\| \leq c$. Suppose that $\theta\in(0,1)$ is chosen so that $\left(L^2(\calZ), W \right)_{\theta,2} \hookrightarrow L^{\infty}(\calZ)$ holds. Then we have $\left(L^2(\calZ), \calH\right)_{\theta,2} \hookrightarrow L^{\infty}(\calZ)$ with embedding norm $\leq c^{\theta}c_0$, where $c_0$ only depends on $\theta,m$ and $d_z$.
\end{lem}
\begin{proof}
We adapt the proof of \cite{kanagawa2018gaussianprocesseskernelmethods}[Lemma A.2, Corollary 4.13].
We have, for any $x \in (L^2(\calZ), W)_{\theta,2}$,
\begin{align*}
    \|x\|_{\left(L^2(\calZ), W\right)_{\theta,2}}&=  \left(\int_{0}^{\infty}\left(t^{-\theta}K\left(t,x;L^2(\calZ),W\right)\right)^{2}\frac{\mathrm{d}t}{t}\right)^{\frac{1}{2}}\\&= \left(\int_{0}^{\infty}\left(t^{-\theta}\inf_{y\in W}\left\{\|x-y\|_{L^2(\calZ)} + t\|y\|_{W}\right\}\right)^{2}\frac{\mathrm{d}t}{t}\right)^{\frac{1}{2}}\\
    &\leq \left(\int_{0}^{\infty}\left(t^{-\theta}\inf_{y\in \calH}\left\{\|x-y\|_{L^2(\calZ)} + t\|y\|_{W}\right\}\right)^{2}\frac{\mathrm{d}t}{t}\right)^{\frac{1}{2}}\\
    &\leq \left(\int_{0}^{\infty}\left(t^{-\theta}\inf_{y\in \calH}\left\{\|x-y\|_{L^2(\calZ)} + tc\|y\|_{\calH}\right\}\right)^{2}\frac{\mathrm{d}t}{t}\right)^{\frac{1}{2}}\\
    &= c^\theta \left(\int_{0}^{\infty}\left(\left(tc \right)^{-\theta}\inf_{y\in \calH}\left\{\|x-y\|_{L^2(\calZ)} + tc \|y\|_{\calH}\right\}\right)^{2}\frac{\mathrm{d}t}{t}\right)^{\frac{1}{2}}\\
    &= c^{\theta}\left(\int_{0}^{\infty}\left(t^{-\theta}K\left(t,x;L^2(\calZ),\calH\right)\right)^{2}\frac{\mathrm{d}t}{t}\right)^{\frac{1}{2}}\\
    &= c^{\theta}\|x\|_{\left(L^2(\calZ), \calH\right)_{\theta,2}} .
\end{align*}
Notice that we correct an error in \cite{kanagawa2018gaussianprocesseskernelmethods} which obtained the exponent $-\theta$ instead of $\theta$ due to an error in the change of variables argument in $(a)$.
Now we have proved that $(L^2(\calZ), \calH)_{\theta,2} \hookrightarrow (L^2(\calZ), W)_{\theta,2}$ and $\left\| (L^2(\calZ), \calH)_{\theta,2} \hookrightarrow (L^2(\calZ), W)_{\theta,2} \right\| \leq c^\theta$.
On the other hand, by assumption, we have $(L^2(\calZ), W)_{\theta,2} \hookrightarrow L^\infty(\calZ)$. Consequently, we have $(L^2(\calZ), \calH)_{\theta,2} \hookrightarrow W(L^2(\calZ), W)_{\theta,2} \hookrightarrow L^{\infty}(\calZ)$. The embedding norm is bounded by $c^{\theta}c_0$, where $c_0$ depends only on $\theta$, $m$, and $d_z$.
\end{proof}

\begin{lem}
\label{lem:anisotropic_sobolev_embedding}
    If $\frac{d_o}{2s_o}< 1 $ and $\frac{d_z}{2s_z} < 1$, then the dominating mixed smoothness Sobolev space $MW^{s_o,s_z}_2([0,1]^{d_o+d_z})$ continuously embeds into $L^{\infty}([0,1]^{d_o+d_z})$.
\end{lem}
\begin{proof}
    See \cite[Equation 1.13]{schmeisser2007recent} for the result where the domain is $\R^{d_z+d_o}$. By \cite{devore1993besov}, there exists a continuous extension operator $\calE: MW^{s_o,s_z}_2([0,1]^{d_o+d_z})\to MW^{s_o,s_z}_2(\R^{d_o + d_z})$ such that $\|\calE[f]\|_{MW^{s_o,s_z}_2(\R^{d_o+d_z})} \leq C' \|f\|_{MW^{s_o,s_z}_2([0,1]^{d_o+d_z})}$ holds for some universal constant $C^\prime$.
    Hence, for $f\in MW^{s_o,s_z}_2([0,1]^{d_o+d_z})$, we thus have $\|f\|_{L^{\infty}([0,1]^{d_o+d_z})}\leq \|\calE[f]\|_{L^\infty}(\R^{d_o+d_z}) \leq C \|\calE[f]\|_{MW^{s_o,s_z}_2(\R^{d_o+d_z})} \leq C' \|f\|_{MW^{s_o,s_z}_2([0,1]^{d_o+d_z})}$, for some universal constants $C,C'>0$. 
\end{proof}

\begin{lem}\label{lem:tensor_interpolation}
    Let $\calZ=[0,1]^{d_z}$ and $\calO = [0,1]^{d_o}$.  The interpolation space $[L^2(\calZ\times\calO), W_2^{s_z}(\calZ) \otimes W_2^{s_o}(\calO)]_{\theta,2} \cong W_2^{s_z \theta}(\calZ) \otimes W_2^{s_o \theta}(\calO)$.
\end{lem}
\begin{proof}
    The proof follows immediately from \cite{defant2000complex}.
\end{proof}

\begin{lem}\label{lem:equivalence_G}
    Let $\calG$ be the vector-valued RKHS of the operator-valued kernel defined in Eq. \eqref{eq:op_val_kernel} with $k_{Z}, k_{O, 1}$ being Sobolev reproducing kernels of smoothness $t_z, t_o$. Then, 
\begin{align}\label{eq:equivalence_RKHS_sobolev_space}
\begin{aligned}
    \calG &\overset{(a)}{\cong} \calH_{X} \otimes \calH_{Z} \otimes \calH_{O,1} \overset{(b)}{\simeq} \calH_{X} \otimes W^{t_z}_2(\calZ;\mathbb{R}) \otimes W^{t_o}_2(\calO;\mathbb{R}) \\
    &\overset{(c)}{\cong} \calH_{X} \otimes MW^{t_z, t_o}_2(\calZ\times\calO;\mathbb{R}) \overset{(d)}{\cong} MW^{t_z,t_o}_2(\calZ\times\calO;\calH_{X}),
\end{aligned}
\end{align}
\end{lem}
\begin{proof}
    $(a)$ is proved in \citet[Theorem 1]{JMLR:v25:23-1663}, $(b)$ holds by definition of Sobolev reproducing kernels, $(c)$ is proved in \citet[Theorem 12.7.2]{aubin2011applied} and $(d)$ is an extension of \citet[Theorem 12.7.1, Theorem 12.4.1]{aubin2011applied}. 
\end{proof}
\newpage
\begin{center}
    {\LARGE \text{Roadmap for Proof of \Cref{thm:upper_rate}}}
\end{center}
The following figures illustrate how \Cref{thm:upper_rate} follows from supporting propositions and lemmas (with arrows indicate that one result is used in the proof of another).
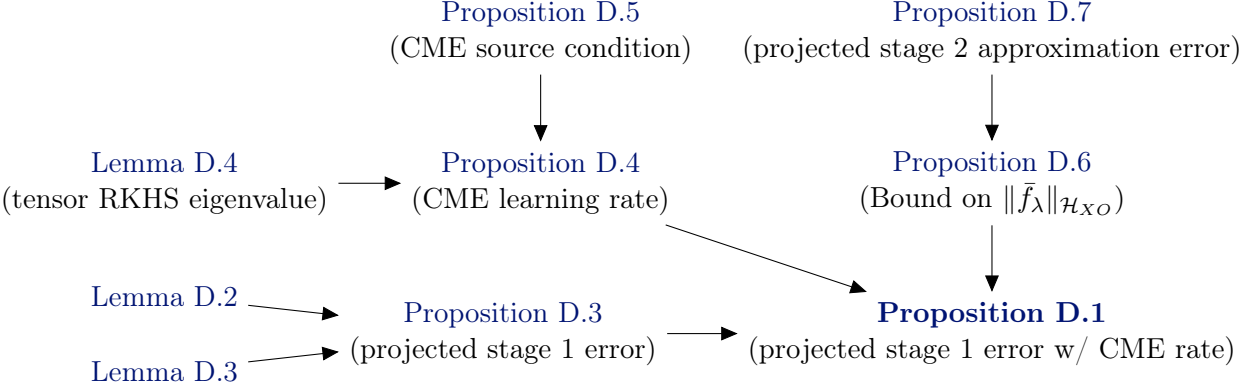
\begin{figure}[!ht]
\centering
\begin{tikzpicture}[node distance=2cm, every node/.style={align=center}]
\node (S1) at (8,0) {\textbf{\Cref{prop:projected_stage_one}}\\(projected stage 1 error w/ CME rate)};
\node (S2) at (8,2) {\Cref{lem:f_bar_lambda_norm_bound}\\(Bound on $\|\bar{f}_{\lambda}\|_{\calH_{XO}}$)};
\node (S3) at (8,4) {\Cref{prop: approx_error}\\(projected stage 2 approximation error)};
\node (S4) at (1.5,0) {\Cref{prop:T_hat_f_bar_f}\\(projected stage 1 error)};
\node (S5) at (-3,-0.5) {\Cref{lem:thm11_meunier}};
\node (S6) at (-3,0.5) {\Cref{lem:thm12_meunier}};
\node (S7) at (2,2) {\Cref{prop:cme_rate}\\(CME learning rate)};
\node (S8) at (2,4) {\Cref{prop:F_ast_sobolev}\\(CME source condition)};
\node (S9) at (-3,2) {\Cref{lem:tensor_evd}\\(tensor RKHS eigenvalue)};

\draw[->] (S2) -- (S1);
\draw[->] (S3) -- (S2);
\draw[->] (S4) -- (S1);
\draw[->] (S5) -- (S4);
\draw[->] (S6) -- (S4);
\draw[->] (S7) -- (S1);
\draw[->] (S8) -- (S7);
\draw[->] (S9) -- (S7);

\end{tikzpicture}
\caption{Roadmap for proof of \Cref{prop:projected_stage_one}.}
\end{figure}
\vspace{-5mm}
\begin{figure}[!ht]
\centering
\begin{tikzpicture}[node distance=2cm, every node/.style={align=center}]
\node (S1) at (8,0) {\textbf{\Cref{prop: projected upper rate}}\\(projected stage 2 error)};
\node (S2) at (8,2) {\Cref{prop: approx_error}\\(projected stage 2 approximation error)};
\node (S3) at (8,4) {\Cref{lem:hang_prop_3}\\(Gaussian convolution)};
\node (S4) at (2,0) {\Cref{prop: projected_est_error}\\(projected stage 2 estimation error)};
\node (S5) at (-4,0) {\Cref{lem: C_f_op_norm_lower_bound}\\(lower bound on $\|C_{FO}\|$)};
\node (S6) at (2,3) {\Cref{prop:eff_dim}\\(effective dimension of $\calH_{FO}$)};
\node (S7) at (-2,2) {\Cref{prop:emb_main}\\(embedding property of $\calH_{FO}$)};

\draw[->] (S2) -- (S1);
\draw[->] (S3) -- (S2);
\draw[->] (S4) -- (S1);
\draw[->] (S5) -- (S4);
\draw[->] (S6) -- (S4);
\draw[->] (S7) -- (S4);
\draw[->] (S2) -- (S4);

\end{tikzpicture}
\caption{Roadmap for proof of \Cref{prop: projected upper rate}.}
\end{figure}
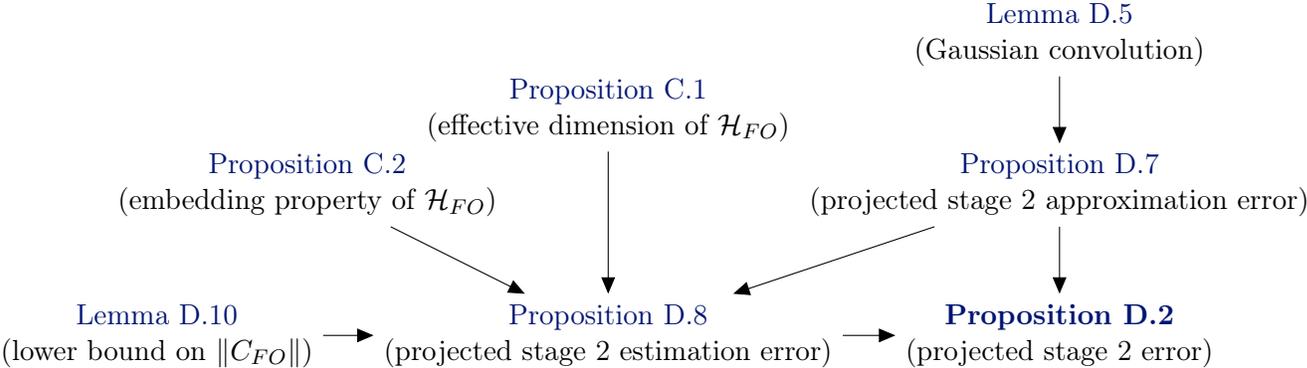
\vspace{-5mm}
\begin{figure}[!ht]
\centering
\begin{tikzpicture}[node distance=2cm, every node/.style={align=center}]
\node (S1) at (2,0) {\textbf{\Cref{thm:upper_rate}}\\($L^2$-upper rate)};
\node (S2) at (2,2) {Eq. \eqref{eq:f_aux_norm}\\(Bound on $\|f_{\mathrm{aux}}\|_{\calH_{XO}}$)};
\node (S3) at (0,4) {\Cref{lem:hang_prop_4}\\\Cref{cor:hang_prop_4_obs_cov}\\(Gaussian convolution $\calH_{XO}$-norm)};
\node (S4) at (4,4) {\Cref{lem:mixed_convolution}};
\node (S5) at (-3,0) {\Cref{lem:l2toprojected_illp}\\(projected rate to $L^2$ rate)};
\node (S6) at (-5,2) {\Cref{lem:unit_ball_grkhs}\\($\|\cdot\|_{\calH_{XO}}$ w/ Fourier transform)};
\node (S7) at (7,1) {\Cref{prop:projected_stage_one}\\(projected stage one error)};
\node (S8) at (7,-1) {\Cref{prop: projected upper rate}\\(projected stage two error)};

\draw[->] (S2) -- (S1);
\draw[->] (S3) -- (S2);
\draw[->] (S4) -- (S2);
\draw[->] (S5) -- (S1);
\draw[->] (S6) -- (S5);
\draw[->] (S7) -- (S1);
\draw[->] (S8) -- (S1);

\end{tikzpicture}
\caption{Roadmap for proof of \Cref{thm:upper_rate}.}
\end{figure}
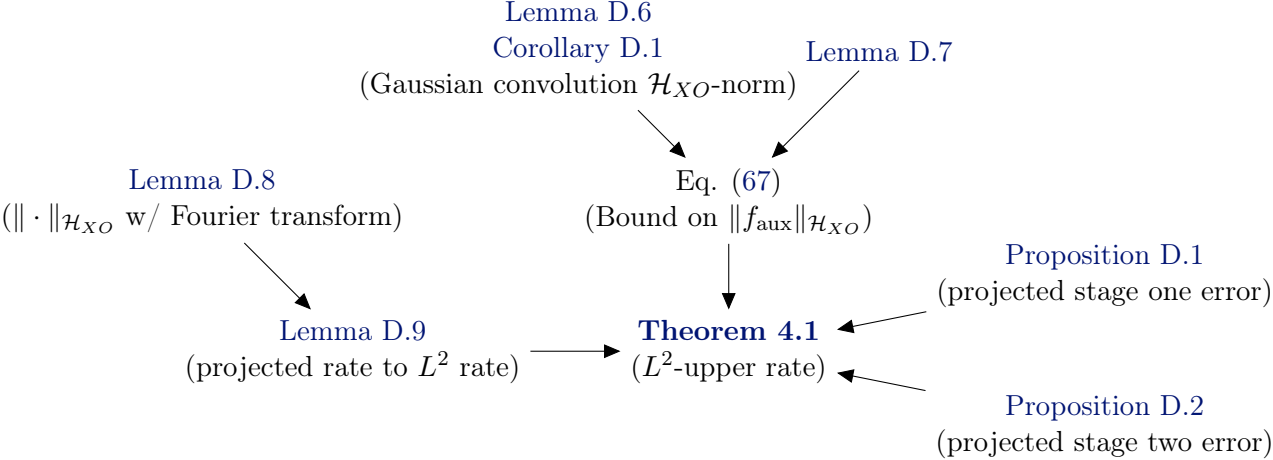

\newpage
\section{Proof of Theorem~\ref{thm:upper_rate} in the main text}\label{sec:proof_upper}

We first construct the auxiliary RKHS function $f_{\mathrm{aux}}$ defined in Eq. \eqref{eq:f_aux} in the main text, and give an upper bound on $\|f_{\mathrm{aux}}\|_{\calH_{\gamma_x,\gamma_o}}$. Let $r = \max\{ \lfloor s_x\rfloor, \lfloor s_o \rfloor\} + 1 $. For $\bgamma = (\gamma_1,\dots, \gamma_d)$, we define an \emph{approximate identity} \citep[Section 4.1.2]{Giné_Nickl_2015} $K_{\bgamma}: \R^{d}\to \R$ via
\begin{align}
\label{eq:K_function}
    K_{\mathbf{1}}(\bx) := \sum_{j=1}^r\binom{r}{j}(-1)^{1-j} \frac{1}{j^{d}}\left(\frac{2}{\pi}\right)^{\frac{d}{2}}  \exp\left( -2 \sum_{i=1}^{d} \frac{x_i^2}{j^2} \right) , \quad
    K_{\bgamma}(\bx) :=  K_{\mathbf{1}}\left(\frac{\bx}{\bgamma}\right)\prod_{i=1}^{d}\frac{1}{\gamma_i}
\end{align}
Note that Eq. \eqref{eq:K_function} reduces to \cite{eberts2013optimal}[Eq. (8)] when $\gamma_1 = \dots = \gamma_{d}$, $s_1 = \dots, s_{d}$. 
We define $K_{\gamma_x} : \R^{d_x}\to \R$ to be $K_{\bgamma_x}$ with $\bgamma_x = \left[\gamma_x, \ldots,\gamma_x\right]\in\R^{d_x}$ (note this agrees with Eq. \eqref{eq:K_function_main} in the main) and 
$K_{\gamma_o} : \R^{d_o}\to \R$ to be $K_{\bgamma_o}$ with $\bgamma_o = \left[\gamma_o, \ldots,\gamma_o\right]\in\R^{d_o}$. 
Define $\iota_{x, \gamma_x^{-1}}:\R^{d_x} \to \R$ as an indicator function of $\{\bomega_x: \|\bomega_x\| \leq \gamma_x^{-1}\}$. 

By the Young's Convolution Inequality and the fact that $\calF$ is unitary, we have
\begin{align*}
    \left\|f_{\ast}(\cdot, \bo)\ast \calF^{-1}[\iota_{x,\gamma_x^{-1}}]\right\|_{L^2(\mathbb{R}^{d_x})} \leq \|f_{\ast}(\cdot, \bo)\|_{L^1(\mathbb{R}^{d_x})} \|\iota_{x,\gamma_x^{-1}}\|_{L^2(\mathbb{R}^{d_x})}
\end{align*}
The right hand side is finite by \Cref{ass:f_ast} in the main text. Hence we apply the Fourier operator $\calF$ to $f_{\ast}(\cdot, \bo)\ast \calF^{-1}[\iota_{x,\gamma_x^{-1}}]$ to find that, $(\forall \bo\in \calO)$,
\begin{align*}
    \calF[f_{\ast}(\cdot, \bo)\ast \calF^{-1}[\iota_{x,\gamma_x^{-1}}]] = \calF[f_{\ast}(\cdot,\bo)]\cdot \iota_{x,\gamma_x^{-1}}.
\end{align*}
Therefore, $f_{\ast}\ast \calF^{-1}[\iota_{x,\gamma_x^{-1}}] = f_{\ast, \mathrm{low}}$ where $f_{\ast, \mathrm{low}}$ is defined in Eq. \eqref{eq:f_ast_low}. We define $f_{\ast, \mathrm{high}} = f_{\ast} - f_{\ast, \mathrm{low}}$. Then we have, $(\forall \bo\in \calO)$, 
\begin{align*}
    \calF[f_{\ast, \mathrm{high}}(\cdot, \bo)] = \calF[f_{\ast}(\cdot, \bo)] - \calF[f_{\ast, \mathrm{low}}(\cdot, \bo)]
\end{align*}
satisfies Eq. \eqref{eq:f_ast_high} in the main text. We let $f_{\mathrm{aux}}$ be as defined in Eq. \eqref{eq:f_aux} in the main text. We have
\begin{align}\label{eq:f_aux_norm}
    &\quad \|f_{\mathrm{aux}}\|_{\calH_{\gamma_x,\gamma_o}} 
    \\ &\leq \left\| f_{\ast} \ast K_{\gamma_o} \ast \calF^{-1}[\iota_{x, \gamma_x^{-1}}] \right\|_{\calH_{\gamma_x,\gamma_o}} + \left\| K_{\gamma_x} \ast K_{\gamma_o} \ast (f_{\ast} - f_{\ast, \mathrm{low}}) \right\|_{\calH_{\gamma_x,\gamma_o}} \nonumber \\
    &\stackrel{(a)}{\leq} n^{\frac{1}{2} \frac{1+\frac{d_o}{s_o}(\frac{s_x}{d_x}+\eta_1)}{1+(2+\frac{d_o}{s_o})(\frac{s_x}{d_x}+\eta_1)}} \|f_\ast\|_{L^2(\R^{d_x+d_o})} + n^{\frac{1}{2} \frac{1+\frac{d_o}{s_o}(\frac{s_x}{d_x}+\eta_1)}{1+(2+\frac{d_o}{s_o})(\frac{s_x}{d_x}+\eta_1)}}
    \|f_{\ast} - f_{\ast,\mathrm{low}} \|_{L^2(\R^{d_x+d_o})}  \nonumber \\
    &\stackrel{(b)}{\leq}  3 n^{\frac{1}{2} \frac{1+\frac{d_o}{s_o}(\frac{s_x}{d_x}+\eta_1)}{1+(2+\frac{d_o}{s_o})(\frac{s_x}{d_x}+\eta_1)}}.
\end{align}
In the above chain of derivations, $(a)$ holds by applying \Cref{lem:mixed_convolution} to the first term and applying \Cref{cor:hang_prop_4_obs_cov} to the second term, $(b)$ holds by $\|f_{\ast, \mathrm{low}}\|_{L^2(\mathbb{R}^{d_x+d_o})}\leq \|f_{\ast}\|_{L^2(\mathbb{R}^{d_x+d_o})} \leq 1$ by Plancherel's Theorem and \Cref{ass:f_ast}.
Notice that,
\begin{align}
\label{eq:appendix_ast_triangular}
    (\ast) := \left\|\left[\hat{f}_{\lambda}\right] - f_{\ast}\right\|_{L^2(P_{XO})} \leq \left\|\left[\hat{f}_{\lambda}\right] - f_{\mathrm{aux}} \right\|_{L^2(P_{XO})} + \left\|f_\ast - f_{\mathrm{aux}} \right\|_{L^2(P_{XO})} .
\end{align}
By definition of $\hat{f}_{\lambda}$ in Eq.~\eqref{eq:hat_f_lambda} in the main text, it satisfies 
\begin{align}\label{eq:trivial_bound}
    \lambda \|\hat{f}_{\lambda}\|_{\calH_{\gamma_x,\gamma_o}}^2 + \frac{1}{n}\sum_{i=1}^n \left(y_i - \left\langle f, \hat{F}_{\xi}(\bz_i,\bo_i) \otimes \phi_{\gamma_o}(o_i) \right\rangle_{\calH_{\gamma_x,\gamma_o}} \right)^2 \leq \frac{1}{n}\sum_{i=1}^n y_i^2.
\end{align}
Notice that
\begin{align*}
    \frac{1}{n}\sum_{i=1}^{n}y_i^2 = \frac{1}{n}\sum_{i=1}^{n}((Tf_{\ast})(\bz_i, \bo_i) + \upsilon_i)^2  \lesssim 2 + \frac{2}{n}\sum_{i=1}^{n}\upsilon_i^2,
\end{align*}
where the inequality follows from 
\begin{align*}
    \|Tf_{\ast}\|_{L^{\infty}(\calZ\times \calO)}\leq \|f_{\ast}\|_{L^{\infty}(\calX\times \calO)} \leq \|f_{\ast}\|_{L^{\infty}(\mathbb{R}^{d_x+d_o})} \lesssim 1,
\end{align*}
by \Cref{ass:f_ast}. By \Cref{ass:subgaussian}, $\upsilon_1, \ldots, \upsilon_n$ are $n$ i.i.d. mean zero $\sigma$-sub-Gaussian random variables so $\upsilon_1^2, \ldots, \upsilon_n^2$ are $n$ i.i.d. $\sigma^2$-sub-exponential random variables~\citep[Lemma 2.7.6]{vershynin2018high} with mean $\E[\upsilon_1^2] < \infty$. By Exercise 2.7.10 (Centering) of \cite{vershynin2018high} we know that $\upsilon_1^2 - \E[\upsilon_1^2], \ldots, \upsilon_n^2 - \E[\upsilon_n^2]$ are $n$ i.i.d. $C\sigma^2$-sub-exponential random variables for some universal constant $C$. By \citet[Theorem 2.8.1 (Bernstein's inequality)]{vershynin2018high}, we have
\begin{align*}
    (\forall t\geq 0), \;P\left(\left|\sum_{i=1}^{n}\upsilon_i^2\right| \geq n\left(1+\mathbb{E}\left[\upsilon_i^2\right]\right)\right) \leq 2\exp\left(-c\min\left(\frac{n}{C^2\sigma^4}, \frac{n}{C\sigma^2}\right)\right),
\end{align*}
where $c$ is a universal constant. For a fixed $\tau\geq 1$, for sufficiently large $n\geq 1$, the above right hand side $\leq 2\exp(-\tau)$. Thus for a fixed $\tau\geq 1$, with $P^n$-probability $\geq 1 - 2e^{-\tau}$, for sufficiently large $n\geq 1$, we have
\begin{align*}
    \frac{1}{n}\sum_{i=1}^{n}y_i^2  \lesssim 2 + \frac{2}{n}\sum_{i=1}^{n}\upsilon_i^2 \leq 4 + 2\mathbb{E}\left[\upsilon_i^2\right].
\end{align*}
Under the same high probability event, from Eq.~\eqref{eq:trivial_bound}, using $\lambda = \frac{1}{n}$, we have
\begin{align}
\label{eq:mean_y_squared_concentration}
    \|\hat{f}_{\lambda}\|_{\calH_{\gamma_x,\gamma_o}} \lesssim \sqrt{4 + 2\mathbb{E}\left[\upsilon_i^2\right]}\sqrt{n}. 
\end{align}
Also, we have 
\begin{align}
\label{eq:appendix_f_aux_upper_sqrtn}
    \|f_{\mathrm{aux}}\|_{\calH_{\gamma_x,\gamma_o}}\leq 2\sqrt{n}
\end{align}
proved above in Eq.~\eqref{eq:f_aux_norm}. Continuing from Eq. \eqref{eq:appendix_ast_triangular}, we apply \Cref{lem:l2toprojected_illp} to $\|[\hat{f}_{\lambda}] - f_{\mathrm{aux}}\|_{L^2(P_{XO})}$, where we notice that $\hat{f}_{\lambda}, f_{\mathrm{aux}}\in \calH_{\gamma_x,\gamma_o}$, and we use Eq. \eqref{eq:mean_y_squared_concentration} and Eq. \eqref{eq:appendix_f_aux_upper_sqrtn} to verify the assumption of that lemma. We find, with $P^n$-probability $\geq 1 - 2e^{-\tau}$,
\begin{small}
\begin{align*}
    (\ast) \leq \gamma_x^{-d_x\eta_0} (\log n)^{\frac{d_x\eta_0}{2}} \left\| T \left[\hat{f}_{\lambda}\right] -  T f_{\mathrm{aux}} \right\|_{L^2(P_{ZO})} + n^{-\frac{\frac{s_x}{d_x} }{1+(2 + \frac{d_o}{s_o})(\frac{s_x}{d_x}+\eta_1)}} + \left\|f_{\ast} -  f_{\mathrm{aux}} \right\|_{L^2(P_{XO})} .
\end{align*}
\end{small}
From Eq.~\eqref{eq:f_f_aux_diff}, $\|f_{\ast} -  f_{\mathrm{aux}}\|_{L^2(P_{XO})}$ can be upper bounded (up to a constant) by the second last term above and hence subsumed. 
Through a triangular inequality, we have with $P^n$-probability $\geq 1 - 2e^{-\tau}$,
\begin{align}
    (\ast) &\leq \gamma_x^{-d_x\eta_0} (\log n)^{\frac{d_x\eta_0}{2}} \left(\left\| T \left[\hat{f}_{\lambda}\right] -  T f_{\ast} \right\|_{L^2(P_{ZO})}  +  \left\| T f_\ast -  T f_{\mathrm{aux}} \right\|_{L^2(P_{ZO})}\right) \nonumber \\
    &\qquad \qquad + n^{-\frac{\frac{s_x}{d_x} }{1+(2 + \frac{d_o}{s_o})(\frac{s_x}{d_x}+\eta_1)}} \label{eq:unprojected_amplified_projected_example} .
\end{align}
We prove in Eq.~\eqref{eq:Tf_Tf_aux} in \Cref{sec:approx_main} that $
\left\| T f_\ast -  T f_{\mathrm{aux}} \right\|_{L^2(P_{ZO})} \leq n^{-\frac{\frac{s_x}{d_x}+\eta_1}{1+(2+\frac{d_o}{s_o})(\frac{s_x}{d_x}+\eta_1)}}$, so we continue from above to obtain, with $P^n$-probability $\geq 1 - 2e^{-\tau}$, 
\begin{align}
\label{eq:eq_to_refer_main_pf_sketch}
    (\ast) \lesssim \gamma_x^{-d_x\eta_0} (\log n)^{\frac{d_x\eta_0}{2}} \left\| T \left[\hat{f}_{\lambda}\right] -  T f_{\ast} \right\|_{L^2(P_{ZO})} + n^{-\frac{\frac{s_x}{d_x}+\eta_1 - \eta_0}{1+(2+\frac{d_o}{s_o})(\frac{s_x}{d_x}+\eta_1)}} (\log n)^{\frac{d_x\eta_0}{2}} .
\end{align}
Note the last term of Eq.~\eqref{eq:unprojected_amplified_projected_example} is subsumed by the second term above since $\eta_0 \geq \eta_1$.
Recall the definition of $\bar{f}_{\lambda}$ in Eq. \eqref{eq:bar_f_lambda_main} in the main text. Then, we have, through a triangular inequality, with $P^n$-probability $\geq 1 - 2e^{-\tau}$, the following holds
\begin{align*}
    (\ast) &\leq \gamma_x^{-d_x\eta_0} (\log n)^{\frac{d_x\eta_0}{2}} \left(\underbrace{\left\| T \left[\hat{f}_{\lambda}\right] -  T \left[\bar{f}_{\lambda}\right] \right\|_{L^2(P_{ZO})}}_{\text{Projected Stage I Error}} +  \underbrace{\left\| T \left[\bar{f}_{\lambda}\right] -  T f_{\ast} \right\|_{L^2(P_{ZO})}}_{\text{Projected Stage II Error}}\right) \\
    &\qquad + n^{-\frac{\frac{s_x}{d_x}+\eta_1 - \eta_0}{1+(2+\frac{d_o}{s_o})(\frac{s_x}{d_x}+\eta_1)}} (\log n)^{\frac{d_x\eta_0}{2}} .
\end{align*}
The projected stage I error can be upper bounded by \Cref{prop:projected_stage_one} as $ n > A_{\lambda,\tau}$ is satisfied for sufficiently large $n \geq 1$ proved in Eq.~\eqref{eq:n_large_than_A}. 
\begin{align*}
    \text{Projected Stage I Error} = \left\| T \left[\hat{f}_{\lambda}\right] -  T \left[\bar{f}_{\lambda}\right] \right\|_{L^2(P_{ZO})} \leq \tau n^{-\frac{\frac{s_x}{d_x}+\eta_1}{1+(2+\frac{d_o}{s_o})(\frac{s_x}{d_x}+\eta_1)}}
\end{align*}
holds with $P^{n+\stageonesamples}$-probability $\geq 1-  34e^{-\tau}$.
The projected stage II error can be upper bounded by \Cref{prop: projected upper rate}.
\begin{align*}
    \text{Projected Stage II Error} = \left\| T \left[\bar{f}_{\lambda}\right] -  T f_{\ast} \right\|_{L^2(P_{ZO})} \leq \tau n^{-\frac{\frac{s_x}{d_x}+\eta_1}{1+(2+\frac{d_o}{s_o})(\frac{s_x}{d_x}+\eta_1)}} \cdot (\log n)^{\frac{d_x+d_o+1}{2}}
\end{align*}
holds with $P^n$-probability $\geq 1 - 4e^{-\tau}$. Combine the above two upper bounds, and we have
\begin{align}
\label{eq:eq_to_refer_main_pf_sketch_final}
     (\ast) = \left\|\left[\hat{f}_{\lambda}\right] - f_{\ast}\right\|_{L^2(P_{XO})} \leq \tau n^{-\frac{\frac{s_x}{d_x}+\eta_1-\eta_0}{1+(2+\frac{d_o}{s_o})(\frac{s_x}{d_x}+\eta_1)}} (\log n)^{\frac{d_x+d_o+1 + d_x\eta_0}{2}}
\end{align}
holds with $P^{n+\stageonesamples}$-probability $1 - 40e^{-\tau}$.  So the proof is concluded.
\qed

\begin{prop}[Projected stage-I error]
\label{prop:projected_stage_one}
Suppose that the assumptions of \Cref{prop:cme_rate} hold. Suppose that $n > A_{\lambda,\tau}$ with $A_{\lambda,\tau}$ defined in Eq.~\eqref{eq:FS_thm_16_const}. Suppose Assumptions~\ref{assn: technical} hold.  Suppose that $\stageonesamples \geq 1$ satisfies Eq. \eqref{eq:m_sufficient_large}. Let $\lambda \asymp n^{-1}$ and 
\begin{align*}
    \gamma_x = n^{-\frac{\frac{1}{d_x}}{1+(2+\frac{d_o}{s_o})(\frac{s_x}{d_x}+\eta_1)}}, \quad \gamma_o = n^{-\frac{\frac{1}{s_o}(\frac{s_x}{d_x}+\eta_1)}{1+(2+\frac{d_o}{s_o})(\frac{s_x}{d_x}+\eta_1)}}. 
\end{align*} Then, with $P^{n+\stageonesamples}$-probability $\geq 1-  34e^{-\tau}$, we have
\begin{align*}
    \left\|T\left(\left[\bar{f}_{\lambda}\right] - \left[\hat{f}_{\lambda}\right]\right)\right\|_{L^2(P_{ZO})} \lesssim \tau n^{-\frac{\frac{s_x}{d_x}+\eta_1}{1+(2+\frac{d_o}{s_o})(\frac{s_x}{d_x}+\eta_1)}} .
\end{align*}
\end{prop}
\begin{proof}
From \Cref{prop:T_hat_f_bar_f}, with $P^{n+\stageonesamples}$-probability $\geq 1 - 28e^{-\tau}$, we have
\begin{align*}
    &\left\|T\left(\left[\bar{f}_{\lambda}\right] - \left[\hat{f}_{\lambda}\right]\right)\right\|_{L^2(P_{ZO})}\\ &\lesssim \tau \lambda^{-\frac{1}{2}} \underbrace{\left(\frac{\left\|\hat{F}_{\xi}-F_*\right\|_{\calG}}{\sqrt{n}}+\left\|F_*-[\hat{F}_{\xi}] \right\|_{L^2\left(\calZ\times\calO; \mathcal{H}_{X,\gamma_x}\right)} \right)}_{(\ddag)} \left(1 + \|\bar{f}_{\lambda}\|_{\calH_{XO, \gamma_x, \gamma_0}}\right) .
\end{align*}
By \Cref{lem:f_bar_lambda_norm_bound}, we have with $P^n$-probability $\geq 1 - 2e^{-\tau}$
\begin{align*}
    \|\bar{f}_{\lambda}\|_{\calH_{\gamma_x,\gamma_o}}\lesssim \tau n^{\frac{1}{2}\frac{1 + \frac{d_o}{s_o}(\frac{s_x}{d_x}+\eta_1)}{1 + (2 + \frac{d_o}{s_o})(\frac{s_x}{d_x} + \eta_1)}}.
\end{align*}
Under the probabilistic event that \Cref{prop:T_hat_f_bar_f} holds, the bounds in \Cref{prop:cme_rate} also hold, so we have
\begin{align*}
    \left\|\left[\hat{F}_\xi \right]-F_*\right\|_{L^2(\calZ\times\calO; \calH_{X,\gamma_x})} \leq J \tau \stageonesamples^{-\frac{1}{2}\frac{\meffective}{\meffective + \deffective/2 + \zeta}} , \quad \left\| \hat{F}_\xi -F_*\right\|_{\calG} \leq J \tau \stageonesamples^{-\frac{1}{2}\frac{\meffective - 1}{\meffective + \deffective/2 + \zeta}} ,
\end{align*}
hold for any $\zeta > 0$. $J$ is some constant independent of $\stageonesamples, n$. 
Thus under this probabilistic event, a sufficient condition so that $(\ddag)\lesssim \frac{1}{n}$ is given by
\begin{align*}
    \stageonesamples \geq (J\tau n)^{2\frac{\meffective+\deffective/2+\zeta}{\meffective}} \vee (J^2\tau^2n)^{\frac{\meffective + \deffective/2 + \zeta}{\meffective - 1}}.
\end{align*}
This is satisfied since $\tilde{n}$ satisfies Eq. \eqref{eq:m_sufficient_large}.
Hence,
\begin{align}\label{eq:m_sufficient_large_CME}
    \left\|\left[\hat{F}_\xi \right]-F_*\right\|_{L^2(\calZ\times\calO; \calH_{X,\gamma_x})} \leq n^{-1} , \quad \left\| \hat{F}_\xi - F_*\right\|_{\calG} \leq n^{-\frac{1}{2}} .
\end{align}
By the union bound, we have that with $P^{n+\stageonesamples}$-probability $\geq 1 - 30 e^{-\tau}$, the following bound holds
\begin{align*}
    \left\|T\left(\left[\bar{f}_{\lambda}\right] - \left[\hat{f}_{\lambda}\right]\right)\right\|_{L^2(P_{ZO})} \lesssim \tau n^{-\frac{\frac{s_x}{d_x}+\eta_1}{1+(2+\frac{d_o}{s_o})(\frac{s_x}{d_x}+\eta_1)}}.
\end{align*}
Hence the proof is concluded. 
\end{proof}

\begin{prop}[Projected stage-II error]
\label{prop: projected upper rate}
Suppose Assumptions~\ref{ass:T_contractivity}, \ref{ass:T_frequency_ill_posedness}, \ref{ass:T_injective}, \ref{ass:f_ast}, \ref{assn: technical} hold. Let $\lambda \asymp n^{-1}$ and 
\begin{align*}
    \gamma_x = n^{-\frac{\frac{1}{d_x}}{1+(2+\frac{d_o}{s_o})(\frac{s_x}{d_x}+\eta_1)}}, \quad \gamma_o = n^{-\frac{\frac{1}{s_o}(\frac{s_x}{d_x}+\eta_1)}{1+(2+\frac{d_o}{s_o})(\frac{s_x}{d_x}+\eta_1)}}. 
\end{align*}
Then, with $P^n$-probability $\geq 1-  4e^{-\tau}$, for sufficiently large $n\geq 1$, we have
\begin{align*}
    \left\|T\left(f_{\ast} - \left[\bar{f}_{\lambda}\right]\right)\right\|_{L^2(P_{ZO})} \leq  2C\tau (\log n)^{\frac{d_x+d_o+1}{2}} n^{-\frac{\frac{s_x}{d_x}+\eta_1}{1+(2+\frac{d_o}{s_o})(\frac{s_x}{d_x}+\eta_1)}}
\end{align*}
for some constant $C>0$ independent of $n$.
\end{prop}
\begin{proof}
    The proposition is proved through the following triangular inequality
    \begin{align*}
        \left\|T\left(f_{\ast} - \left[\bar{f}_{\lambda}\right]\right)\right\|_{L^2(P_{ZO})} &\leq \left\|T\left(f_{\ast} - \left[f_{\lambda}\right]\right)\right\|_{L^2(P_{ZO})} + \left\|T\left(\left[\bar{f}_{\lambda}\right] - [f_{\lambda}]\right)\right\|_{L^2(P_{ZO})} \\
        &\leq n^{-\frac{\frac{s_x}{d_x}+\eta_1}{1+(2+\frac{d_o}{s_o})(\frac{s_x}{d_x}+\eta_1)}} + n^{-\frac{\frac{s_x}{d_x}+\eta_1}{1+(2+\frac{d_o}{s_o})(\frac{s_x}{d_x}+\eta_1)}} (\log n)^{\frac{d_x+d_o+1}{2}} \\
        &\leq 2 n^{-\frac{\frac{s_x}{d_x}+\eta_1}{1+(2+\frac{d_o}{s_o})(\frac{s_x}{d_x}+\eta_1)}} (\log n)^{\frac{d_x+d_o+1}{2}}.
    \end{align*}
    The second last inequality holds by using \Cref{prop: approx_error} for the first term and \Cref{prop: projected_est_error} for the second term, which holds with $P^n$-probability $\geq 1-  4e^{-\tau}$.
\end{proof}

\subsection{Projected Stage I Error}\label{sec:proj_stage_one_error}
With the introduction of $\calH_{FO}$ and the partial isometry $V:\calH_{XO}\to \calH_{FO}$ in \Cref{sec:defi_H_FO}, we define
\begin{align*}
    \hat{h}_{\lambda} &:= \argmin_{h\in \calH_{FO}}  \frac{1}{n}\sum_{i=1}^{n}  \left(y_i - \left\langle h, V\Big(\hat{F}_\xi (\bz_i, \bo_i)\otimes \phi_{\gamma_o}(\bo_i) \Big) \right\rangle_{ \calH_{FO}} \right) ^2 + \lambda \|h\|_{\calH_{FO}}^2 \\
    \bar{h}_{\lambda} &= \argmin_{h \in \calH_{FO}}  \frac{1}{n}\sum_{i=1}^{n}  \left(y_i - \left\langle h, V \Big( F_\ast(\bz_i, \bo_i) \otimes \phi_{\gamma_o}(\bo_i) \Big) \right\rangle_{ \calH_{FO}} \right) ^2 + \lambda \|h\|_{\calH_{FO}}^2 .
\end{align*}
As shown in \Cref{sec:defi_H_FO}, we have $\hat{h}_{\lambda} = V\hat{f}_{\lambda}$ and $\bar{h}_{\lambda} = V\bar{f}_{\lambda}$, where $V: \calH_{XO}\to \calH_{FO}$ is the metric surjection map introduced in \Cref{sec:defi_H_FO}.  $\hat{f}_{\lambda} \in \calH_{\gamma_x,\gamma_o}$ is defined in Eq.~\eqref{eq:hat_f_lambda} and $\bar{f}_{\lambda} \in \calH_{\gamma_x,\gamma_o}$ is defined in Eq.~\eqref{eq:bar_f_lambda_main} in the main text.

We further define
\begin{align}\label{eq:bPhi_FO}
\begin{aligned}
    \bPhi_{\hat{F}O} : \calH_{FO} \to \R^{n} &= \left[V \left( \hat{F}_\xi(\bz_1, \bo_1) \otimes \phi_{\gamma_o}(\bo_1) \right),\dots, V \left( \hat{F}_\xi (\bz_n, \bo_n) \otimes \phi_{\gamma_o}(\bo_n) \right) \right]^{\ast}\\
    \bPhi_{FO} : \calH_{FO} \to \R^{n} &= \left[V \left( F_{\ast}(\bz_1, \bo_1) \otimes \phi_{\gamma_o}(\bo_1) \right) ,\dots, V \left( F_{\ast}(\bz_n, \bo_n) \otimes \phi_{\gamma_o}(\bo_n) \right)  \right]^{\ast} \\
    \bY \in \R^n &, \quad \bY = [y_1, \ldots, y_n]^\top,
\end{aligned}
\end{align}
and the following operators on $\calH_{FO}$
\begin{align*}
    C_{\hat{F} O} := \E \left[  V\left( \hat{F}_\xi(Z, O) \otimes \phi_{\gamma_o}(O) \right)  \otimes  V\left( \hat{F}_\xi(Z, O) \otimes \phi_{\gamma_o}(O) \right)  \right] , &\quad \hat{C}_{\hat{F}O} := \frac{1}{n} \bPhi_{\hat{F} O}^\ast \bPhi_{\hat{F} O} \\
    C_{F O} := \E \left[  V\left( F_\ast(Z, O) \otimes \phi_{\gamma_o}(O) \right)  \otimes  V\left( F_\ast(Z, O) \otimes \phi_{\gamma_o}(O) \right)  \right] , &\quad \hat{C}_{FO} := \frac{1}{n} \bPhi_{F O}^{\ast} \bPhi_{F O}.
\end{align*}
Hence, we have closed form expression for 
\begin{align}\label{eq:closed_form_h_lambdas}
    \hat{h}_{\lambda} = \frac{1}{n}\left( \hat{C}_{\hat{F} O} + \lambda \right)^{-1}\bPhi_{\hat{F} O}^\ast \bY, \quad \bar{h}_{\lambda} = \frac{1}{n}\left( \hat{C}_{F O} + \lambda \right)^{-1}\bPhi_{FO}^\ast \bY .
\end{align}

\begin{prop}\label{prop:T_hat_f_bar_f}
Fix $\tau\geq 1$. Suppose that the assumptions of \Cref{prop:cme_rate} hold. Suppose that $n > A_{\lambda,\tau}$ with $A_{\lambda,\tau}$ defined in Eq.~\eqref{eq:FS_thm_16_const}, and $\stageonesamples \geq 1$ satisfies Eq. \eqref{eq:m_sufficient_large} in the main text. We have with $P^{n+\stageonesamples}$-probability $\geq 1 - 28e^{-\tau}$, the following inequality holds
\begin{align*}
    &\left\|T [\hat{f}_{\lambda}] - T [\bar{f}_{\lambda}] \right\|_{L^2(P_{ZO})}\\ &\lesssim \tau \lambda^{-\frac{1}{2}} \left(\frac{\left\|\hat{F}_{\xi}-F_*\right\|_{\calG}}{\sqrt{n}}+\left\|\left[F_*-\hat{F}_{\xi}\right]\right\|_{L^2\left(\calZ\times\calO; \mathcal{H}_{X,\gamma_x}\right)}\right) \left(1 + \|\bar{f}_{\lambda}\|_{\calH_{\gamma_x,\gamma_o}}\right).
\end{align*}
Under the same probabilistic event, the bounds in \Cref{prop:cme_rate} also hold. 
\end{prop}

\begin{proof}
We find
\begin{align*}
    &\left\|T [\hat{f}_{\lambda}] - T [\bar{f}_{\lambda}] \right\|_{L^2(P_{ZO})} = \left\|[\hat{h}_{\lambda}] - [\bar{h}_{\lambda}] \right\|_{L^2(P_{ZO})} \leq \left\|C_{FO}^{\frac{1}{2}}\left(\hat{h}_{\lambda} - \bar{h}_{\lambda}\right)\right\|_{\calH_{FO}} .
\end{align*}
The last step follows from Lemma 12 of \cite{fischer2020sobolev}. We can proceed by plugging in the closed form expressions from Eq.~\eqref{eq:closed_form_h_lambdas} to have
\begin{small}
\begin{align*}
   &\left\|C_{FO}^{\frac{1}{2}}\left(\hat{h}_{\lambda} - \bar{h}_{\lambda}\right)\right\|_{\calH_{FO}}\\ = & \left\|C_{FO}^{\frac{1}{2}}\left(\left(\hat{C}_{\hat{F}O} + \lambda\right)^{-1}\frac{1}{n}\bPhi_{\hat{F}O}^{\ast}\bY - \left(\hat{C}_{FO} + \lambda\right)^{-1}\frac{1}{n}\bPhi_{FO}^{\ast}\bY\right)\right\|_{\calH_{FO}}\\
    \leq & \left\|C_{FO}^{\frac{1}{2}}(C_{FO} + \lambda)^{-\frac{1}{2}}\right\|\cdot \left\|(C_{FO} + \lambda)^{\frac{1}{2}}\left(\left(\hat{C}_{\hat{F}O} + \lambda\right)^{-1}\frac{1}{n}\bPhi_{\hat{F}O}^{\ast}\bY - \left(\hat{C}_{FO} + \lambda\right)^{-1}\frac{1}{n}\bPhi_{FO}^{\ast}\bY\right)\right\|_{\calH_{FO}}\\
    \leq & \left\|(C_{FO} + \lambda)^{\frac{1}{2}}\left(\left(\hat{C}_{\hat{F}O} + \lambda\right)^{-1}\frac{1}{n}\bPhi_{\hat{F}O}^{\ast}\bY - \left(\hat{C}_{FO} + \lambda\right)^{-1}\frac{1}{n}\bPhi_{FO}^{\ast}\bY\right)\right\|_{\calH_{FO}}\\
    \leq & S_{-1} + S_{0},
\end{align*}
\end{small}
where we define
\begin{align*}
    S_{-1} &:= \left\|(C_{FO} + \lambda)^{\frac{1}{2}}\left(\hat{C}_{\hat{F}O} + \lambda\right)^{-1} \frac{1}{n}\left(\bPhi_{\hat{F}O} - \bPhi_{FO}\right)^{\ast}\bY\right\|_{\calH_{FO}}\\
    S_0 &= \left\|(C_{FO} + \lambda)^{\frac{1}{2}}\left(\hat{C}_{\hat{F}O} + \lambda\right)^{-1}\left(\hat{C}_{FO} - \hat{C}_{\hat{F}O}\right) \left(\hat{C}_{FO} + \lambda\right)^{-1}\frac{1}{n}\bPhi_{FO}^{\ast}\bY\right\|_{\calH_{FO}}
\end{align*}
We bound $S_{-1}$ with $P^{n+\stageonesamples}$-high probability by \Cref{lem:thm11_meunier} and $S_{0}$ in $P^{n+\stageonesamples}$-high probability by \Cref{lem:thm12_meunier}. We thus have, with $P^{n+\stageonesamples}$-probability $\geq 1 - 28e^{-\tau}$, the following bound
\begin{align*}
    &\quad \left\|C_{FO}^{\frac{1}{2}}\left(\hat{h}_{\lambda} - \bar{h}_{\lambda}\right)\right\|_{\calH_{FO}} \\
    &\leq c \tau\sqrt{n}\left(\frac{\left\| \hat{F}_{\xi} -F_\ast \right\|_{\calG}}{\sqrt{n}} + \left\| \left[\hat{F}_\xi -F_*\right] \right\|_{L^2\left(\calZ\times\calO; \mathcal{H}_{X,\gamma_x}\right)}\right)\left(\|\bar{h}_\lambda\|_{\mathcal{H}_{FO}} + 1\right).
\end{align*}
$\|\bar{h}_\lambda\|_{\mathcal{H}_{FO}} = \|\bar{f}_{\lambda}\|_{\calH_{\gamma_x,\gamma_o}}$ since $V$ is a metric surjection.  
\end{proof}

\begin{prop}[CME rate]\label{prop:cme_rate}
Suppose Assumption \ref{assn: technical} in the main text holds. Suppose $k_\calO: \calO\times\calO \to \R$ and $k_\calZ: \calZ\times\calZ \to \R$ are Mat\'{e}rn reproducing kernels whose RKHSs $\calH_O$ and $\calH_Z$ are norm equivalent to $W_2^{t_o}(\calO)$ and $W_2^{t_z}(\calZ)$ with $m_o > t_o > d_o/2$, $m_z > t_z > d_z/2$ and $\frac{m_z}{t_z} \wedge \frac{m_o}{t_o} \leq 2$. Define $\meffective=(m_z t_z^{-1}) \wedge (m_o t_o^{-1})$ and $\deffective= (d_z t_z^{-1}) \vee(d_o t_o^{-1})$. 
For $\zeta > 0$ arbitrarily small, we take $\xi = \stageonesamples^{-\frac{1}{\meffective + \deffective/2 + \zeta}}$, then with $P^{\tilde{n}}$-probability at least $1- 4e^{-\tau}$, we have
\begin{align*}
    \left\|\left[\hat{F}_\xi -F_*\right]\right\|_{L^2(\calZ\times\calO; \calH_{X,\gamma_x})} \leq J\tau \stageonesamples^{-\frac{1}{2}\frac{ \meffective}{\meffective + \deffective/2 + \zeta}} , \quad \left\|\hat{F}_\xi -F_*\right\|_{\calG} \leq J \tau \stageonesamples^{-\frac{1}{2}\frac{\meffective - 1}{\meffective + \deffective/2 + \zeta}} .
\end{align*}
$J$ is a constant independent of $\tilde{n}$.
\end{prop}
\begin{proof}
We are going to apply \citet[Theorem 3]{JMLR:v25:23-1663} to obtain the desired result. To this end, we need to verify the assumptions (EVD) and (SRC) made in \cite{JMLR:v25:23-1663}. Note that we let the assumption (EMB) of \cite{JMLR:v25:23-1663} be trivially verified with $\alpha = 1$, as we prove in \Cref{prop:F_ast_sobolev} that $F_{\ast} \in [\calG]^\beta$ with $\beta > 1$, therefore $\beta + p > 1 + p > 1 = \alpha$ falls in Case 2 of Theorem 3 of \cite{JMLR:v25:23-1663}. 

\vspace{3mm}
\noindent
\underline{\textit{Verification of (EVD)}}
Let $\phi_{Z}$ (resp. $\phi_{O}$) denote the feature map of $\calH_{Z}$ (resp. $\calH_{O}$). Define $C_{ZO}:\calH_Z \otimes \calH_O \to \calH_Z \otimes \calH_O$ as the covariance operator.
\begin{align*}
    C_{ZO} = \iint_{\calZ\times\calO} (\phi_Z(\bz) \otimes \phi_O(\bo)) \otimes (\phi_Z(\bz) \otimes \phi_O(\bo)) p_{ZO}(\bz, \bo) \;\mathrm{d}\bz \;\mathrm{d}\bo .
\end{align*}
We also define another two auxiliary covariance operators $\bar{C}_Z:\calH_Z\to \calH_Z$ and $\bar{C}_O:\calH_O\to \calH_O$.
\begin{align*}
    \bar{C}_Z = \int_\calZ \phi_Z(\bz) \otimes \phi_Z(\bz) p_
    {Z}(\bz)\;\mathrm{d}\bz, \quad \bar{C}_O = \int_\calO \phi_O(\bo) \otimes \phi_O(\bo)p_O(\bo)\; \mathrm{d}\bo .
\end{align*}
Since $k_Z$ and $k_O$ are bounded, $C_{ZO}, \bar{C}_Z, \bar{C}_O$ are all self-adjoint compact operators. 
From \cite{edmunds1996function} and \cite{fischer2020sobolev}[Section 4], we have $\lambda_i(\bar{C}_Z) \asymp i^{-2 t_z/d_z}$ and $\lambda_i(\bar{C}_O) \asymp i^{-2 t_o/d_o}$. 
We know from Assumption \ref{assn: technical} that $P_{ZO}$ is equivalent to the Lebesgue measure, so by Jensen's inequality, we have $C_{ZO} \leq \bar{C}_Z \otimes \bar{C}_O$, where $\otimes$ here denotes an operator tensor product. Hence, from Lemma 17 of \cite{meunier2024nonparametricinstrumentalregressionkernel} we have $\lambda_i(C_{ZO}) \leq \lambda_i(\bar{C}_Z \otimes \bar{C}_O)$. Finally, from \Cref{lem:tensor_evd} we have 
\begin{align*}
    \lambda_i(C_{ZO}) \leq \lambda_i(\bar{C}_Z \otimes \bar{C}_O) \lesssim i^{-2 t_z/d_z \wedge 2 t_o/d_o + \zeta}
\end{align*}
for any $\zeta > 0$.
Therefore, we have proved that (EVD) hold with $1/p = 2 t_z/d_z \wedge 2 t_o/d_o - \zeta$.
Hence, $p = \deffective/2 + \zeta$ for any $\zeta > 0$.

\vspace{3mm}
\noindent
\underline{\textit{Verification of (SRC)}}
Let $\beta = \meffective$. By the definition of vector-valued interpolation space in \cite{JMLR:v25:23-1663}[Definition 2] and the Assumption that $P_{ZO}$ is equivalent to the Lebesgue measure, we have that
\begin{align*}
    [\calG]^{\beta} &\cong \calH_{X,\gamma_x} \otimes [L^2(\calZ\times\calO), \calH_Z \otimes \calH_O]_{\beta, 2} \\
    &\stackrel{(a)}{\cong} \calH_{X,\gamma_x} \otimes [L^2(\calZ\times\calO), W^{t_z}_2(\calZ) \otimes W^{t_o}_2(\calO)]_{\beta, 2} \\
    &\stackrel{(b)}{\cong} \calH_{X,\gamma_x} \otimes W^{t_z \beta}_2(\calZ) \otimes W^{t_o \beta}_2(\calO)\\
    &\stackrel{(\ast)}{\supseteq} \calH_{X,\gamma_x}\otimes W^{m_z}_2(\calZ) \otimes W^{m_o}_2(\calO)\\
    &\stackrel{(c)}{\cong} \calH_{X,\gamma_x}\otimes MW^{m_z,m_o}_{2}(\calZ\times \calO;\mathbb{R})\\
    &\stackrel{(d)}{\cong} MW^{m_z,m_o}_{2}(\calZ\times \calO;\calH_{X,\gamma_x}).
\end{align*}
where in $(a)$ we use Corollary 10.13 and Theorem 10.46 of \cite{wendland2004scattered}, in $(b)$ we use \Cref{lem:tensor_interpolation}, in $(c)$ we use proposition 3.1 of \cite{SICKEL2009748}, which shows $MW^{m_z, m_o}_2(\calZ\times\calO) \cong W^{m_z}_2(\calZ) \otimes W^{m_o}_2(\calO)$, and in $(d)$ we use \Cref{lem:vv_mss}. Furthermore, the inclusion in $(\ast)$ is a continuous embedding. We show in \Cref{prop:F_ast_sobolev} that under \Cref{assn: technical}, $\|F_{\ast}\|_{MW^{m_z,m_o}_{2}(\calZ\times \calO;\calH_{X,\gamma_x})}$ is bounded by a constant. Hence we deduce from the above embedding that 
\begin{align}\label{eq:F_ast_power_space}
    F_{\ast}\in [\calG]^{\beta}, \quad \beta = \meffective = \frac{m_z}{t_z} \wedge \frac{m_o}{t_o}.
\end{align}
Hence $\|F_{\ast}\|_{[\calG]^{\beta}} \leq C$ for some constant $C>0$. Hence $F_{\ast}$ satisfies (SRC) in \cite{JMLR:v25:23-1663} with $\beta=\meffective$.

Now we use Case 2 in Theorem 3 of \cite{JMLR:v25:23-1663} to obtain the following result: if we take $\xi = \stageonesamples^{-\frac{1}{\meffective + \deffective/2 + \zeta}}$, then for any $0\leq \gamma < \meffective$, the following bound on the $\gamma$-norm
\begin{align*}
    \left\|\left[\hat{F}_\xi -F_*\right]\right\|_{\gamma}^2 \leq \tau^2 \stageonesamples^{-\frac{\meffective - \gamma}{\meffective + \deffective/2 + \zeta}} .
\end{align*}
holds with $P^{\tilde{n}}$-probability at least $1 - 4e^{-\tau}$.
Therefore, noting that $\meffective> 1$, by taking $\gamma=0$ and $\gamma=1$, we obtain with  $P^{\tilde{n}}$-probability at least $1 - 4e^{-\tau}$, 
\begin{align*}
    \left\|\left[\hat{F}_\xi -F_*\right]\right\|_{L^2(\calZ\times\calO; \calH_{X,\gamma_x})} \leq \tau \stageonesamples^{-\frac{1}{2}\frac{\meffective}{\meffective + \deffective/2 + \zeta}} , \quad \left\|\hat{F}_\xi -F_*\right\|_{\calG} \leq \tau \stageonesamples^{-\frac{1}{2}\frac{\meffective - 1}{\meffective + \deffective/2 + \zeta}} .
\end{align*}
\end{proof}

\begin{prop}\label{prop:F_ast_sobolev}
Suppose that Assumption \ref{assn: technical} in the main text is satisfied. Then,
\begin{align*}
    \|F_{\ast}\|_{MW^{m_z, m_o}_2(\calZ\times\calO;\calH_{X,\gamma_x})} \leq \sqrt{\binom{m_z + d_z -1}{d_z -1}} \sqrt{ \binom{m_o + d_o-1}{d_o -1}} \rho.
\end{align*}
\end{prop}
\begin{proof}
Notice that, by definition of $F_{\ast}$, 
\begin{align*}
    \|F_{\ast}\|_{MW^{m_z, m_o}_2(\calZ\times\calO;\calH_{X,\gamma_x})} &= \left\|\int_{\calX} \phi_{\gamma_x}(\bx)p(\bx|\cdot, \cdot)\;\mathrm{d} \bx \right\|_{MW^{m_z, m_o}_2(\calZ\times\calO;\calH_{X,\gamma_x})} \\
    &\leq \int_{\calX} \left\|\phi_{\gamma_x}(\bx)p(\bx|\cdot,\cdot )\right\|_{MW^{m_z, m_o}_2(\calZ\times\calO;\calH_{X,\gamma_x})}\;\mathrm{d} \bx .
\end{align*}
In the last inequality above, we use \cite{Hytönen2016AnalysisBanachSpaces}[Proposition 1.2.11].
Next, notice that
\begin{align*}
    &\quad \left\|\phi_{\gamma_x}(\bx) p(\bx \mid \cdot,\cdot)\right\|_{MW^{m_z, m_o}_2(\calZ\times\calO;\calH_{X,\gamma_x})}^2 \\
    &= \sum_{|\balpha|\leq m_z} \sum_{|\bbeta|\leq m_o} \int_{\calZ\times\calO} \left\|\partial^{\balpha}_\bz \partial^{\bbeta}_\bo \left(\phi_{\gamma_x}(\bx) p(\bx\mid \bz, \bo) \right) \right\|_{\calH_{X,\gamma_x}}^2 \;\mathrm{d} \bz \;\mathrm{d} \bo \\
    &= \sum_{|\balpha|\leq m_z} \sum_{|\bbeta|\leq m_o} \int_{\calZ\times\calO} \left\|\phi_{\gamma_x}(\bx) \partial^{\balpha}_\bz \partial^{\bbeta}_\bo p(\bx \mid \bz, \bo) \right\|_{\calH_{X,\gamma_x}}^2 \;\mathrm{d} \bz \;\mathrm{d} \bo \\
    &\leq \sum_{|\balpha|\leq m_z} \sum_{|\bbeta|\leq m_o} \int_{\calZ\times \calO} \left|\partial^{\balpha}_\bz \partial^{\bbeta}_\bo p(\bx| \bz, \bo) \right|^2 \;\mathrm{d} \bz \;\mathrm{d} \bo \\ 
    &\leq \binom{m_z + d_z -1}{d_z -1} \binom{m_o + d_o-1}{d_o -1} \rho^2  .
\end{align*}
In the last step, $\binom{m_z + d_z -1}{d_z -1}$ shows up as the evaluation of $\sum_{|\balpha| \leq m_z} 1$.
\end{proof}

\subsubsection{RKHS norm of $\bar{f}_{\lambda}$}
We now provide a refined control of the RKHS norm of $\bar{f}_{\lambda}$ by invoking an oracle inequality for the RKHS $\calH_{FO}$. We remind the reader that $V\bar{f}_{\lambda}\in \calH_{FO}$, i.e. the image of $\bar{f}_{\lambda}$ under the metric surjection $V$, is the solution to a KRR problem with respect to the RKHS $\calH_{FO}$ (see Eq. \eqref{eq:bhl_bvfl}). In \Cref{sec:capacity} we characterized the rate of decay of the entropy numbers of $\calH_{FO}$. The only remaining technical hurdle is that the noise $\upsilon = Y - (Tf_{\ast})(Z,O)$ is unbounded; however, under \Cref{ass:subgaussian} it is $\sigma$-subgaussian. Concretely, we introduce a logarithmically growing sequence of clipping values $M_n$, which facilitates a key step in the derivation of the oracle inequality and ensures that the conclusions of \citet[Theorem 7.23]{steinwart2008support} are applicable. 

Firstly, by \citet[Eq. (2.14)]{vershynin2018high}, $P(|\upsilon|\geq t) \leq 2\exp(-\frac{ct^2}{\sigma^2})$ for some universal constant $c$. 
Define $\upsilon_i=y_i - (Tf_{\ast})(\bz_i, \bo_i)$ for stage II samples $\{\bz_i, \bo_i, y_i\}_{i=1}^n$. 
Hence
\begin{align*}
    &P^{n}\left( \{ \bz_i, \bo_i, y_i\}_{i=1}^n \in (\calZ \times \calO \times Y)^n : \max_i |\upsilon_i|\leq t\}\right)\\
    \geq & 1 - \sum_{i=1}^{n}P(|\upsilon_i|\geq t)\\
    \geq & 1 - 2\exp\left(\ln n - \frac{ct^2}{\sigma^2}\right).
\end{align*}
Hence, with $P^{n}$-probability $\geq 1 - 2\exp(-\hat{\rho})$, we have
\begin{align*}
    \max_{1\leq i\leq n}|\upsilon_i|\leq \sigma\sqrt{\frac{\ln n + \hat{\rho}}{c}}.
\end{align*}
By \Cref{ass:f_ast}, we have $\|Tf_{\ast}\|_{\infty}\leq \|f_{\ast}\|_{\infty}\leq 1$. For $n\geq 1$, we define
\begin{align*}
    M_n = 1 + \sigma\sqrt{\frac{\ln n + \hat{\rho}}{c}}.
\end{align*}
Hence $y_i \in [-M_n, M_n]$ for all $1\leq i\leq n$ with $P^n$-probability $\geq 1 - 2\exp(-\hat{\rho})$. 

Secondly, we verify the assumptions of \citet[Theorem 7.23]{steinwart2008support}. By \citet[Example 7.3]{steinwart2008support}, the \emph{supremum bound} is satisfied for $B = 4M_n^2$ and the \emph{variance bound} is satisfied for $V = 16M_n^2$ and $\nu=1$. 
Define $D_{ZO}=\{\bz_i, \bo_i\}_{i=1}^n$ and $L^2(D_{ZO})$ as the $L^2$ space with respect to the empirical data measure $D_{ZO}$. 
We are about to show that $(\forall n\geq 1) (\exists p\in (0,1))(\exists a\geq B = 4M_n^2)$ such that 
\begin{align}
\label{eq:f_bar_lambda_hnorm_rts}
   (\forall i\geq 1)\; \mathbb{E}_{D_{ZO}\sim P_{ZO}^{n}}e_i(\mathrm{id}:\calH_{FO}\to L^2(D_{ZO}))\leq ai^{-\frac{1}{2p}}.
\end{align}
To this end, we invoke \citet[Corollary 7.31]{steinwart2008support}, \Cref{cor:entropy_bounded_obscov} and \Cref{lem:reduced_rkhs_entropy}. We conclude that, there exists a constant $c_p>0$ only depending on $p$, such that $(\forall i\geq 1) (\forall n\geq 1)$, we have
\begin{align*}
    &\mathbb{E}_{D_{ZO}\sim P^n_{ZO}}e_i(\mathrm{id}:\calH_{FO}\hookrightarrow L^2(D_{ZO}))\\
    \leq & c_p(3C)^{\frac{1}{2p}}\left(\frac{d_x+d_o+1}{2ep}\right)^{\frac{d_x+d_o+1}{2p}}\left(\gamma_x^{d_x}\gamma_o^{d_o}\right)^{-\frac{1}{2p}}(\min\{i,n\})^{\frac{1}{2p}}i^{-\frac{1}{p}}\\
    \leq & c_p(3C)^{\frac{1}{2p}}\left(\frac{d_x+d_o+1}{2ep}\right)^{\frac{d_x+d_o+1}{2p}}\left(\gamma_x^{d_x}\gamma_o^{d_o}\right)^{-\frac{1}{2p}}i^{-\frac{1}{2p}}.
\end{align*}
We define a new constant $\tilde{c}_{p}:= c_p(3C)^{\frac{1}{2p}}\left(\frac{d_x+d_o+1}{2ep}\right)^{\frac{d_x+d_o+1}{2p}}$. Since for all $n\geq 1$, $\left(\gamma_x^{d_x}\gamma_o^{d_o}\right)^{-\frac{1}{2p}}\geq 1$, we set 
\begin{align*}
    a = \max\{4M_n^2, \tilde{c}_{p}\}\left(\gamma_x^{d_x}\gamma_o^{d_o}\right)^{-\frac{1}{2p}}
\end{align*}
in Eq. \eqref{eq:f_bar_lambda_hnorm_rts}, which satisfies $a \geq B = 4M_n^2$ for all $n\geq1$. 

We thus apply \citet[Theorem 7.23]{steinwart2008support} restricted to the probabilistic event where $y_i\in [-2M_n, 2M_n]$ for all $1\leq i\leq n$. 
We find, with $P^{n}$-probability $\geq 1 - 5\exp(-\hat{\rho})$, 
\begin{align*}
    &\lambda \|\bar{f}_{\lambda}\|^2_{\calH_{XO}}\leq 9\left(\lambda \|f_0\|^2_{\calH_{XO}} + \|T(f_0 - f_{\ast})\|^2_{L^2(P_{ZO})}\right) \\ & + K\left(\frac{\max\{4M_n^2, \tilde{c}_{p}\}^{2p}\gamma_x^{-d_x}\gamma_o^{-d_o}}{\lambda^p n}\right) + \frac{216M_n^2 \hat{\rho}}{n} + \frac{15B_0\hat{\rho}}{n}
\end{align*}
where $K\geq 1$ is a constant depending on $p, M_n$, $f_0\in \calH_{\gamma_x,\gamma_o}$. $B_0\geq B$ is a constant that satisfies
\begin{align*}
    \|L\circ (Tf_{0})\|_{\infty} &:= \sup_{(\bz,\bo,y)\in \calZ\times \calO\times [-2M_n, 2M_n]}(y - (Tf_0)(\bz,\bo))^2 \leq B_0,
\end{align*}
By checking the dependence of $K$ on $p, M_n$ in the proof of \citet[Theorem 7.23]{steinwart2008support}, we find that $K\leq c(p)M_n^2$ for some constant $c(p)$ depending only on $p$. 
We choose $f_0 = f_{\lambda}$, where $f_{\lambda} \in \calH_{\gamma_x,\gamma_o}$ is defined in Eq. \eqref{eq:f_lambda}. We have, since $Y = [-2M_n, 2M_n]$, 
\begin{align*}
    \|L\circ (Tf_{\lambda})\|_{\infty} = & \sup_{(\bz,\bo,y)\in \calZ\times \calO\times [-2M_n, 2M_n]} (y - (Tf_{\lambda})(\bz,\bo))^2 
    \leq  (2M_n + \|Tf_{\lambda}\|_{\infty})^2 \\ \leq & (2M_n + \|f_{\lambda\|_{\infty}})^2 \leq (2M_n + \|f_{\lambda}\|_{\calH_{\gamma_x,\gamma_o}})^2 .
\end{align*}
We choose thus 
\begin{align*}
    B_0  = (2M_n + \|f_{\lambda}\|_{\calH_{\gamma_x,\gamma_o}})^2 \geq 4M_n^2 = B.
\end{align*}
Hence we deduce from the above high probability inequality that
\begin{align*}
    &\lambda \|\bar{f}_{\lambda}\|^2_{\calH_{XO}}\leq C\left(\lambda \|f_{\lambda}\|^2_{\calH_{XO}} + \|T(f_{\lambda} - f_{\ast})\|^2_{L^2(P_{ZO})}\right) \\ & + C\left( M_n^{2+2p}\left(\frac{\max\{4, \tilde{c}_{p}\}^{2p}\gamma_x^{-d_x}\gamma_o^{-d_o}}{\lambda^p n}\right) + \frac{(M_n + \|f_{\lambda}\|_{\calH_{XO}})^2 \hat{\rho}}{n}\right) .
\end{align*}
for some constant $C$ that is independent of $\lambda, M_n, n, \gamma_x,\gamma_o, B, V, \hat{\rho}$. We now state the main Proposition in this subsection. 
\begin{prop}
\label{lem:f_bar_lambda_norm_bound}
Suppose Assumptions~\ref{assn: technical}, \ref{ass:f_ast}, \ref{ass:T_injective}, \ref{ass:T_frequency_ill_posedness} and \ref{ass:T_contractivity} in the main text hold. 
Let $\lambda = n^{-1}$ and 
\begin{align*}
    \gamma_x = n^{-\frac{\frac{1}{d_x}}{1+(2+\frac{d_o}{s_o})(\frac{s_x}{d_x}+\eta_1)}}, \quad \gamma_o = n^{-\frac{\frac{1}{s_o}(\frac{s_x}{d_x}+\eta_1)}{1+(2+\frac{d_o}{s_o})(\frac{s_x}{d_x}+\eta_1)}}. 
\end{align*} Fix an arbitrary $p>0$. Then with probability $\geq 1-5e^{-\tau}$, for sufficiently large $n\geq 1$, the following bound holds
    \begin{align*}
        \|\bar{f}_{\lambda}\|^2_{\calH_{\gamma_x,\gamma_o}}\lesssim \tau n^{\frac{1 + \frac{d_o}{s_o}(\frac{s_x}{d_x}+\eta_1)}{1 + (2 + \frac{d_o}{s_o})(\frac{s_x}{d_x} + \eta_1)} + p}.
    \end{align*}
\end{prop}
\begin{proof}
Fix $p>0$. We combine the results of \Cref{prop: approx_error} with the inequality immediately preceding this proposition to find, for $\hat{\rho}\geq 1$, there exists a constant $C_1, C_2, C_3$ independent of $\lambda, M_n, n, \gamma_x, \gamma_o, B, V, \hat{\rho}$ such that , with probability $\geq 1 - 5e^{-\hat{\rho}}$ the following bound holds
    \begin{align*}
        &\quad \lambda \|\bar{f}_{\lambda}\|^2_{\calH_{XO}} \\
        &\stackrel{(a)}{\leq} C_1\hat{\rho}\left(\lambda \|f_{\lambda}\|^2_{\calH_{XO}} + \|T(f_{\lambda} - f_{\ast})\|^2_{L^2(P_{ZO})} + M_n^{2+2p}\left(\frac{\max\{4, \tilde{c}_p\}^{2p}\gamma_x^{-d_x}\gamma_o^{-d_o}}{\lambda^{p}n}\right) + \frac{M_n^2}{n}\right)\\
        &\stackrel{(b)}{\leq}  C_2\hat{\rho}\left(n^{-\frac{2(\frac{s_x}{d_x}+\eta_1)}{1+(2+\frac{d_o}{s_o})(\frac{s_x}{d_x}+\eta_1)}} + M_n^{2+2p}\left(\frac{\max\{4, \tilde{c}_p\}^{2p}\gamma_x^{-d_x}\gamma_o^{-d_o}}{\lambda^{p}n}\right) + \frac{M_n^2}{n}\right)\\
        &\stackrel{(c)}{=} C_2\hat{\rho}\left(n^{-\frac{2(\frac{s_x}{d_x}+\eta_1)}{1+(2+\frac{d_o}{s_o})(\frac{s_x}{d_x}+\eta_1)}}\left(1 + M_n^{2+2p}\frac{\max\{4, \tilde{c}_p\}^{2p}}{\lambda^{p}}\right) + \frac{M_n^2}{n}\right)\\
        &\stackrel{(d)}{\leq}C_3 \hat{\rho}n^{-\frac{2(\frac{s_x}{d_x}+\eta_1)}{1+(2+\frac{d_o}{s_o})(\frac{s_x}{d_x}+\eta_1)}} n^{2p}.
    \end{align*}
    In the above derivations, $(a)$ holds by $\lambda=n^{-1}$, $(b)$ holds by the conclusion of \Cref{prop: approx_error}, $(c)$ holds by the choice of $\gamma_x,\gamma_o$ as functions of $n$, $(d)$ holds \emph{for sufficiently large $n\geq 1$}, since $p$ is fixed and $n^{p} > M_n^{2+2p} = \left(1 + \sigma\sqrt{\frac{\ln n + \hat{\rho}}{c}}\right)^{2+2p}$ for sufficiently large $n\geq 1$, and $n^{-\frac{2(\frac{s_x}{d_x}+\eta_1)}{1+(2+\frac{d_o}{s_o})(\frac{s_x}{d_x}+\eta_1)}} \geq \frac{M_n^2}{n}$ for sufficiently large $n\geq 1$. In particular, the constant $C_3$ depends on $p$. 
\end{proof}

\subsection{Projected approximation error in Stage II}
\label{sec:approx_main}
\begin{prop}
\label{prop: approx_error}
Suppose Assumptions~\ref{assn: technical}, \ref{ass:f_ast}, \ref{ass:T_injective}, \ref{ass:T_frequency_ill_posedness} and \ref{ass:T_contractivity} in the main text hold. 
Let $\lambda = n^{-1}$ and 
\begin{align*}
    \gamma_x = n^{-\frac{\frac{1}{d_x}}{1+(2+\frac{d_o}{s_o})(\frac{s_x}{d_x}+\eta_1)}}, \quad \gamma_o = n^{-\frac{\frac{1}{s_o}(\frac{s_x}{d_x}+\eta_1)}{1+(2+\frac{d_o}{s_o})(\frac{s_x}{d_x}+\eta_1)}}. 
\end{align*} Then we have 
\begin{align*}
    \lambda \| f_\lambda\|_{\calH_{\gamma_x,\gamma_o}}^2 + \left\|T\left(f_{\ast} - \left[f_{\lambda}\right]\right)\right\|_{L^2(P_{ZO})}^2 \leq C n^{-\frac{2(\frac{s_x}{d_x}+\eta_1)}{1+(2+\frac{d_o}{s_o})(\frac{s_x}{d_x}+\eta_1)}},
\end{align*}
for some constant $C>0$ independent of $n$. 
\end{prop}

\begin{proof}
Recall $f_{\mathrm{aux}}\in \calH_{\gamma_x,\gamma_o}$ defined in Eq.~\eqref{eq:f_aux}. We write 
\begin{align*}
    f_{\mathrm{aux}} = f_{\ast} \ast K_{\gamma_o} - (f_{\ast} - f_{\ast , \mathrm{low}}) \ast K_{\gamma_o} + (f_{\ast} - f_{\ast, \mathrm{low}}) \ast K_{\gamma_o} \ast K_{\gamma_x}
.\end{align*}
By definition of $f_\lambda$, we have
\begin{align}
    &\quad \lambda \| f_\lambda\|_{\calH_{\gamma_x,\gamma_o}}^2 + \left\|T\left(f_{\ast} - \left[f_{\lambda}\right]\right)\right\|_{L^2(P_{ZO})}^2 
    \nonumber \\
    &=\inf_{f\in \calH_{\gamma_x,\gamma_o}}\lambda \|f\|_{\calH_{\gamma_x,\gamma_o}}^2 + \left\|Tf_{\ast} - T[f]\right\|_{L^2(P_{ZO})}^2\nonumber \\
    &\leq \lambda \|f_{\mathrm{aux}}\|_{\calH_{\gamma_x,\gamma_o}}^2 + \left\|Tf_{\ast} - T[f_{\mathrm{aux}}]\right\|_{L^2(P_{ZO})}^2 \nonumber \\
    &\leq \lambda \|f_{\mathrm{aux}} \|_{\calH_{\gamma_x,\gamma_o}}^2 + \left\|Tf_{\ast} - T \left(f_{\ast} \ast K_{\gamma_o} \right)\right\|_{L^2(P_{ZO})}^2 \nonumber \\
    &+\left\| T\left((f_{\ast} - f_{\ast , \mathrm{low}}) \ast K_{\gamma_o} - (f_{\ast} - f_{\ast, \mathrm{low}}) \ast K_{\gamma_o} \ast K_{\gamma_x}\right) \right\|_{L^2(P_{ZO})}^2. \label{eq:proj_approx_error}
\end{align}
For the first term, we know from Eq.~\eqref{eq:f_aux_norm} and \Cref{ass:f_ast} in the main text that 
\begin{align*}
    \lambda\|f_{\mathrm{aux}}\|_{\calH_{\gamma_x,\gamma_o}}^2 \leq n^{-1} \cdot 2 n^{\frac{1+\frac{d_o}{s_o}(\frac{s_x}{d_x}+\eta_1)}{1+(2+\frac{d_o}{s_o})(\frac{s_x}{d_x}+\eta_1)}} \|f_\ast\|_{L^2(\R^{d_x+d_o})} \leq 2 n^{-\frac{2(\frac{s_x}{d_x}+\eta_1)}{1+(2+\frac{d_o}{s_o})(\frac{s_x}{d_x}+\eta_1)}} .
\end{align*}
For the second term, we have
\begin{align*}
    \left\|Tf_{\ast} - T \left(f_{\ast} \ast K_{\gamma_o} \right)\right\|_{L^2(P_{ZO})}^2 &\leq \left\|f_{\ast} - f_{\ast} \ast K_{\gamma_o} \right\|_{L^2(P_{XO})}^2 \\
    &\lesssim \left\|f_{\ast} - f_{\ast} \ast K_{\gamma_o} \right\|_{L^2(\calX\times\calO)}^2 \\
    &\lesssim |f_{\ast}|^2_{B^{s_x,s_o}_{2,q}(\mathbb{R}^{d_x+d_o})} \cdot \max\{0, \gamma_o^{2s_o}\}\\
    &= |f_{\ast}|^2_{B^{s_x,s_o}_{2,q}(\mathbb{R}^{d_x+d_o})}  n^{-\frac{2(\frac{s_x}{d_x}+\eta_1)}{1+(2+\frac{d_o}{s_o})(\frac{s_x}{d_x}+\eta_1)}}.
\end{align*}
The first inequality above holds because $T$ is a bounded operator;
the second inequality holds by Assumption \ref{assn: technical} that $P_{XO}$ admits a bounded density. 
The second last inequality above holds by using \Cref{lem:hang_prop_3} for $\bgamma = (\underbrace{0,\dots, 0}_{d_x},\underbrace{\gamma_o,\dots,\gamma_o}_{d_o})$. 

\noindent
For the third term, we note that $(\forall \bo\in \mathbb{R}^{d_o})$, 
\begin{align*}
    \calF[\left((f_{\ast} - f_{\ast , \mathrm{low}}) \ast K_{\gamma_o} - (f_{\ast} - f_{\ast, \mathrm{low}}) \ast K_{\gamma_o} \ast K_{\gamma_x}\right)(\cdot, \bo)] 
\end{align*}
is supported on the complement of $\{\bx: \|\bx\| \leq \gamma_x^{-1}\}$ (see~Eq. \eqref{eq:f_ast_low}). Thus it follows from \Cref{ass:T_contractivity} that
\begin{align*}
    &\quad \left\| T\left((f_{\ast} - f_{\ast , \mathrm{low}}) \ast K_{\gamma_o} - (f_{\ast} - f_{\ast, \mathrm{low}}) \ast K_{\gamma_o} \ast K_{\gamma_x}\right) \right\|_{L^2(P_{ZO})}^2 \\
    &\leq \gamma_x^{2d_x\eta_1}\left\| (f_{\ast} - f_{\ast , \mathrm{low}}) \ast K_{\gamma_o} - (f_{\ast} - f_{\ast, \mathrm{low}}) \ast K_{\gamma_o} \ast K_{\gamma_x} \right\|_{L^2(P_{XO})}^2\\
    &\stackrel{(i)}{\lesssim} \gamma_x^{2d_x\eta_1+ 2s_x}\cdot \left|(f_{\ast} - f_{\ast , \mathrm{low}}) \ast K_{\gamma_o}\right|^2_{B^{s_x,s_o}_{2,q}(\mathbb{R}^{d_x+d_o})}\\
    &\stackrel{(ii)}{\leq} \gamma_x^{2d_x\eta_1+ 2s_x}\cdot \left|f_{\ast} \ast K_{\gamma_o}\right|^2_{B^{s_x,s_o}_{2,q}(\mathbb{R}^{d_x+d_o})}\\
    &\stackrel{(iii)}{\leq} \gamma_x^{2d_x\eta_1+ 2s_x}\cdot \left|f_{\ast}\right|^2_{B^{s_x,s_o}_{2,q}(\mathbb{R}^{d_x+d_o})} \|K_{\gamma_o}\|^2_{L^1(\mathbb{R}^{d_o})}\\
    &\stackrel{(iv)}{\lesssim} n^{-\frac{2(\frac{s_x}{d_x}+\eta_1)}{1+(2+\frac{d_o}{s_o})(\frac{s_x}{d_x}+\eta_1)}}. 
\end{align*}
In the above derivations, $(i)$ follows by using \Cref{lem:hang_prop_3} for $\bgamma = (\underbrace{\gamma_x,\dots,\gamma_x}_{d_x},\underbrace{0,\dots,0}_{d_o})$ and the Assumption that $P_{XO}$ admits a bounded density, $(ii)$ follows from the proof of \Cref{lem:new_besov_mask}, $(iii)$ follows from \Cref{lem:besov_young}, and $(iv)$ follows by the fact that $\|K_{\gamma_o}\|_{L^1(\mathbb{R})}= 1$ \citep[Section 4.1.2]{Giné_Nickl_2015} and \Cref{ass:f_ast} in the main text.

Combine the upper bound on the three terms in Eq.~\eqref{eq:proj_approx_error} and we obtain
\begin{align}\label{eq:lambda_f_lambda_T_L2}
    \lambda \| f_\lambda\|_{\calH_{\gamma_x,\gamma_o}}^2 + \left\|T\left(f_{\ast} - \left[f_{\lambda}\right]\right)\right\|_{L^2(P_{ZO})}^2 \leq n^{-\frac{2(\frac{s_x}{d_x}+\eta_1)}{1+(2+\frac{d_o}{s_o})(\frac{s_x}{d_x}+\eta_1)}}.
\end{align}
The proof concludes here.
In addition to that, from Eq.~\eqref{eq:proj_approx_error}, we also have
\begin{align}\label{eq:Tf_Tf_aux}
    \left\|Tf_{\ast} - T[f_{\mathrm{aux}}]\right\|_{L^2(P_{ZO})} \leq C n^{-\frac{\frac{s_x}{d_x}+\eta_1}{1+(2+\frac{d_o}{s_o})(\frac{s_x}{d_x}+\eta_1)}} .
\end{align}
Also, if we follow the same derivations as above, we obtain
\begin{align}\label{eq:f_f_aux_diff}
    \left\|f_{\ast} - [f_{\mathrm{aux}}] \right\|_{L^2(P_{XO})} \leq C n^{-\frac{\frac{s_x}{d_x}}{1+(2+\frac{d_o}{s_o})(\frac{s_x}{d_x}+\eta_1)}} .
\end{align}
\end{proof}

\subsection{Projected estimation error in Stage II}
\label{sec:esti_main}
\begin{prop}
\label{prop: projected_est_error}
Suppose Assumptions~\ref{assn: technical}, \ref{ass:f_ast}, \ref{ass:T_injective}, \ref{ass:T_frequency_ill_posedness} and \ref{ass:T_contractivity} in the main text hold. 
Let $\lambda \asymp n^{-1}$ and 
\begin{align*}
    \gamma_x = n^{-\frac{\frac{1}{d_x}}{1+(2+\frac{d_o}{s_o})(\frac{s_x}{d_x}+\eta_1)}}, \quad \gamma_o = n^{-\frac{\frac{1}{s_o}(\frac{s_x}{d_x}+\eta_1)}{1+(2+\frac{d_o}{s_o})(\frac{s_x}{d_x}+\eta_1)}}. 
\end{align*} With $P^n$-probability $\geq 1-  4e^{-\tau}$, for sufficiently large $n\geq 1$, we have
\begin{align*}
\left\|T\left(\left[\bar{f}_{\lambda}\right] - [f_{\lambda}]]\right)\right\|_{L^2(P_{ZO})}^2 \leq C \tau^2 n^{-\frac{2(\frac{s_x}{d_x}+\eta_1)}{1+(2+\frac{d_o}{s_o})(\frac{s_x}{d_x}+\eta_1)}} (\log n)^{d_x+d_o+1},
\end{align*}
for some constant $C$ independent of $n$.
\end{prop}
\begin{proof}
Recall the definition of $\bar{h}_{\lambda}$ defined in Eq.~\eqref{eq:hat_bhl_bvfl} and $h_\lambda$ defined in Eq.~\eqref{eq:bhl_bvfl}.
\begin{align*}
    \bar{h}_{\lambda} &= \argmin_{h\in \calH_{FO}}\lambda \|h\|_{\calH_{FO}}^2 + \frac{1}{n}\sum_{i=1}^{n}(h(\bz_i,\bo_i) - y_i)^2 . \\
    h_{\lambda} &= \argmin_{h\in\calH_{FO}}\lambda \|h\|^2_{\calH_{FO}} + \|h_{\ast} - h\|^2_{L^2(P_{ZO})}.
\end{align*} 
Recall that we have proved that $\|T([\bar{f}_{\lambda}] - [f_{\lambda}])\|_{L^2(P_{ZO})} = \|[\bar{h}_{\lambda}] - [h_{\lambda}] \|_{L^2(P_{ZO})}$ in Eq.~\eqref{eq:error_translate}, so the proof of \Cref{prop: projected_est_error} is translated to the estimation error of a standard kernel ridge regression with hypothesis space $\calH_{FO}$ and the target function $h_\ast := T f_\ast \in L^2(P_{ZO})$. Next, we are going to apply existing results, mainly Theorem 16 of \cite{fischer2020sobolev}, to our setting.

From Eq. \eqref{eq:error_translate} and \cite{fischer2020sobolev}[Lemma 12], we have
\begin{align}
\label{eq:prop_first_line_TFF}
     \left\| [\bar{h}_{\lambda}] - [h_{\lambda}]\right\|_{L^2(P_{ZO})} \leq \left\|C_{FO}^{\frac{1}{2}}\left(\bar{h}_{\lambda} - h_{\lambda}\right)\right\|_{\calH_{FO}} =: (\ast)
\end{align}
Next, we will upper bound $(\ast)$ using Theorem 16 from \citet{fischer2020sobolev}. To apply this result, we must verify the underlying assumptions and control the auxiliary quantities specified in Theorem 16 from \citet{fischer2020sobolev}. 

First, we are going to verify the (MOM) condition.
By \Cref{ass:subgaussian}, we know that
\begin{align*}
    \int_{\mathbb{R}}\left|y- (T f_\ast)(\bz, \bo)\right|^m p(y \mid \bz, \bo)dy &= \mathbb{E}\left[|\upsilon|^m\mid Z=\bz,O=\bo\right] \\
    &\leq (\sqrt{2}C\sigma)^m(m/2)^{\frac{m}{2}} \leq \frac{1}{2}(2C\sigma)^mm! 
\end{align*}
by \citet[Eq. (2.15)]{vershynin2018high}, for some universal constant $C>0$. Hence the (MOM) condition of \cite{fischer2020sobolev} is satisfied.

Next, we control the auxiliary quantities in \cite{fischer2020sobolev}[Theorem 16]. Recall $m_z,m_o$ as defined in Eq. \eqref{def:m_z,m_o}. By definition, they satisfy
\begin{align*}
    \frac{d_o}{2m_o}<\frac{1 + 2\left(\frac{s_x}{d_x}+\eta_1\right)}{1+2\left(\frac{s_x}{d_x}+\eta_1\right) + \frac{d_o}{s_o}\left(\frac{s_x}{d_x}+\eta_1\right)}<1 , \quad \frac{d_z}{2m_z}\leq \frac{d_o}{2m_o} < 1 .
\end{align*}
Rearranging, we deduce that
\begin{align*}
    \frac{d_o}{2m_o} &< \left( 2m_o\frac{\frac{1}{s_o}(\frac{s_x}{d_x}+\eta_1)}{1+(2+\frac{d_o}{s_o})(\frac{s_x}{d_x}+\eta_1)} + 1\right)^{-1}\\ \frac{d_o}{2m_o} &< \frac{1+(2+\frac{d_o}{s_o})(\frac{s_x}{d_x}+\eta_1)}{1+2(\frac{s_x}{d_x}+\eta_1)+\frac{2m_o+d_o}{s_o}(\frac{s_x}{d_x}+\eta_1)}.
\end{align*}
Hence there exists $\theta\in (0,1)$ such that the following inequalities hold simultaneously
\begin{align}\label{eq:theta_cond}
\begin{aligned}
    \frac{d_o}{2m_o\theta} < 1, \quad \frac{d_z}{2m_z\theta} < 1 \\
    \theta \leq \frac{1+(2+\frac{d_o}{s_o})(\frac{s_x}{d_x}+\eta_1)}{1+2(\frac{s_x}{d_x}+\eta_1)+\frac{2m_o+d_o}{s_o}(\frac{s_x}{d_x}+\eta_1)}\\
    \theta\left( 2m_o\frac{\frac{1}{s_o}(\frac{s_x}{d_x}+\eta_1)}{1+(2+\frac{d_o}{s_o})(\frac{s_x}{d_x}+\eta_1)} + 1\right) < 1
\end{aligned}
\end{align}
Next, we define 
\begin{align}\label{eq:FS_thm_16_const}
\begin{aligned}
    g_{\lambda}&:= \log\left(2e\calN_{FO}(\lambda)\frac{\|C_{FO}\|+\lambda}{\|C_{FO}\|}\right)\\
    A_{\lambda,\tau}&:= 8\left\|k^{\theta}_{FO}\right\|_{\infty}^2\tau g_{\lambda}\lambda^{-\theta}\\
    L_{\lambda}&:=\max\left\{L, \left\|h_{\ast} - [h_{\lambda}]\right\|_{L^{\infty}(P_{ZO})}\right\}
\end{aligned}
\end{align}

\textit{Controlling $g_{\lambda}$ and $A_{\lambda,\tau}$:  } 
From \Cref{prop:eff_dim}, we know that
\begin{align*}
    \calN_{FO}(\lambda) \lesssim (\log n)^{d_x+d_o+1} \left(\gamma_x^{d_x}\gamma_o^{d_o}\right)^{-1} = (\log n)^{d_x+d_o+1} n^{\frac{1+\frac{d_o}{s_o}(\frac{s_x}{d_x}+\eta_1)}{1+(2+\frac{d_o}{s_o})(\frac{s_x}{d_x}+\eta_1)}} .
\end{align*}
From \Cref{lem: C_f_op_norm_lower_bound}, we know that
\begin{align*}
    \|C_{FO}\| \geq a_f^{-1}\left(\frac{\sqrt{\pi}}{4}\right)^{\frac{d_x+d_o}{2}} n^{-\frac{1}{2}\frac{1+\frac{d_o}{s_o}(\frac{s_x}{d_x}+\eta_1)}{1+(2+\frac{d_o}{s_o})(\frac{s_x}{d_x}+\eta_1)}} .
\end{align*}
Therefore, since $\lambda \asymp n^{-1}$, we have
\begin{align*}
    g_{\lambda } &= \log\left(2e\calN_{FO}(\lambda) \left(1 + \lambda \|C_{FO}\|^{-1} \right) \right) \\
    &\lesssim \log\left((\log n)^{d_x+d_o+1} n^{\frac{1+\frac{d_o}{s_o}(\frac{s_x}{d_x}+\eta_1)}{1+(2+\frac{d_o}{s_o})(\frac{s_x}{d_x}+\eta_1)}} \left(1 + n^{-1} \cdot n^{\frac{1}{2}\frac{1+\frac{d_o}{s_o}(\frac{s_x}{d_x}+\eta_1)}{1+(2+\frac{d_o}{s_o})(\frac{s_x}{d_x}+\eta_1)} }\right)\right)\\
    &\leq \log\left(2(\log n)^{d_x+d_o+1} n^{\frac{1+\frac{d_o}{s_o}(\frac{s_x}{d_x}+\eta_1)}{1+(2+\frac{d_o}{s_o})(\frac{s_x}{d_x}+\eta_1)}} \right)\\
    &= (d_x + d_o + 1)\log(\log n) + \frac{1+\frac{d_o}{s_o}(\frac{s_x}{d_x}+\eta_1)}{1+(2+\frac{d_o}{s_o})(\frac{s_x}{d_x}+\eta_1)} \log(n)\\
    &\leq 2\frac{1+\frac{d_o}{s_o}(\frac{s_x}{d_x}+\eta_1)}{1+(2+\frac{d_o}{s_o})(\frac{s_x}{d_x}+\eta_1)}\log(n).
\end{align*}
The last step holds because $\log(n)$ dominates a constant term and $\log(\log n)$ for sufficiently large $n$. Since $\frac{d_o}{2m_o\theta}<1, \frac{d_z}{2m_z\theta}<1$, by \Cref{prop:emb_main}, we have
\begin{align}
    \left\|k_{FO}^{\theta}\right\|_{\infty} \lesssim \rho^{\theta}\gamma_o^{-\theta m_o}\label{eq: kfo_bound}.
\end{align}
Therefore, 
\begin{align}\label{eq:n_large_than_A}
    A_{\lambda,\tau} &\lesssim \left\|k_{FO}^{\theta}\right\|_{\infty}^2\tau \log(n)n^{\theta} \lesssim \gamma_o^{-2m_o\theta}\tau \log(n)n^{\theta} = \tau \log(n) n^{\theta\left( 2m_o\frac{\frac{1}{s_o}(\frac{s_x}{d_x}+\eta_1)}{1+(2+\frac{d_o}{s_o})(\frac{s_x}{d_x}+\eta_1)} + 1\right)}.
\end{align}
Since $\theta ( 2m_o\frac{\frac{1}{s_o}(\frac{s_x}{d_x}+\eta_1)}{1+(2+\frac{d_o}{s_o})(\frac{s_x}{d_x}+\eta_1)} + 1) < 1$ from Eq.~\eqref{eq:theta_cond}, $n>  A_{\lambda,\tau}$ holds for sufficiently large $n \geq 1$. 

\textit{Controlling $L_{\lambda}$:  } 
By Eq. \eqref{eq:lambda_f_lambda_T_L2}, we have
\begin{align*}
    \lambda \|f_{\lambda}\|^2_{\calH_{\gamma_x,\gamma_o}} \leq n^{-\frac{2(\frac{s_x}{d_x}+\eta_1)}{1+(2+\frac{d_o}{s_o})(\frac{s_x}{d_x}+\eta_1)}} \|f_\ast\|_{B^{s_x,s_o}_{2,q}(\R^{d_x+d_o})}^2 .
\end{align*}
Thus 
\begin{align}
\label{eq:f_lambda_hnorm_bound}
    \|f_{\lambda}\|^2_{\calH_{\gamma_x,\gamma_o}} \leq n^{\frac{1+\frac{d_o}{s_o}(\frac{s_x}{d_x}+\eta_1)}{1+(2+\frac{d_o}{s_o})(\frac{s_x}{d_x}+\eta_1)}} \|f_\ast\|_{B^{s_x,s_o}_{2,q} (\R^{d_x+d_o})}^2 .
\end{align}
Thus we have
\begin{align}
    L_{\lambda} &= \max\left\{2C\sigma, \|T(f_{\ast} - [f_{\lambda}])\|_{L^{\infty}(\calZ\times \calO)}\right\}\nonumber\\
    &\leq \max\left\{2C\sigma, \|f_{\ast} - [f_{\lambda}]\|_{L^{\infty}(\calX\times \calO)}\right\}\nonumber\\
    &\leq  \max\left\{2C\sigma, \|f_{\ast}\|_{L^{\infty}(\calX\times \calO)} + \|[f_{\lambda}]\|_{L^{\infty}(\calX\times \calO)}\right\}\nonumber\\
    &\stackrel{(i)}{\leq} \max\left\{2C\sigma, \|f_{\ast}\|_{L^{\infty}(\calX\times \calO)} + \|f_{\lambda}\|_{\calH_{\gamma_x,\gamma_o}}\right\}\nonumber\\
    &\stackrel{(ii)}{\leq} \max\left\{2C\sigma, \|f_{\ast}\|_{L^{\infty}(\calX\times \calO)} + n^{\frac{1}{2}\frac{1+\frac{d_o}{s_o}(\frac{s_x}{d_x}+\eta_1)}{1+(2+\frac{d_o}{s_o})(\frac{s_x}{d_x}+\eta_1)}} \|f_\ast\|_{B^{s_x,s_o}_{2,q} (\R^{d_x+d_o})} \right\}\nonumber\\
    &\stackrel{(iii)}{\leq} 2 n^{\frac{1}{2}\frac{1+\frac{d_o}{s_o}(\frac{s_x}{d_x}+\eta_1)}{1+(2+\frac{d_o}{s_o})(\frac{s_x}{d_x}+\eta_1)}} \|f_\ast\|_{B^{s_x,s_o}_{2,q} (\R^{d_x+d_o})} \label{eq:l_lambda_bound},
\end{align}
where $(i)$ holds by the reproducing property, $(ii)$ holds by Eq. \eqref{eq:f_lambda_hnorm_bound}, and $(iii)$ holds for sufficiently large $n$.

\textit{Bounding $(\ast)$:  } Now we are ready to apply \cite{fischer2020sobolev}[Theorem 16] to upper bound $(\ast)$. 
By \cite{fischer2020sobolev}[Theorem 16], with $P^n$-probability $\geq 1 - 4e^{-\tau}$, for $n\geq A_{\lambda,\tau}$,
\begin{small}
\begin{align*}
    (\ast) &= \left\|C_{FO}^{\frac{1}{2}}\left(\bar{h}_{\lambda} - h_{\lambda}\right)\right\|_{\calH_{FO}}^2 
    \\ \leq &\frac{576\tau^2}{n} \left(\sigma^2\calN_{FO}(\lambda) + \left\|k^{\theta}_{FO}\right\|_{\infty}^2\frac{\left\|h_{\ast} - [h_{\lambda}]\right\|_{L^2(P_{ZO})}^2}{\lambda^{\theta}} + 2\left\|k^{\theta}_{FO}\right\|_{\infty}^2\frac{L_{\lambda}^2}{n\lambda^{\theta}} \right) \\
    \stackrel{(a)}{\lesssim} &\frac{\tau^2}{n} \left((\log n)^{d_x+d_o+1} n^{\frac{1+\frac{d_o}{s_o}(\frac{s_x}{d_x}+\eta_1)}{1+(2+\frac{d_o}{s_o})(\frac{s_x}{d_x}+\eta_1)}} + \left\|k^{\theta}_{FO}\right\|_{\infty}^2n^{\theta}\left(\left\|h_{\ast} - [h_{\lambda}]\right\|_{L^2(P_{ZO})}^2 + \frac{L_{\lambda}^2}{n}\right)\right)\\
    \stackrel{(b)}{\lesssim} & \tau^2\left(n^{-\frac{2(\frac{s_x}{d_x}+\eta_1)}{1+(2+\frac{d_o}{s_o})(\frac{s_x}{d_x}+\eta_1)}} (\log n)^{d_x+d_o+1} + \gamma_o^{-2m_o\theta}  n^{\theta-1} \cdot n^{-\frac{2(\frac{s_x}{d_x}+\eta_1)}{1+(2+\frac{d_o}{s_o})(\frac{s_x}{d_x}+\eta_1)}}\right)\\
    \stackrel{(c)}{\lesssim} & \tau^2\left(n^{-\frac{2(\frac{s_x}{d_x}+\eta_1)}{1+(2+\frac{d_o}{s_o})(\frac{s_x}{d_x}+\eta_1)}} (\log n)^{d_x+d_o+1} + n^{-\frac{2(\frac{s_x}{d_x}+\eta_1)}{1+(2+\frac{d_o}{s_o})(\frac{s_x}{d_x}+\eta_1)}} \right)\\
    \lesssim & \tau^2 n^{-\frac{2(\frac{s_x}{d_x}+\eta_1)}{1+(2+\frac{d_o}{s_o})(\frac{s_x}{d_x}+\eta_1)}} (\log n)^{d_x+d_o+1}.
\end{align*}
\end{small}
We use Proposition~\ref{prop:eff_dim} to upper bound $\calN_{FO}(\lambda)$ in $(a)$,
we use Eq. \eqref{eq:l_lambda_bound} to upper bound $L_\lambda$, Eq. \eqref{eq: kfo_bound} to upper bound $\|k^{\theta}_{FO}\|_{\infty}$ and \Cref{prop: approx_error} to upper bound $ \|h_{\ast} - [h_{\lambda}]\|_{L^2(P_{ZO})} = \|T f_{\ast} - T[f_{\lambda}]\|_{L^2(P_{ZO}}$ in $(b)$, and we use the following fact in $(c)$: by Eq.~\eqref{eq:theta_cond}, we have
\begin{align*}
    &\theta \leq \frac{1+(2+\frac{d_o}{s_o})(\frac{s_x}{d_x}+\eta_1)}{1+2(\frac{s_x}{d_x}+\eta_1)+\frac{2m_o+d_o}{s_o}(\frac{s_x}{d_x}+\eta_1)} \\
    \iff & n^{\frac{\frac{2m_o \theta}{s_o}(\frac{s_x}{d_x}+\eta_1)}{1+(2+\frac{d_o}{s_o})(\frac{s_x}{d_x}+\eta_1)}} n^{\theta-1} \leq 1\\
    \iff & \gamma_o^{-2m_o\theta}n^{\theta-1} \leq 1. 
\end{align*}
The proof is concluded.
\end{proof}

\subsection{Auxiliary Lemmas for \Cref{sec:proof_upper}}
\begin{lem}
    \label{lem:vv_mss}
    Let $\Omega_1,\Omega_2$ be two open sets in $\mathbb{R}^{n_1}$ and $\mathbb{R}^{n_2}$ respectively. Let $\calH$ be a Hilbert space. Then we have
    \begin{align*}
        MW^{m,p}(\Omega_1\times \Omega_2, \calH) \cong W^{m}(\Omega_1, W^{p}(\Omega_2,\calH)) \cong W^m(\Omega_1)\otimes W^p(\Omega_2)\otimes \calH
    \end{align*}
\end{lem}
\begin{proof}
    We adapt the proof of \cite{aubin2011applied}[Theorem 12.7.2]. Let $f\in W^{m}(\Omega_1, W^{p}(\Omega_2,\calH))$. Then $f\in L^2(\Omega_1, W^p(\Omega_2,\calH))$ such that $\partial_1^{\alpha}f\in L^2(\Omega_1, W^p(\Omega_2,\calH))$ for $|\alpha|\leq m$. As explained in \cite{Hytönen2016AnalysisBanachSpaces}[Definition 2.5.1], to say that $\partial_1^{\alpha}f\in L^2(\Omega_1, W^p(\Omega_2,\calH))$ is equivalent to saying that for every $\phi \in C^{\infty}_{c}(\Omega_1)$, we have
    \begin{align}
    \label{eq:vv_weak}
        \int_{\Omega_1}f(x,\cdot)\partial^{\alpha}_{1}\phi(x)\;\mathrm{d}x = (-1)^{|\alpha|}\int_{\Omega_1}(\partial^{\alpha}_{1}f)(x,\cdot)\phi(x)\;\mathrm{d}x
    \end{align}
    belongs to $W^p(\Omega_2,\calH)$. But this exactly says that $f\in MW^{m,p}(\Omega_1\times \Omega_2,\calH)$. Thus the first isometric isomorphism is proved. The second isomorphism is proved by two successive applications of \citet[Theorem 12.7.1]{aubin2011applied}, namely the fact that the Hilbert-space valued Sobolev space $W^m(\Omega,\calH)$ is isometrically isomorphic to the Hilbertian tensor product $W^m(\Omega)\otimes \calH$.   
\end{proof}

\begin{lem}\label{lem:thm12_meunier}
Fix $\tau\geq 1$. Suppose that the assumptions of \Cref{prop:cme_rate} hold. Let $\lambda = \frac{1}{n}$. We assume that $n\geq 1$ satisfies $n> A_{\lambda,\tau}$ with $A_{\lambda,\tau}$ defined in Eq. \eqref{eq:FS_thm_16_const} and $\stageonesamples \geq 1$ satisfies Eq.~\eqref{eq:m_sufficient_large}. Then, with $P^{n+\stageonesamples}$-probability $\geq 1-10e^{-\tau}$
    \begin{align*}
        S_0 \leq c^{\prime} \tau\sqrt{n}\left(\frac{\left\|\hat{F}_{\xi} -F_\ast \right\|_{\calG}}{\sqrt{n}}+\left\| \left[\hat{F}_{\xi}\right] - F_\ast \right\|_{L^2\left(\calZ\times\calO; \mathcal{H}_{X,\gamma_x}\right)}\right)\left\|\bar{h}_\lambda\right\|_{\mathcal{H}_{FO}},
    \end{align*}
    where $c'$ is a constant independent of $n,\tilde{n}$. Under the same probabilistic event, the bounds in \Cref{prop:cme_rate} hold. 
\end{lem}
\begin{proof}
We adapt the proof of \citet{meunier2024nonparametricinstrumentalregressionkernel}[Theorem 12]. 
\begin{align*}
\begin{aligned}
S_0
= & \left\|\left(C_{FO}+\lambda \right)^{1 / 2}\left(\hat{C}_{\hat{F}O}+\lambda \right)^{-1}\left(\hat{C}_{\hat{F}O}-\hat{C}_{FO}\right)\left(\hat{C}_{FO}+\lambda \right)^{-1} \frac{1}{n} \boldsymbol{\Phi}_{FO}^\ast \boldsymbol{Y} \right\|_{\mathcal{H}_{FO}} \\
\leq & \lambda^{-1 / 2} \underbrace{\left\|\left(C_{FO}+\lambda \right)^{1 / 2}\left(\hat{C}_{FO}+\lambda \right)^{-1 / 2}\right\|}_{(A)} \\
&\cdot \underbrace{\left\|\left(\hat{C}_{FO}+\lambda \right)^{1 / 2}\left(\hat{C}_{\hat{F} O}+\lambda \right)^{-1 / 2}\right\|}_{(B)} \cdot \underbrace{\left\|\hat{C}_{\hat{F}O}-\hat{C}_{FO}\right\|}_{(C)} \cdot \left\|\bar{h}_\lambda\right\|_{\mathcal{H}_{FO}}
\end{aligned}
\end{align*}
\underline{Upper bound for term $(A)$}. Notice that all assumptions from \citet[Lemma 17]{fischer2020sobolev} have been checked already in the analysis of projected estimation error in \Cref{sec:esti_main}.
Since $n\geq A_{\lambda,\tau}$ with $A_{\lambda,\tau}$ defined in Eq. \eqref{eq:FS_thm_16_const}, with $P^n$-probability $\geq 1 - 2e^{-\tau}$,
\begin{align*}
    \left\|\left(C_{FO} +\lambda\right)^{-1 / 2}\left(C_{FO} -\hat{C}_{FO}\right) \left(C_{FO}+\lambda\right)^{-1 / 2} \right\| &\leq \frac{4\left\|k^{\theta}_{FO}\right\|_{\infty}^2 \tau g_\lambda}{3 n \lambda^\theta}+\sqrt{\frac{2\left\|k^{\theta}_{FO}\right\|_{\infty}^2 \tau g_\lambda}{n \lambda^\theta}} \\
    &\leq \frac{2}{3}.
\end{align*}
By \citet[Proposition 7]{rudi2015less}, we have
\begin{align*}
    (A) \leq \left(1 - \left\|\left(C_{FO} +\lambda\right)^{-1 / 2}\left(C_{FO} -\hat{C}_{FO}\right) \left(C_{FO}+\lambda\right)^{-1 / 2} \right\|\right)^{-\frac{1}{2}} \leq (1/3)^{-\frac{1}{2}} \leq 2.
\end{align*}
\underline{Upper bound for term $(C)$}. Define
\begin{align*}
    \tilde{\bPhi}_{\hat{F}O} : \calH_{\gamma_x,\gamma_o} \to \R^{n} &= \left[ \left( \hat{F}_\xi(\bz_1, \bo_1) \otimes \phi_{\gamma_o}(\bo_1) \right) ,\dots, \left( \hat{F}_\xi(\bz_n, \bo_n) \otimes \phi_{\gamma_o}(\bo_n) \right) \right]^{\ast}\\
    \tilde{\bPhi}_{FO} : \calH_{\gamma_x,\gamma_o} \to \R^{n} &= \left[\left( F_{\ast}(\bz_1, \bo_1)\otimes \phi_{\gamma_o}(\bo_1) \right),\dots, \left( F_{\ast}(\bz_n, \bo_n) \otimes \phi_{\gamma_o}(\bo_1) \right) \right]^{\ast} .
\end{align*}
An immediate consequence is that $V \tilde{\bPhi}_{\hat{F}O}^\ast = \bPhi_{\hat{F}O}^\ast$ and $V \tilde{\bPhi}_{FO}^\ast = \bPhi_{FO}^\ast$ for $V$ defined following \Cref{defi:H_F} and $\bPhi_{\hat{F}O}, \bPhi_{FO}$ defined in Eq.~\eqref{eq:bPhi_FO}.
Hence, we have
\begin{align*}
    (C) &= \left\|\hat{C}_{\hat{F}O}-\hat{C}_{FO}\right\| \\
    &= \left\| \frac{1}{n} \bPhi_{F O}^{\ast} \bPhi_{F O} - \frac{1}{n} \bPhi_{\hat{F} O}^\ast \bPhi_{\hat{F} O} \right\| \\
    &= \left\| \frac{1}{n} V \tilde{\bPhi}_{F O}^{\ast} \tilde{\bPhi}_{F O} V^\ast - \frac{1}{n} V \tilde{\bPhi}_{\hat{F} O}^\ast \tilde{\bPhi}_{\hat{F} O} V^\ast \right\| \\
    &= \left\| \frac{1}{n} \tilde{\bPhi}_{F O}^{\ast} \tilde{\bPhi}_{F O} - \frac{1}{n} \tilde{\bPhi}_{\hat{F} O}^\ast \tilde{\bPhi}_{\hat{F} O} \right\|\\ 
    &= \left\|\frac{1}{n}\sum_{i=1}^{n}\left(\hat{F}_\xi(\bz_i,\bo_i)\otimes \phi_{\gamma_o}(\bo_i)\right) \otimes \left(\hat{F}_\xi(\bz_i,\bo_i)\otimes \phi_{\gamma_o}(\bo_i)\right) \right. \\
    &\quad\quad \left. - \frac{1}{n}\sum_{i=1}^{n} \left(F_{\ast}(\bz_i,\bo_i)\otimes\phi_{\gamma_o}(\bo_i)\right)\otimes \left(F_{\ast}(\bz_i,\bo_i)\otimes\phi_{\gamma_o}(\bo_i)\right) \right\|.
\end{align*}
The second last equality holds because $V:\calH_{XO} \to \calH_{FO}$ is an isometry. 
We start with the following decomposition
\begin{align*}
    &\left(\hat{F}_\xi(\bz_i,\bo_i)\otimes \phi_{\gamma_o}(\bo_i)\right)\otimes\left(\hat{F}_\xi(\bz_i,\bo_i)\otimes \phi_{\gamma_o}(\bo_i)\right) \\
    & \quad\quad - \left(F_{\ast}(\bz_i,\bo_i)\otimes\phi_{\gamma_o}(\bo_i)\right) \otimes \left(F_{\ast}(\bz_i,\bo_i)\otimes\phi_{\gamma_o}(\bo_i)\right) \\
    =&\left(\left(\hat{F}_\xi(\bz_i,\bo_i) - F_{\ast}(\bz_i,\bo_i) \right) \otimes\phi_{\gamma_o}(\bo_i)\right) \otimes \left(\left(\hat{F}_\xi(\bz_i,\bo_i) - F_{\ast}(\bz_i,\bo_i) \right)\otimes\phi_{\gamma_o}(\bo_i)\right) \\ &+ \left(\left(\hat{F}_\xi(\bz_i,\bo_i) - F_{\ast}(\bz_i,\bo_i) \right)\otimes \phi_{\gamma_o}(\bo_i)\right)\otimes (F_{\ast}(\bz_i,\bo_i)\otimes \phi_{\gamma_o}(\bo_i)) \\
    &+ (F_{\ast}(\bz_i,\bo_i)\otimes \phi_{\gamma_o}(\bo_i)) \otimes \left(\left(\hat{F}_\xi(\bz_i,\bo_i) - F_{\ast}(\bz_i,\bo_i) \right)\otimes \phi_{\gamma_o}(\bo_i)\right).
\end{align*}
Thus we have
\begin{small}
\begin{align*}
    (C)&\leq \left\|\frac{1}{n}\sum_{i=1}^{n}\left(\left(\hat{F}_\xi(\bz_i,\bo_i) - F_{\ast}(\bz_i,\bo_i) \right)\otimes\phi_{\gamma_o}(\bo_i)\right) \otimes \left(\left(\hat{F}_\xi(\bz_i,\bo_i) - F_{\ast}(\bz_i,\bo_i) \right)\otimes\phi_{\gamma_o}(\bo_i)\right) \right\| \nonumber \\
    &+ 2\left\|\frac{1}{n}\sum_{i=1}^{n}\left(\left(\hat{F}_\xi(\bz_i,\bo_i) - F_{\ast}(\bz_i,\bo_i) \right)\otimes \phi_{\gamma_o}(\bo_i)\right)\otimes (F_{\ast}(\bz_i,\bo_i)\otimes \phi_{\gamma_o}(\bo_i))\right\|\nonumber\\
    &\leq \frac{1}{n}\sum_{i=1}^{n} \left\| \left(\left(\hat{F}_\xi(\bz_i,\bo_i) - F_{\ast}(\bz_i,\bo_i) \right)\otimes \phi_{\gamma_o}(\bo_i)\right) \otimes \left(\left(\hat{F}_\xi(\bz_i,\bo_i) - F_{\ast}(\bz_i,\bo_i) \right)\otimes\phi_{\gamma_o}(\bo_i)\right) \right\| \nonumber \\
    &+ 2\frac{1}{n}\sum_{i=1}^{n} \left\|\left(\left(\hat{F}_\xi(\bz_i,\bo_i) - F_{\ast}(\bz_i,\bo_i) \right)\otimes \phi_{\gamma_o}(\bo_i)\right)\otimes (F_{\ast}(\bz_i,\bo_i)\otimes \phi_{\gamma_o}(\bo_i))\right\|\nonumber  \\
    &\leq \frac{1}{n}\sum_{i=1}^{n}\left\|\hat{F}_\xi (\bz_i,\bo_i) -F_{\ast}(\bz_i,\bo_i) \right\|_{\calH_{X,\gamma_x}}^2  + \frac{2}{n}\sum_{i=1}^{n}\left\|\hat{F}_\xi (\bz_i,\bo_i) -F_{\ast}(\bz_i,\bo_i)\right\|_{\calH_{X,\gamma_x}},
\end{align*}
\end{small}
where in the last step we use $\|a\otimes b\|_{H_1\otimes H_2} = \|a\|_{H_1}\|b\|_{H_2}$ for $a\in H_1$, $b\in H_2$ for Hilbert spaces  $H_1,H_2$~\citep[Eq. (3)]{gretton2005measuring}, $\|\phi_{\gamma_o}(\bo_i) \otimes \phi_{\gamma_o}(\bo_i) \| \leq 1$ and $\|F_{\ast}(\bz_i,\bo_i)\|_{\calH_{X,\gamma_x}} \leq 1$. 
Note that the last line is exactly analyzed in the proof of \citet[Lemma 7]{meunier2024nonparametricinstrumentalregressionkernel}. Their analysis employs a Hoeffding's concentration bounds with respect to Stage 2 samples $\calD_2$, and \emph{conditioned} on Stage 1 samples $\calD_1$ to show that, conditioned on $\calD_1$, 
\begin{align}\label{eq:bound_C}
    (C) \leq J_0 \left( \sqrt{\frac{\tau}{n}}\left\|F_{\ast}-\hat{F}_{\xi}\right\|_{\calG}+\left\|F_*-[\hat{F}_{\xi}] \right\|_{L^2\left(\calZ\times\calO; \mathcal{H}_{X,\gamma_x}\right)}  \right)
\end{align}
with $P^n$-probability $\geq 1 - 4e^{-\tau}$, under the assumptions that 
\begin{align}\label{eq:F_ast_F_xi_small_1}
    \left\|F_\ast -[\hat{F}_{\xi}] \right\|_{L^2\left(\calZ\times\calO; \mathcal{H}_{X,\gamma_x}\right)} \vee \left\|F_\ast -\hat{F}_{\xi}\right\|_{\calG} \leq 1.
\end{align} 
We note that the independence of $\calD_1$ and $\calD_2$ is implicitly used in this step to ensure that $\calD_2$ remains i.i.d. after conditioning on $\calD_1$. $J_0$ is a constant independent of $n,\tilde{n}$.
Denote as $\mathfrak{D}$ the $P^{\tilde{n}}$-probabilistic event that the bounds in \Cref{prop:cme_rate} hold (which has $P^{\tilde{n}}$-probability $\geq 1 - 4e^{-\tau}$). Under this event $\mathfrak{D}$ along with the fact that $\tilde{n}$ satisfies Eq. \eqref{eq:m_sufficient_large}, Eq.~\eqref{eq:F_ast_F_xi_small_1} is satisfied so Eq.~\eqref{eq:bound_C} holds. Additionally, from Eq.~\eqref{eq:bound_C}, a sufficient condition for $(C)\leq \frac{1}{6n}$ is given by $
\stageonesamples \geq (6 J J_0 \tau^{\frac{3}{2}} n)^{2\frac{\meffective+\deffective/2+\zeta}{\meffective}} \vee (36 J^2 J_0^2 \tau^2 n)^{\frac{\meffective + \deffective/2 + \zeta}{\meffective - 1}}$, which is satisfied since $\tilde{n}$ satisfies Eq. \eqref{eq:m_sufficient_large}.

\underline{Upper bound for term $(B)$}. By \citet[Proposition 7]{rudi2015less}, we have $
(B)\leq (1-t)^{-\frac{1}{2}}$, where
\begin{small}
\begin{align*}
    t&:=\left\|\left(\hat{C}_{FO} + \lambda\right)^{-\frac{1}{2}}\left(\hat{C}_{FO} - \hat{C}_{\hat{F}O}\right)\left(\hat{C}_{FO} + \lambda \right)^{-\frac{1}{2}}\right\| \leq \lambda^{-1}\left\|\hat{C}_{FO} - \hat{C}_{\hat{F}O}\right\| = n \cdot (C) \leq \frac{1}{6}, 
\end{align*}
\end{small}
from where we have $(B)\leq \frac{6}{5}$. Under the event $\mathfrak{D}$, the upper bounds $(A)\leq 2$, $(B)\leq 6/5$, $(C)\leq \frac{1}{6}$ hold simultaneously with $P^{n}$-probability $\geq 1 - 6e^{-\tau}$. Since $\mathfrak{D}$ holds with $P^{\tilde{n}}$-probability $\geq 1 - 4e^{-\tau}$, by independence of Stage 1 and Stage 2 samples (which is a consequence of the sample splitting strategy), we have that the above upper bounds hold simultaneously with $P^{n+\stageonesamples}$-probability $\geq \left(1 - 6e^{-\tau}\right)\left(1 - 4e^{-\tau}\right) \geq 1 - 10e^{-\tau}$.
\end{proof}

\begin{lem}\label{lem:thm11_meunier}
Fix $\tau\geq 1$. Suppose that the assumptions of \Cref{prop:cme_rate} hold. Let $\lambda = \frac{1}{n}$. We assume that $n\geq 1$ satisfies $n\geq A_{\lambda,\tau}$ with $A_{\lambda,\tau}$ defined in Eq. \eqref{eq:FS_thm_16_const} and $\stageonesamples \geq 1$ satisfies Eq.~\eqref{eq:m_sufficient_large}. Then, with $P^{n+\stageonesamples}$-probability $\geq 1-18e^{-\tau}$
    \begin{align*}
        S_{-1} \leq c \tau\sqrt{n}\left(\frac{\left\|\hat{F}_{\xi} -F_\ast \right\|_{\calG}}{\sqrt{n}}+\left\| \left[\hat{F}_{\xi}\right] - F_\ast \right\|_{L^2\left(\calZ\times\calO; \mathcal{H}_{X,\gamma_x}\right)}\right)\left\|\bar{h}_\lambda\right\|_{\mathcal{H}_{FO}} .
    \end{align*}
    where $c$ is a constant independent of $\stageonesamples, n$. Under the same probabilistic event, the bounds in \Cref{prop:cme_rate} hold. 
\end{lem}
\begin{proof}
    We omit the proof since it is similar to \citet[Theorem 11]{meunier2024nonparametricinstrumentalregressionkernel}, with adaptations similar to those in the proof of \Cref{lem:thm12_meunier}. 
\end{proof}

\begin{lem}\label{lem:tensor_evd}
Suppose the eigenvalues of the operator $\Sigma_1$ satisfy $\mu_{1, i} \asymp i^{-1/p_1}$ and the eigenvalues of the operator $\Sigma_2$ satisfy $\mu_{2, i} \asymp i^{-1/p_2}$. Then the eigenvalues of their tensor product $\Sigma_1 \otimes \Sigma_2$ satisfy, for any $\zeta > 0$,
\begin{align*}
    \lambda_{i}(\Sigma_1 \otimes \Sigma_2) \leq i^{- 1/ p_1 \wedge 1/ p_2  + \zeta}.
\end{align*}
\end{lem}
\begin{proof}
Suppose $\left( \mu_{1, i} \right)_{i \in \N^+}$ (resp. $\left( \mu_{2, j} \right)_{j \in \N^+}$) are the eigenvalues of $\Sigma_1$ (resp. $\Sigma_2$) with corresponding eigenfunctions $\left( q_{1, i} \right)_{i \in \N^+}$ (resp. $\left( q_{2, j} \right)_{j \in \N^+}$). 
By definition of operator tensor product, we have
\begin{align*}
    \left( \Sigma_1 \otimes \Sigma_2 \right) (q_{1, i} \otimes q_{2, j} ) = \left( \Sigma_1  q_{1, i} \right) \otimes \left( \Sigma_2 q_{2, j} \right) = \mu_{1, i} \mu_{2, j} (q_{1, i} \otimes q_{2, j} ) .
\end{align*}
Therefore, for any $i \in \N^+, j \in \N^+$ we obtain that $\mu_{1, i} \cdot \mu_{2, j}$ is the eigenvalue of $\Sigma_1 \otimes_{op} \Sigma_2$ with corresponding eigenfunction $q_{1, i} \otimes q_{2, j}$. 
Therefore the eigenvalues of $\Sigma_1 \otimes_{op} \Sigma_2$ equal the tensor product of two sequences $(\mu_{1, i})_{i\in \N^+}$ and $(\mu_{2, j})_{j \in \N^+}$. 

Without loss of generality, we assume $p_1 \leq p_2$ so that $\mu_{1, i} \leq \mu_{2, i}$ for $i$ large enough. 
Denote $(\sigma_n)_{n\in \N^+}$ as the tensor product of two sequences $(\mu_{2, i})_{i\in \N^+}$ and $(\mu_{2, j})_{j \in \N^+}$ rearranged in a non-increasing order. Hence, we have $\lambda_{n}(\Sigma_1 \otimes \Sigma_2) \leq \sigma_n$.
From \cite{krieg2018tensor}, we have 
\begin{align*}
    \sigma_n \leq n^{- 1/p_2} (\log n)^{1/p_2} \leq n^{- 1/p_2 + \zeta},
\end{align*}
for any $\zeta > 0$.
It concludes the proof.
\end{proof}

\begin{lem}
\label{lem:hang_prop_3}
Let $\calX = [0,1]^d$. Let $f : \R^{d} \to \R$ and $f \in B_{2,q}^{\bs}(\calX) \cap L^1(\R^d)$ with $\bs = [s_1, \ldots, s_d]^\top$. Recall $K_{\bgamma}$ defined in Eq. \eqref{eq:K_function} with $\bgamma = [\gamma_1, \ldots, \gamma_d]^\top \in (0,1]^d$.  Then, we have 
\begin{align*}
    \left\|\left[K_{\bgamma}\ast f\right] - f\right\|_{L^2(\calX)} \leq C_0 |f|_{B_{2,q}^{\bs}(\calX)} \cdot \max\{\gamma_1^{s_1}, \ldots, \gamma_d^{s_d}\} 
\end{align*}
for some constant $C_0$ that only depends on $d, \bs$. 
\end{lem}
\begin{proof}
Notice that
\begin{align*}
\begin{aligned}
&\quad K_{\bgamma} * f(\bx) =\int_{\R^d} \sum_{j=1}^r\binom{r}{j}(-1)^{1-j} \frac{1}{j^d}\left(\frac{2}{\pi}\right)^{\frac{d}{2}} \left( \prod_{i=1}^{d} \frac{1}{\gamma_i} \right) \exp\left( -2 \sum_{i=1}^d \frac{(x_i - t_i)^2}{j^2 \gamma_i^2} \right) f(\bt) \;\mathrm{d}\bt \\
& =\sum_{j=1}^r\binom{r}{j}(-1)^{1-j} \frac{1}{j^d} \left(\frac{2}{\pi}\right)^{\frac{d}{2}} \left( \prod_{i=1}^{d} \frac{1}{\gamma_i} \right) \int_{\R^d} \exp\left( -2 \sum_{i=1}^d \frac{(x_i - t_i)^2}{j^2 \gamma_i^2} \right) f(\bt) \;\mathrm{d}\bt \\
& =\sum_{j=1}^r\binom{r}{j}(-1)^{1-j} \frac{1}{j^d}\left(\frac{2}{\pi}\right)^{\frac{d}{2}} \left( \prod_{i=1}^{d} \frac{1}{\gamma_i} \right) \int_{\R^d} \exp\left( -2 \sum_{i=1}^d \frac{h_i^2}{\gamma_i^2} \right)  f(\bx + j \bh) j^d \;\mathrm{d}\bh \\
& =\sum_{j=1}^r\binom{r}{j}(-1)^{1-j} \left(\frac{2}{\pi}\right)^{\frac{d}{2}} \int_{\R^d} \exp\left( -2 \| \bh\|^2 \right)  f(\bx + j \bh \odot \bgamma) \;\mathrm{d}\bh \\
& = \int_{\R^d}\left(\frac{2}{\pi}\right)^{\frac{d}{2}} \exp\left( -2 \| \bh\|^2 \right) \left(\sum_{j=1}^r \binom{r}{j}(-1)^{1-j} f( \x + j \bh \odot \bgamma )\right) \;\mathrm{d}\bh .
\end{aligned}
\end{align*}
Next, since $\int_{\R^d} \left(\frac{2}{\pi}\right)^{\frac{d}{2}} \exp\left( -2 \| \bh\|^2 \right) \;\mathrm{d}\bh $ = 1, we have that
\begin{align*}
    &\quad \left\| [K_{\bgamma} * f] - f\right\|_{L^2(\calX)}^2 \\
    &= \int_{\calX} \left( \int_{\R^d}\left(\frac{2}{\pi}\right)^{\frac{d}{2}} \exp\left( -2 \| \bh\|^2 \right) \left(\sum_{j=0}^r \binom{r}{j}(-1)^{1-j} f( \x + j \bh \odot \bgamma )\right) \;\mathrm{d}\bh \right)^2 \;\mathrm{d}\bx \\
    &= \int_\calX \left( \int_{\R^d}\left(\frac{2}{\pi}\right)^{\frac{d}{2}} \exp\left( -2 \| \bh\|^2 \right) \left(\sum_{j=0}^r \binom{r}{j}(-1)^{2r + 1 - j} f( \x + j \bh \odot \bgamma )\right) \;\mathrm{d}\bh \right)^2 \;\mathrm{d}\bx \\
    &= \int_\calX \left( \int_{\R^d} (-1)^{r+1} \left(\frac{2}{\pi}\right)^{\frac{d}{2}} \exp\left( -2 \| \bh\|^2 \right) \Delta_{\bh \odot \bgamma}^r f(\bx) \;\mathrm{d}\bh \right)^2 \;\mathrm{d}\bx .
\end{align*}
Where the last step follows from the definition of modulus of smoothness in Eq. \eqref{eq:module_smoothness}.
Then, from Cauchy-Schwarz inequality, we have
\begin{small}
\begin{align}\label{eq:K_gamma_f_f_diff_Delta}
    &\leq  \int_\calX \left( \int_{\R^d} \left(\frac{2}{\pi}\right)^{\frac{d}{2}} \exp\left( -2 \| \bh\|^2 \right) \;\mathrm{d}\bh\right) \left( \int_{\R^d} \left(\frac{2}{\pi}\right)^{\frac{d}{2}} \exp\left( -2 \| \bh\|^2 \right) \left| \Delta_{\bh \odot \bgamma}^r f(\bx) \right|^2 \;\mathrm{d}\bh \right) \;\mathrm{d}\bx \nonumber \\
    &= \int_\calX \int_{\R^d} \left(\frac{2}{\pi}\right)^{\frac{d}{2}} \exp\left( -2 \| \bh\|^2 \right) \left| \Delta_{\bh \odot \bgamma}^r f(\bx) \right|^2 \;\mathrm{d}\bh  \;\mathrm{d}\bx \nonumber \\
    &= \int_{\R^d} \left(\frac{2}{\pi}\right)^{\frac{d}{2}} \exp\left( -2 \| \bh\|^2 \right) \left\| \Delta_{\bh \odot \bgamma}^r f \right\|_{L^2(\calX)}^2  \;\mathrm{d}\bh .
\end{align}
\end{small}
Define $N := \min\{\gamma_1^{-s_1}, \ldots, \gamma_d^{-s_d}\}$.
Since $f \in B_{2,q}^{\bs}(\calX)$, 
we have
\begin{align*}
    &|f|_{B_{2,q}^{\bs}(\calX)} \geq |f|_{B_{2,\infty}^{\bs}(\calX)} := \sup_{t} \left( t^{-1} \omega_{r, 2} \left(f, t^{\frac{1}{s_1}}, \ldots, t^{\frac{1}{s_d}}, \calX \right) \right) \\
    &\geq \left( \left( \sum_{i=1}^d h_i^{s_i} \right) N^{-1} \right)^{-1} \omega_{r, 2} \left(f, \left( \sum_{i=1}^d h_i^{s_i} \right)^{\frac{1}{s_1}} N^{-\frac{1}{s_1}}, \ldots, \left( \sum_{i=1}^d h_i^{s_i} \right)^{\frac{1}{s_d}} N^{-\frac{1}{s_d}}, \calX \right) \\
    &\geq N \left( \sum_{i=1}^d h_i^{s_i} \right)^{-1} \omega_{r, 2} \left(f, h_1 \gamma_1, \ldots, h_d \gamma_d, \calX \right) \\
    &\geq N \left( \sum_{i=1}^d h_i^{s_i} \right)^{-1} \left\| \Delta_{\bh \odot \bgamma}^r f \right\|_{L^2(\calX)} .
\end{align*}
As a result, we have $ \| \Delta_{\bh \odot \bgamma}^r f\|_{L^2(\calX)} \leq |f|_{B_{2,q}^{\bs}(\calX)} N^{-1} (\sum_{i=1}^d h_i^{s_i})$.
Plugging the above back to Eq.~\eqref{eq:K_gamma_f_f_diff_Delta}, we obtain
\begin{align*}
    \left\| K_{\bgamma} * f - f\right\|_{L^2(\calX)}^2 &\leq \int_{\R^d} \left(\frac{2}{\pi}\right)^{\frac{d}{2}} \exp\left( -2 \| \bh\|^2 \right) \left( |f|_{B_{2,q}^{\bs}(\calX)} N^{-1} \left(\sum_{i=1}^d h_i^{s_i}\right) \right)^2 \;\mathrm{d}\bh \\
    &\leq |f|_{B_{2,q}^{\bs}(\calX)}^2 N^{-2} d \int_{\R^d} \left(\frac{2}{\pi}\right)^{\frac{d}{2}} \exp\left( -2 \| \bh\|^2 \right) \left(\sum_{i=1}^d h_i^{2 s_i}\right) \;\mathrm{d}\bh .
\end{align*}
Notice that
\begin{align*}
    C_0^2 := d \int_{\R^d} \left(\frac{2}{\pi}\right)^{\frac{d}{2}} \exp\left( -2 \| \bh\|^2 \right) \left(\sum_{i=1}^d h_i^{2 s_i}\right) \;\mathrm{d}\bh = d \sum_{i=1}^d \int_{\R} \left(\frac{2}{\pi}\right)^{\frac{1}{2}} \exp\left( -2 h_i^2 \right) h_i^{2 s_i }dh_i,
\end{align*}
which is the sum of $s_i$-th moment of a one dimensional Gaussian distribution $\calN(0, 1/4)$ from $i=1$ to $i=d$. 
Finally, we obtain
\begin{align*}
    \left\| K_{\bgamma} * f - f\right\|_{L^2(\calX)}^2 &\leq C_0^2 |f|_{B_{2,q}^{\bs}(\calX)}^2 N^{-2} = C_0^2 |f|_{B_{2,q}^{\bs}(\calX)}^2 \left( \max\{\gamma_1^{s_1}, \ldots, \gamma_d^{s_d}\} \right)^2 .
\end{align*}
\end{proof}

\begin{lem}\label{lem:hang_prop_4}
Given a function $f:\R^{d}\to\R$, we have
\begin{align*}
    \left\|K_{\bgamma}\ast f \right\|_{\calH_{X,\bgamma}} \leq \left(\prod_{i=1}^{d}\gamma_i^{-\frac{1}{2}}\right)\pi^{-\frac{d}{4}}(2^r - 1)\|f\|_{L^2(\R^d)}.
\end{align*}
\end{lem}
\begin{proof}
Following the same arguments in the proof of proposition 4 in \cite{hang2021optimal}, we have $ K_{\bgamma} \ast f \in \calH_{X,\bgamma}$.
For a function $f:\R^{d}\to\R$, we define $\tau_{\bgamma}f(\bx) = f(\bgamma\odot \bx)$. We have, for all $\bx \in \R^d$, 
\begin{align*}
    \tau_{\bgamma}\left(K_{\bgamma} \ast f\right)(\bx) &= \int_{\R^d}K_{\bgamma} \left( \bgamma \odot \bx - \bt \right) f(\bt) \;\mathrm{d}\bt \\
    &=  \int_{\R^d}K_{\mathbf{1}}\left(\frac{\bgamma \odot \bx - \bt}{\bgamma}\right)\left(\prod_{i=1}^{d}\frac{1}{\gamma_i} \right)f(\bt)\;\mathrm{d}\bt\\
    &= \int_{\R^d}K_{\mathbf{1}}\left(\bx - \frac{\bt}{\bgamma}\right)\left(\prod_{i=1}^{d}\frac{1}{\gamma_i} \right)f\left( \bt \right)\;\mathrm{d}\bt\\
    &= \int_{\R^d}K_{\mathbf{1}}\left(\bx - \bt\right)f\left(\bt \odot\bgamma\right)\;\mathrm{d}\bt\\
    &= (K_{\mathbf{1}}\ast (\tau_{\bgamma}f))(\bx) .
\end{align*}
Proposition 4.37 of \cite{steinwart2008support} can be generalized to anisotropic case which shows that for any  
$f \in \calH_{X,\bgamma}$, $\tau_{\bgamma} f \in \calH_{X, \mathbf{1}}$ and 
$\tau_{\bgamma}: \calH_{X,\bgamma} \to \calH_{X, \mathbf{1}}$ is an isometric isomorphism.  Hence we have
\begin{align*}
    \left\|K_{\bgamma}\ast f \right\|_{\calH_{X,\bgamma}} &= \left\|\tau_{\bgamma}\left(K_{\bgamma}\ast f\right) \right\|_{\calH_{X, \mathbf{1}}}\\
    &= \left\|K_{\mathbf{1}}\ast (\tau_{\bgamma}f)\right\|_{\calH_{X, \mathbf{1}}}\\
    &\leq \pi^{-\frac{d}{4}}(2^r - 1) \|\tau_{\bgamma}f\|_{L^2(\R^d)}\\
    &= \left(\prod_{i=1}^{d}\gamma_i^{-\frac{1}{2}}\right)\pi^{-\frac{d}{4}}(2^r - 1)\|f\|_{L^2(\R^d)},
\end{align*}
where the second last inequality follows from \cite{eberts2013optimal}[Theorem 2.3] setting $\gamma=1$.
\end{proof}
\begin{cor}
\label{cor:hang_prop_4_obs_cov}
    Let $f:\R^{d_x+d_o}\to \R$ and 
    \begin{align*}
        \gamma_x = n^{-\frac{\frac{1}{d_x}}{1+(2+\frac{d_o}{s_o})(\frac{s_x}{d_x}+\eta_1)}},\quad \gamma_o = n^{-\frac{\frac{1}{s_o}(\frac{s_x}{d_x}+\eta_1)}{1+(2+\frac{d_o}{s_o})(\frac{s_x}{d_x}+\eta_1)}}.
    \end{align*}
    Recall $K_{\gamma_x}:\R^{d_x} \to \R$ and $K_{\gamma_o}:\R^{d_o} \to \R$ defined in Eq.~\eqref{eq:K_function}, then we have
    \begin{align*}
        \|K_{\gamma_x} \ast K_{\gamma_o} \ast f\|_{\calH_{\gamma_x,\gamma_o}} \leq C_{1} 
        n^{\frac{1}{2} \frac{1+\frac{d_o}{s_o}(\frac{s_x}{d_x}+\eta_1)}{1+(2+\frac{d_o}{s_o})(\frac{s_x}{d_x}+\eta_1)}}
        \|f\|_{L^2(\R^{d_x+d_o})},
    \end{align*}
    for some constant $C_1$ that only depends on $d_x,d_o,s_x,s_o$.
\end{cor}
\begin{proof}
    The proof is a direct application of \Cref{lem:hang_prop_4}.
\end{proof}

\begin{lem}\label{lem:mixed_convolution}
    Let $\gamma_x = n^{-\frac{\frac{1}{d_x}}{1+(2+\frac{d_o}{s_o})(\frac{s_x}{d_x}+\eta_1)}}$, $\gamma_o = n^{-\frac{\frac{1}{s_o}(\frac{s_x}{d_x}+\eta_1)}{1+(2+\frac{d_o}{s_o})(\frac{s_x}{d_x}+\eta_1)}}$. 
    Let $\iota_{x,\gamma_x^{-1}} : \R^{d_x} \to \R$ be an indicator function over $\{\bomega_x: \|\bomega_x\|_2 \leq \gamma_x^{-1} \}$ and $K_{\gamma_o}: \R^{d_x} \to \R$ be as defined in Eq.~\eqref{eq:K_function}.
    Then we have
    \begin{align*}
        \|\calF^{-1}[\iota_{x,\gamma_x^{-1}}] \ast K_{\gamma_o} \ast f\|_{\calH_{\gamma_x,\gamma_o}} \lesssim 2^{r} n^{\frac{1}{2} \frac{1+\frac{d_o}{s_o}(\frac{s_x}{d_x}+\eta_1)}{1+(2+\frac{d_o}{s_o})(\frac{s_x}{d_x}+\eta_1)}} \|f\|_{L^2(\R^{d_x+d_o})} .
    \end{align*}
\end{lem}

\begin{proof}
In contrast to the notation used in the rest of the document (with the exception of \Cref{lem:l2toprojected_illp}), in this proof, we use $\calF$ to denote the unitary Fourier operator on $L^2(\mathbb{R}^{d_x+d_o})$, we use $\calF_x$ to denote the unitary Fourier operator on $L^2(\mathbb{R}^{d_x})$, and we use $\calF_o$ to denote the unitary Fourier operator on $L^2(\mathbb{R}^{d_o})$. We let $\otimes: L^2(\mathbb{R}^{d_x})\times L^2(\mathbb{R}^{d_o})\to L^2(\mathbb{R}^{d_x+d_o})$ denotes the tensor product mapping, where $(f\otimes g)(\bx,\bo) = f(\bx)g(\bo)$ for all $(\bx,\bo)\in \mathbb{R}^{d_x+d_o}$. By considering a convergent sequence to $f_1$ in $L^1(\mathbb{R}^{d_x})\cap L^2(\mathbb{R}^{d_x})$, and similarly for $f_2$, we can show that 
\begin{align*}
    \calF[f\otimes g] = \calF_x[f]\otimes \calF_o[g]
\end{align*}
for $f\in L^2(\mathbb{R}^{d_x})$ and $g\in L^2(\mathbb{R}^{d_o})$. At the start of \Cref{sec:proof_upper}, we proved that for $f\in L^1(\mathbb{R}^{d_x+d_o})$, $f\ast \calF^{-1}[\iota_{x,\gamma_x^{-1}}] \in L^2(\mathbb{R}^{d_x+d_o})$. Since $\|K_{\gamma_o}\|_{L^1(\mathbb{R}^{d_o})}= 1$ \citep[Section 4.1.2]{Giné_Nickl_2015}, we have
\begin{align*}
    \|(f\ast \calF^{-1}[\iota_{x,\gamma_x^{-1}}])\ast K_{\gamma_o}\|_{L^2(\mathbb{R}^{d_x+d_o})}^2 &\stackrel{(a)}{=} \int_{\mathbb{R}^{d_x}}\|(f\ast \calF^{-1}[\iota_{x,\gamma_x^{-1}}])(\bx,\cdot) \ast K_{\gamma_o}\|^2_{L^2(\mathbb{R}^{d_o})}\;\mathrm{d}\bx\\
    &\stackrel{(b)}{\leq} \int_{\mathbb{R}^{d_x}} \|(f\ast \calF^{-1}[\iota_{x,\gamma_x^{-1}}])(\bx,\cdot)\|^2_{L^2(\mathbb{R}^{d_o})} \|K_{\gamma_o}\|^2_{L^1(\mathbb{R}^{d_o})}\\
    & = \|f\ast \calF^{-1}[\iota_{x,\gamma_x^{-1}}]\|^2_{L^2(\mathbb{R}^{d_x+d_o})}.
\end{align*}
where in $(a)$, $(f\ast \calF^{-1}[\iota_{x,\gamma_x^{-1}}])(\bx,\cdot)$ is a slice function of any representative of $f\ast \calF^{-1}[\iota_{x,\gamma_x^{-1}}]\in L^2(\mathbb{R}^{d_x})$, and $(b)$ follows from Young's Convolution Inequality. Hence we can apply $\calF$ to $f\ast \calF^{-1}[\iota_{x,\gamma_x^{-1}}]\in L^2(\mathbb{R}^{d_x+d_o})$ and we find 
\begin{align*}
    \calF\left[\calF_x^{-1}[\iota_{x,\gamma_x^{-1}}] \ast K_{\gamma_o} \ast f\right] = (\iota_{x,\gamma_x^{-1}}\otimes \calF_o[K_{\gamma_o}]) \cdot \calF[f].
\end{align*}
For $\phi(\bx,\bo) = \exp(- \gamma_x^{-2} \|\bx\|_2^2) \cdot \exp(- \gamma_o^{-2} \|\bo\|_2^2)$, its Fourier transform is 
\begin{align*}
    \calF[\phi](\bomega_x,\bomega_o) = \pi^{d_x/2+d_o/2} \gamma_x^{d_x} \exp\left(-\frac{1}{4} \gamma_x^{2} \|\bomega_x\|_2^2\right) \cdot \gamma_o^{d_o}\exp\left(-\frac{1}{4} \gamma_o^{2} \|\bomega_o\|_2^2\right).
\end{align*}
Hence, by definition of Gaussian RKHS norm~\citep[Theorem 10.12]{wendland2004scattered}, 
\begin{align*}
    &\quad \left\| \calF_x^{-1}[\iota_{x,\gamma_x^{-1}}]\ast K_{\gamma_o} \ast f\right\|_{\calH_{\gamma_x,\gamma_o}}^2 \\
    &= \int_{\R^{d_x+d_o}} \frac{\left(\iota_{x,\gamma_x^{-1}}(\bomega_x) \cdot \calF[f](\bomega_x,\bomega_o) \cdot \calF_o[K_{\gamma_o}](\bomega_o)\right)^2}{\pi^{d_x/2+d_o/2} \gamma_x^{d_x} \exp\left(-\frac{1}{4} \gamma_x^{2} \|\bomega_x\|_2^2\right) \cdot \gamma_o^{d_o}\exp(-\frac{1}{4} \gamma_o^{2} \|\bomega_o\|_2^2)} \;\mathrm{d}\bomega_x \;\mathrm{d}\bomega_o \\
    &\lesssim \gamma_x^{-d_x} \gamma_o^{-d_o} \int_{\R^{d_o}}\int_{\{\bomega_x: \|\bomega_x\|_2 \leq \gamma_x^{-1} \}} \frac{\left( \calF[f](\bomega_x,\bomega_o) \cdot \calF_o[K_{\gamma_o}](\bomega_o)\right)^2}{\exp\left(-\frac{1}{4} \gamma_x^{2} \|\bomega_x\|_2^2\right) \cdot \exp(-\frac{1}{4} \gamma_o^{2} \|\bomega_o\|_2^2)} \;\mathrm{d}\bomega_x \;\mathrm{d}\bomega_o \\
    &\lesssim \gamma_x^{-d_x} \gamma_o^{-d_o} \int_{\R^{d_x+d_o}} \exp\left(\frac{1}{4} \gamma_o^{2} \|\bomega_o\|_2^2\right)\left(\calF[f](\bomega_x,\bomega_o) \cdot \calF_o[K_{\gamma_o}](\bomega_o)\right)^2 \;\mathrm{d}\bomega_x \;\mathrm{d}\bomega_o .
\end{align*}
Since $K_{\gamma_o}(\bo) :=  \sum_{j=1}^r\binom{r}{j}(-1)^{1-j} \frac{1}{j^{d_o}\gamma_o^{d_o}}\left(\frac{2}{\pi}\right)^{\frac{d}{2}}  \exp( -2 j^{-2} \gamma_o^{-2} \|\bo\|_2^2)$, so we have 
\begin{align*}
    \calF_o[K_{\gamma_o}](\bomega_o) = \sum_{j=1}^r\binom{r}{j}(-1)^{1-j} \exp\left(-\frac{1}{8} j^2 \gamma_o^2 \|\bomega_o\|_2^2\right) .
\end{align*}
Consequently, we have
\begin{align*}
    &\quad \left\| \calF^{-1}[\iota_{x,\gamma_x^{-1}}]\ast K_{\gamma_o} \ast f\right\|_{\calH_{\gamma_x,\gamma_o}}^2 \\
    &\lesssim \gamma_x^{-d_x} \gamma_o^{-d_o} \int_{\R^{d_x+d_o}} \frac{\left(\calF[f](\bomega_x,\bomega_o) \cdot \sum_{j=1}^r\binom{r}{j}(-1)^{1-j} \exp(-\frac{1}{8} j^2 \gamma_o^2 \|\bomega_o\|_2^2) \right)^2}{\exp(-\frac{1}{4} \gamma_o^{2} \|\bomega_o\|_2^2)} d\bomega_x d\bomega_o \\
    &\leq \gamma_x^{-d_x} \gamma_o^{-d_o} 2^{2r} \int_{\R^{d_x+d_o}} \left(\calF[f](\bomega_x,\bomega_o) \right)^2 d\bomega_x d\bomega_o \\
    &= \gamma_x^{-d_x} \gamma_o^{-d_o} 2^{2r} \| f\|_{L^2(\R^{d_x+d_o})}^2 .
\end{align*}
where the last equality follows by Plancherel's Theorem. 
\end{proof}

\begin{lem}
\label{lem:unit_ball_grkhs}
    Let $\bgamma = (\gamma_1,\dots,\gamma_d)\in (0,1)^{d}$. Let $k_{\bgamma}$ be the anisotropic Gaussian kernel on $\R^d$ and let $\calH_{X,\bgamma}$ be the RKHS associated with $k_{\bgamma}$. Then 
\begin{align*}
    \int_{\R^{d}} |\calF[f](\bomega)|^2\exp\left(\frac{\|\bomega\odot \bgamma\|_{2}^{2}}{4}\right) \;\mathrm{d}\bomega = 2^{d} \pi^{d/2} \left(\prod_{i=1}^{d} \gamma_i\right) \|f\|_{\calH_{X,\bgamma}}^2
\end{align*}
\end{lem}
\begin{proof}
We define $\phi_{X,\bgamma}(\bx) = \exp(-\sum_{j=1}^{d} \gamma_j^{-2} x_j^2)$.
Then
\begin{align*}
    \calF[\phi_{X,\bgamma}](\bomega) &= \int_{\R^{d}}\phi_{X,\bgamma}(\bx)\exp(-i\langle \bx, \bomega\rangle) \mathrm{d}\bx \\
    &= \left(\prod_{i=1}^{d}\gamma_i\right) \int_{\R^{d}}\phi_{X,1}(\bx)\exp\left(-i\left\langle \bx, \bomega \odot \bgamma\right\rangle\right)\mathrm{d}\bx\\
    &= \pi^{d/2} \left(\prod_{i=1}^{d}\gamma_i\right)\exp\left(-\frac{\|\bomega\odot \bgamma\|_{2}^{2}}{4}\right)
\end{align*}
where the last step follows from \cite{wendland2004scattered}[Theorem 5.20]. By \cite{kanagawa2018gaussianprocesseskernelmethods}[Theorem 2.4], we have
\begin{align*}
    \|f\|_{\calH_{X,\bgamma}}^2 &= (2\pi)^{-d} \int_{\R^d} \frac{|\calF[f](\bomega)|^2}{\calF[\phi_{X,\bgamma}](\bomega)} d \bomega \\
    &=(2\pi)^{-d} \pi^{d/2} \int_{\R^d} |\calF[f](\bomega)|^2 \left(\prod_{i=1}^{d}\gamma_i\right)^{-1}\exp\left(\frac{\|\bomega\odot \bgamma\|_{2}^{2}}{4}\right) \mathrm{d}\bomega\; \\
    &= 2^{-d} \pi^{-d/2} \left(\prod_{i=1}^{d} \gamma_i\right)^{-1} \int  |\calF[f](\bomega)|^2 \exp\left(\frac{\|\bomega\odot \bgamma\|_{2}^{2}}{4}\right) \mathrm{d}\bomega
\end{align*}
whence the Lemma follows. 
\end{proof}

\begin{lem}\label{lem:l2toprojected_illp}
Suppose that \Cref{ass:T_frequency_ill_posedness} in the main text holds, and suppose $\gamma_x^{\delta_x} = \gamma_o^{\delta_o}$ for some positive constant $\delta_x, \delta_o$. 
Suppose $f\in \calH_{\gamma_x,\gamma_o}$ with $\|f\|_{\calH_{\gamma_x,\gamma_o}}^2 \leq Cn$ for some constant $C>0$. Then,
\begin{align}
    \|f\|_{L^2(P_{XO})} &\lesssim \gamma_x^{-d_x\eta_0} (\log (\gamma_x^{-\delta_x}))^{\frac{d_x\eta_0}{2}} \|Tf\|_{L^2(P_{ZO})} \nonumber \\
    &\qquad + \gamma_x^{-d_x\eta_0} (\log (\gamma_x^{-\delta_x}))^{\frac{d_x\eta_0}{2}}\gamma_x^{\frac{1}{8} \delta_x} (n \gamma_x^{d_x}\gamma_o^{d_o})^{\frac{1}{2}} \label{eq:term_lemma_D_18} .
\end{align}
In particular, for 
\begin{align*}
    \gamma_x = n^{-\frac{\frac{1}{d_x}}{1+(2+\frac{d_o}{s_o})(\frac{s_x}{d_x}+\eta_1)}},\quad \gamma_o = n^{-\frac{\frac{1}{s_o}(\frac{s_x}{d_x}+\eta_1)}{1+(2+\frac{d_o}{s_o})(\frac{s_x}{d_x}+\eta_1)}}. 
\end{align*}
and $\frac{1}{8} \delta_x = 2s_x + d_x\eta_1 + d_x \eta_0$, $\delta_x = \delta_o \frac{d_x}{s_o}(\frac{s_x}{d_x}+\eta_1)$. 
Then, we have
\begin{align*}
    \|f\|_{L^2(P_{XO})} \lesssim \gamma_x^{-d_x\eta_0} (\log n)^{\frac{d_x\eta_0}{2}} \|Tf\|_{L^2(P_{ZO})} + n^{-\frac{\frac{s_x}{d_x} }{1+(2 + \frac{d_o}{s_o})(\frac{s_x}{d_x}+\eta_1)}}.
\end{align*}
\end{lem}
\begin{proof}
Define the following sets 
\begin{align*}
    I_{\bx} &:= \left\{\bomega_{x}\in\R^{d_x}: \|\bomega_x\|_2^2 \gamma_x^{2} (\log (\gamma_x^{-\delta_x}))^{-1} \leq 1 \right\}\\ \quad I_{\bo} &:= \left\{\bomega_{o}\in\R^{d_o}: \|\bomega_{o}\|^2_{2}\gamma_o^{2} (\log (\gamma_o^{-\delta_o}))^{-1} \leq 1\right\}\\ I &:= I_{\bx} \times I_{\bo}\\
    \tilde{I} &:= \left\{(\bomega_x,\bomega_o)\in\R^{d_x+d_o}: \|\bomega_x\|_2^2 \gamma_x^{2} (\log (\gamma_x^{-\delta_x}))^{-1} + \|\bomega_{o}\|^2\gamma_o^{2} (\log (\gamma_o^{-\delta_o}))^{-1} \leq 1 \right\} .
\end{align*}
Note that $\tilde{I}\subset I$. Let $\iota_{x} : \R^{d_x}\to \{0,1\}$ denote the indicator function on $I_{\bx}$, let $\iota_{o} : \R^{d_o}\to \{0,1\}$ denote the indicator function on $I_{\bo}$, let $\iota:  \R^{d_x+d_o}\to \{0,1\}$ denote the indicator function on $I$, which satisfies $\iota = \iota_{x} \cdot \iota_{o}$. 
Thus we have
\begin{small}
\begin{align}
    &\quad \|f\|_{L^2(P_{XO})} \nonumber \\
    &\leq \left\|f\ast \calF^{-1}(\iota)\right\|_{L^2(P_{XO})} + \left\|f\ast \calF^{-1}(\iota) - f\right\|_{L^2(P_{XO})} \nonumber \\
    &\stackrel{(i)}{\leq} \gamma_x^{-d_x\eta_0} (\log (\gamma_x^{-\delta_x}))^{\frac{d_x\eta_0}{2}} \|T(f\ast \calF^{-1}(\iota))\|_{L^2(P_{ZO})} + \left\|f\ast \calF^{-1}(\iota) - f\right\|_{L^2(P_{XO})} \nonumber\\
    &\stackrel{(ii)}{\leq} \gamma_x^{-d_x\eta_0} (\log (\gamma_x^{-\delta_x}))^{\frac{d_x\eta_0}{2}} \|Tf\|_{L^2(P_{ZO})} \nonumber \\
    &\qquad + \left(1 + \gamma_x^{-d_x\eta_0} (\log (\gamma_x^{-\delta_x}))^{\frac{d_x\eta_0}{2}} \right)\left\|f\ast \calF^{-1}(\iota) - f\right\|_{L^2(P_{XO})} 
    \nonumber \\
    &\lesssim \gamma_x^{-d_x\eta_0} (\log (\gamma_x^{-\delta_x}))^{\frac{d_x\eta_0}{2}} \|Tf\|_{L^2(P_{ZO})}  + \gamma_x^{-d_x\eta_0} (\log (\gamma_x^{-\delta_x}))^{\frac{d_x\eta_0}{2}} \left\|f\ast \calF^{-1}(\iota) - f\right\|_{L^2(\calX\times\calO)} 
    \label{eq:intermediate_reverse_link}.
\end{align}
\end{small}
In the above derivations,
\begin{itemize}
    \item $(i)$ follows from \Cref{ass:T_frequency_ill_posedness}, and for $\calF_x,\calF_o$ defined in the proof of \Cref{lem:mixed_convolution},
\begin{align*}
    (\forall \bo\in \calO),\; \mathrm{supp}(\calF_x[(f\ast \calF^{-1}(\iota))(\cdot,\bo)]) = \mathrm{supp}(\calF_x[(f\ast \calF_o^{-1}[\iota_o])(\cdot, \bo)] \cdot  \iota_x) \subseteq I_{\bx},
\end{align*}
where we note that the relevant Fourier transforms exist since $f\in \calH_{\gamma_x,\gamma_o}\subset L^2(\mathbb{R}^{d_x+d_o})$. 
    \item $(ii)$ follows from a triangular inequality and a Jensen's inequality. 
\end{itemize}
Next, we bound $(\ast) =\|f\ast \calF^{-1}(\iota) - f\|_{L^2([0,1]^{d_x+d_o})}$. We have
\begin{align*}
    (\ast) &= \left\|f\ast \calF^{-1}(\iota) - f\right\|_{L^2([0,1]^{d_x+d_o})} \leq \left\|f\ast \calF^{-1}(\iota) - f\right\|_{L^2(\R^{d_x+d_o})}\\ &\stackrel{(a)}{=} \left\|\calF\left[f\ast \calF^{-1}(\iota) - f\right]\right\|_{L^2(\R^{d_x+d_o})}= \|\calF[f]\iota - \calF[f]\|_{L^2(\R^{d_x+d_o})}\\ &= \left(\int_{I^{c}} |\calF[f](\bomega)|^2\;\mathrm{d}\bomega\right)^{\frac{1}{2}} \stackrel{(b)}{\leq} \left(\int_{\tilde{I}^{c}} |\calF[f](\bomega)|^2\;\mathrm{d}\bomega\right)^{\frac{1}{2}} ,
\end{align*}
where $(a)$ holds by Plancherel's Theorem, $(b)$ holds by $\tilde{I}\subseteq I$. 
Recall that for $\ba, \mathbf{b}\in\R^d$, $\ba \odot \mathbf{b} := (a_1 b_1, \ldots, a_d b_d) \in \R^{d}$. For $\bomega=(\bomega_x, \bomega_o)$, we proceed from above to have,
\begin{small}
\begin{align*}
    &\leq \left(\int_{\tilde{I}^{c}} |\calF[f](\bomega)|^2 \exp\left(\frac{\|(\bomega_x \odot \bgamma_x, \bomega_o \odot \bgamma_o )\|^2_{2}}{4}\right)\exp\left(-\frac{\|(\bomega_x \odot \bgamma_x, \bomega_o \odot \bgamma_o ) \|^2_{2}}{4}\right)\;\mathrm{d}\bomega\right)^{\frac{1}{2}}\\
    &\leq \sup_{\bomega\in \tilde{I}^{c}}\exp\left(-\frac{\|(\bomega_x \odot \bgamma_x, \bomega_o \odot \bgamma_o )\|^2_{2}}{8}\right)\left(\int_{\R^{d_x+d_o}} |\calF[f](\bomega)|^2 \exp\left(\frac{\|(\bomega_x \odot \bgamma_x, \bomega_o \odot \bgamma_o ) \|^2_{2}}{4}\right)\;\mathrm{d}\bomega\right)^{\frac{1}{2}}\\
    &\stackrel{(c)}{\leq} \sup_{\bomega\in \tilde{I}^{c}}\exp\left(-\frac{\|(\bomega_x \odot \bgamma_x, \bomega_o \odot \bgamma_o ) \|^2_{2}}{8}\right)\left( \|f\|_{\calH_{XO,\bgamma_x,\bgamma_o}}^2 2^{-\frac{d_x+d_o}{2}} \gamma_x^{d_x}\gamma_o^{d_o}\right)^{\frac{1}{2}}\\
    &\lesssim \sup_{\|\bomega_x\|_2^2 \gamma_x^{2} (\log (\gamma_x^{-\delta_x}))^{-1} + \|\bomega_{o}\|^2\gamma_o^{2} (\log (\gamma_o^{-\delta_o}))^{-1} \geq 1} \exp\left(-\frac{ \gamma_x^2 \|\bomega_x\|_2^2 + \gamma_o^2 \|\bomega_o\|_2^2 }{8}\right)\left(n  \gamma_x^{d_x}\gamma_o^{d_o}\right)^{\frac{1}{2}} \\
    &\leq \exp\left(-\frac{1}{8} \log(\gamma_x^{-\delta_x} ) \right)\left(n \gamma_x^{d_x} \gamma_o^{d_o}\right)^{\frac{1}{2}} \\
    &\asymp \gamma_x^{\frac{1}{8} \delta_x} \left(n \gamma_x^{d_x}\gamma_o^{d_o} \right)^{1/2},
\end{align*}
\end{small}
where $(c)$ holds by Lemma~\ref{lem:unit_ball_grkhs}.
Now we have proved the first claim of this lemma. Next, we proceed to prove the second claim by plugging in the specific values of $\delta_x,\delta_o$ and $\gamma_x,\gamma_o$. 
Recall that $\frac{1}{8} \delta_x = 2s_x + d_x\eta_1 + d_x \eta_0 + d_x$, we obtain 
\begin{align*}
    (\ast) \lesssim  n^{-\frac{\frac{\delta_x}{8}\frac{1}{d_x}}{1+(2+\frac{d_o}{s_o})(\frac{s_x}{d_x}+\eta_1)}} \cdot n^{\frac{\frac{s_x}{d_x}+\eta_1}{1+(2+\frac{d_o}{s_o})(\frac{s_x}{d_x}+\eta_1)}} = n^{\frac{-\frac{s_x}{d_x} - \eta_0 - 1}{1+(2+\frac{d_o}{s_o})(\frac{s_x}{d_x}+\eta_1)}} .
\end{align*}
We plug it back to Eq.~\eqref{eq:intermediate_reverse_link} and obtain
\begin{align*}
    \|f\|_{L^2(P_{XO})} &\lesssim \gamma_x^{-d_x\eta_0} (\log (\gamma_x^{-\delta_x}))^{\frac{d_x\eta_0}{2}} \|Tf\|_{L^2(P_{ZO})} \\
    &\qquad + \gamma_x^{-d_x\eta_0} (\log (\gamma_x^{-\delta_x}))^{\frac{d_x\eta_0}{2}} n^{\frac{-\frac{s_x}{d_x} - \eta_0 - 1}{1+(2+\frac{d_o}{s_o})(\frac{s_x}{d_x}+\eta_1)}} \\
    &\lesssim \gamma_x^{-d_x\eta_0} (\log n)^{\frac{d_x\eta_0}{2}} \|Tf\|_{L^2(P_{ZO})} + n^{-\frac{\frac{s_x}{d_x} + 1}{1+(2 + \frac{d_o}{s_o})(\frac{s_x}{d_x}+\eta_1)}}(\log n)^{\frac{d_x\eta_0}{2}} \\
    &\lesssim \gamma_x^{-d_x\eta_0} (\log n)^{\frac{d_x\eta_0}{2}} \|Tf\|_{L^2(P_{ZO})} + n^{-\frac{\frac{s_x}{d_x} }{1+2(\frac{s_x}{d_x}+\eta_1)+ \frac{d_o}{s_o}(\frac{s_x}{d_x}+\eta_1)}}.
\end{align*}
So the proof is concluded.
\end{proof}

\begin{lem}
\label{lem: C_f_op_norm_lower_bound}
Suppose \Cref{assn: technical} hold, and let
\begin{align*}
    \gamma_x = n^{-\frac{\frac{1}{d_x}}{1+(2+\frac{d_o}{s_o})(\frac{s_x}{d_x}+\eta_1)}}, \quad \gamma_o = n^{-\frac{\frac{1}{s_o}(\frac{s_x}{d_x}+\eta_1)}{1+(2+\frac{d_o}{s_o})(\frac{s_x}{d_x}+\eta_1)}}.
\end{align*}
Then, for sufficiently large $n\geq 1$, we have 
\begin{align*}
    \|C_{FO}\| \geq a\left(\frac{\sqrt{\pi}}{4}\right)^{\frac{d_x+d_o}{2}} n^{-\frac{1}{2}\frac{1+\frac{d_o}{s_o}(\frac{s_x}{d_x}+\eta_1)}{1+(2+\frac{d_o}{s_o})(\frac{s_x}{d_x}+\eta_1)}} .
\end{align*}
\end{lem}
\begin{proof}
By definition of the operator norm, we have
\begin{align*}
    &\quad \|C_{FO}\|^2 \\
    &= \sup_{\|f\|_{\calH_{\gamma_x,\gamma_o}} = 1}\left\langle f, \E[(F_{\ast}(Z, O) \otimes \phi_{\gamma_o}(O))\otimes (F_{\ast}(Z, O) \otimes \phi_{\gamma_o}(O))]f\right\rangle_{\calH_{\gamma_x,\gamma_o}}\\ 
    =& \sup_{\|f\|_{\calH_{\gamma_x,\gamma_o}} = 1} \E \left[ \left(\E[f(X,O)|Z,O] \right)^2 \right]\\
    \geq & \sup_{\|f\|_{\calH_{\gamma_x,\gamma_o}} = 1} \left( \E_{ZO \sim P_{ZO}} \left[\E[f(X,O)|Z,O]\right] \right)^2\\
    =&  \sup_{\|f\|_{\calH_{\gamma_x,\gamma_o}} = 1} \E[f(X,O)]^2 \\
    =& \|\mu_{XO}\|_{\calH_{\gamma_x,\gamma_o}}^2 , 
\end{align*}
where $\mu_{XO} := \E_{P_{XO}}[\phi_{\gamma_x}(X) \otimes \phi_{\gamma_o}(O)]$. 
Notice that 
\begin{align*}
    &\|\mu_{XO} \|_{\calH_{\gamma_x,\gamma_o}}^2\\ &= \left\langle \iint_{\calO\times\calX} \phi_{\gamma_x}(\bx) \otimes \phi_{\gamma_o}(\bo) p_{XO}(\bx,\bo)\;\mathrm{d}\bx \;\mathrm{d}\bo, \right. \\
    &\qquad\qquad \left. \iint_{\calO\times\calX} \phi_{\gamma_x}(\bx') \otimes \phi_{\gamma_o}(\bo') p_{XO}(\bx',\bo')\;\mathrm{d}\bx'\;\mathrm{d}\bo'\right\rangle_{\calH_{\gamma_x,\gamma_o}}\\
    &= \iint_{\calO\times\calX} \iint_{\calO\times\calX} K_{\gamma_o}(\bo,\bo') K_{\gamma_x}(\bx,\bx')p_{XO}(\bx,\bo) p_{XO}(\bx', \bo')\;\mathrm{d}\bx\;\mathrm{d}\bx'\;\mathrm{d}\bo\;\mathrm{d}\bo'\\
    &\stackrel{(a)}{\geq} a^2\int_{[1/4,3/4]^{d_o}}\int_{[1/4,3/4]^{d_o}} K_{\gamma_o}(\bo,\bo')\;\mathrm{d}\bo\;\mathrm{d}\bo'\int_{[1/4,3/4]^{d_x}}\int_{[1/4,3/4]^{d_x}} K_{\gamma_x}(\bx,\bx') \;\mathrm{d}\bx\;\mathrm{d}\bx'\\
    &\stackrel{(b)} = a^2 \left(\gamma_x \sqrt{\pi} \left[\frac{\mathrm{erf}\left(\frac{1}{2\gamma_x}\right)}{2} + \gamma_x\frac{\exp\left(-\frac{1}{4\gamma_x^2}\right)-1}{\sqrt{\pi}}\right]\right)^{d_x} \\
    &\qquad\qquad \cdot \left(\gamma_o \sqrt{\pi} \left[\frac{\mathrm{erf}\left(\frac{1}{2\gamma_o}\right)}{2} + \gamma_o\frac{\exp\left(-\frac{1}{4\gamma_o^2}\right)-1}{\sqrt{\pi}}\right]\right)^{d_o},
\end{align*}
where $\mathrm{erf}(x) := \frac{1}{\sqrt{\pi}}\int^{x}_{-x}\exp\left(-t^2\right)\;\mathrm{d}t$
is the standard error function of normal distribution. Step $(a)$ holds by \Cref{assn: technical} and step $(b)$ holds by using \Cref{lem:gaussian_kernel_embedding}. We plug in $\gamma_x = n^{-\frac{\frac{1}{d_x}}{1+(2+\frac{d_o}{s_o})(\frac{s_x}{d_x}+\eta_1)}}$, $\gamma_o = n^{-\frac{\frac{1}{s_o}(\frac{s_x}{d_x}+\eta_1)}{1+(2+\frac{d_o}{s_o})(\frac{s_x}{d_x}+\eta_1)}}$ to obtain
\begin{align*}
    \|\mu_{XO}\|^2_{\calH_{\gamma_x,\gamma_o}} &\stackrel{(i)}{\geq} a^2\left(\sqrt{\pi}\gamma_x/4\right)^{d_x}\left(\sqrt{\pi}\gamma_o/4\right)^{d_o} = a^2 (\sqrt{\pi}/4)^{d_x+d_o} n^{-\frac{1+\frac{d_o}{s_o}(\frac{s_x}{d_x}+\eta_1)}{1+(2+\frac{d_o}{s_o})(\frac{s_x}{d_x}+\eta_1)}},
\end{align*}
where $(i)$ holds for sufficiently large $n\geq 1$ because $\mathrm{erf}(x)\geq \frac{1}{2}$ when $x\geq 1$. 
\end{proof}

\begin{lem}\label{lem:gaussian_kernel_embedding}
    Let $k_{\gamma} : \R^{d}\times \R^{d}\to\R$ be the Gaussian kernel with length scale $\gamma\in (0,1]$. Then we have,
    \begin{align*}
        \int_{[1/4,3/4]^{d}}\int_{[1/4,3/4]^{d}} k_{\gamma}(\bx, \bx') \;\mathrm{d}\bx \mathrm{d} \bx' = \left(\gamma \sqrt{\pi} \left[\frac{\mathrm{erf}\left(\frac{1}{2\gamma}\right)}{2} + \gamma\frac{\exp\left(-\frac{1}{4\gamma^2}\right)-1}{\sqrt{\pi}}\right]\right)^{d_x}.
    \end{align*}
    Here, $\mathrm{erf}(x) := \frac{1}{\sqrt{\pi}}\int^{x}_{-x}\exp\left(-t^2\right)\;\mathrm{d}t$ is the standard error function.
\end{lem}
\begin{proof}
    Since $k_{\gamma}(\bx, \bx^\prime) = \prod_{i=1}^{d} \exp\left(-\frac{(x_i - x_i')^2}{\gamma^2}\right)$, it suffices to prove the following result
    \begin{align*}
        \int_{1/4}^{3/4} \int_{1/4}^{3/4} k_{\gamma}(x_i, x_i') \;\mathrm{d} x_i \mathrm{d} x_i' =  \gamma^2 \sqrt{\pi} \left[\frac{\mathrm{erf}\left(\frac{1}{2\gamma}\right)}{2\gamma} + \frac{\exp\left(-\frac{1}{4\gamma^2}\right)-1}{\sqrt{\pi}}\right].
    \end{align*}
    Notice that 
    \begin{align*}
        &\int_{1/4}^{3/4} \exp \left(-\frac{(x_i - x_i')^2}{\gamma^2} \right) \;\mathrm{d} x_i\\ &= \gamma \int_{(1/4-x_i^\prime) / \gamma}^{(3/4-x_i^\prime) / \gamma}  \exp(-x^2) \;\mathrm{d}x \\
        &=\gamma \int^{(3/4-x_i')/\gamma}_{-\infty}\exp(-x^2) \;\mathrm{d}x - \gamma \int^{(1/4-x_i')/\gamma}_{-\infty}\exp(-x^2)\;\mathrm{d}x\\
        &=\gamma\frac{\sqrt{\pi}}{2}\left(1+\mathrm{erf}\left(\frac{3/4-x_i'}{\gamma}\right)\right) - \gamma\frac{\sqrt{\pi}}{2}\left(1-\mathrm{erf}\left(\frac{x_i'-1/4}{\gamma}\right)\right)\\
        &=\gamma\frac{\sqrt{\pi}}{2}\left(\mathrm{erf}\left(\frac{3/4-x_i'}{\gamma}\right) + \mathrm{erf}\left(\frac{x_i'-1/4}{\gamma}\right)\right).
    \end{align*}
    Therefore, we have
    \begin{align*}
        &\quad \int_{1/4}^{3/4} \int_{1/4}^{3/4} \exp \left(-\frac{(x_i - x_i')^2}{\gamma^2} \right) \;\mathrm{d} x_i \;\mathrm{d} x_i^\prime \\
        &= \gamma \frac{\sqrt{\pi}}{2} \int_{1/4}^{3/4} \mathrm{erf}\left(\frac{3/4-x_i'}{\gamma}\right) + \mathrm{erf}\left(\frac{x_i'-1/4}{\gamma}\right)  \;\mathrm{d} x_i^\prime \\
        &= \gamma^2 \sqrt{\pi} \int_{0}^{\frac{1}{2\gamma}} \mathrm{erf}(y)  \;\mathrm{d} y \\
        &= \gamma^2 \sqrt{\pi} \left[\frac{\mathrm{erf}\left(\frac{1}{2\gamma}\right)}{2\gamma} + \frac{\exp\left(-\frac{1}{4\gamma^2}\right)-1}{\sqrt{\pi}}\right] . 
    \end{align*}
    where the last inequality holds by using the identity
    \begin{align*}
        \int \mathrm{erf}(x) d x = x \cdot \mathrm{erf}(x) + e^{-x^2} / \sqrt{\pi} + C.
    \end{align*}
\end{proof}

\section{Proof of Theorem~\ref{thm:lower_rate} in the main text}
\label{sec:lower}

\subsection{Relationship Between the NPIR-O Model and the NPIV-O Model}
\label{sec:npir_reduction}

Following \citet[Section 3]{chen2011rate}, we first establish that the NPIV-O model is no more informative than the reduced form nonparametric indirect regression with observed confounders (NPIR-O) model. 
\begin{defi}[Restricted NPIV-O model]
\label{def:appendix_restricted_npivo}
    Let $\sigma_0>0$ be a finite constant. Recall $\mathfrak{S}$ as defined in \Cref{ass:f_ast} in the main text. Let $\calC$ be a set of elements $(P_{\epsilon ZXO}, f)$ such that the following property holds: $(\forall f\in \mathfrak{S})\;(\exists P_{\epsilon ZXO}, f)\in \calC$ such that $P_{ZY}$ is determined by $P_{\epsilon ZX}$ and $f$, and that
    \begin{align*}
        Y_i - \mathbb{E}[Y_i\mid Z = \bz_i, O = \bo_i] = f(\bz_i,\bo_i) + \epsilon_i - (Tf)(\bz_i,\bo_i)
    \end{align*}
    given $Z_i = \bz_i, O_i = \bo_i$ is $\calN(0, \sigma^2(\bz_i,\bo_i))$-distributed with $\sigma^2(\bz_i,\bo_i) \geq \sigma_0^2$. 
\end{defi}
For an NPIV-O model as defined in \Cref{def:appendix_restricted_npivo}, we specify the reduced form NPIR-O model as
\begin{align*}
    Y_i = (Tf)(Z_i, O_i) + \upsilon_i, \quad i = 1,\dots, n
\end{align*}
with $(Z_i, O_i, \upsilon_i)$ i.i.d., $P_{\upsilon_i\mid Z_i = \bz_i,O_i = \bo_i} = \calN(0,\sigma^2(\bz_i, \bo_i))$, $f\in \mathfrak{S}$ the unknown structural function, and $T: L^2(P_{XO})\to L^2(P_{ZO})$ a known operator satisfying \Cref{ass:T_injective}. The observations corresponding to the NPIR are $\{(Y_i, \bz_i, \bo_i)\}_{i=1}^{n}$. 
\begin{defi}[NPIR-O model class]
\label{def:npiro_model_class}
    Let $\calC$ be as defined in \Cref{def:appendix_restricted_npivo}. The NPIR-O model class $\calC_0$ consists of all model parameters $(P_{Z'O'}, \sigma(\cdot, \cdot), f)$ such that $(\exists (P_{\epsilon ZX}, f)\in \calC)$ with the following properties: $P_{ZO} = P_{Z'O'}$, $\sigma^2(\bz,\bo)\geq \sigma_0^2 > 0$, the conditional law $P_{X\mid Z,O}$ is prescribed according to $T$, and $P_{\epsilon\mid ZOX}$ is arbitrary among the conditions imposed in $\calC$. 
\end{defi}
The following Lemma is \citet[Lemma 1]{chen2007large}, by augmenting relevant variables to include observed confounders $O$. 
\begin{lem}
\label{lem:npir_o_more_informative}
    The NPIR-O model is more informative than the NPIV-O model in the sense that for each estimator $\hat{f}_n$ for the NPIV-O model, there is an estimator $\tilde{f}_n$ for the NPIR-O model with 
    \begin{align*}
        \sup_{(P_{ZO}, \sigma(\cdot, \cdot), f)\in \calC_0}\mathbb{E}_{(P_{ZO},\sigma(\cdot, \cdot), f)}[\|\tilde{f}_n - f\|^2_{L^2(P_{XO})}] \leq \sup_{(P_{\epsilon ZXO}, f)\in \calC}\mathbb{E}_{(P_{\epsilon ZXO, f})}[\|\hat{f}_n - f\|^2_{L^2(P_{XO})}].
    \end{align*}
\end{lem}

In this section, we provide a lower bound for the NPIR-O model class defined in \Cref{def:npiro_model_class}, which by the above discussion implies a lower bound for the (restricted) NPIV-O model class defined in \Cref{def:appendix_restricted_npivo}.

\subsection{The Lower Bound for NPIR-O Model}
\label{sec:NPIR_lower_appendix}

\underline{\emph{Step One.}}
We take $\mathfrak{m}$ to be the smallest even integer such that $\mathfrak{m} > s_x \vee s_o$.
To help us construct $f_{\bv}$, we need to introduce several functions.
We define
\begin{align}
\label{eq:Bspline_main}
    M_{\mathfrak{K}, \boldsymbol{-\frac{\mathfrak{m}}{2}}}(\bx) := \prod_{i=1}^{d_x}\iota_\mathfrak{m}\left(2^{\left\lfloor \mathfrak{K}\frac{\underline{s}}{s_x}\right\rfloor}x_i + \frac{\mathfrak{m}}{2} \right), \quad M_{\mathfrak{K},\bl_o}(\bo) := \prod_{j=1}^{d_o}\iota_\mathfrak{m}\left(2^{\left\lfloor \mathfrak{K}\frac{\underline{s}}{s_o}\right\rfloor}o_i - \ell_{o,j}\right).
\end{align}
Define
\begin{align}\label{eq:calL}
    \mathcal{L} := \left\{(\bl_x,\bl_o): \bl_x \in \left\{0,1,\dots, \left\lfloor \frac{0.8\pi}{\zeta}2^{\frac{\mathfrak{K}\underline{s}}{s_x}}\right\rfloor\right\}^{d_x},  \bl_o \in \left(m\mathbb{Z}\cap \left\{1, \dots, 2^{\left\lfloor \mathfrak{K}\frac{\underline{s}}{s_o}\right\rfloor}\right\}\right)^{d_o}\right\},
\end{align}
then we can compute the size of $\calL$ as 
\begin{align}\label{eq:size_L}
    |\calL| \asymp \zeta^{-d_x}2^{\mathfrak{K}\underline{s}\left(\frac{d_x}{s_x} + \frac{d_o}{s_o}\right)}
\end{align}
For $(\bl_x, \bl_o) \in \calL$, define $\one_{\bl_x}$ as an indicator function of the set
\begin{align}
\label{eq:I_bl_x_defi}
    I_{\bl_x} := \bigtimes_{j=1}^{d_x}\left[1.1\pi + \zeta \ell_{x,j}2^{-\mathfrak{K}\frac{\underline{s}}{s_x}}, 1.1\pi + (\zeta \ell_{x,j} + 1)2^{-\mathfrak{K}\frac{\underline{s}}{s_x}}\right],
\end{align}
Next, we define
\begin{align}
\label{eq:app_omega_klx_defi}
    \Omega_{\mathfrak{K},\bl_x}(\bx) := \left(M_{0,\boldsymbol{-\frac{\mathfrak{m}}{2}}} \ast \calF^{-1}[\one_{\bl_x}]\right)\left(2^{\frac{\mathfrak{K}\underline{s}}{s_x}}\bx\right), \quad \Omega_{\mathfrak{K}(\bl_x,\bl_o)}(\bx,\bo) &:= \Omega_{\mathfrak{K},\bl_x}(\bx)M_{\mathfrak{K},\bl_o}(\bo). 
\end{align}
Now we are ready to define $f_{\bv}$. Note that a vector $\bv \in \{0,1\}^{|\calL|}$ canonically associates to each point $(\bl_x,\bl_o)$ a value $\beta_{\bv(\bl_x,\bl_o)}\in \{0,1\}$. We define 
\begin{align}
\label{eq: defn_fv}
    f_{\bv} := \epsilon_0 2^{-\mathfrak{K}\underline{s}\left(1 - \frac{d_x}{2s_x}\right)} \sum_{\bl_x,\bl_o}\beta_{\bv(\bl_x,\bl_o)}\Omega_{\mathfrak{K}(\bl_x,\bl_o)},
\end{align}
\steve{where $\epsilon_0 > 0$ is a fixed scaler to be chosen later, to ensure that $\|f_{\bv}\|_{B^{s_x,s_o}_{2,\infty}(\mathbb{R}^{d_x+d_o})}\leq 1$, thus ensuring that $f_{\bv}$ satisfies \Cref{ass:f_ast}.} The function class $\mathfrak{F}:=\{ f_{\bv},  \bv \in \{0,1\}^{|\calL|}\}$. 
For each $f_{\bv}$, consider the joint data generating distribution $P_{ZXOY}$ specified as follows: 
\begin{enumerate}
    \item The marginal distribution $P_{ZO}$ is the product of independent distributions $P_{Z}$ and $P_{O}$, where both are uniform distributions on $[0,1]^{d_z}$ and $[0,1]^{d_o}$.
    \item The marginal distribution $P_{XO}$ is the product of independent distributions $P_{X}$ and $P_{O}$, where $P_{O}$ is the uniform distribution on $[0,1]^{d_o}$, and $P_X$ is supported on $[-1/2,1/2]^{d_x}$ and admits the following density function:
    \begin{align}\label{eq:p_X_bump}
        p_X(\bx)\propto \prod_{i=1}^{d_x} g(x_i), \quad g(x_i) := \exp\left(-\frac{2}{1 - 4x_i^2}\right) \one_{x_i \in [-1/2, 1/2]} .
    \end{align}
    \item The conditional distribution $P_{X\mid Z, O}$ satisfies \Cref{assn: technical} and it induces an operator $T: L^2\left(P_{XO}\right) \rightarrow L^2\left(P_{ZO}\right)$ that satisfies \Cref{ass:T_contractivity}. 
    \item The conditional distribution $P_{Y\mid Z=\bz, O=\bo}$ is a Gaussian distribution $\calN((Tf_{\bv})(\bz,\bo),\sigma^2)$ for any $\bz\in\calZ$ and $\bo\in\calO$. 
\end{enumerate}

\vspace{3mm}
\noindent
\underline{\emph{Step Two.}}
Now we are going to present some properties of the basis $\Omega_{\mathfrak{K}(\bl_x,\bl_o)}$, which will be used later on to prove some properties of $f_{\bv}$. By Young's Convolution Theorem, we have 
\begin{align*}
    \left\|M_{0,\boldsymbol{-\frac{\mathfrak{m}}{2}}}\ast \calF^{-1}[\one_{\bl_{x}}]\right\|_{L^2(\mathbb{R}^{d_x})} \leq \|M_{0,\boldsymbol{-\frac{\mathfrak{m}}{2}}}\|_{L^1(\mathbb{R}^{d_x})}\|\one_{\bl_{x}}\|_{L^2(\mathbb{R}^{d_x})} < \infty,
\end{align*}
hence its Fourier transform is well-defined. We have
\begin{align}
    \calF[\Omega_{\mathfrak{K},\bl_x}](\bomega_x) &= 2^{-\mathfrak{K}\frac{\underline{s}d_x}{s_x}}\calF\left[M_{0,\boldsymbol{-\frac{\mathfrak{m}}{2}}}\ast \calF^{-1}[\one_{\bl_{x}}]\right]\left(2^{-\frac{\mathfrak{K}\underline{s}}{s_x}}\bomega_x\right) \nonumber\\
    &= 2^{-\mathfrak{K}\frac{\underline{s}d_x}{s_x}} \cdot \calF\left[M_{0,\boldsymbol{-\frac{\mathfrak{m}}{2}}}\right] \left(2^{-\frac{\mathfrak{K}\underline{s}}{s_x}}\bomega_x\right)  \cdot \one_{\bl_x}\left(2^{-\frac{\mathfrak{K}\underline{s}}{s_x}}\bomega_x\right) \nonumber \\
    &= \calF\left[M_{\mathfrak{K},\boldsymbol{-\frac{\mathfrak{m}}{2}}}\right] \left(\bomega_x\right)  \cdot \one_{\bl_x}\left(2^{-\frac{\mathfrak{K}\underline{s}}{s_x}}\bomega_x\right)
    \label{eq:fourier_omega_k_lx}.
\end{align}
Also, since $\calF[\iota_\mathfrak{m}](\omega)= \exp{(- \mathfrak{m} i \omega / 2)} \cdot \sin (\omega / 2)^{\mathfrak{m}} \cdot (\omega / 2)^{-\mathfrak{m}}$, we have 
\begin{align}\label{eq:fourier_M00}
    \calF\left[M_{0,\boldsymbol{-\frac{\mathfrak{m}}{2}}}\right](\bomega_x) = \prod_{i=1}^{d_x} \frac{\sin (\omega_{i,x} / 2)^{\mathfrak{m}}}{(\omega_{i,x} / 2)^{\mathfrak{m}}} .
\end{align}
Now we can see that we pick the location vector to be $-\boldsymbol{\frac{\mathfrak{m}}{2}}$ such that the Fourier transform of $M_{0,\boldsymbol{-\frac{\mathfrak{m}}{2}}}$ is a real valued function. Note that since $\zeta\geq 1$, we have $(\zeta \ell_{x,j} + 1)2^{-\mathfrak{K}\frac{\underline{s}}{s_x}} \leq \zeta (\ell_{x,j}+1)2^{-\mathfrak{K}\frac{\underline{s}}{s_x}}$ for $j \in \{1, \ldots, d_x \}$, hence $I_{\bl_x}\cap I_{\bl_x'} = \emptyset$ if $\bl_x \neq \bl_x'$ which means that the support of $\calF[\Omega_{\mathfrak{K},\bl_x}]$ is disjoint for $\bl_x \neq \bl_x'$.
Also, note that 
\begin{align*}
    \mathrm{supp}(M_{\mathfrak{K},\bl_o}) = \bigtimes_{j=1}^{d_o}[ \ell_{o,j} 2^{-\mathfrak{K}\frac{\underline{s}}{s_o}}, (\ell_{o,j} + \mathfrak{m})2^{-\mathfrak{K}\frac{\underline{s}}{s_o}}]
\end{align*}
and $ \ell_{o,j}$ are multiples of $\mathfrak{m}$ by definition of $\calL$ in Eq.~\eqref{eq:calL}, we have $\mathrm{supp}(M_{\mathfrak{K},\bl_o})\cap \mathrm{supp}(M_{k,\bl_o'}) = \emptyset $ for any $\bl_o \neq \bl_o^\prime$. 

\vspace{3mm}
\noindent
\underline{\emph{Step Three.}}
In this step, we are going to prove that $f_{\bv}\in B_{2,\infty}^{s_x,s_o}(\mathbb{R}^{d_x+d_o})$ and its Besov norm is bounded by $\epsilon_0$.
We have, $(\forall\bo\in \calO)$, 
\begin{align}\label{eq:fourier_f_v}
    &\calF[f_{\bv}(\cdot,\bo)](\bomega_x) = \epsilon_0 2^{-\mathfrak{K}\underline{s}\left(1 - \frac{d_x}{2s_x}\right)} \sum_{\bl_x,\bl_o}\beta_{\bv(\bl_x,\bl_o)}\calF[\Omega_{\mathfrak{K}(\bl_x,\bl_o)}(\cdot, \bo)](\bomega_x) \nonumber \\ 
    &= \epsilon_0 2^{-\mathfrak{K}\underline{s}\left(1 - \frac{d_x}{2s_x}\right)} \sum_{\bl_x,\bl_o}\beta_{\bv(\bl_x,\bl_o)}\calF[\Omega_{\mathfrak{K},\bl_x}](\bomega_x) \cdot M_{\mathfrak{K},\bl_o}(\bo) \nonumber \\
    &= \epsilon_0 2^{-\mathfrak{K}\underline{s}\left(1 - \frac{d_x}{2s_x}\right)} \sum_{\bl_x, \bl_o} \beta_{\bv(\bl_x,\bl_o)} 2^{-\mathfrak{K}\frac{\underline{s}d_x}{s_x}}\calF[M_{0,\boldsymbol{-\frac{\mathfrak{m}}{2}}}] \left(2^{-\mathfrak{K}\frac{\underline{s}}{s_x}}\bomega_{x}\right) \cdot \one_{\bl_x}\left(2^{-\mathfrak{K}\frac{\underline{s}}{s_x}}\bomega_{x}\right) \cdot M_{\mathfrak{K},\bl_o}(\bo) \nonumber \\
    &= \epsilon_0 2^{-\mathfrak{K}\underline{s}\left(1 - \frac{d_x}{2s_x}\right)} \sum_{\bl_o} \calF[M_{\mathfrak{K},\boldsymbol{-\frac{\mathfrak{m}}{2}}}](\bomega_x) \cdot M_{\mathfrak{K},\bl_o}(\bo) \cdot \left(\sum_{\bl_x}\beta_{\bv(\bl_x,\bl_o)}\one_{\bl_x}\left(2^{-\mathfrak{K}\frac{\underline{s}}{s_x}}\bomega_{x}\right) \right).
\end{align}   
Note that, for fixed $\bl_o$,  $\bomega_x \mapsto \sum_{\bl_x}\beta_{\bv(\bl_x,\bl_o)}\one_{\bl_x}(2^{-\mathfrak{K}\frac{\underline{s}}{s_x}}\bomega_{x})$ is the indicator function on
\begin{align*}
    \bigcup_{\bl_x: \beta_{\bv(\bl_x,\bl_o)}>0} \bigtimes_{j=1}^{d_x}\left[1.1\pi 2^{\mathfrak{K}\frac{\underline{s}}{s_x}} + \zeta \ell_{x,j}, 1.1\pi 2^{\mathfrak{K}\frac{\underline{s}}{s_x}} + (\zeta \ell_{x,j} + 1)\right] ,
\end{align*}
The above observation as applied to the right hand side of Eq. \eqref{eq:fourier_f_v} means that we can apply \Cref{lem:new_besov_mask} to obtain
\begin{align}
\label{eq:f_bv_besov_norm_estimate}
    \|f_{\bv}\|_{B^{s_x,s_o}_{2,q}(\mathbb{R}^{d_x+d_o})} &\lesssim \epsilon_0 2^{-\mathfrak{K}\underline{s}\left(1 - \frac{d_x}{2s_x}\right)} \cdot  2^{\mathfrak{K}\underline{s}\left(1 - \frac{d_x}{2s_x}\right)} = \epsilon_0 .
\end{align}
\steve{Hence we can choose $\epsilon_0$ to be a small scalar such that $\|f_{\bv}\|_{B^{s_x,s_o}_{2,q}(\mathbb{R}^{d_x+d_o})} \leq 1$. The rest of the conditions that $f_{\bv} \in L^\infty(\R^{d_x+d_o}) \cap L^1(\R^{d_x+d_o}) \cap C^0(\R^{d_x+d_o})$ are all trivial to verify. Therefore, $f_{\bv}$ satisfies \Cref{ass:f_ast}. }

\vspace{3mm}
\noindent
\underline{\emph{Step Four.}}
By the Gilbert-Varshamov Bound \cite{Tsybakov}[Lemma 2.9], for $|\calL|\geq 8$, there exists a subset $V_\mathfrak{K}$ of $\{0,1\}^{|\calL|}$ such that $\boldsymbol{0}\in V_\mathfrak{K}$ and
\begin{align}
\label{eq: gv_beta_sep}
    \sum_{\bl_x,\bl_o}\left|\beta_{\bv (\bl_x,\bl_o)} - \beta_{\bv' (\bl_x,\bl_o)}\right|^2 \geq \frac{|\calL|}{8}
\end{align}
for $\bv\neq \bv'\in V_\mathfrak{K}$ and $|V_\mathfrak{K}|\geq 2^{\frac{|\calL|}{8}}$. 
Note that since $\sqrt{p_X}$ is an even function over $[-1/2,1/2]^{d_x}$, its Fourier transform $q$ is a real valued function. Define
\begin{align}\label{eq:defi_tilde_q}
    q(\bomega_x) := \calF[\sqrt{p_{X}}](\bomega_x) , \quad
    \Tilde{q}(\bomega_x) := \calF[\sqrt{p_X}](\bomega_x) 1_{[-\zeta/3, \zeta/3]^{d_x}}(\bomega_x).
\end{align}
For coefficients $\alpha_{\bv(\bl_x,\bl_o)} \in \mathbb{R}$, we have
\begin{small}
\begin{align}
    &\left\|\sum_{\bl_x,\bl_o} \alpha_{\bv(\bl_x,\bl_o)}\Omega_{\mathfrak{K}(\bl_x,\bl_o)}\right\|^2_{L^2(P_{XO})}\nonumber\\
    =& \int_{[0,1]^{d_o}}\int_{\mathbb{R}^{d_x}}\sum_{\substack{\bl_x,\bl_o\nonumber\\ \Tilde{\bl}_x, \Tilde{\bl_o}}} \alpha_{\bv(\bl_x,\bl_o)}\alpha_{\bv(\Tilde{\bl_x}, \Tilde{\bl}_o)} \Omega_{\mathfrak{K},\bl_x}(\bx)\Omega_{k\Tilde{\bl}_x}(\bx)M_{\mathfrak{K},\bl_o}(\bo)M_{k\Tilde{\bl}_o}(\bo)p_X(\bx)\;\mathrm{d}\bx\;\mathrm{d}\bo\nonumber\\
    = & \int_{\mathbb{R}^{d_x}}\sum_{\substack{\bl_x,\bl_o\\ \Tilde{\bl}_x, \Tilde{\bl_o}}} \alpha_{\bv(\bl_x,\bl_o)}\alpha_{\bv(\Tilde{\bl_x}, \Tilde{\bl}_o)} \Omega_{\mathfrak{K},\bl_x}(\bx)\Omega_{k\Tilde{\bl}_x}(\bx) p_X(\bx)\left(\int_{[0,1]^{d_o}}M_{\mathfrak{K},\bl_o}(\bo)M_{k\Tilde{\bl}_o}(\bo)\;\mathrm{d}\bo\right)\;\mathrm{d}\bx\nonumber\\
    \stackrel{(a)}{=} & \sum_{\bl_o}\|M_{\mathfrak{K},\bl_o}\|^2_{L^2(\mathbb{R}^{d_o})}\int_{\mathbb{R}^{d_x}} \left(\sum_{\bl_x}\alpha_{\bv(\bl_x,\bl_o)}\Omega_{\mathfrak{K},\bl_x}(\bx)\sqrt{p(\bx)}\right)^2\;\mathrm{d}\bx\nonumber\\
    \stackrel{(b)}{=} & \sum_{\bl_o}\|M_{\mathfrak{K},\bl_o}\|^2_{L^2(\mathbb{R}^{d_o})}\int_{\mathbb{R}^{d_x}} \left(\sum_{\bl_x}\alpha_{\bv(\bl_x,\bl_o)}\calF\left[\Omega_{\mathfrak{K},\bl_x}\sqrt{p}\right](\bomega_x)\right)^2\;\mathrm{d}\bomega_x\nonumber\\
    =&2^{-\frac{\mathfrak{K}\underline{s}d_o}{s_o}}\|M_{0\boldsymbol{0}}\|^2_{L^2(\mathbb{R}^{d_o})}\int_{\mathbb{R}^{d_x}} \sum_{\bl_o}\left(\sum_{\bl_x}\alpha_{\bv(\bl_x,\bl_o)}\calF\left[\Omega_{\mathfrak{K},\bl_x}\sqrt{p}\right](\bomega_x)\right)^2\;\mathrm{d}\bomega_x .\label{eq: gen_coeff_prob_l2}
\end{align}
\end{small}
where $(a)$ follows from the fact that $M_{\mathfrak{K},\bl_o}$ have pairwise disjoint support verified in \underline{\emph{Step Two}} above; and $(b)$ follows from the Plancherel's Theorem. By Eq. \eqref{eq: gen_coeff_prob_l2}, we have
\begin{align}\label{eq:L2_norm_A_B}
    &\quad \|f_{\bv} - f_{\bv'}\|_{L^2(P_{XO})}^2 \nonumber \\
    &= \epsilon_0 2^{-2\mathfrak{K}\underline{s}\left(1 - \frac{d_x}{2s_x}\right)} 2^{-\frac{\mathfrak{K}\underline{s}d_o}{s_o}}\|M_{0\boldsymbol{0}}\|_{L^2(\mathbb{R}^{d_o})}^2 \nonumber \\
    &\qquad \cdot \int_{\mathbb{R}^{d_x}} \sum_{\bl_o}\left(\sum_{\bl_x}\left(\beta_{\bv(\bl_x,\bl_o)} - \beta_{\bv'(\bl_x,\bl_o)}\right)\calF\left[\Omega_{\mathfrak{K},\bl_x}\sqrt{p_X}\right](\bomega_x)\right)^2\;\mathrm{d}\bomega_x \nonumber \\
    &\geq \epsilon_0 2^{-2\mathfrak{K}\underline{s}\left(1 - \frac{d_x}{2s_x}\right)} 2^{-\frac{\mathfrak{K}\underline{s}d_o}{s_o}}\|M_{0\boldsymbol{0}}\|_{L^2(\mathbb{R}^{d_o})}^2\left((A)^2/2 - (B)^2\right),
\end{align}
where the last step follows from the reverse triangular inequality $(a+b)^2 \geq \frac{a^2}{2} - b^2$, and we define
\begin{align*}
    (A)^2&:= \sum_{\bl_o}\left\|\sum_{\bl_x}\left(\beta_{\bv(\bl_x,\bl_o)} - \beta_{\bv'(\bl_x,\bl_o)}\right) \calF[\Omega_{\mathfrak{K},\bl_x}]  \ast \Tilde{q}\right\|_{L^2(\mathbb{R}^{d_x})}^2\\
    (B)^2&:=\sum_{\bl_o}\left\|\sum_{\bl_x}\left(\beta_{\bv(\bl_x,\bl_o)} - \beta_{\bv'(\bl_x,\bl_o)}\right)  \calF[\Omega_{\mathfrak{K},\bl_x}]   \ast (q-\Tilde{q})\right\|_{L^2(\mathbb{R}^{d_x})}^2.
\end{align*}
In order to lower bound Eq.~\eqref{eq:L2_norm_A_B}, we need to lower bound $(A)^2$ and upper bound $(B)^2$.
First, we are going to lower bound $(A)^2$.
From Eq.~\eqref{eq:fourier_omega_k_lx}, we know that the support of $\calF[\Omega_{\mathfrak{K},\bl_x}]$ is 
\begin{align}\label{eq:tilde_I}
    \tilde{I_{\bl_x}} := \bigtimes_{j=1}^{d_x}\left[1.1\pi 2^{\mathfrak{K} \frac{\underline{s}}{s_x}} + \zeta \ell_{x,j}, 1.1\pi 2^{\mathfrak{K} \frac{\underline{s}}{s_x}} + (\zeta \ell_{x,j} + 1)\right].
\end{align}
Since the support of $\Tilde{q}$ is $[-\zeta/3, \zeta/3]^{d_x}$, by the standard fact that $\mathrm{supp}(f\ast g)\subseteq \mathrm{supp}(f) + \mathrm{supp}(g)$, we have
\begin{align}\label{eq:support_F_conv_tilde_q}
    \mathrm{supp}\left( \calF[\Omega_{\mathfrak{K},\bl_x}]  \ast \Tilde{q} \right) &\subseteq \bigtimes_{j=1}^{d_x}\left[1.1\pi 2^{\mathfrak{K} \frac{\underline{s}}{s_x}} + \zeta (\ell_{x,j} - 1/3), 1.1\pi 2^{\mathfrak{K} \frac{\underline{s}}{s_x}} + (\zeta (\ell_{x,j} + 1/3) + 1)\right] \nonumber \\
    &=: \Lambda_\mathfrak{K}
\end{align}
Note that for $\zeta \geq 3$, we have $\zeta(\ell_j + 1/3) + 1 \leq \zeta(\ell_j + 1 - 1/3)$, hence $\mathrm{supp}\left( \calF[\Omega_{\mathfrak{K},\bl_x}]  \ast \Tilde{q} \right)$ is pairwise disjoint with respect to different $\bl_x$. 
Hence we obtain that
\begin{align*}
    (A)^2 &= \sum_{\bl_o}\sum_{\bl_x}(\beta_{\bv(\bl_x,\bl_o)} - \beta_{\bv'(\bl_x,\bl_o)})^2 \left\| \calF[\Omega_{\mathfrak{K},\bl_x}]   \ast \Tilde{q}\right\|^2_{L^2(\mathbb{R}^{d_x})}.
\end{align*}
Notice that
\begin{align}\label{eq:F_conv_tilde_q_1}
    &\quad \left\|\calF[\Omega_{\mathfrak{K},\bl_x}]   \ast \Tilde{q}\right\|_{L^2(\mathbb{R}^{d_x})}^2 \nonumber \\
    &= \int_{\Lambda_\mathfrak{K}} \left|\int_{\mathbb{R}^{d_x}} \calF[\Omega_{\mathfrak{K},\bl_x}](\bomega_x^\prime) \tilde{q}(\bomega_x - \bomega_x^\prime) \;\mathrm{d}\bomega_x^\prime\right|^2 \;\mathrm{d}\bomega_x \nonumber \\
    &= 2^{-\frac{2\mathfrak{K}\underline{s}d_x}{s_x}} \int_{\Lambda_\mathfrak{K}} \left|\int_{\mathbb{R}^{d_x}} \calF[M_{0,\boldsymbol{-\frac{\mathfrak{m}}{2}}}]\left(2^{-\frac{\mathfrak{K}\underline{s}}{s_x}} \bomega_x^\prime \right) \cdot \one_{\bl_x} \left(2^{-\frac{\mathfrak{K}\underline{s}}{s_x}} \bomega_x^\prime \right) \cdot \tilde{q}(\bomega_x - \bomega_x^\prime) \;\mathrm{d}\bomega_x^\prime\right|^2 \;\mathrm{d}\bomega_x \nonumber \\
    &\stackrel{(i)}{\geq} 2^{-\frac{2\mathfrak{K}\underline{s}d_x}{s_x}} \int_{\tilde{I_{\bl_x}}} \left|\int_{\mathbb{R}^{d_x}} \calF[M_{0,\boldsymbol{-\frac{\mathfrak{m}}{2}}}]\left(2^{-\frac{\mathfrak{K}\underline{s}}{s_x}} \bomega_x^\prime \right) \cdot \one_{\bl_x} \left(2^{-\frac{\mathfrak{K}\underline{s}}{s_x}} \bomega_x^\prime \right) \cdot \tilde{q}(\bomega_x - \bomega_x^\prime) \;\mathrm{d}\bomega_x^\prime\right|^2 \;\mathrm{d}\bomega_x \nonumber \\
    &\stackrel{(ii)}{=} 2^{-\frac{2\mathfrak{K}\underline{s}d_x}{s_x}} \int_{\tilde{I_{\bl_x}}} \left|\int_{\tilde{I_{\bl_x}}} \calF[M_{0,\boldsymbol{-\frac{\mathfrak{m}}{2}}}]\left(2^{-\frac{\mathfrak{K}\underline{s}}{s_x}} \bomega_x^\prime \right) \cdot \tilde{q}(\bomega_x - \bomega_x^\prime) \;\mathrm{d}\bomega_x^\prime\right|^2 \;\mathrm{d}\bomega_x .
\end{align}
$(i)$ above holds because $\tilde{I_{\bl_x}} \subset \Lambda_\mathfrak{K}$; and $(ii)$ holds because of the indicator function.
Notice, for any $\bomega_x,  \bomega_x^\prime \in \tilde{I_{\bl_x}}$, we have for $j \in \{1, \ldots, d_x\}$, $
1 \geq \bomega_{x, j} - \bomega^\prime_{x, j} \geq -1$. 
Since for any $-1 \leq t \leq 1$, $i \in \{1, \ldots, d_x\}$ and $g$ defined in Eq.~\eqref{eq:p_X_bump}, there is $\calF[\sqrt{g}](t) = \int_{-1/2}^{1/2} \sqrt{g}(x_i) \exp(-x_i t)dx_i = \int_{-1/2}^{1/2} \sqrt{g}(x_i) \cos( x_i t) dx_i > 0$.
So we have, for $\zeta > 3$,
\begin{align*}
    \tilde{q}(\bomega_x - \bomega_x^\prime) = q(\bomega_x - \bomega_x^\prime)= \calF[\sqrt{p_X}](\bomega_x - \bomega_x^\prime) =\prod_{i=1}^{d_x} \calF[\sqrt{g}](\bomega_{x,i} - \bomega^\prime_{x,i}) > 0.
\end{align*}
Since both $\calF[M_{0,\boldsymbol{-\frac{\mathfrak{m}}{2}}}] (2^{-\frac{\mathfrak{K}\underline{s}}{s_x}} \bomega_x^\prime)$ and $\tilde{q}(\bomega_x - \bomega_x^\prime)$ are positive and real, we continue from Eq.~\eqref{eq:F_conv_tilde_q_1} to have
\begin{align*}
    \geq 2^{-\frac{2 \mathfrak{K}\underline{s}d_x}{s_x}} \left( \inf_{\bomega_x \in I_{\bl_x}} \calF[M_{0,\boldsymbol{-\frac{\mathfrak{m}}{2}}}] (\bomega_x) \right)^2 \cdot \underbrace{\int_{\tilde{I_{\bl_x}}} \left( \int_{\tilde{I_{\bl_x}}} \tilde{q}\left( \bomega_x - \bomega_x^{\prime} \right) \;\mathrm{d}\bomega_x^{\prime}  \right)^2 \;\mathrm{d}\bomega_x}_{(\ast)} .
\end{align*}
Notice that in the integration in the term $(\ast)$ above, only the difference of $\bomega_x - \bomega_x^{\prime}$ show up. 
So we can obtain
\begin{align*}
    (\ast) &= \int_{[0,1]^{d_x}} \left( \int_{[0,1]^{d_x}} \tilde{q}\left( \bomega_x - \bomega_x^{\prime} \right) \;\mathrm{d}\bomega_x^{\prime}  \right)^2 \;\mathrm{d}\bomega_x \\
    &= \prod_{i=1}^{d_x} \int_{0}^{1} \left( \int_0^1  \tilde{q}_i\left( \bomega_{x, i} - \bomega_{x, i}^{\prime} \right) \;\mathrm{d}\bomega_{x, i}^{\prime}  \right)^2 \;\mathrm{d}\bomega_{x, i} \\
    &= \prod_{i=1}^{d_x} \int_{0}^{1} \left( \int_{\bomega_{x, i}-1}^{\bomega_{x, i}} \tilde{q}_i\left( \bomega_{x, i}^{\prime\prime} \right) \;\mathrm{d}\bomega_{x, i}^{\prime\prime}  \right)^2 \;\mathrm{d}\bomega_{x, i} \\
    &= \prod_{i=1}^{d_x} \int_{0}^{1} \left( \int_{\bomega_{x, i}-1}^{\bomega_{x, i}} q_i\left( \bomega_{x, i}^{\prime\prime} \right) \;\mathrm{d}\bomega_{x, i}^{\prime\prime}  \right)^2 \;\mathrm{d}\bomega_{x, i}.
\end{align*}
The second last equality holds by change of variables, and the last equality holds because $\tilde{q}_i(\omega)=q_i(\omega)$ for $\omega \in [-\zeta/3, \zeta/3]$. 
So $(\ast)$ is a strictly positive constant independent of $\zeta, k, n$. Plugging it back to above, we obtain
\begin{align*}
    \left\|\calF[\Omega_{\mathfrak{K},\bl_x}] \ast \Tilde{q}\right\|_{L^2(\mathbb{R}^{d_x})}^2 \geq 2^{-\frac{2 \mathfrak{K}\underline{s}d_x}{s_x}} \left( \inf_{\bomega_x \in I_{\bl_x}} \calF[M_{0,\boldsymbol{-\frac{\mathfrak{m}}{2}}}](\bomega_x) \right)^2 \cdot (\ast). 
\end{align*}
Notice that, since $I_{\bl_x}\subseteq [1.1\pi, 1.95\pi]^{d_x}$ and we use Eq.~\eqref{eq:fourier_M00} to obtain, since $M_{0,\boldsymbol{-\frac{\mathfrak{m}}{2}}}$ is an even function and thus has real-valued Fourier transform
\begin{align*}
    \inf_{\bomega_x \in I_{\bl_x}} \left| \calF[M_{0,\boldsymbol{-\frac{\mathfrak{m}}{2}}}](\bomega_x) \right|^2 = \inf_{\bomega_x \in I_{\bl_x}} \left| \prod_{i=1}^{d_x} \frac{\sin(\bomega_{x,i}/2)^{\mathfrak{m}}}{ (\bomega_{x,i}/2)^{\mathfrak{m}}} \right|^2 \geq \inf_{\omega\in [1.1\pi, 1.95\pi]} \left( \frac{\sin(\omega/2)}{\omega/2} \right)^{2\mathfrak{m}d_x} > 0.
\end{align*}
Define the following positive constant, independent of $\zeta$, $k$ or $n$,
\begin{align}
\label{eq:lower_bound_c_chi_defi}
    C_{\chi} :=  (\ast) \cdot \inf_{\omega\in [1.1\pi, 1.95\pi]} \left( \frac{\sin(\omega/2)}{\omega/2} \right)^{2\mathfrak{m}d_x} .
\end{align}
We thus have
\begin{align}
    \label{eq:q_tilde_conv_lb}
    \left\|\calF[\Omega_{\mathfrak{K},\bl_x}] \ast \Tilde{q}\right\|_{L^2(\mathbb{R}^{d_x})}^2 \geq C_{\chi} 2^{-2\mathfrak{K}\underline{s}\frac{d_x}{s_x}}. 
\end{align}
Therefore,
\begin{align*}
    (A)^2 &\geq \sum_{\bl_o,\bl_x}(\beta_{\bv(\bl_x,\bl_o)} - \beta_{\bv'(\bl_x',\bl_o')})^2 C_{\chi}2^{-\frac{2\mathfrak{K}\underline{s}d_x}{s_x}} \stackrel{(a)}{\geq} \frac{C_{\chi}}{8}|\calL| 2^{-\frac{2\mathfrak{K}\underline{s}d_x}{s_x}} \stackrel{(b)}{\gtrsim} \zeta^{-d_x}2^{-\mathfrak{K}\underline{s}\left(\frac{d_x}{s_x} - \frac{d_o}{s_o}\right)}.
\end{align*}
where $(a)$ follows from Eq. \eqref{eq: gv_beta_sep} and $(b)$ follows from Eq.~\eqref{eq:size_L}.

Next, we are going to upper bound $(B)$.
We have
\begin{align}
    (B)^2 &\leq \sum_{\bl_o}\left(\sum_{\bl_x}\left\|\calF[\Omega_{\mathfrak{K},\bl_x}]\ast (q-\Tilde{q})\right\|_{L^2(\mathbb{R}^{d_x})}\right)^2\nonumber\\
    &\stackrel{(a)}{\leq} \|q - \Tilde{q}\|^2_{L^1(\mathbb{R}^{d_x})} \sum_{\bl_o}\left(\sum_{\bl_x} \|\calF[\Omega_{\mathfrak{K},\bl_x}]\|_{L^2(\mathbb{R}^{d_x})}\right)^2 \nonumber\\
    &\stackrel{(b)}{\leq} \|q - \Tilde{q}\|^2_{L^1(\mathbb{R}^{d_x})} \sum_{\bl_o}\left(\sum_{\bl_x} \|M_{\mathfrak{K},\boldsymbol{-\frac{\mathfrak{m}}{2}}}\|_{L^2(\mathbb{R}^{d_x})}\right)^2 \nonumber\\
    &\stackrel{(c)}{=} \|q - \Tilde{q}\|^2_{L^1(\mathbb{R}^{d_x})} \sum_{\bl_o}\left(\sum_{\bl_x} 2^{-\mathfrak{K}\frac{\underline{s}d_x}{2s_x}} \|M_{0,\boldsymbol{-\frac{\mathfrak{m}}{2}}}\|_{L^2(\mathbb{R}^{d_x})}\right)^2 \nonumber\\
    &=  \|q - \Tilde{q}\|^2_{L^1(\mathbb{R}^{d_x})} \sum_{\bl_o}\left(\sum_{\bl_x}1 \right)^2  2^{-\mathfrak{K}\frac{\underline{s}d_x}{s_x}}\|M_{0,\boldsymbol{-\frac{\mathfrak{m}}{2}}}\|_{L^2(\mathbb{R}^{d_x})}^2\nonumber\\
    &\lesssim  \|q - \Tilde{q}\|^2_{L^1(\mathbb{R}^{d_x})} \zeta^{-2d_x}2^{\mathfrak{K}\underline{s}\left(\frac{2d_x}{s_x} + \frac{d_o}{s_o}\right)} 2^{-\mathfrak{K}\frac{\underline{s}d_x}{s_x}}\|M_{0,\boldsymbol{-\frac{\mathfrak{m}}{2}}}\|_{L^2(\mathbb{R}^{d_x})}^2\nonumber\\
    &\stackrel{(d)}{\leq} \|q - \Tilde{q}\|_{L^1(\mathbb{R}^{d_x})}^2 \zeta^{-\frac{1}{2}d_x} 2^{\frac{\mathfrak{K}\underline{s}d_o}{s_o}}2^{\frac{\mathfrak{K}\underline{s}d_x}{s_x}} \|M_{0,\boldsymbol{-\frac{\mathfrak{m}}{2}}}\|_{L^2(\mathbb{R}^{d_x})}^2 \label{eq: b_upper_bound}
\end{align}
In the above derivations, $(a)$ holds by Young's convolution inequality, $(b)$ holds by Plancherel's Theorem and Eq.~\eqref{eq:fourier_omega_k_lx}, $(c)$ holds by the change of variables $\bx\leftarrow 2^{-\frac{\mathfrak{K}\underline{s}}{s_x}}\bx$, and $(d)$ holds for $\zeta\geq 1$. 
We have
\begin{align}
    \|q - \Tilde{q}\|_{L^1(\mathbb{R}^{d_x})} &= \int_{\mathbb{R}^{d_x}} \left|q(\bomega_x) - \Tilde{q}(\bomega_x)\right| \;\mathrm{d}\bomega_x \nonumber\\
    &\stackrel{(i)}{=} 2^{d_x} \prod_{i=1}^{d_x} \int_{\zeta/3}^{\infty} \left|\calF[\sqrt{g}](\bomega_{x,i}) \right| \;\mathrm{d} \bomega_{x,i}\nonumber \\
    &\stackrel{(ii)}{\asymp} \prod_{i=1}^{d_x} \int_{\zeta/3}^{\infty} \bomega_{x,i}^{-3/4} \exp(-\sqrt{\bomega_{x,i}}) \;\mathrm{d}\bomega_x \nonumber\\
    &\leq \zeta^{-\frac{3}{4}d_x} \prod_{i=1}^{d_x}  \int_{\zeta}^\infty \exp(-\sqrt{\bomega_{x,i}}) \;\mathrm{d}\bomega_{x,i} \nonumber \\
    &= \zeta^{-\frac{3}{4}d_x} 2^{d_x}\left(1+\zeta^{1 / 2}\right)^{d_x} \exp\left(-d_x \zeta^{\frac{1}{2}}\right) \nonumber \\
    &\leq 2^{2d_x} \zeta^{-\frac{1}{4}d_x} \exp\left(-d_x \zeta^{\frac{1}{2}}\right),
    \label{eq:q_qt_diff}
\end{align}
where in $(i)$ we use the definition of $\tilde{q}$ in Eq.~\eqref{eq:defi_tilde_q}; and in $(ii)$
we use the asymptotic decay of the Fourier transform of the bump function~\citep[Section 2]{johnson2015saddle}. 
\begin{align}
\label{eq:bump_function_fourier_decay}
    \left|\calF[\sqrt{g}](\omega) \right| \asymp \omega^{-3/4} \exp(-\sqrt{\omega}), \quad \omega \gg 1.
\end{align}
Hence, we plug Eq.~\eqref{eq:q_qt_diff} back to Eq.~\eqref{eq: b_upper_bound} and we find that
\begin{align}\label{eq:upper_B}
    (B)^2 &\lesssim \zeta^{-\frac{1}{2}d_x} 2^{\frac{\mathfrak{K}\underline{s}d_o}{s_o} + \frac{\mathfrak{K}\underline{s}d_x}{s_x}}\left(2^{2d_x} \zeta^{-\frac{1}{4}d_x} \exp\left(-d_x \zeta^{\frac{1}{2}}\right)\right)^2 \|M_{0,\boldsymbol{-\frac{\mathfrak{m}}{2}}}\|_{L^2(\mathbb{R}^{d_x})}^2\nonumber\\ &\lesssim 2^{\frac{\mathfrak{K}\underline{s}d_o}{s_o} + \frac{\mathfrak{K}\underline{s}d_x}{s_x}} \zeta^{-d_x} \exp\left(-2 d_x \zeta^{\frac{1}{2}}\right).
\end{align}
Hence in order for $(B)^2\leq (A)^2/4$ to hold, a sufficient condition on $\zeta$ is given by the following inequality, up to some constants,
\begin{align}
    &2^{\frac{\mathfrak{K}\underline{s}d_o}{s_o} + \frac{\mathfrak{K}\underline{s}d_x}{s_x}} \zeta^{-d_x}\exp\left(-2 d_x \zeta^{\frac{1}{2}}\right)  \lesssim \frac{1}{4} \zeta^{-d_x}2^{-\frac{\mathfrak{K}\underline{s}d_x}{s_x}}2^{\frac{\mathfrak{K}\underline{s}d_o}{s_o}}\nonumber\\
    \iff \qquad\qquad & \exp\left(-2 d_x \zeta^{\frac{1}{2}}\right) \lesssim 2^{-\frac{2\mathfrak{K}\underline{s}d_x}{s_x}}. \label{eq:zeta_one}
\end{align}
Hence for sufficiently large $\zeta = \zeta(n)$, such that Eq. \eqref{eq:zeta_one} is satisfied, we have $(A)\geq (B)/2$, thus 
\begin{align}
\label{eq:f_l2_separate}
    \|f_{\bv} - f_{\bv'}\|^2_{L^2(P_{XO})} & \geq \epsilon_0^2 2^{-2\mathfrak{K}\underline{s}\left(1 - \frac{d_x}{2s_x}\right)} 2^{-\frac{\mathfrak{K}\underline{s}d_o}{s_o}}\|M_{0,\boldsymbol{-\frac{\mathfrak{m}}{2}}}\|_{L^2(\mathbb{R}^{d_o})}^2 \frac{1}{4}(A)^2\nonumber\\
    &\gtrsim \epsilon_0^2 2^{-2\mathfrak{K}\underline{s}\left(1 - \frac{d_x}{2s_x}\right)} 2^{-\frac{\mathfrak{K}\underline{s}d_o}{s_o}}\zeta^{-d_x}2^{-\frac{\mathfrak{K}\underline{s}d_x}{s_x}}2^{\frac{\mathfrak{K}\underline{s}d_o}{s_o}}\nonumber\\
    &= \epsilon_0^2 2^{-2\mathfrak{K}\underline{s}}\zeta^{-d_x}.
\end{align}

\vspace{3mm}
\noindent
\underline{\emph{Step Five}}
Recall the distributions $P_{f_\bv}$ defined above, and recall that $\boldsymbol{0}\in V_\mathfrak{K}$ for $V_\mathfrak{K}$ defined in Eq.~\eqref{eq: gv_beta_sep}. 
In this step, we are going to show that, for $P_{f_\bv}^{\otimes n} := P_{f_\bv} \otimes \dots \otimes P_{f_\bv}$ which is a probability distribution over $(\calZ\times \calO \times\mathbb{R})^{n}$, it satisfies
\begin{align*}
    \frac{1}{|V_\mathfrak{K}|}\sum_{\bv\in V_\mathfrak{K}}\operatorname{KL}\left(P_{f_\bv}^{\otimes n}, P_{f_{\boldsymbol{0}}}^{\otimes n}   \right) \lesssim \epsilon_0^2 n  2^{-2\mathfrak{K}\frac{\underline{s}d_x}{s_x} \eta_1 -2\mathfrak{K}\underline{s}} .
\end{align*}
Notice that KL divergence tensorizes over independent copies at each dimension, we have that  $\operatorname{KL}(P_{f_{\boldsymbol{0}}}^{\otimes n} \| P_{f_\bv}^{\otimes n}) = n \operatorname{KL}(P_{f_{\boldsymbol{0}}} \| P_{f_\bv})$, so we are going to study $\operatorname{KL}(P_{f_{\boldsymbol{0}}} \| P_{f_\bv})$ as follows.
\begin{align}\label{eq:single_kl}
    \operatorname{KL}\left( P_{f_\bv}, P_{f_{\boldsymbol{0}}} \right) &= \E_{(\bz,\bo) \sim P_{ZO}}\left[ \operatorname{KL}\left(  P_{f_\bv} \left(\cdot \mid 
    \bz,\bo \right) \right) , P_{f_{\boldsymbol{0}}} \left( \cdot \mid \bz,\bo \right)\right] \nonumber\\
    &\stackrel{(a)}{=} \frac{ \left\|T f_{\bv}\right\|_{L^2\left( P_{ZO} \right)}^2}{2 \sigma^2} \nonumber\\
    &\stackrel{(b)}{\leq} 2^{-2\mathfrak{K}\frac{\underline{s}d_x}{s_x} \eta_1} \frac{ \left\|f_{\bv}\right\|_{L^2(P_{XO})}^2}{2 \sigma^2} 
\end{align}
In the above chain of derivations, 
we use \cite{blanchard2018optimal}[Proposition 6.2] in $(a)$; and in $(b)$ we use the \Cref{ass:T_contractivity} along with the fact from Eq.~\eqref{eq:fourier_f_v} that the support of the Fourier transform of $f_{\bv}$ is indeed in high frequency.
\begin{align*}
    \mathrm{supp}(\calF[f_{\bv}]) = \cup_{\bl_x, \bl_o \in \calL} \tilde{I}_{\bl_x} \subseteq \left[ 1.1 \pi \cdot 2^{\mathfrak{K}\frac{\underline{s}}{s_x}}, 1.9\pi \cdot 2^{\mathfrak{K}\frac{\underline{s}}{s_x}} \right]^{d_x}.
\end{align*}
Next notice that, by Eq. \eqref{eq: gen_coeff_prob_l2}, we have
\begin{small}
\begin{align}\label{eq:f_v_L2}
    &\|f_{\bv}\|^2_{L^2(P_{XO})} \nonumber \\ &\stackrel{(a)}{\leq} \epsilon_0^2 2^{-2\mathfrak{K}\underline{s}\left(1 - \frac{d_x}{2s_x}\right)} 2^{-\frac{\mathfrak{K}\underline{s}d_o}{s_o}}\|M_{0,\boldsymbol{0}}\|^2_{L^2(\mathbb{R}^{d_o})}\int_{\mathbb{R}^{d_x}} \sum_{\bl_o}\left(\sum_{\bl_x}\beta_{\bv(\bl_x,\bl_o)}\calF\left[\Omega_{\mathfrak{K},\bl_x}\sqrt{p}\right](\bomega_x)\right)^2\;\mathrm{d}\bomega_x \nonumber \\
    &\leq \epsilon_0^2 2^{-2\mathfrak{K}\underline{s}\left(1 - \frac{d_x}{2s_x}\right)} 2^{-\frac{\mathfrak{K}\underline{s}d_o}{s_o}}\|M_{0,\boldsymbol{0}}\|^2_{L^2(\mathbb{R}^{d_o})}\int_{\mathbb{R}^{d_x}} \sum_{\bl_o}\left(\sum_{\bl_x} \left| \calF\left[\Omega_{\mathfrak{K},\bl_x}\sqrt{p}\right](\bomega_x) \right| \right)^2\;\mathrm{d}\bomega_x \nonumber \\
    &\stackrel{(b)}{\lesssim} \epsilon_0^2 2^{-2\mathfrak{K}\underline{s}\left(1 - \frac{d_x}{2s_x}\right)} \int_{\mathbb{R}^{d_x}} \left(\sum_{\bl_x} \left| \calF\left[\Omega_{\mathfrak{K},\bl_x}\sqrt{p}\right](\bomega_x) \right| \right)^2\;\mathrm{d}\bomega_x \nonumber \\
    &\leq \epsilon_0^2 2^{-2\mathfrak{K}\underline{s}\left(1 - \frac{d_x}{2s_x}\right)}  \left\|\sum_{\bl_x} \left| \calF[\Omega_{\mathfrak{K},\bl_x}] \ast (q - \Tilde{q} + \Tilde{q}) \right| \right\|^2_{L^2(\mathbb{R}^{d_x})} \nonumber \\
    &\stackrel{(c)}{\leq} \epsilon_0^2 2^{-2\mathfrak{K}\underline{s}\left(1 - \frac{d_x}{2s_x}\right)}  \left(2\left\|\sum_{\bl_x} \left|\calF[\Omega_{\mathfrak{K},\bl_x}]\ast (q-\Tilde{q}) \right| \right\|^2_{L^2(\mathbb{R}^{d_x})} + 2\left\|\sum_{\bl_x} \left| \calF[\Omega_{\mathfrak{K},\bl_x}]\ast \Tilde{q} \right| \right\|^2_{L^2(\mathbb{R}^{d_x})}\right) .
\end{align}
\end{small}
In the above derivations, $(a)$ holds by Eq. \eqref{eq: gen_coeff_prob_l2}, $(b)$ holds by $\sum_{\bl_o}1 \lesssim 2^{\frac{\mathfrak{K}\underline{s}d_o}{s_o}}$, $(c)$ holds by triangular inequality. 
The first term in Eq.~\eqref{eq:f_v_L2} has already been upper bounded in Eq.~\eqref{eq:upper_B}. Hence for sufficiently large $\zeta$ that satisfies Eq.~\eqref{eq:zeta_one}, 
\begin{align*}
    \left\| \sum_{\bl_x} \left| \calF[\Omega_{\mathfrak{K},\bl_x}]\ast (q-\Tilde{q}) \right| \right\|^2_{L^2(\mathbb{R}^{d_x})} &\lesssim 2^{\frac{\mathfrak{K}\underline{s}d_x}{s_x}} \zeta^{-\frac{d_x}{2}} \exp\left(-2 d_x \zeta^{\frac{1}{2}}\right)\\
    &\lesssim 2^{\frac{\mathfrak{K}\underline{s}d_x}{s_x}}2^{-\frac{2\mathfrak{K}\underline{s}d_x}{s_x}} = 2^{-\frac{\mathfrak{K}\underline{s}d_x}{s_x}}.
\end{align*}
The second term in Eq.~\eqref{eq:f_v_L2} can be upper bounded by
\begin{align*}
    \left\|\sum_{\bl_x} \left| \calF[\Omega_{\mathfrak{K},\bl_x}] \ast \Tilde{q} \right| \right\|^2_{L^2(\mathbb{R}^{d_x})} &\stackrel{(a)}{\lesssim} \sum_{\bl_x}\left\|\calF[\Omega_{\mathfrak{K},\bl_x}] \ast \Tilde{q} \right\|^2_{L^2(\mathbb{R}^{d_x})}\\ &\stackrel{(b)}{\leq} \sum_{\bl_x}\left\|\calF[\Omega_{\mathfrak{K},\bl_x}] \right\|^2_{L^2(\mathbb{R}^{d_x})} \cdot \|\Tilde{q}\|^2_{L^1(\mathbb{R}^{d_x})} \\
    &\stackrel{(c)}{\leq} \sum_{\bl_x}\left\|\calF[M_{\mathfrak{K},\boldsymbol{-\frac{\mathfrak{m}}{2}}}] \right\|^2_{L^2\left(\tilde{I}_{\bl_x}\right)} \cdot \|\Tilde{q}\|^2_{L^1(\mathbb{R}^{d_x})}\\ &\leq \left\|\calF[M_{\mathfrak{K},\boldsymbol{-\frac{\mathfrak{m}}{2}}}] \right\|^2_{L^2(\mathbb{R}^{d_x})} \cdot \|\Tilde{q}\|^2_{L^1(\mathbb{R}^{d_x})} \\
    &\stackrel{(d)}{=} \left\|M_{\mathfrak{K},\boldsymbol{-\frac{\mathfrak{m}}{2}}} \right\|^2_{L^2(\mathbb{R}^{d_x})} \cdot \|\Tilde{q}\|^2_{L^1(\mathbb{R}^{d_x})}\\ &\stackrel{(e)}{=} 2^{-\frac{\mathfrak{K}\underline{s}d_x}{s_x}} \left\|M_{0,\boldsymbol{-\frac{\mathfrak{m}}{2}}} \right\|^2_{L^2(\mathbb{R}^{d_x})} \|\Tilde{q}\|^2_{L^1(\mathbb{R}^{d_x})}\\
    &\lesssim 2^{-\frac{\mathfrak{K}\underline{s}d_x}{s_x}}. 
\end{align*}
In the above chain of derivations, $(a)$ holds because $\calF[\Omega_{\mathfrak{K},\bl_x}] \ast \Tilde{q}$ have disjoint support for different $\bl_x$ proved in Eq.~\eqref{eq:support_F_conv_tilde_q}; $(b)$ holds by Young’s convolution inequality; $(c)$ holds by Eq. \eqref{eq:fourier_omega_k_lx}; $(d)$ holds by the fact that $I_{\bl_x}$ are pairwise disjoint and Plancherel's Theorem and $(e)$ holds by change of variables $\bx\leftarrow 2^{-\frac{\mathfrak{K}\underline{s}}{s_x}}\bx$. 
Therefore, combining the upper bound on the above two terms, we obtain
\begin{align}
\label{eq:f_bv_l2}
    \|f_{\bv}\|^2_{L^2(P_{XO})} \lesssim \epsilon_0^2 2^{-2\mathfrak{K}\underline{s}\left(1 - \frac{d_x}{2s_x}\right)}  2^{-\frac{\mathfrak{K}\underline{s}d_x}{s_x}} = \epsilon_0^2 2^{-2\mathfrak{K}\underline{s}} .
\end{align}
We plug the upper bound on $\|f_{\bv}\|^2_{L^2(P_{XO})}$ back to Eq.~\eqref{eq:single_kl} to obtain
\begin{align}\label{eq:kl_upper_bound_small}
    \operatorname{KL}\left(P_{f_\bv}, P_{f_{\boldsymbol{0}}} \right) \lesssim \epsilon_0^2 \sigma^{-2} 2^{-2\mathfrak{K}\frac{\underline{s}d_x}{s_x} \eta_1 -2\mathfrak{K}\underline{s}}.
\end{align}
And thus,
\begin{align}
\label{eq:KL_upper_bound}
    \frac{1}{|V_\mathfrak{K}|}\sum_{\bv\in V_\mathfrak{K}}\operatorname{KL}\left(P_{f_\bv}^{\otimes n}, P_{f_{\boldsymbol{0}}}^{\otimes n} \right) \leq \epsilon_0^2 n\sigma^{-2} 2^{-2\mathfrak{K}\frac{\underline{s}d_x}{s_x} \eta_1 -2\mathfrak{K}\underline{s}}.
\end{align}

\vspace{3mm}
\noindent
\underline{\emph{Step Six}}
In this step, we are going to show that, for any measurable learning method $(\bz_i, \bo_i,y_i)_{i=1}^n =: D \mapsto \hat{f}_D$, there is a distribution $P$ among $P_{f_\bv}$ with $\bv \in V_\mathfrak{K}$ which is difficult to learn for the considered learning method. We define a measurable mapping
\begin{align}
    \Psi:\left([0,1]^{d_z}\times[0,1]^{d_o}\times\R\right)^n \rightarrow V_\mathfrak{K}, \quad \Psi(D) := \underset{\bv \in V_\mathfrak{K}}{\operatorname{argmin}} \left\| \hat{f}_D -f_{\bv} \right\|_{L^2(P_{XO})} .
\end{align}
For $\bv \in V_\mathfrak{K}$ and $D \in \left([0,1]^{d_z}\times[0,1]^{d_o}\times\R\right)^n $ with $\Psi(D) \neq \bv$, we start from \eqref{eq:f_l2_separate} to have
\begin{align*}
    \epsilon_0 2^{-\mathfrak{K}\underline{s}} \zeta^{-\frac{d_x}{2}} &\lesssim \left\|f_{\Psi(D)} - f_{\bv} \right\|_{L^2(P_{XO})}\\ &\leq\left\|f_{\Psi(D)} -\hat{f}_D \right\|_{L^2(P_{XO})} + \left\|\hat{f}_D - f_{\bv} \right\|_{L^2(P_{XO})}\\ &\leq 2\left\|\hat{f}_D - f_{\bv} \right\|_{L^2(P_{XO})} .
\end{align*}
Consequently, for all $\bv \in V_\mathfrak{K}$ we find
\begin{align*}
    P_{f_\bv}^{\otimes n}(D: \Psi(D) \neq \bv) \leq P_{f_\bv}^{\otimes n} \left(D:\left\|\hat{f}_D - f_{\bv} \right\|_{L^2(P_{XO})} \gtrsim \epsilon_0 2^{-\mathfrak{K}\underline{s}} \zeta^{-\frac{d_x}{2}}\right) .
\end{align*}
Therefore, we have
\begin{align}
    &\max _{\bv \in V_\mathfrak{K}} P_{f_\bv}^{\otimes n} \left(D:\left\|\hat{f}_D - f_{\bv} \right\|_{L^2(P_{XO})} \gtrsim \epsilon_0 2^{-\mathfrak{K}\underline{s}} \zeta^{-\frac{d_x}{2}}\right) \nonumber \\ 
    \geq & \max _{\bv \in V_\mathfrak{K}} P_{f_\bv}^{\otimes n}(D: \Psi(D) \neq \bv) \nonumber \\
    \stackrel{(a)}{\gtrsim} & \frac{\sqrt{|V_\mathfrak{K}|}}{\sqrt{|V_\mathfrak{K}|} + 1} \left( 1 - \frac{\epsilon_0^2 n 2^{-2\mathfrak{K}\frac{\underline{s}d_x}{s_x} \eta_1} 2^{-2\mathfrak{K}\underline{s}} }{ \log |V_\mathfrak{K}| } -\frac{1}{2 \log |V_\mathfrak{K}|}\right) \nonumber \\
    \stackrel{(b)}{\geq} & \frac{\sqrt{|V_\mathfrak{K}|}}{\sqrt{|V_\mathfrak{K}|} + 1} \left( 1 - \frac{\epsilon_0^2 n2^{-2\mathfrak{K}\frac{\underline{s}d_x}{s_x} \eta_1}2^{-2\mathfrak{K}\underline{s}} }{\zeta^{-d_x}2^{\mathfrak{K}\frac{\underline{s}d_x}{s_x} + \mathfrak{K}\frac{\underline{s}d_o}{s_o}}}  -\frac{1}{2 \log |V_\mathfrak{K}|}\right) \nonumber \\
    \geq & \frac{\sqrt{|V_\mathfrak{K}|}}{\sqrt{|V_\mathfrak{K}|} + 1} \left( 1 - \epsilon_0^2 \zeta^{d_x} n 
    2^{-\mathfrak{K}\frac{\underline{s}}{s_x}\left(2d_x\eta_1 + 2s_x + d_x + \frac{d_o}{s_o}s_x\right)} -\frac{1}{2 \log |V_\mathfrak{K}|}\right). \label{eq:V_k_lower_bound}
\end{align}
$(a)$ holds by \citet[Theorem 20]{fischer2020sobolev} and Eq. \eqref{eq:KL_upper_bound}. $(b)$ holds by the construction of $V_\mathfrak{K}$ in Eq.~\eqref{eq: gv_beta_sep} with $|V_\mathfrak{K}|\geq 2^{\frac{\calL}{8}}$ so $\log |V_\mathfrak{K}|\gtrsim \zeta^{-d_x}2^{\mathfrak{K}\underline{s}(\frac{d_x}{s_x} + \frac{d_o}{s_o})}$. Next, we choose $\zeta$ and $k$ as the following function of $n$:
\begin{align}
\begin{aligned}
\label{eq:choice_k_lb}
    \zeta = \left( \frac{1}{2} \frac{1}{2d_x\eta_1 + 2s_x + d_x + \frac{d_o}{s_o} s_x} \log n\right)^2, \quad 2^{-\mathfrak{K}\underline{s}} = n^{-\frac{s_x}{2d_x\eta_1 + 2s_x + d_x + \frac{d_o}{s_o} s_x}} \zeta^{-s_x} .
\end{aligned}
\end{align}
The choice of $\zeta$ ensures that Eq. \eqref{eq:zeta_one} is satisfied for sufficiently large $n\geq 1$. 
The choice of $k$ as a function of $n$ ensures that we can proceed from Eq.~\eqref{eq:V_k_lower_bound} to obtain
\begin{align*}
    &\quad \max _{\bv \in V_\mathfrak{K}} P_{f_\bv}^{\otimes n} \left(D: \left\| \hat{f}_D - f_{\bv} \right\|_{L^2(P_{XO})} \geq \epsilon_0 n^{-\frac{s_x}{2s_x + 2d_x\eta_1 + d_x + \frac{d_o}{s_o}s_x}} (\log n)^{-2s_x - d_x} \right) \\
    &\gtrsim \frac{\sqrt{|V_\mathfrak{K}|}}{\sqrt{|V_\mathfrak{K}|} + 1} \left( 1 - \epsilon_0^2 - \frac{1}{2 \log |V_\mathfrak{K}|}\right) \\
    &\geq 1 - 2 \epsilon_0^2.
\end{align*}
The last inequality holds for sufficiently large $n$ because $ |V_\mathfrak{K}| \to \infty$ as $n \to \infty$.
We set $2\epsilon_0^2 = \tau^2$ and relabelling the constants, to obtain
\begin{align*}
    \inf_{D \mapsto \hat{f}_D} \sup_{f \in B_{2,q}^{s_x,s_o}\left(\R^{d_x+d_o}\right)} \left\| \hat{f}_D - f \right\|_{L^2(P_{XO})} &\geq \inf_{D \mapsto \hat{f}_D} \max_{\bv \in V_\mathfrak{K}} \left\| \hat{f}_D - f_{\bv}\right\|_{L^2(P_{XO})} \\
    &\geq \tau n^{-\frac{s_x}{2s_x + 2d_x\eta_1 + d_x + \frac{d_o}{s_o}s_x}} (\log n)^{-2s_x - d_x} \\
    &= \tau n^{-\frac{\frac{s_x}{d_x}}{1+2(\frac{s_x}{d_x}+\eta_1)+\frac{d_o}{s_o}\frac{s_x}{d_x}}} (\log n)^{-2s_x - d_x} ,
\end{align*}
holds for sufficiently large $n$ with probability $\geq 1 - C \tau^2$.

\subsection{Auxiliary Results}
\begin{lem}\label{lem:besov_young}
Let $f:\R^d\to\R$, $f\in B_{2,q}^{\bs}(\R^d)$ and $K\in L^1(\R^d)$, then 
\begin{align*} 
    \|K\ast f \|_{B_{2,q}^{\bs}(\R^d)} \leq \|K\|_{L^1(\R^d)}\|f\|_{B_{2,q}^{\bs}(\R^d)}
\end{align*}
\end{lem}
\begin{proof}
Define $(\tau_{\mathbf{h}}f)(\bx) := f(\bx + \mathbf{h})$. We have
\begin{align*}
    \tau_{\mathbf{h}}(K\ast f)(\bx) = (K\ast f)(\bx + \mathbf{h}) = (K \ast (\tau_{\mathbf{h}}f))(\bx) .
\end{align*}
Since $\Delta^{r}_{\mathbf{h}} = (\tau_{\mathbf{h}} - \mathrm{id})^r$, we have $\Delta^{r}_{\mathbf{h}}(K\ast f) = K\ast \left(\Delta^{r}_{\mathbf{h}}f\right)$. Thus for any $r > \max\{s_1, \ldots, s_d\}$, 
\begin{align}
    \omega_{r,2}\left(K\ast f, t^{\frac{1}{s_1}},\dots, t^{\frac{1}{s_d}}, \R^d\right) &= \sup_{0<|h_i|<t^{\frac{1}{s_i}}}\left\|\Delta^{r}_{\mathbf{h}}(K\ast f)\right\|_{L^2(\R^d)}\nonumber\\
    &= \sup_{0<|h_i|<t^{\frac{1}{s_i}}}\left\|K\ast \Delta^{r}_{\mathbf{h}}( f)\right\|_{L^2(\R^d)}\nonumber\\
    &\stackrel{(i)}{\leq} \sup_{0<|h_i|< t^{\frac{1}{s_i}}}\|K\|_{L^1(\R^d)} \left\|\Delta^{r}_{\mathbf{h}}f\right\|_{L^2(\R^d)}\label{eq:lem_young_besov_modified}\\
    &= \|K\|_{L^1(\R^d)}\omega_{r,2}\left(f,t^{\frac{1}{s_1}},\dots, t^{\frac{1}{s_d}},\R^d\right), \nonumber
\end{align}
where we use Young's convolution inequality in $(i)$. Hence
\begin{align*}
    |K\ast f|_{B_{2,q}^{\bs}(\R^d)} &= \left(\int_{0}^{1}\left[t^{-1}\omega_{r,2}\left(K\ast f,t^{\frac{1}{s_1}},\dots, t^{\frac{1}{s_d}},\R^d\right)\right]^q\;\frac{\mathrm{d}t}{t}\right)^{\frac{1}{q}}\\
    &\leq \|K\|_{L^1(\R^d)}\left(\int_{0}^{1}\left[t^{-1}\omega_{r,2}\left(f,t^{\frac{1}{s_1}},\dots, t^{\frac{1}{s_d}},\R^d\right)\right]^q\;\frac{\mathrm{d}t}{t}\right)^{\frac{1}{q}} \\
    &= \|K\|_{L^1(\R^d)}|f|_{B^{\bs}_{2,q}(\R^d)} .
\end{align*}
Hence
\begin{align*}
    \|K\ast f\|_{B^{\bs}_{2,q}(\mathbb{R}^d)} &= \|K\ast f\|_{L^2(\mathbb{R}^d)} + |K\ast f|_{B^{\bs}_{2,q}(\mathbb{R}^d)}\\
    &\stackrel{(i)}{\leq}  \|K\|_{L^1(\mathbb{R}^d)}\|f\|_{L^2(\mathbb{R}^d)} + \|K\|_{L^1(\mathbb{R}^d)}|f|_{B^{\bs}_{2,q}(\mathbb{R}^d)} = \|K\|_{L^1(\mathbb{R}^d)}\|f\|_{B^{\bs}_{2,q}(\mathbb{R}^d)},
\end{align*}
where $(i)$ is due to Young's convolution inequality.
\end{proof}

\begin{lem}\label{lem:new_besov_mask}
Let $M_{k,-\frac{\mathfrak{m}}{2}}: \R^{d_x}\to\R, M_{\mathfrak{K},\bl_o}:\R^{d_o}\to\R$ be defined in Eq.~\eqref{eq:Bspline_main}. 
Let $\one_{S(\bl_o)}:\R^{d_x}\to\R$ be an indicator function over a measurable set $S(\bl_o)\subset\R^{d_x}$ for each $\bl_o$ in Eq.~\eqref{eq:calL}. 
Define $f(\bx,\bo):= \sum_{\bl_o}(M_{k,-\frac{\mathfrak{m}}{2}} \ast \calF^{-1}[\one_{S(\bl_o)}])(\bx) \cdot M_{\mathfrak{K},\bl_o}(\bo)$. Then we have
    $\|f\|_{B^{s_x,s_o}_{2,\infty}(\mathbb{R}^{d_x+d_o})}\lesssim 2^{\mathfrak{K}\underline{s} - \mathfrak{K}\frac{\underline{s}}{s_x}\frac{d_x}{2}}$.
\end{lem}
\begin{proof}
By definition of Besov norm in \citet{leisner_nonlinear_wavelet_approximation_2003}[Eq. (2.2)], for any $r > \max\{s_x, s_o\}$, we have
\begin{small}
\begin{align*}
    \|f\|_{B^{s_x,s_o}_{2,\infty}(\mathbb{R}^{d_x+d_o})} &\asymp \underbrace{\sup_{0<t<1}\left[t^{-1}\omega^{\bx}_{r,2}(f,t^{1/s_x},\dots,t^{1/s_x})\right]}_{(I)} + \underbrace{\sup_{0<t<1}\left[t^{-1}\omega^{\bo}_{r,2}(f,t^{1/s_o},\dots,t^{1/s_o})\right]}_{(II)},
\end{align*}
\end{small}
where $\omega_{r,2}^{\bx}$ (respectively $\omega_{r,2}^{\bo}$) denotes the \emph{partial} modulus of smoothness in the $\bx$ (respectively $\bo$)-direction, and we restrict the supremum from $t\in (0,\infty)$ to $t\in (0,1)$ by \citet[Theorem 10.1]{devore1993constructive}. Specifically, define $(\tau_{\mathbf{h}}g)(\bx) := g(\bx + \mathbf{h})$ and $\Delta^{r}_{\mathbf{h}} := (\tau_{\mathbf{h}} - \mathrm{id})^r$, we write
\begin{align*}
    \omega^{\bx}_{r,2}(f, t^{1/s_x},\dots, t^{1/s_x})
    :=\sup_{|h_i|\leq t^{1/s_x}}\left(\int_{\mathbb{R}^{d_o}}\int_{\mathbb{R}^{d_x}} \left|\Delta^{r}_{\bh}(f(\cdot, \bo))(\bx)\right|^2 \;\mathrm{d}\bx\;\mathrm{d}\bo\right)^{\frac{1}{2}},
\end{align*}
and $\omega_{r,2}^{\bo}$ is defined similarly. We first bound $(I)$.  We have
\begin{align*}
    \tau_{\mathbf{h}}(K\ast g)(\bx) = (K\ast g)(\bx + \mathbf{h}) = (K \ast (\tau_{\mathbf{h}}g))(\bx) .
\end{align*}
We thus have $\Delta^{r}_{\mathbf{h}}(K\ast g) = K\ast \left(\Delta^{r}_{\mathbf{h}}g\right)$. Thus we have
\begin{align*}
    &\omega^{\bx}_{r,2}(f, t^{1/s_x},\dots, t^{1/s_x})\\
    =&\sup_{|h_i|\leq t^{1/s_x}}\left(\int_{\mathbb{R}^{d_o}}\int_{\mathbb{R}^{d_x}} \left|\sum_{\bl_o} ( (\Delta_{\bh}^r M_{k,-\frac{\mathfrak{m}}{2}}) \ast \calF^{-1}[\one_{S(\bl_o)}])(\bx) M_{\mathfrak{K},\bl_o}(\bo)\right|^2 \;\mathrm{d}\bx\;\mathrm{d}\bo\right)^{\frac{1}{2}}\\
    =&\sup_{|h_i|\leq t^{1/s_x}}\left(\sum_{\bl_o,\bl_o'}\int_{\mathbb{R}^{d_x}} ( (\Delta_{\bh}^r M_{k,-\frac{\mathfrak{m}}{2}}) \ast \calF^{-1}[\one_{S(\bl_o)}])(\bx) \overline{( (\Delta_{\bh}^r M_{k,-\frac{\mathfrak{m}}{2}}) \ast \calF^{-1}[\one_{S(\bl_o')}])(\bx)} \;\mathrm{d}\bx  \right. \\
    &\qquad \left. \cdot \int_{\mathbb{R}^{d_o}}M_{\mathfrak{K},\bl_o}(\bo) M_{k,\bl_o'}(\bo) \;\mathrm{d}\bo\right)^{\frac{1}{2}}\\
    \stackrel{(a)}{=}&\sup_{|h_i|\leq t^{1/s_x}}\left(\sum_{\bl_o}\|\calF^{-1}[\one_{S(\bl_o)}]\ast (\Delta^r_{\bh}M_{k,-\frac{\mathfrak{m}}{2}})\|_{L^2(\mathbb{R}^{d_x})}^2 \|M_{\mathfrak{K},\bl_o}\|_{L^2(\mathbb{R}^{d_o})}^2\right)^{\frac{1}{2}}\\
    \stackrel{(b)}{\leq}& \sup_{|h_i|\leq t^{1/s_x}}\left(\sum_{\bl_o}\|\Delta^r_{\bh}M_{k,-\frac{\mathfrak{m}}{2}}\|_{L^2(\mathbb{R}^{d_x})}^2 \|M_{\mathfrak{K},\bl_o}\|_{L^2(\mathbb{R}^{d_o})}^2\right)^{\frac{1}{2}}\\
    =& \left(\sup_{|h_i|\leq t^{1/s_x}}\|\Delta^r_{\bh}M_{k,-\frac{\mathfrak{m}}{2}}\|_{L^2(\mathbb{R}^{d_x})}\right)\left(\sum_{\bl_o}\|M_{\mathfrak{K},\bl_o}\|_{L^2(\mathbb{R}^{d_o})}^2\right)^{\frac{1}{2}}\\
    \stackrel{(c)}{\leq} & \sup_{|h_i|\leq t^{1/s_x}}\|\Delta^r_{\bh}M_{k,-\frac{\mathfrak{m}}{2}}\|_{L^2(\mathbb{R}^{d_x})}\\
    =& \omega_{r,2}(M_{k,-\frac{\mathfrak{m}}{2}},t^{1/s_x},\dots,t^{1/s_x}).
\end{align*}
In the above derivations,
\begin{itemize}
    \item we use $\mathrm{supp}(M_{\mathfrak{K},\bl_o})\cap \mathrm{supp}(M_{\mathfrak{K},\bl_o}) = \emptyset$ if $\bl_o\neq \bl_o'$ in $(a)$,
    \item we use in $(b)$ the inequality 
\begin{align}
\begin{aligned}
\label{eq:plancherel_indicator_inequality}
    &\quad \|\calF^{-1}[\one_{S(\bl_o)}] \ast (\Delta^r_{\bh}M_{k,-\frac{\mathfrak{m}}{2}})\|^2_{L^2(\mathbb{R}^{d_x})} = \|\one_{S(\bl_o)} \cdot \calF[\Delta^r_{\bh}M_{k,-\frac{\mathfrak{m}}{2}}]\|^2_{L^2(\mathbb{R}^{d_x})} \\
    &\leq \|\calF[\Delta_{\bh}^rM_{k,-\frac{\mathfrak{m}}{2}}]\|^2_{L^2(\mathbb{R}^{d_x})} = \|\Delta_{\bh}^rM_{k,-\frac{\mathfrak{m}}{2}}\|^2_{L^2(\mathbb{R}^{d_x})},
\end{aligned}
\end{align}
\item we use $
\sum_{\bl_o}\|M_{\mathfrak{K},\bl_o}\|^2_{L^2(\mathbb{R}^{d_o})} \leq 2^{\frac{\mathfrak{K}\underline{s}d_o}{s_o}}2^{-\frac{\mathfrak{K}\underline{s}d_o}{s_o}}\|M_{0\boldsymbol{0}}\|^2_{L^2(\mathbb{R}^{d_o})} \leq 1$ in $(c)$. Indeed, note that $\|M_{0\boldsymbol{0}}\|^2_{L^2(\mathbb{R}^{d_o})} = \|\iota_{m}\|^{2d_o}_{L^2(\mathbb{R})} \leq \|\iota_{m}\|^{2d_o}_{L^1(\mathbb{R})} = 1$. 
\end{itemize}
Thus we have
\begin{align*}
    (I)&\leq \sup_{t>0}\left[t^{-1}\omega_{r,2}\left(M_{k,-\frac{\mathfrak{m}}{2}},t^{\frac{1}{s_x}},\dots,t^{\frac{1}{s_x}}\right)\right] \asymp |M_{k,-\frac{\mathfrak{m}}{2}}|_{B^{s_x}_{2,\infty}(\mathbb{R}^{d_x})} \asymp 2^{\mathfrak{K}\underline{s}\left(1 -\frac{d_x}{2s_x}\right)},
\end{align*}
where we bound $(I)$ by an isotropic Besov norm, and we use the sequential Besov norm equivalence \citep{devore1988interpolation}[Theorem 5.1] (see also \citet[Theorem 3.4]{leisner_nonlinear_wavelet_approximation_2003} applied to $B^{s_x}_{2,\infty}(\mathbb{R}^{d_x})$) in the last equality. Now we bound $(II)$. Notice that
\begin{align*}
    &\omega^{\bo}_{r,2}(f,t^{1/s_o},\dots, t^{1/s_o})\\
    \stackrel{(a)}{\lesssim} &\omega_{s_o,2}^{\bo}(f,t^{1/s_o},\dots, t^{1/s_o})\\
    \stackrel{(b)}{\lesssim} &2^{\mathfrak{K}\underline{s}}\omega^{\bo}_{s_o,2}\left(f,2^{-\frac{\mathfrak{K}\underline{s}}{s_o}}t^{1/s_o},\dots, 2^{-\frac{\mathfrak{K}\underline{s}}{s_o}}t^{1/s_o}\right)\\
    =&2^{\mathfrak{K}\underline{s}} \sup_{|h_i|\leq 2^{-\frac{\mathfrak{K}\underline{s}}{s_o}}t^{1/s_o}}\left(\int_{\mathbb{R}^{d_x+d_o}} \left|\sum_{\bl_o}(\calF^{-1}[\one_{S(\bl_o)}] \ast M_{k,-\frac{\mathfrak{m}}{2}})(\bx) \cdot (\Delta^{s_o}_{\bh} M_{\mathfrak{K},\bl_o})(\bo) \right|^2 \;\mathrm{d}\bx\;\mathrm{d}\bo\right)^{\frac{1}{2}}\\
    =&2^{\mathfrak{K}\underline{s}}\sup_{|h_i|\leq 2^{-\frac{\mathfrak{K}\underline{s}}{s_o}} t^{1/s_o}}\left(\sum_{\bl_o,\bl_o'}\int_{\mathbb{R}^{d_x}} (\calF^{-1}[\one_{S(\bl_o)}]\ast M_{k,-\frac{\mathfrak{m}}{2}})(\bx) \cdot \overline{\calF^{-1}[\one_{S(\bl_o^\prime)}]\ast M_{k,-\frac{\mathfrak{m}}{2}}}(\bx)\;\mathrm{d}\bx \right. \\
    &\qquad \left. \cdot \int_{\mathbb{R}^{d_o}}(\Delta^{s_o}_{\bh}M_{\mathfrak{K},\bl_o})(\bo)\cdot (\Delta^{s_o}_{\bh}M_{k,\bl_o'})(\bo) \;\mathrm{d}\bo\right)^{\frac{1}{2}}\\
    \stackrel{(c)}{=}& 2^{\mathfrak{K}\underline{s}}\sup_{|h_i|\leq 2^{-\frac{\mathfrak{K}\underline{s}}{s_o}}t^{1/s_o}}\left(\sum_{\bl_o}\|\calF^{-1}[\one_{S(\bl_o)}] \ast M_{k,\bl_x}\|^2_{L^2(\mathbb{R}^{d_x})} \cdot \|\Delta^{s_o}_{\bh}M_{\mathfrak{K},\bl_o}\|^2_{L^2(\mathbb{R}^{d_o})}\right)^{\frac{1}{2}}\\
    \stackrel{(d)}{\leq} & 2^{\mathfrak{K}\underline{s}}\sup_{|h_i|\leq 2^{-\frac{\mathfrak{K}\underline{s}}{s_o}}t^{1/s_o}}\left(\sum_{\bl_o}\|M_{k,-\frac{\mathfrak{m}}{2}}\|^2_{L^2(\mathbb{R}^{d_x})} \cdot \|\Delta^{s_o}_{\bh}M_{\mathfrak{K},\bl_o}\|^2_{L^2(\mathbb{R}^{d_o})}\right)^{\frac{1}{2}}\\
    = & 2^{\mathfrak{K}\underline{s}}\|M_{k,-\frac{\mathfrak{m}}{2}}\|_{L^2(\mathbb{R}^{d_x})}\left( \sup_{|h_i|\leq 2^{-\frac{\mathfrak{K}\underline{s}}{s_o}} t^{1/s_o}}\sum_{\bl_o}\|\Delta^{s_o}_{\bh}M_{\mathfrak{K},\bl_o}\|^2_{L^2(\mathbb{R}^{d_o})}\right)^{\frac{1}{2}}\\
    \stackrel{(e)}{=} & 2^{\mathfrak{K}\underline{s}}\|M_{k,-\frac{\mathfrak{m}}{2}}\|_{L^2(\mathbb{R}^{d_x})}\left( \sup_{|h_i|\leq 2^{-\frac{\mathfrak{K}\underline{s}}{s_o}} t^{1/s_o}}\left\|\sum_{\bl_o}\left(\Delta^{s_o}_{\bh}M_{\mathfrak{K},\bl_o}\right)\right\|^2_{L^2(\mathbb{R}^{d_o})}\right)^{\frac{1}{2}}\\
    \stackrel{(f)}{=} & 2^{\mathfrak{K}\underline{s}}\|M_{k,-\frac{\mathfrak{m}}{2}}\|_{L^2(\mathbb{R}^{d_x})}\left( \sup_{|h_i|\leq 2^{-\frac{\mathfrak{K}\underline{s}}{s_o}} t^{1/s_o}}\left\|\Delta^{s_o}_{\bh}\left(\sum_{\bl_o}M_{\mathfrak{K},\bl_o}\right)\right\|^2_{L^2(\mathbb{R}^{d_o})}\right)^{\frac{1}{2}}\\
    = &2^{\mathfrak{K}\underline{s}}\|M_{k,-\frac{\mathfrak{m}}{2}}\|_{L^2(\mathbb{R}^{d_x})}\left(\omega_{s_o,2}\left(\sum_{\bl_o}M_{\mathfrak{K},\bl_o},2^{-\frac{\mathfrak{K}\underline{s}}{s_o}}t^{\frac{1}{s_o}},\dots,2^{-\frac{\mathfrak{K}\underline{s}}{s_o}}t^{\frac{1}{s_o}}\right)^2\right)^{\frac{1}{2}}.
\end{align*}
In the above derivations,
\begin{itemize}
    \item we use $r=s_x\vee s_o \geq s_o$ and Minkowski's inequality in $(a)$ (this is also sometimes referred to as the reverse Marchaud inequality, see \cite{KOLOMOITSEV2020105423}[Property 8]),
    \item we use \cite{leisner_nonlinear_wavelet_approximation_2003}[Theorem 2.1.1] in $(b)$, which we can apply since $2^{\frac{\mathfrak{K}\underline{s}}{s_o}}\geq 1$,
    \item we deduce from Eq. \eqref{eq:Bspline_main} that
    \begin{align*}
        \mathrm{supp}(M_{k\bl_o}) = \bigtimes_{j=1}^{d_o}\left[2^{-\left\lfloor\frac{\mathfrak{K}\underline{s}}{s_o}\right\rfloor}\ell_{o,j} , 2^{-\left\lfloor\frac{\mathfrak{K}\underline{s}}{s_o}\right\rfloor}(\mathfrak{m}+\ell_{o,j})\right].
    \end{align*}
    We also see that 
    \begin{align*}
        \mathrm{supp}(\Delta^{s_o}_{\bh}M_{k\bl_o}) = \bigtimes_{j=1}^{d_o}\left[2^{-\left\lfloor\frac{\mathfrak{K}\underline{s}}{s_o}\right\rfloor}\ell_{o,j} - s_oh_j , 2^{-\left\lfloor\frac{\mathfrak{K}\underline{s}}{s_o}\right\rfloor}(\mathfrak{m}+\ell_{o,j})\right]
    \end{align*}
    We deduce that a sufficient condition to guarantee that \begin{align*}
        \mathrm{supp}\left(\Delta^{s_o}_{\bh}M_{\mathfrak{K},\bl_o}\right)\cap\mathrm{supp}\left(\Delta^{s_o}_{\bh}M_{k,\bl_o'}\right) = \emptyset
    \end{align*}
    if $\bl_o\neq \bl_o'$ is given by 
    \begin{align*}
        (\forall j=1,\dots,d_o),\; |h_j|\leq 2^{-\left\lfloor\frac{\mathfrak{K}\underline{s}}{s_o}\right\rfloor}\frac{\mathfrak{m}-1}{s_o}.
    \end{align*}
    Since $\mathfrak{m}-1\geq s_o$ by choice, and we have for all $i = 1,\dots, d_o$, $|h_i|\leq 2^{-\frac{\mathfrak{K}\underline{s}}{s_o}}t^{1/s_o}\leq 2^{-\frac{\mathfrak{K}\underline{s}}{s_o}}$, this sufficient condition is satisfied. Hence 
    \begin{align}
        \label{eq:lemedotfiveortho}
        \left\langle \Delta^{s_o}_{\bh}M_{\mathfrak{K},\bl_o},\Delta^{s_o}_{\bh}M_{k,\bl_o'}\right\rangle_{L^2(\mathbb{R}^{d_o})} = \left\|\Delta^{s_o}_{\bh}M_{\mathfrak{K},\bl_o}\right\|_{L^2(\mathbb{R}^{d_o})}^2\delta_{\bl_o,\bl_o'}. 
    \end{align}
    \item we use Plancherel's Theorem in $(d)$, in a similar way as in Eq. \eqref{eq:plancherel_indicator_inequality},
    \item we use Eq. \eqref{eq:lemedotfiveortho} again for step $(e)$,
    \item we use linearity of $\Delta^{s_o}_{\bh}$ for step $(f)$. 
\end{itemize}
Thus we have shown that, for $q=\infty$, 
\begin{align*}
    (II) &\lesssim 2^{\mathfrak{K}\underline{s}}\|M_{k,-\frac{\mathfrak{m}}{2}}\|_{L^2(\mathbb{R}^{d_x})}\sup_{t>0}\left[t^{-1}\omega_{s_o,2}\left(\sum_
    {\bl_o}M_{\mathfrak{K},\bl_o}, 2^{-\frac{\mathfrak{K}\underline{s}}{s_o}}t^{\frac{1}{s_o}},\dots, 2^{-\frac{\mathfrak{K}\underline{s}}{s_o}}t^{\frac{1}{s_o}}\right)\right]\\
    &\lesssim \|M_{k,-\frac{\mathfrak{m}}{2}}\|_{L^2(\mathbb{R}^{d_x})}\sup_{t>0}\left[\left(t2^{-\mathfrak{K}\underline{s}}\right)^{-1}\omega_{s_o,2}\left(\sum_
    {\bl_o}M_{\mathfrak{K},\bl_o}, 2^{-\frac{\mathfrak{K}\underline{s}}{s_o}}t^{\frac{1}{s_o}},\dots, 2^{-\frac{\mathfrak{K}\underline{s}}{s_o}}t^{\frac{1}{s_o}}\right)\right]\\
    &\stackrel{(a)}{\lesssim} \|M_{k,-\frac{\mathfrak{m}}{2}}\|_{L^2(\mathbb{R}^{d_x})}\sup_{t>0}\left[\int_{t}^{\infty} w^{-s_o}\omega_{s_o+1}\left(\sum_{\bl_o}M_{\mathfrak{K},\bl_o}, w,\dots,w\right)\;\frac{\mathrm{d}w}{w}\right]\\
    &\leq \|M_{k,-\frac{\mathfrak{m}}{2}}\|_{L^2(\mathbb{R}^{d_x})}\int_{0}^{\infty} w^{-s_o}\omega_{s_o+1}\left(\sum_{\bl_o}M_{\mathfrak{K},\bl_o}, w,\dots,w\right)\;\frac{\mathrm{d}w}{w}\\
    &\asymp \|M_{k,-\frac{\mathfrak{m}}{2}}\|_{L^2(\mathbb{R}^{d_x})}\left\|\sum_{\bl_o}M_{\mathfrak{K},\bl_o}\right\|_{B^{s_o}_{2,1}(\mathbb{R}^{d_o})}\\
    &\stackrel{(b)}{\lesssim} 2^{-\frac{\mathfrak{K}\underline{s}d_x}{2s_x}}2^{\frac{\mathfrak{K}\underline{s}}{s_o}(s_o - d_o/2)}\left(\sum_{\bl_o}1\right)^{\frac{1}{2}}\\
    &\lesssim 2^{\mathfrak{K}\underline{s}\left(1 - \frac{d_x}{2s_x}\right)},
\end{align*}
where we use the Marchaud-type estimate \citet[Chapter 2, Eq. (10.3)]{devore1993constructive} in $(a)$, we use the sequential Besov norm equivalence \cite{devore1988interpolation}[Theorem 5.1] (see also \citep[Theorem 3.3.3]{leisner_nonlinear_wavelet_approximation_2003}) in $(b)$. 
\end{proof}
\end{appendix}

\end{document}